\documentclass[twoside,11pt]{article}

\usepackage{jmlr2e}
\usepackage{tikz} 
\definecolor{lightgray}{gray}{0.92}
\definecolor{mydarkblue}{rgb}{0.0, 0.0, 0.7}
\hypersetup{ 
	pdftitle={},
	pdfkeywords={},
	pdfborder=0 0 0,
	pdfpagemode=UseNone,
	colorlinks=true,
	linkcolor=mydarkblue,
	citecolor=mydarkblue,
	filecolor=mydarkblue,
	urlcolor=mydarkblue,
	pdfview=FitH,
	pdfauthor={Anonymous},
}

\usepackage{caption}
\usepackage{subcaption}

\makeatletter
\newcommand\footnoteref[1]{\protected@xdef\@thefnmark{\ref{#1}}\@footnotemark}
\makeatother

\usepackage{blkarray}

\usepackage{comment}
\usepackage[dvipsnames,table]{xcolor}

\usepackage{booktabs} % To thicken table lines
\usepackage{wrapfig}
\usepackage{caption}
\usepackage{floatrow}

\usepackage{amsmath,amsfonts,bm}
\usepackage{bbm}

\newcommand{\eqobs}{\sim_\textnormal{obs}}
\newcommand{\eqdiff}{\sim_\textnormal{diff}}
\newcommand{\eqcon}{\sim_\textnormal{con}}
\newcommand{\eqperm}{\sim_\textnormal{perm}}
\newcommand{\eqlin}{\sim_\textnormal{lin}}

\newcommand{\supp}{\textnormal{supp}}
\renewcommand{\subset}{\subseteq}

\def\1{\bm{1}}

\def\bfT{{\bm{s}}}

\def\Pa{{\bf Pa}}
\def\Ch{{\bf Ch}}

\def\bflambda{{\bm{\lambda}}}

\def\vmu{{\bm{\mu}}}
\def\vtheta{{\bm{\theta}}}
\def\vdelta{{\bm{\delta}}}
\def\vsigma{{\bm{\sigma}}}
\def\vgamma{{\bm{\gamma}}}
\def\va{{\bm{a}}}
\def\vb{{\bm{b}}}
\def\vc{{\bm{c}}}
\def\vd{{\bm{d}}}
\def\ve{{\bm{e}}}
\def\vf{{\bm{f}}}
\def\vg{{\bm{g}}}
\def\vh{{\bm{h}}}

\def\vn{{\bm{n}}}

\def\vs{{\bm{s}}}

\def\vv{{\bm{v}}}
\def\vw{{\bm{w}}}
\def\vx{{\bm{x}}}
\def\vy{{\bm{y}}}
\def\vz{{\bm{z}}}

\def\bff{{\bm{f}}}

\def\bfv{{\bm{v}}}

\def\mA{{\bm{A}}}
\def\mB{{\bm{B}}}
\def\mC{{\bm{C}}}
\def\mD{{\bm{D}}}
\def\mE{{\bm{E}}}

\def\mG{{\bm{G}}}
\def\mH{{\bm{H}}}
\def\mI{{\bm{I}}}

\def\mK{{\bm{K}}}
\def\mL{{\bm{L}}}
\def\mM{{\bm{M}}}

\def\mP{{\bm{P}}}

\def\mS{{\bm{S}}}

\def\mU{{\bm{U}}}
\def\mV{{\bm{V}}}
\def\mW{{\bm{W}}}
\def\mX{{\bm{X}}}

\DeclareMathAlphabet{\mathsfit}{\encodingdefault}{\sfdefault}{m}{sl}
\SetMathAlphabet{\mathsfit}{bold}{\encodingdefault}{\sfdefault}{bx}{n}

\def\gA{{\mathcal{A}}}
\def\gB{{\mathcal{B}}}

\def\gG{{\mathcal{G}}}

\def\gI{{\mathcal{I}}}
\def\gJ{{\mathcal{J}}}

\def\gL{{\mathcal{L}}}

\def\gN{{\mathcal{N}}}
\def\gO{{\mathcal{O}}}

\def\gX{{\mathcal{X}}}

\def\gZ{{\mathcal{Z}}}

\def\sE{{\mathbb{E}}}
\def\sI{{\mathbb{I}}}

\def\sN{{\mathbb{N}}}

\def\sP{{\mathbb{P}}}

\def\sR{{\mathbb{R}}}

\def\sV{{\mathbb{V}}}

\def\sZ{{\mathbb{Z}}}

\newcommand{\vecspan}{\mathrm{span}}

\usepackage{enumitem}
\usepackage{times}

\makeatletter
\let\BlackBox\relax
\makeatother
\usepackage{thm-restate}
\newcommand{\BlackBox}{\rule{1.5ex}{1.5ex}}

\newcommand{\seb}[1]{\textcolor{red}{[#1]}}

\makeatletter
\let\c@example\relax

\let\c@theorem\relax
\let\thetheorem\relax

\makeatother

\newtheorem{example}{Example} 
\newtheorem{theorem}{Theorem}
\newtheorem{lemma}{Lemma} 
\newtheorem{proposition}{Proposition} 
\newtheorem{remark}{Remark}
\newtheorem{corollary}{Corollary}
\newtheorem{definition}{Definition}
\newtheorem{conjecture}{Conjecture}
\newtheorem{axiom}{Axiom}
\newtheorem{assumption}{Assumption}

\usepackage{lastpage}
\jmlrheading{27}{2026}{1-\pageref{LastPage}}{5/24; Revised
2/26}{5/26}{24-0771}{Sébastien Lachapelle, Pau {Rodríguez López}, Yash Sharma, Katie Everett, Rémi {Le Priol}, Alexandre Lacoste, Simon Lacoste-Julien}
\ShortHeadings{Nonparametric Partial Disentanglement via Mechanism Sparsity}{Lachapelle et al.}%{Lachapelle, {Rodríguez López}, Sharma, Everett, {Le Priol}, Lacoste and Lacoste-Julien}
\firstpageno{1}

\DeclareMathSymbol{\shortminus}{\mathbin}{AMSa}{"39}

\begin{document}

%\title{Nonparametric Partial Disentanglement via Sparse Actions, Interventions and Sparse Temporal Dependencies}

\title{Nonparametric Partial Disentanglement via Mechanism Sparsity: Sparse Actions, Interventions and Sparse Temporal Dependencies}

\author{
\name \hspace{-0.08cm}Sébastien Lachapelle \email s.lachapelle@samsung.com \\
\addr Samsung AI Lab, Montreal; Mila \& DIRO, Université de Montréal
\AND
\name Pau {Rodríguez López}\\
 \addr ServiceNow Research%
 \AND
 \name Yash Sharma \\
 \addr Tübingen AI Center, University of T\"ubingen %
 \AND
 \name Katie Everett \\
 \addr MIT CSAIL %
 \AND
 \name Rémi {Le Priol} \\
 \addr Mila \& DIRO, Université de Montréal
 \AND
 \name Alexandre Lacoste \\
 \addr ServiceNow Research%
 \AND
 \name Simon Lacoste-Julien \\
 \addr Samsung AI Lab, Montreal; Mila \& DIRO, Université de Montréal; Canada CIFAR AI Chair
}

\editor{Elias Bareinboim}

\maketitle

\begin{abstract}%   <- trailing '%' for backward compatibility of .sty file
This work introduces a novel principle for disentanglement we call \textit{mechanism sparsity regularization}, which applies when the latent factors of interest depend sparsely on observed auxiliary variables and/or past latent factors. We propose a representation learning method that induces disentanglement by \textit{simultaneously} learning the latent factors and the sparse causal graphical model that explains them. We develop a nonparametric identifiability theory that formalizes this principle and shows that the latent factors can be recovered by regularizing the learned causal graph to be sparse, \textcolor{black}{under some assumptions such as the absence of instantaneous causal effects between latent factors}. More precisely, we show identifiability up to a novel equivalence relation we call \textit{consistency}, which allows some latent factors to remain entangled (hence the term \textit{partial} disentanglement). To describe the structure of this entanglement, we introduce the notions of \textit{entanglement graphs} and \textit{graph preserving functions}. We further provide a graphical criterion which guarantees \textit{complete} disentanglement, that is identifiability up to permutations and element-wise transformations. We demonstrate the scope of the mechanism sparsity principle as well as the assumptions it relies on with several worked out examples. For instance, the framework shows how one can leverage multi-node interventions with unknown targets on the latent factors to disentangle them. We further draw connections between our nonparametric results and the now popular exponential family assumption. Lastly, we propose an estimation procedure based on variational autoencoders and a sparsity constraint and demonstrate it on various synthetic datasets. This work is meant to be a significantly extended version of \citet{lachapelle2022disentanglement}.
\end{abstract}

\begin{keywords}
  identifiable representation learning, causal representation learning, disentanglement, nonlinear independent component analysis, causal discovery
\end{keywords}

\newpage

\tableofcontents

\newpage

\section{Introduction}
\label{sec:introduction}
It has been proposed that causal reasoning will be central in moving modern machine learning algorithms beyond their current shortcomings, such as their lack of \textit{robustness}, \textit{transferability}, and \textit{interpretability}~\citep{Pearl2018TheSP, scholkopf2019causality,InducGoyal2021}. To achieve this, the field of \textit{causal representation learning}~(CRL)~\citep{scholkopf2021causal} %aims to reconcile the causal graphical model~(CGM) formalism~\citep{pearl2009causality, peters2017elements}, which operates on \textit{semantically meaningful} variables, with deep neural networks~\citep{Goodfellow-et-al-2016}, which excel on unstructured, low-level, high-dimensional observations, e.g. images. 
aims to learn representations of high-dimensional observations, such as images, that are suitable to perform causal reasoning, such as predicting the effect of unseen interventions and answering counterfactual queries. 
A popular formalism to do so is to assume that the observations $\vx \in \sR^{d_z}$ are sampled from a generative model of the form $\vx = \vf(\vz)$ where $\vz \in \sR^{d_z}$ is a random vector of \textit{unobserved} and \textit{semantically meaningful} variables, also called latent factors, distributed according to an unknown \textit{causal graphical model} (CGM) \citep{pearl2009causality, peters2017elements} and transformed by a potentially highly nonlinear \textit{decoder}, or \textit{mixing function}, $\vf$~\citep{kocaoglu2018causalgan, causeOccam2021,lachapelle2022disentanglement,lippe2022icitris,brehmer2022weakly,ahuja2023interventional,buchholz2023learning,vonkugelgen2023nonparametric,zhang2023identifiability,jiang2023learning}. The goal is then to recover the latent factors $\vz_i$ up to permutation and rescaling, as well as the causal relationships that explain them. This is closely related to the problem of \textit{disentanglement}~\citep{bengio2013representation,Higgins2017betaVAELB,pmlr-v119-locatello20a}, which also aims to extract interpretable variables from high-dimensional observations, but without the emphasis on modeling their causal relations. Such problems are plagued by the difficult question of \textit{identifiability}, which is of crucial importance in the classical settings of \textit{causal discovery}~\citep{pearl2009causality,peters2017elements}, where $\vf$ is assumed to be the identity, and \textit{independent component analysis} (ICA)~\citep{ICAbook,hyvarinen2023nonlinear}, where the causal graph over latents is assumed empty. In the former, one can only identify the Markov equivalence class of the causal graph (assuming faithfulness), thus leaving some edge orientations ambiguous~\citep{pearl2009causality}, while in the latter, identifiability of the ground-truth latent factors is impossible when assuming a general nonlinear $\vf$,~\citep{HYVARINEN1999429}. The general CRL problem inherits the difficulties from both of these settings, which makes identifiability especially challenging. Various strategies to improve identifiability have been contributed to the literature, such as assuming access to \textit{interventional data} in which latent factors are targeted by interventions~\citep{lachapelle2022disentanglement,CITRIS,lippe2022icitris,ahuja2023interventional}, or access to an \textit{auxiliary variable} $\va$ that makes factors $\vz_i$ mutually independent when conditioned on~\citep{HyvarinenST19,iVAEkhemakhem20a,ice-beem20}. A valid auxiliary variable $\va$ must be observed and could correspond, for instance, to a time or environment index, an action in an interactive environment, or even a previous observation if the data have temporal structure. See Section~\ref{sec:lit_review} for a more extensive review of existing approaches to the identification of latent variables. %\seb{should add "disentanglement" somewhere}

The present paper introduces\footnote{\label{note1}A shorter version of this work originally appeared in \citet{lachapelle2022disentanglement}.} \textit{mechanism sparsity regularization} as a new principle for the identification of latent variables. We show that if (i) an auxiliary variable $\va$ is observed and affects the latent variables \textit{sparsely} and/or (ii) the latent variables $\vz$ present \textit{sparse} temporal dependencies, then the latent variables can be recovered by learning a graphical model for $\vz$ and $\va$ and regularizing it to be sparse (Theorems~\ref{thm:nonparam_dis_cont_a}, \ref{thm:nonparam_dis_disc_a}, \ref{thm:nonparam_dis_z} \& \ref{thm:expfam_dis_z}). 
%Our method aims to recover both the latent factors $\vz$ and the causal graph that relates them to the auxiliary variable $\va$ and to themselves across different time steps.
More specifically, we consider models of the form $\vx^t = \vf(\vz^t) + \vn^t$, where $\vn^t$ is independent noise (Assumption~\ref{ass:diffeomorphism}) and the latent factors $\vz_i^t$ are mutually independent given the past factors and the auxiliary variables, i.e. ${p(\vz^t \mid \vz^{<t}, \va^{<t}) = \prod_{i=1}^{d_z} p(\vz_i^t \mid \vz^{<t}, \va^{<t})}$ (Assumption~\ref{ass:cond_indep}). 
Crucially, we leverage the assumption that these mechanisms are sparse in the sense that $p(\vz^t \mid \vz^{<t}, \va^{<t})$ factorizes according to a sparse causal graph $\mG$ (Assumption~\ref{ass:graph}).   
Interestingly, if $\va$ corresponds to an intervention index, our framework explains how interventions targeting unknown subsets of latent factors can identify them (Section~\ref{sec:unknown-target}). We emphasize that the settings where the data have no temporal dependencies or no auxiliary variable $\va$ are special cases of our framework. Our identifiability results are summarized in Table~\ref{tab:summary_results}. %\seb{Say something about fact that relying on sparsity allows us to consider simpler models that do not rely on third-order derivatives/differences of $\log p(\vz^t_i \mid \vz^{<t}, \va^{<t})$ unlike iVAE and other approaches, which allows for Homoscedastic Gaussian models. Should also probably mention conditional independence.}

This work is meant to be an extended version of~\citet{lachapelle2022disentanglement} in which we generalize along two main axes: First, we relax the \textit{exponential family} assumption by providing a fully \textit{nonparametric} treatment. Secondly, our results drop the graphical criterion of \citet{lachapelle2022disentanglement} and, thus, allow for \textit{arbitrary} latent causal graphs. As a consequence of this relaxation, instead of guaranteeing identifiability up to permutation and element-wise transformation, we guarantee identifiability up to what we call \textit{$\va$-consistency} or \textit{$\vz$-consistency} (Definitions~\ref{def:a_consistent_models}~\&~\ref{def:z_consistent_models}), which might allow certain latent variables to remain entangled. Our results thus have the following flavor: Given a specific ground-truth causal graph $\mG$ over $\vz$ and $\va$, we describe precisely the structure of the entanglement between latent factors via what we call an \textit{entanglement graph} (Definition~\ref{def:entanglement_graphs}) and \textit{graph preserving functions} (Definition~\ref{def:g_preserving_map}). See Figure~\ref{fig:Ga} for examples. Interestingly, the stronger identifiability up to permutation and element-wise transformation arises as a simple consequence of our theory when the graphical criterion of \citet{lachapelle2022disentanglement} is assumed to hold. In addition to these two main axes of generalization, we provide extensive examples illustrating the scope of our framework, our assumptions, and the consequences of our results (see Table~\ref{tab:examples} for a list). When it comes to the learning algorithm, we replace the sparsity \textit{penalty} by a sparsity \textit{constraint}, which improves the learning dynamics and is more interpretable, which results in easier hyperparameter tuning.

\def\HorizNodeSpace{1cm}
\def\VertNodeSpace{0.3cm}
\def\GraphSpace{1.5cm}
\def\VertLabelSpace{1.2cm}
\def\Padding{2cm}
\usetikzlibrary{positioning, shapes.misc, fit,arrows.meta}
\usetikzlibrary{calc}
\usetikzlibrary{decorations.pathreplacing}
\begin{figure}[t]
    \centering
    \resizebox{\linewidth}{!}{
    \input{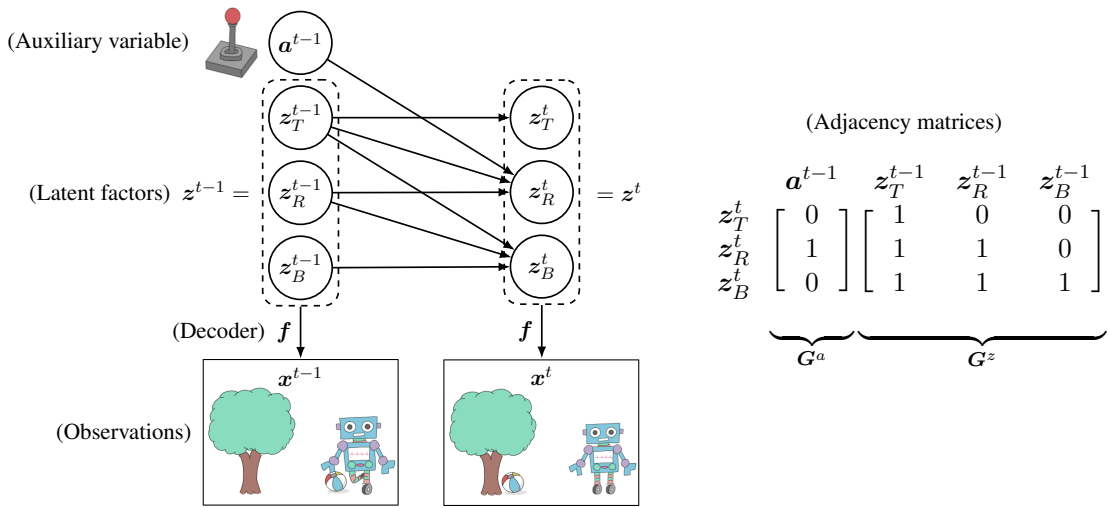}}
    \caption{A minimal motivating example. The latent factors $\vz_T^{t}$, $\vz_R^t$ and $\vz_B^t$ represent the $x$-positions of the tree, the robot and the ball at time $t$, respectively. Only the image of the scene $\vx^t$ and the action $\va^{t-1}$ are observed. See end of Section~\ref{sec:model} for details. \vspace{-2mm}
    }
    \label{fig:working_example}
\end{figure}

%The present paper introduces \textit{mechanism sparsity regularization} as a new path to disentanglement. By building on the recent theoretical developments in ICA, we show that if the high-level variables have a \textit{sparse} temporal structure and/or an action is observed and affects the high-level variables \textit{sparsely}, then the latent variables can be recovered by regularizing the inferred graphical model to have sparse dependencies (Thm.~\ref{thm:combined}). In estimating the latent variables, the presented methodology estimates the causal graph describing them and their relation to the action $A$ (when available). A very similar disentanglement method based on graph sparsity was proposed independently by~\citet{causeOccam2021}, but this concurrent work does not analyze identifiability formally (Sec.~\ref{sec:lit_review}). In contrast, our theory provides precise conditions, e.g. on the ground-truth graph, to ensure identifiability, thus extending the domain of known cases where latent variables can be recovered.

The hypothesis that \textit{high-level concepts are related by a sparse dependency graph} has been described and leveraged for out-of-distribution generalization by~\citet{consciousBengio} and \citet{goyal2021recurrent}, which were early sources of inspiration for this work. To the best of our knowledge, the theory developed in the present work and in its preliminary version~\citep{lachapelle2022disentanglement} is the first to show formally that this inductive bias can sometimes be enough to recover ground-truth latent factors.

Figure~\ref{fig:working_example} shows a minimal motivating example in which our approach could be used to extract the high-level variables (such as the $x$-position of the three objects) and learn their dynamics (how the objects move and affect each other) from a time series of images and agent actions, $(\vx^t, \va^t)$. Theorems~\ref{thm:nonparam_dis_cont_a}, \ref{thm:nonparam_dis_disc_a}, \ref{thm:nonparam_dis_z} \& \ref{thm:expfam_dis_z} show how the sparse dependencies between the objects and the action can be leveraged to identify the latent variables and the graph describing their dynamics. The learned CGM could be used subsequently to simulate interventions on semantic variables~\citep{pearl2009causality, peters2017elements}, such as changing the torque of the robot or the weight of the ball. %Interventions allow an agent to imagine situations it has never seen before, which would not be possible without a disentangled representation~\citep{scholkopf2019causality}. 
Moreover, disentanglement could be useful to interpret what caused the actions of an agent~\citep{Pearl2018TheSP}. Following~\citet{lachapelle2022disentanglement}, empirical works demonstrated that disentangled representations with sparse mechanisms can adapt to unseen interventions faster in the context of single-cell biology~\citep{lopez2022learninglong} and synthetic video data \citep{lei2022variational}.

   \begin{table}[]
    %\centering
    \scriptsize
    \rowcolors{2}{white}{lightgray} 
    \begin{tabular}{cccccccc}
        \toprule
        %& & \multicolumn{2}{c}{\textbf{Sparse $\hat\mG^a$}} & \textbf{Sparse $\hat\mG^z$} & & & \\
        %\midrule
         & \begin{tabular}{@{}c@{}}\textbf{Parametric} \\ \textbf{assumption}\end{tabular} & \textbf{Continuous $\va$} &  \begin{tabular}{@{}c@{}}\textbf{Discrete $\va$} \\ \textbf{(interventions)}\end{tabular} & \begin{tabular}{@{}c@{}}\textbf{Temporal} \\ \textbf{dependencies} \end{tabular} & \begin{tabular}{@{}c@{}}\textbf{Sufficient} \\ \textbf{influence}\end{tabular} & \begin{tabular}{@{}c@{}}\textbf{Identifiable} \\ \textbf{up to}\end{tabular} & \textbf{Examples}  \\
        \midrule
        \textbf{Thm.~\ref{thm:nonparam_dis_cont_a}} & None & Required & -- & Optional & Ass.~\ref{ass:nonparam_suff_var_a_cont} & Def.~\ref{def:a_consistent_models} & \ref{ex:single_node_complete_dis}, \ref{ex:a_target_one_z}, \ref{ex:multi_target_a}, \ref{ex:a_target_one_z_cont_a}, \ref{ex:multi_target_a_cont_a}\\
        \textbf{Thm.~\ref{thm:nonparam_dis_disc_a}} & None & -- & Required & Optional & Ass.~\ref{ass:nonparam_suff_var_a} & Def.~\ref{def:a_consistent_models} & \ref{ex:single_node_complete_dis}, \ref{ex:a_target_one_z}, \ref{ex:multi_target_a}, \ref{ex:single_node_complete_dis_2}, \ref{ex:multinode_linear_gauss}, \ref{ex:group_interv}\\
        \textbf{Thm.~\ref{thm:nonparam_dis_z}} & None & Optional & Optional & Required & Ass.~\ref{ass:nonparam_suff_var_z} & Def.~\ref{def:z_consistent_models} & \ref{ex:diagonal_deps}, \ref{ex:lower_triangular_no_action}, \ref{ex:temporal_partial}, \ref{ex:lower_triangular_no_action_exp}\\
        \textbf{Thm.~\ref{thm:linear}} & Exp. fam. & Optional & Optional & Optional & Ass.~\ref{ass:suff_var_expfam} & Def.~\ref{def:linear_eq} & \ref{ex:linear}\\
        \textbf{Thm.~\ref{thm:expfam_dis_z}} & Exp. fam. & Optional & Optional & Required & Ass.~\ref{ass:temporal_suff_var_main} & Def.~\ref{def:z_consistent_models}, \ref{def:linear_eq} & \ref{ex:lower_triangular_no_action_exp_expfam}\\
        \bottomrule
    \end{tabular}
    \caption{Summary of our identifiability results.}
    \label{tab:summary_results}
\end{table}

\vspace{-0.1cm}
\paragraph{Summary of our contributions:}
\begin{enumerate}[topsep=-1ex,itemsep=-1ex,partopsep=1ex,parsep=1ex]
    \item We introduce\footnoteref{note1} a new principle for disentanglement based on \textit{mechanism sparsity regularization} motivated by rigorous and novel \textit{identifiability guarantees} (Theorems~\ref{thm:nonparam_dis_cont_a}, \ref{thm:nonparam_dis_disc_a}, \ref{thm:nonparam_dis_z} \& \ref{thm:expfam_dis_z}).
    \item We extend \citet{lachapelle2022disentanglement} by providing a \textit{nonparametric} treatment and allowing for \textit{arbitrary latent graphs}, \textcolor{black}{as long as it has no edges between latent factors from the same time step (no instantaneous causal links)}. Given a latent ground-truth graph, our theory predicts the structure of the entanglement between variables, which we formalize with \textit{entanglement graphs} (Definition~\ref{def:entanglement_graphs}), \textit{graph preserving maps} (Definition~\ref{def:g_preserving_map}) and novel \textit{equivalence relations}~(Definitions~\ref{def:a_consistent_models}~\&~\ref{def:z_consistent_models}).
    \item We provide several examples to illustrate the generality of our results and to get a better understanding of their various assumptions and consequences (summarized in Table~\ref{tab:examples}). For instance, we show how multi-node interventions with unknown-targets can yield disentanglement, both with and without temporal dependencies (Examples~\ref{ex:multinode_linear_gauss}~\&~\ref{ex:group_interv}).
    \item We introduce an evaluation metric denoted by $R_\text{con}$ that quantifies how close two representations are to being $\va$-consistent or $\vz$-consistent (Section~\ref{sec:evaluation}). 
    \item We implement a learning approach based on variational autoencoders (VAEs)~\citep{Kingma2014} that learns the mixing function $\vf$, the transition distribution $p(\vz^t \mid \vz^{<t}, \va^{<t})$ and the causal graph $\mG$. The latter is learned using binary masks and regularized for sparsity using a \textit{constraint} as opposed to a penalty as in~\citet{lachapelle2022disentanglement}.
    \item We perform experiments on synthetic datasets to validate the prediction of our theory.
\end{enumerate}
\vspace{2ex}

\paragraph{Overview.} Section~\ref{sec:main} introduces the model (Section~\ref{sec:model}), entanglement maps and graphs (Section~\ref{sec:entanglement_graphs}), the notion of identifiability (Section~\ref{sec:ident_obseq}), equivalence up to diffeomorphism (Section~\ref{sec:diffeq}) and disentanglement formally (Section~\ref{sec:disent_permeq}). Section~\ref{sec:nonparam_ident} provides mathematical intuition for why mechanism sparsity yields disentanglement (Section~\ref{sec:insight_main}); introduces the machinery of graph preserving maps (Section~\ref{sec:consistency_equ}) which are key to establish identifiability up to $\va$-consistency (Section~\ref{sec:nonparam_ident_action}) and $\vz$-consistency (Section~\ref{sec:nonparam_ident_time}), i.e. \textit{partial} disentanglement. Section~\ref{sec:nonparam_ident} also discusses the relationship with interventions (Section~\ref{sec:unknown-target}), provides a graphical criterion guaranteeing \textit{complete} disentanglement (Section~\ref{sec:graph_crit}), and introduces and discusses extensively the \textit{sufficient influence assumptions} on which these results are critically based (Sections~\ref{sec:proofs} \& \ref{sec:examples}). Section~\ref{sec:exponential_family} draws connections between our \textit{nonparametric} theory and the \textit{exponential family} assumption sometimes used in the literature. Section~\ref{sec:estimation} presents the VAE-based learning algorithm with sparsity constraint. Section~\ref{sec:evaluation} introduces our novel $R_\text{con}$ metric. Section~\ref{sec:lit_review} reviews the literature on identifiability in representation learning. Section~\ref{sec:exp} presents the empirical results.

\paragraph{Notation.} Scalars are denoted in lower-case and vectors in lower-case bold, e.g. $x \in \sR$ and $\vx \in \sR^n$. Note that these will sometimes denote random variables depending on the context. We maintain an analogous notation for scalar-valued and vector-valued functions, e.g., $f$ and $\vf$. The $i$th coordinate of the vector $\vx$ is denoted by $\vx_i$. The set containing the first $n$ integers excluding $0$ is denoted by $[n]$. Given a subset of indices $S \subset [n]$, $\vx_S$ denotes the subvector consisting of entries $\vx_i$ for $i \in S$. Given a sequence of $T$ random vectors $(\vx^1, \dots, \vx^T)$, the subsequence consisting of the first $t$ elements is denoted by $\vx^{\leq t} := (\vx^1, ..., \vx^t)$, and analogously for $\vx^{<t}$. We will sometimes combine these notation to get $\vx^{\leq t}_S := (\vx_S^1, \dots, \vx_S^t)$. Given a function $\vf: \sR^n \rightarrow \sR^m$, its Jacobian matrix evaluated at $\vx \in \sR^n$ is denoted by $D\vf(\vx) \in \sR^{m\times n}$. See Table~\ref{tab:TableOfNotation} in the Appendix for more.

\begin{table}[]
    %\centering
    \rowcolors{2}{white}{lightgray} 
    \scriptsize
    \begin{tabular}{c|l|l|l}
    \toprule
        \textbf{Examples} & \textbf{Type of disentanglement} & \textbf{Auxiliary variable} & \textbf{Time dependencies} \\
        \midrule
         \ref{ex:single_node_complete_dis} & Complete & Yes (single target) & Optional \\
         \ref{ex:a_target_one_z} & Partial & Yes (single target) & Optional \\
         \ref{ex:multi_target_a} & Complete & Yes (multi-target) & Optional \\
         \ref{ex:diagonal_deps} & Complete & Optional & Yes (independent factors) \\
         \ref{ex:lower_triangular_no_action} & Complete & Optional & Yes (dependent factors) \\
         \ref{ex:temporal_partial} & Partial & Optional & Yes (dependent factors) \\
         \ref{ex:a_target_one_z_cont_a} & Partial & Yes (single-target continuous ) & Yes \\
         \ref{ex:multi_target_a_cont_a} & Complete & Yes (multi-target continuous)  & No \\
         \ref{ex:single_node_complete_dis_2} & Complete & Yes (single-target interventions) & No \\
         \ref{ex:multinode_linear_gauss} & Complete & Yes (multi-target interventions) & Yes \\
         \ref{ex:group_interv} & Complete & Yes (grouped multi-target interventions) & No \\
         \ref{ex:lower_triangular_no_action_exp} & Complete & No & Yes (non-Markovian) \\
         %\ref{ex:linear} & Illustrating the sufficient variability assumption of Theorem~\ref{thm:linear} for exponential families\\
         \ref{ex:lower_triangular_no_action_exp_expfam} & Complete & No & Yes (Markovian) \\
         \bottomrule
    \end{tabular}
    \caption{List of examples illustrating the scope of our theory, its assumptions and its consequences.}
    \label{tab:examples}
\end{table}

\section{Problem Setting, Entanglement Graphs \& Disentanglement}
\label{sec:main}

In this section, we introduce the latent variable model under consideration (Section~\ref{sec:model}), entanglement graphs (Section~\ref{sec:entanglement_graphs}), identifiability and observational equivalence (Section~\ref{sec:ident_obseq}), equivalence up to diffeomorphism (Section~\ref{sec:diffeq}) as well as permutation equivalence (Section~\ref{sec:disent_permeq}).

\subsection{An Identifiable Latent Causal Model} \label{sec:model}
\begin{comment}
\begin{figure}%{l}{\textwidth}
\centering
\includegraphics[scale=0.74]{figures/overview.pdf}
\caption{
\textcolor{red}{[TODO: Position this, and change theorem number.]}
Graphical representation of the model described in Sec.~\ref{sec:model}. Thm.~\ref{thm:cont_c_sparse}~\&~\ref{thm:sparse_temporal} provide novel paths to disentanglement via \textit{mechanism sparsity regularization}.
}
\label{fig:overview}
\end{figure}
\end{comment}

We now specify the setting under consideration. Assume that we observe the realization of a sequence of $d_x$-dimensional random vectors $\{\vx^t\}_{t=1}^{T}$ and a sequence of $d_a$-dimensional auxiliary vectors $\{\va^{t}\}_{t=0}^{T-1}$. The coordinates of $\va^{t}$ are discrete or continuous and can potentially represent, for example, an action taken by an agent or the index of an intervention or environment (see Section~\ref{sec:unknown-target}). Observations $\{\vx^t\}$ are assumed to be explained by a sequence of continuous random vectors $d_z$ of hidden dimensions $\{\vz^t\}_{t=1}^T$ using a ground-truth decoder function $\vf$.
\begin{assumption}[Observation model]\label{ass:diffeomorphism}
    For all $t \in [T]$, the observations $\vx^t$ are given by
    \begin{align*}
        \vx^t = \bff(\vz^t) + \vn^t \,,
    \end{align*}
    where $\vn^t \sim \mathcal{N}(0, \sigma^2 \mI)$ are mutually independent across time and independent of all $\vz^t$ and $\va^{t}$ with $\sigma^2 \geq 0$. Moreover, $d_z \leq d_x$ and $\bff: \sR^{d_z} \rightarrow \sR^{d_x}$ is a diffeomorphism onto its image\footnote{A \textit{diffeomorphism} is a $C^1$ bijection with a $C^1$ inverse. Generally, given a map $\vh: A \rightarrow \sR^m$ where $A \subseteq \sR^n$, saying $\vh$ is $C^k$ is typically only well defined if $A$ is an open set of $\sR^n$. Throughout, if $A \subseteq \sR^n$ is arbitrary (not necessarily open), we say $\vh$ is $C^k$ if there exists a $C^k$ map $\tilde \vh: U \rightarrow \sR^m$ defined on an open set $U$ of $\sR^n$ containing $A$ such that $\vh = \tilde \vh$ on $A$. Note that it is then meaningful for $\bff^{-1}: \bff(\sR^{d_z}) \rightarrow \sR^{d_z}$ to be $C^1$ even when $\bff(\sR^{d_z})$ is not open in $\sR^{d_x}$. Moreover, it can be shown that $\vf:\sR^{d_z} \rightarrow \sR^{d_x}$ is a diffeomorphism onto its image if $\vf$ is an homeomorphism onto its image, i.e. continuous in both directions, and has a full rank Jacobian everywhere on its domain \citep[Sec. 23 \& Thm. 24.1]{munkres1991analysis}.}. Lastly, assume that $\vf(\sR^{d_z})$ is closed in $\sR^{d_x}$.
\end{assumption}
   Importantly, we suppose that each factor $\vz_i^t$ contains interpretable information about the observation, e.g., for high-dimensional images, the coordinates $\vz_i^t$ might be the position of an object, its color, or its orientation in space. This idea that there exists a \textit{ground-truth decoder} $\vf$ that captures the relationship between the so-called ``natural factors of variations'' and the observations $\vx$ is of capital importance, since it is the very basis for a mathematical definition of disentanglement (Definition~\ref{def:disentanglement}). Appendix~\ref{sec:diffeo_f} discusses the implications of the diffeomorphism assumption (see also \citet{mansouri2022objectcentric}). We denote $\vz^{\leq t} := [\vz^1\ \cdots\ \vz^t] \in \sR^{d_z \times t}$ and analogously for $\vz^{< t}$ and other random vectors.

In a similar spirit to previous works on nonlinear ICA~\citep{HyvarinenST19, iVAEkhemakhem20a}, we assume that the latent factors $\vz^t_i$ are conditionally independent given the past.

\begin{assumption}[Conditionally independent latent factors]\label{ass:cond_indep}
    The latent factors $\vz_i^t$ are conditionally mutually independent given $\vz^{<t}$ and $\va^{<t}$: 
    \begin{align}\label{eq:cond_indep}
    p(\vz^t \mid \vz^{<t}, \va^{<t}) = \prod_{i = 1}^{d_z} p(\vz_i^t \mid \vz^{<t}, \va^{<t}) \, ,
    \end{align}
    where $p(\vz^t \mid \vz^{<t}, \va^{<t})$ is a density function w.r.t. the Lebesgue measure on $\sR^{d_z}$. We assume that the support of $p(\vz_i^t \mid \vz^{<t}, \va^{<t})$ is $\sR$ for all $\vz^{<t}$ and $\va^{<t}$. The support of $p(\vz^t \mid \vz^{<t}, \va^{<t})$ is thus given by $\sR^{d_z}$. \textcolor{black}{At $t=1$, we set $p(\vz^1 \mid \vz^{<1}, \va^{<1}) = p(\vz^1 \mid \va^{0}) = \prod_{i = 1}^{d_z} p(\vz_i^1 \mid \va^{0})$.}
\end{assumption}

We will refer to the l.h.s. of~\eqref{eq:cond_indep} as the \textit{transition model} and to the factors $p(\vz_i^t \mid \vz^{<t}, \va^{<t})$ as \textit{mechanisms}. \textcolor{black}{In a causal setting, this conditional independence assumption states that latent variables cannot exert influence on other latent variables instantaneously.  We believe this to be a reasonable assumption when the time steps are sampled at a sufficiently high rate since, at a fundamental level, causal effects cannot travel faster than the speed of light. Of course in many practical scenarios, a low sampling rate can entail instantaneous dependencies which, if too strong, might result in model misspecification. In Section~\ref{sec:lit_review}, we discuss and contrast with causal representation learning methods which allow for instantaneous influence.}  

Notice that we do not assume the dynamical system to be \textcolor{black}{\textit{Markovian}\footnote{\textcolor{black}{We use the term ``Markovian'' in the stochastic processes sense, as in Markov chains, where the future is independent of the past given the present. We will use this meaning throughout. This is to be contrasted with the Markovian property in structural models which states that exogenous noises are mutually independent~\citep{Bareinboim2022}.}}}, i.e. the distribution over future states can depend on the whole history of latents and auxiliary variables $(\vz^{<t}, \va^{<t})$. In addition, this model can represent \textit{non-homogeneous} processes by taking the auxiliary variable $\va$ to be a time index~\citep{HyvarinenST19}.

We are going to describe the dependency structure of the latent and auxiliary variables over time using a \textit{probabilistic directed graphical model} composed of two bipartite graphs, $\mG^z \in \{0,1\}^{d_z \times d_z}$, which relates $\vz^{<t}$ to $\vz^t$, and $\mG^a \in \{0,1\}^{d_z \times d_a}$, which relates $\va^{<t}$ to $\vz^t$. A directed edge points from $\vz_j^{<t}$ to $\vz^t_i$ if and only if $\mG^z_{i, j} = 1$. Analogously, a directed edge points from $\va^{<t}_\ell$ to $\vz^t_i$, if and only if $\mG^a_{i, \ell} = 1$. Figure~\ref{fig:working_example} shows an example of such graphs together with its adjacency matrix $\mG := [\mG^z,\mG^a]$. The following assumption specifies the relationship between these graphs and the transition model.
\begin{assumption}[Transition model $p$ follows $\mG$]\label{ass:graph}
    For every mechanism $i \in [d_z]$,
    \begin{align}\label{eq:parents}
    p(\vz_i^t \mid \vz^{<t}, \va^{<t}) = p(\vz_i^t \mid \vz^{<t}_{\Pa_i^z}, \va^{<t}_{\Pa_i^a}) \,,
\end{align}
where $\Pa^z_i \subseteq [d_z]$ and $\Pa^a_i \subseteq [d_a]$ are the sets of parents of $\vz_i^t$ in $\mG^z$ and $\mG^a$, respectively.
\end{assumption}
The graph $\mG$ thus encodes a set of conditional independence statements about the latent and auxiliary variables.\footnote{\textcolor{black}{This notion of graph is not standard since, typically, directed edges point from scalar variables to scalar variables, whereas here a directed edge points from $\vz_j^{<t}$ (a vector) to $\vz^t_i$ (a scalar). This can be understood as having all directed edges $\vz_j^{\tau} \rightarrow \vz^t_i$ for all $\tau < t$. This does not imply that all edges are necessarily ``active''. We made this choice to simplify the presentation and hypothesize that analogous results could be derived for standard graphs.}} We will say that \textit{mechanisms are sparse} when the graphs $\mG^a$ and $\mG^z$ are sparse.  %\citep[Definition 6.21]{peters2017elements}. 

This model has three components that need to be learned: (i) the decoder function $\vf$, (ii) the transition model over latent variables $p$, and (iii) the dependency graph $\mG$. We gather all these components into $\vtheta := (\vf, p, \mG)$. Everything else in the model, i.e. $d_z$ and $\sigma^2$, is assumed to be known. We assume that $\sigma^2$ is known here mainly for simplicity, since, when it is not, it can be identified as shown by \citet[Appendix A.4.1]{lachapelle2022disentanglement}, as long as $d_x > d_z$.

Notice how we have not specified any model for the auxiliary variable $\va^t$. We do not intend to do so in this work, as we are solely interested in modeling the \textit{conditional} distribution of $\vx^{\leq T}$ and $\vz^{\leq T}$ given $\va^{< T}$. We denote by $\gA \subseteq \sR^{d_a}$ the set of possible values for the auxiliary variable $\va^t$. We thus have that, for all values of $\va^{<T} \in \gA^T$, our model induces a conditional distribution 
\begin{align*}
    p(\vx^{\leq T} \mid \va^{<T}) = \int \prod_{t=1}^{T} p(\vx^t \mid \vz^t)p(\vz^t \mid \vz^{< t}, \va^{<t}) d\vz^{\leq T} \,,
\end{align*}
where $p(\vx^t \mid \vz^{t}) = \gN(\vx^t; \vf(\vz^t), \sigma^2\mI)$. We note that if $\sigma^2 = 0$, the conditional distribution of $\vx^t$ given $\vz^t$ is a Dirac centered at $\vf(\vz^t)$ and thus has no density w.r.t. the Lebesgue measure. Even if, in that case, the above integral does not make sense, the conditional distribution of $\vx^{\leq T}$ given $\va^{<T}$ is still well-defined and all the results of this work still hold since none of the proofs requires $\sigma^2 > 0$.

\paragraph{A motivating example.} Figure~\ref{fig:working_example} represents a minimal example in which our theory applies. The environment consists of three objects: a tree, a robot and a ball with $x$-positions $\vz_T^t$, $\vz_R^t$ and $\vz_B^t$, respectively. Together, they form the vector $\vz^t$ of high-level latent variables, i.e., $\vz^t=(\vz_T^t, \vz_R^t, \vz_B^t)$. A remote controls the direction in which the wheels of the robot turn. The vector $\va^t$ records these actions, which might be taken by a human or an artificial agent trained to accomplish some goal. The only observations are the actions $\va^t$ and the images $\vx^t$ representing the scene which is given by $\vx^t = \bff(\vz^t) + \vn^t$. The dynamics of the environment is governed by the transition model~$p$, which, e.g., could be given by a Gaussian model of the form $p(\vz^t_i \mid \vz^{<t}, \va^{<t}) = \mathcal{N}(\vz_i^t; \mu_i(\vz^{t-1}, \va^{t-1}), \sigma_z^2)$. Plausible connectivity graphs $\mG^z$ and $\mG^a$ are shown in Figure~\ref{fig:working_example} showing how the latent factors are related and how the controller affects them. For every object, its position at the time step $t$ depends on its position at $t-1$. The position of the tree, $\vz_T^t$, is not affected by anything, since neither the robot nor the ball can change its position. The robot, $\vz_R^t$, changes its position based both on the action $\va^{t-1}$ and the position of the tree $\vz_T^{t-1}$ (in case of collision). The position of the ball, $\vz_B^t$, is affected by both the robot, which can move around it by running into it, and the tree, which can bounce. The key observations here are that (i) the different objects interact \textit{sparsely} with one another and (ii) the action $\va^t$ affects very few objects (in this case, only one). The theorems of Section~\ref{sec:nonparam_ident} show how one can leverage this sparsity for disentanglement.

\subsection{Entanglement Maps \& Entanglement Graphs} \label{sec:entanglement_graphs}
In this section, we define \textit{entanglement maps}, which describe the functional relationship between the learned and ground-truth representations, and \textit{entanglement graphs}, which describe their entanglement structure. 

\begin{definition}[Entanglement maps]\label{def:entanglement_maps}
Let $\vf$ and $\tilde\vf$ be two diffeomorphisms from $\sR^{d_z}$ to their images such that $\vf(\sR^{d_z}) = \tilde\vf(\sR^{d_z})$. The \textbf{entanglement map} of the pair $(\vf, \tilde\vf)$ is given by
\begin{align}\label{eq:v_map}
    \vv := \vf^{-1} \circ \tilde\vf \,.
\end{align}
\end{definition}

This map will be crucial throughout this work, especially in defining disentanglement. Intuitively, the entanglement map for a pair of decoders $(\vf, \tilde\vf)$ translates the representation of one model to that of the other. In general, the entanglement maps of $(\vf, \tilde\vf)$ and $(\tilde\vf, \vf)$ are different.

We now define the {\em dependency graph} of some function $\vh$ so that each edge indicates that some input $i$ influences some output $j$:

\begin{definition}[Functional dependency graph]\label{def:dep_graph}
    Let $\vh$ be a function from $\sR^n$ to $\sR^m$. The \textbf{dependency graph} of $\vh$ is a bipartite directed graph from $[n]$ to $[m]$ with adjacency matrix $\mH \in \{0,1\}^{m \times n}$ such that
    %\begin{equation}
    %    \mH_{i,j} = 1 \iff \text{There exists}\ \va_{-j} \in \sR^{n-1}\,, \text{ s.t. } \vh_i(\va_j, \va_{-j})\ \text{is not a constant function of}\ \va_j \,, \footnote{\text{Notice the slight abuse of notation: by $\vh_i(\va_j, \va_{-j})$, we mean $\vh_i(\va)$.}}
    %\end{equation}
    %where $\va_{-j}$ is $\va$ with its $j$th coordinate removed and $A_{-j} := \{\va_{-j} \mid \va \in A\}$.
    \begin{equation*}
        \mH_{i,j} = 0 \iff \text{There is a function $\bar\vh$ such that, for all $\va \in \sR^n$, } \vh_i(\va) = \bar\vh_i(\va_{-j})\,,
    \end{equation*}
    where $\va_{-j}$ is $\va$ with its $j$th coordinate removed.
\end{definition}

\begin{example}[Dependency graph of a linear map]\label{ex:dep_graph_lin}
Let $\vh(\vz) := \mW\vz$ where $\mW \in\sR^{m\times n}$ and let $\mH$ be the dependency graph of $\vh$. Then $\mH_{i,j} = 0 \iff \mW_{i,j} = 0$. 
\end{example}

We will be particularly interested in the dependency graph of the entanglement map $\vv := \vf^{-1} \circ \hat\vf$, denoted by $\mV$.
\begin{definition}[Entanglement graphs]\label{def:entanglement_graphs}
    Let $\vf$ and $\tilde\vf$ be two diffeomorphisms from $\sR^{d_z}$ to their images such that $\vf(\sR^{d_z}) = \tilde\vf(\sR^{d_z})$. The \textbf{entanglement graph} of the pair $(\vf, \tilde\vf)$ is the dependency graph (Definition~\ref{def:dep_graph}) of their entanglement map $\vv:=\vf^{-1} \circ \tilde\vf$, which we denote $\mV \in \{0,1\}^{d_z \times d_z}$.
\end{definition}

We now relate the dependency graph of a function to the zeros of its Jacobian matrix. A proof can be found in Appendix~\ref{sec:V_and_Dv}.

\begin{restatable}[Linking dependency graph and Jacobian]{proposition}{JacobianEntanglement}\label{prop:V_and_Dv}
Let $\vh$ be a $C^1$ function, i.e. continuously differentiable, from $\sR^n$ to $\sR^m$ and let $\mH$ be its dependency graph (Definition~\ref{def:dep_graph}). Then, %If its domain $A$ is convex and regular closed\footnote{A subset $A \subseteq \sR^n$ is \textit{regular closed} when it is equal to the closure of its interior, i.e. $\overline{A^\circ} = A$. Throughout, we use the following useful fact: If $\vh:A\rightarrow B$ is $C^k$ and $A$ is regular closed, then the $k$ first derivatives of $\vh$ are uniquely defined everywhere on $A$, in the sense that they do not depend on the choice of $C^k$ extension of $\vh$. We prove this in Lemma~\ref{lem:equal_derivatives} in appendix.}, then
\begin{align}
    \mH_{i,j} = 0 \iff \text{For all}\ \va \in \sR^{n}\,, D\vh(\va)_{i,j} = 0 \,. \label{eq:alternative_def_H}
\end{align}
\end{restatable}
The equivalence \eqref{eq:alternative_def_H} can be seen as an equivalent definition of dependency graph for differentiable functions.

\subsection{Identifiability and Observational Equivalence}\label{sec:ident_obseq} %\label{sec:rep_equ_nonparam}

To analyze formally whether a specific algorithm is expected to yield a disentangled representation, we will rely on the notion of \textit{identifiability}. Before we define what we mean by identifiability, we need the notion of \textit{observationally equivalent} models. Two models are observationally equivalent if both models represent the same distribution over observations. The following formalizes this definition.

\begin{definition}[Observational equivalence]\label{def:eqobs}
 We say that two models $\vtheta := (\bff, p, \mG)$ and $\tilde{\vtheta}:= (\tilde \bff,\tilde p,\tilde \mG)$ satisfying Assumption~\ref{ass:diffeomorphism} are \textbf{observationally equivalent}, denoted $\vtheta \eqobs \tilde\vtheta$, if and only if, for all $\va^{< T} \in \gA^T$ and all $\vx^{\leq T} \in \sR^{d_x \times T}$,
    \begin{align*}
        p(\vx^{\leq T} \mid \va^{<T}) = \tilde{p}(\vx^{\leq T} \mid \va^{<T}) \,.
    \end{align*}
\end{definition}

Formally, we say that a parameter $\vtheta$ is \textbf{identifiable up to some equivalence relation $\sim$}, when 
\begin{align}
    \vtheta \eqobs \tilde{\vtheta} \implies \vtheta \sim \tilde{\vtheta} \, . \label{eq:general_ident}
\end{align}
This work is mainly concerned with proving statements of the above form by making assumptions both on $\vtheta$ and $\hat\vtheta$. The stronger the assumptions on $\vtheta$ and $\hat\vtheta$ are, the stronger the equivalence relation $\sim$ will be. The following sections present two equivalence relations over models, namely $\eqdiff$ and $\eqperm$. We note that the equivalence relation $\eqperm$ will help us formalize disentanglement.

Practically speaking, observational equivalence between the learned model $\hat\vtheta$ and the ground-truth model $\vtheta$ can be achieved via maximum likelihood estimation in the infinite data regime. Thus, identifiability results of the form of \eqref{eq:general_ident} guaranty that if the learned model is perfectly fitted to the data (assumed infinite), its parameter $\hat\vtheta$ is $\sim$-equivalent to that of the ground-truth model, $\vtheta$.

\subsection{Equivalence up to Diffeomorphism} \label{sec:diffeq}

We start by defining \textit{equivalence up to diffeomorphism}. This equivalence relation is important since we will show later on that it is actually the same as observational equivalence and will thus be our first step in all our identifiability results. In what follows, we overload the notation and write $\vv(\vz^{<t}) := [\vv(\vz^{1}), \dots, \vv(\vz^{t-1})]$, and similarly for other functions.

\begin{definition}[Equivalence up to diffeomorphism] \label{def:eqdiff}
 We say that two models $\vtheta := (\bff, p, \mG)$ and $\tilde{\vtheta}:= (\tilde \bff,\tilde p,\tilde \mG)$ satisfying Assumption~\ref{ass:diffeomorphism} are \textbf{equivalent up to diffeomorphism}, denoted $\vtheta \eqdiff \tilde\vtheta$, if and only if $\vf(\sR^{d_z}) = \tilde\vf(\sR^{d_z})$ and, for all $t \in [T]$, all $\va^{<t} \in \gA^t$ and all $\vz^{\leq t} \in \sR^{d_z\times t}$,
        \begin{align}\label{eq:density_v}
        \tilde p(\vz^{t} \mid \vz^{<{t}}, \va^{<{t}}) = p(\bfv(\vz^{t}) \mid \bfv(\vz^{<{t}}), \va^{<{t}}) |\det D\bfv(\vz^{{t}}) |\, ,
        \end{align}
    where $\vv := \vf^{-1} \circ \tilde\vf$ (entanglement map) is a diffeomorphism and $D\vv$ denotes its Jacobian matrix.
\end{definition}

The fact that the relation $\eqdiff$ is indeed an \textit{equivalence} comes from the fact that the set of diffeomorphisms from a set to itself forms a group under composition.

To better understand the above definition, let $\vz^t := \vg(\vz^{<t}, \va^{<t}; \epsilon^t)$ and $\tilde\vz^t := \tilde\vg(\tilde\vz^{<t}, \va^{<t}; \tilde\epsilon^t)$ where $\epsilon^t$ and $\tilde\epsilon^t$ are noise variables and $\vg$ and $\tilde\vg$ are functions such that random variables $\vz^t$ and $\tilde\vz^t$ have conditional densities given by $p(\vz^t \mid \vz^{<t}, \va^{<t})$ and $\tilde p(\tilde\vz^t \mid \tilde\vz^{<t}, \va^{<t})$, respectively. Using the change-of-variable formula for densities, one can rewrite \eqref{eq:density_v} as
\begin{align}
    \tilde\vg(\tilde\vz^{<t}, \va^{<t}; \tilde\epsilon^t) \overset{d}{=} \vv^{-1} \circ \vg(\vv(\tilde\vz^{<t}), \va^{<t}; \epsilon^t)\, , \label{eq:imitator}
\end{align}
where ``$\overset{d}{=}$'' denotes equality in distribution. This equation has a nice interpretation: applying the latent transition model $\tilde\vtheta$ to go from $(\tilde\vz^{<t}, \va^{<t})$ to $\tilde\vz^t$ is the same as first applying $\vv$, then applying the latent transition model $\vtheta$ and finally applying $\vv^{-1}$. Equation~\eqref{eq:imitator} is reminiscent of \citet{ahuja2022properties}, in which the mechanism $\tilde\vg$ would be called an \textit{imitator} of $\vg$. \citet{ahuja2022properties} showed that $\eqobs$ and $\eqdiff$ are actually one and the same. For completeness, we present an analogous argument here. We start by showing that $\vtheta \eqdiff \tilde\vtheta$ implies $\vtheta \eqobs \tilde\vtheta$. 
\begin{align*}
    p(\vx^{\leq T} \mid \va^{<T}) &= \int \prod_{t=1}^{T} \big[p(\vx^t \mid \vz^t)p(\vz^t \mid \vz^{< t}, \va^{<t})\big] d\vz^{\leq T} \\
    &= \int \prod_{t=1}^{T} \big[p(\vx^t \mid \vv(\vz^t))p(\vv(\vz^t) \mid \vv(\vz^{< t}), \va^{<t})\big] |\det D\vv(\vz^{\leq T})| d\vz^{\leq T} \\
    &= \int \prod_{t=1}^{T} \big[p(\vx^t \mid \vv(\vz^t))p(\vv(\vz^t) \mid \vv(\vz^{< t}), \va^{<t})|\det D\vv(\vz^{t})|\big] d\vz^{\leq T} \\
    &= \int \prod_{t=1}^{T} \big[\tilde{p}(\vx^t \mid \vz^t)\tilde{p}(\vz^t \mid \vz^{< t}, \va^{<t})\big] d\vz^{\leq T} = \tilde{p}(\vx^{\leq T} \mid \va^{<T})\,,
\end{align*}
where the second equality used the change-of-variable formula, the third equality used the fact that the Jacobian of $\vv(\vz^{\leq T})$ is block-diagonal (each block corresponds to a time step $t$) and the next to last equality used the definition of $\eqdiff$ and the fact that 
$$p(\vx^t \mid \vv(\vz^t)) = \gN(\vx^t; \vf (\vf^{-1} \circ \tilde\vf (\vz^t)), \sigma^2\mI) = \gN(\vx^t; \tilde\vf (\vz^t), \sigma^2\mI) = \tilde{p}(\vx^t \mid \vz^t)\,.$$ 

The following proposition establishes the converse, i.e. that $\vtheta \eqobs \tilde\vtheta$ implies $\vtheta \eqdiff \tilde\vtheta$. Since its proof is more involved, we present it in the Appendix~\ref{app:identDiffeo}. Note that this first identifiability result is relatively weak and should be seen as a first step towards stronger guarantees. A very similar result was shown by \citet[Theorem 3.1]{ahuja2022properties} to highlight the fact that the representation $\vf$ is identifiable up to the equivariances $\vv$ of the transition model $p$.

\begin{restatable}[Identifiability up to diffeomorphism]{proposition}{identDiffeo}
\label{prop:identDiffeo}
    Let $\vtheta := (\bff, p, \mG)$ and $\hat{\vtheta}:= (\hat \bff,\hat p,\hat \mG)$ be two models satisfying Assumption~\ref{ass:diffeomorphism}. If $\vtheta \eqobs \hat\vtheta$ (Def.~\ref{def:eqobs}), then $\vtheta \eqdiff \hat\vtheta$ (Def.~\ref{def:eqdiff}).
\end{restatable}

Intuitively, Proposition~\ref{prop:identDiffeo} shows that if two models agree on the distribution of the observations, then their ``data manifold'' $\vf(\sR^{d_z})$ and $\hat\vf(\sR^{d_z})$ are equal and their respective transition models are related via $\vv := \vf^{-1} \circ \hat\vf$.

\subsection{Disentanglement and Equivalence up to Permutation}\label{sec:disent_permeq}

A disentangled representation is often defined intuitively as a representation in which the coordinates are in one-to-one correspondence with \textit{natural factors of variation} in the data. We are going to assume that these natural factors are captured by an unknown ground-truth decoder $\vf$. Given a learned decoder $\hat\vf$ such that $\vf(\sR^{d_z}) = \hat \vf(\sR^{d_z})$, the entanglement map $\vv := \vf^{-1} \circ \hat\vf$ gives a correspondence between the learned representation $\hat\vf$ and the natural factors of variations of $\vf$. The following equivalence relation will help us define disentanglement.

\begin{definition}[Equivalence up to permutation]\label{def:perm_equ}
 We say that two models $\vtheta := (\bff, p, \mG)$ and $\tilde{\vtheta}:= (\tilde \bff,\tilde p,\tilde \mG)$ satisfying Assumptions~\ref{ass:diffeomorphism}, \ref{ass:cond_indep} \& \ref{ass:graph} are \textbf{equivalent up to permutation}, denoted $\vtheta \eqperm \tilde\vtheta$, if and only if there exists a permutation matrix $\mP$ such that
\begin{enumerate}
    \item $\vtheta \eqdiff \tilde\vtheta$ (Def.~\ref{def:eqdiff}) and $\tilde\mG^a = \mP \mG^a$ and $\tilde\mG^z =  \mP \mG^z \mP^\top$\, ; and
    \item The entanglement map $\vv := \vf^{-1}\circ\tilde\vf$ can be written as $\vv = \vd \circ \mP^\top$, where $\vd$ is element-wise, i.e. $\vd_i(\vz)$ depends only on $\vz_i$, for all $i$. In other words, the entanglement graph is $\mV = \mP^\top$.
\end{enumerate}
\end{definition}

The fact that the relation $\eqperm$ is an equivalence relation is actually a special case of a more general result that we present later in Section~\ref{sec:nonparam_ident_action}.

This allows us to give a formal definition of (complete) disentanglement. Note that we use the term \textit{complete} to contrast with \textit{partial} disentanglement.

\begin{definition}[Complete disentanglement]\label{def:disentanglement}
Given a ground-truth model $\vtheta$, we say that a learned model $\hat{\vtheta}$ is \textbf{completely disentangled} when $\vtheta \eqperm \hat{\vtheta}$.
\end{definition}

Intuitively, a learned representation is \textit{completely disentangled} when there is a one-to-one correspondence between its coordinates and those of the ground-truth representation (see Figure~\ref{fig:disentanglement_illustration}).

\begin{figure}
    \centering
    \includegraphics[width=0.5\linewidth]{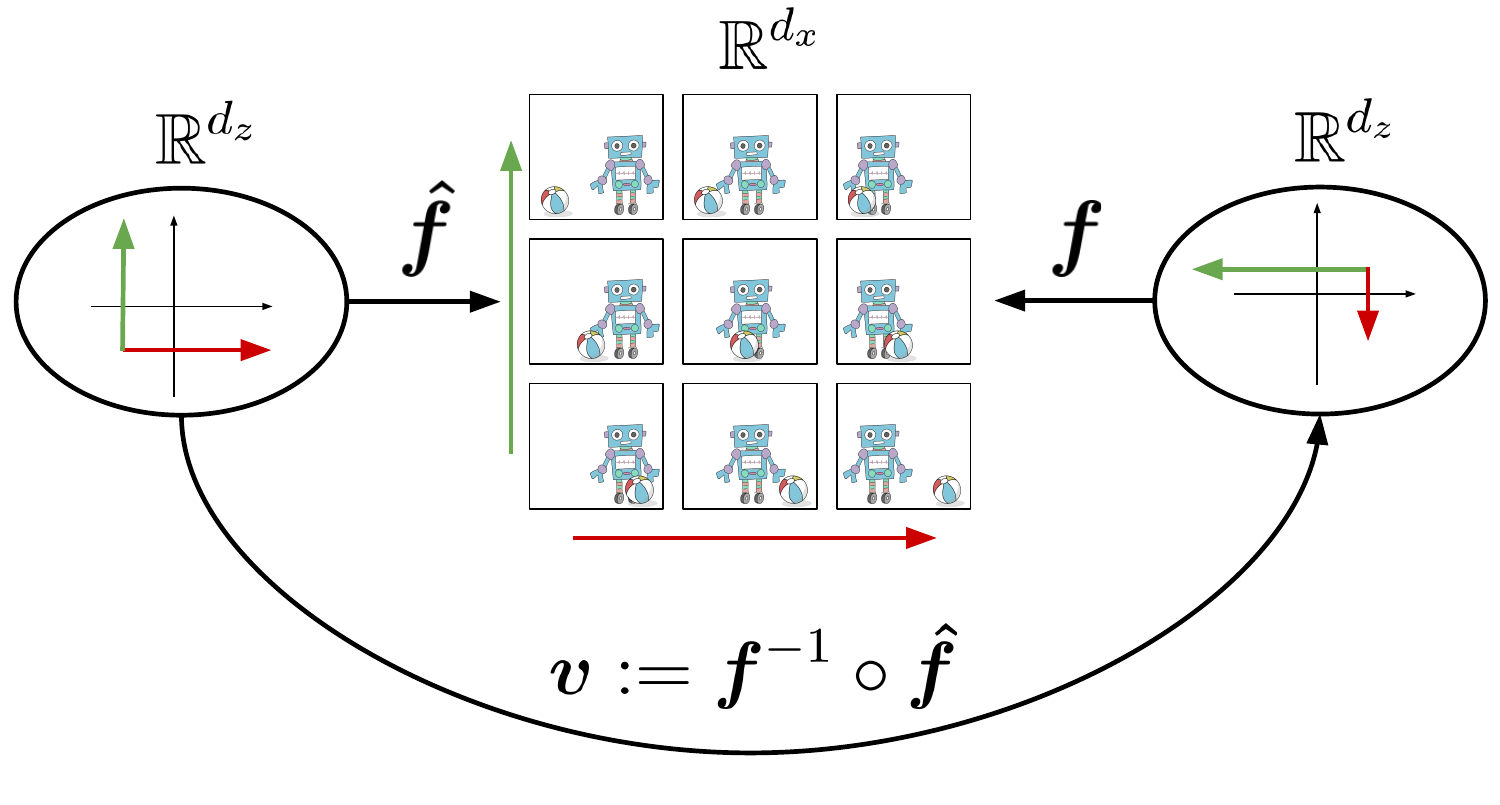}
    \caption{An illustration of disentanglement (Definition~\ref{def:disentanglement}). The ground-truth decoder $\vf$ captures the ``natural factors of variations'', which here are the $x$-positions of the robot and ball. The learned decoder $\hat\vf$ is disentangled here because each of its latent coordinates corresponds exactly to one object in the scene, in a one-to-one fashion. Mathematically, this is captured by the special structure of the entanglement map $\vv:= \vf^{-1} \circ \hat\vf$, which is a permutation composed with an element-wise invertible transformation.}
    \label{fig:disentanglement_illustration}
\end{figure}

We define \textit{partial} disentanglement as something that lives strictly between equivalence up to diffeomorphism and equivalence up to permutation:

\begin{definition}[Partial disentanglement]\label{def:partial_disentanglement}
Given a ground-truth model $\vtheta$, we say that a learned model $\hat\vtheta$ is \textbf{partially disentangled} when $\vtheta \eqdiff \hat\vtheta$ with an entanglement graph $\mV$ (Definition~\ref{def:entanglement_graphs}) that is neither a permutation nor the complete graph.
\end{definition}
This definition of partial disentanglement ranges from models that are almost completely entangled, i.e. those with very dense entanglement graphs $\mV$, to ones that are very close to being completely disentangled, i.e. those with a very sparse $\mV$. The following section will make more precise how one can learn a completely or partially disentangled representation from data and exactly what form the entanglement graph is going to take. 

\section{Nonparametric Partial Disentanglement via Mechanism Sparsity}\label{sec:nonparam_ident}
In this section, we show that without further assumptions the model introduced so far is unidentifiable (Section~\ref{sec:couter_example}), provide a first insight as to why mechanism sparsity can make the model identifiable and thus lead to disentanglement (Section~\ref{sec:insight_main}), introduce the machinery of $\mG$-preserving maps (Section~\ref{sec:consistency_equ}) which leads up to theorems showing identifiability up to $\va$-consistency (Section~\ref{sec:nonparam_ident_action}) and $\vz$-consistency (Section~\ref{sec:nonparam_ident_time}), which corresponds to partial disentanglement.  %when one regularizes $\hat\mG^\va$ or $\hat\mG^\vz$ to be sparse,
We also relate these results to interventions (Section~\ref{sec:unknown-target}), show how to combine both regularization on $\hat\mG^a$ and $\hat\mG^z$ to obtain stronger guarantees (Section~\ref{sec:combine_Ga_Gz}) and introduce a graphical criterion guaranteeing complete disentanglement~(Section~\ref{sec:graph_crit}). Finally, we introduce the \textit{sufficient influence assumptions} and prove the identifiability results (Section~\ref{sec:proofs}), and provide multiple examples to build intuition~(Section~\ref{sec:examples}).

\subsection{Conditional Independence is Insufficient for Disentanglement and Graph Discovery}\label{sec:couter_example}

In the setting where the sequence $\{\vz^t\}_{t=1}^T$ is fully observed, the absence of instantaneous causal effects (Assumption~\ref{ass:cond_indep}) combined with faithfulness (when the only conditional independencies in the distribution are those implied by the causal graph) makes the causal graph fully identifiable \citep[Section 10.3.1]{peters2017elements}. In our setting where only the sequence $\{\vf(\vz^t) + \vn^t\}_{t=1}^T$ is observed, conditional independence alone is insufficient to identify the latent causal graph or the latent variables, as the following example demonstrates.
\begin{example}
    Suppose $p(\vz^t \mid \vz^{<t}, \va^{<t}) := \gN(\vz^t ; \mW\vz^{t-1} + \va^{t-1}\ve_1, \eta^2 I)$ where $\va^t \in \gA = \sR$, $\mW \in \sR^{d\times d}$ is symmetric and $\ve_1 \in \sR^{d}$ is the first one-hot vector. Let $\mG^z$ and $\mG^a$ be the graph entailed by this model. Note that $\mG^z$ can be very dense while $\mG^a$ contains only a single edge. Choose $\vf$ to be any diffeomorphism on its image. This defines a model $\vtheta = (\vf, p, \mG)$ that satisfies the conditional independence assumption. We now construct a second model $\tilde\vtheta$ that is observationally equivalent to $\vtheta$, but has a drastically different graph $\tilde\mG$ and different latent factors. Since $\mW$ is symmetric, it can be diagonalized by an orthogonal matrix, i.e., $\mW = \mU \mD \mU^\top$, where $\mU$ is orthogonal and $\mD$ is diagonal. Define $\tilde\vf := \vf \circ \mU$, so that $\vf(\sR^d) = \tilde\vf(\sR^d)$ and $\vv := \vf^{-1} \circ \tilde\vf = \mU$. Define
    \begin{align*}
        \tilde p(\vz^t \mid \vz^{<t}, \va^{<t}) :=  p(\mU\vz^t \mid \mU \vz^{<t}, \va^{<t}) \,,
    \end{align*}
    which is a well-defined conditional density by the change-of-variable formula (using $|\det \mU| = 1$). We can see $\tilde p$ as the conditional density resulting from first applying $\mU$ to $\vz^{t-1}$, taking a step using $p$ to get $\vz^{t}$ and finally transforming $\vz^{t}$ using $\mU^\top$, analogously to \eqref{eq:imitator}. It is clear that $\tilde p$ satisfies the conditional independence (Assumption~\ref{ass:cond_indep}) since an orthogonal transformation of an isotropic Gaussian remains an isotropic Gaussian. Furthermore, the mean of $\tilde p(\vz^t \mid \vz^{<t}, \va^{<t})$ is given by $\mU^\top \mW \mU \vz^{t-1} + \va^{t-1}\mU^\top\ve_1 = \mD \vz^{t-1} + \va^{t-1}(\mU^\top)_{:, 1}$. We find that $\tilde p$ is Markov w.r.t. $\tilde\mG^z = \mI$, since $\mD$ is diagonal, and $\tilde\mG^a := \mathbbm{1}$, assuming $\mU$ is dense. Note that $\tilde\mG^z$ will be in general much sparser than $\mG^z$ and $\tilde\mG^a$ is potentially much denser than $\mG^a$ (how extreme these differences are depends on the exact choice of $\mW$). Moreover, in general, the representation is not disentangled, since $\vv = \mU$ can be chosen to be dense. Nevertheless, these two models are observationally equivalent since they are equivalent up to diffeomorphism by construction~(Section~\ref{sec:diffeq}). 
\end{example}

\textcolor{black}{The preceding example illustrates some difficulties one might face when it comes to the identifiability of the model introduced so far. This serves as a motivation to explore additional inductive biases such as mechanism sparsity, which we do next.}

\subsection{A First Mathematical Insight for Disentanglement via Mechanism Sparsity}\label{sec:insight_main}

In this section, we derive a first insight pointing towards how mechanism sparsity regularization, i.e., regularizing $\hat\mG$ to be sparse, can promote disentanglement.

Before going further, we briefly introduce an abuse of notation that will be handy throughout: we will sometimes use vectors and matrices as sets of indices corresponding to their supports.
\begin{definition}[Vectors \& matrices as index sets]
    Let $\va \in \sR^{n}$ and $\mA \in \sR^{m \times n}$. We will sometimes use $\va$ to denote the set of indices corresponding to the support of the vector $\va$, i.e.
    \begin{align*}
        \va \sim \{i \in [n] \mid \va_i \not= 0\}\,.
    \end{align*}
    This will allow us to write things like $i \in \va$ or $\va \subseteq \vb$, where $\vb \in \sR^n$. We will use an analogous convention for matrices, i.e.,
    \begin{align*}
        \mA \sim \{(i,j) \in [m]\times[n] \mid \mA_{i,j} \not= 0\}\,,
    \end{align*}
    This will allow us to write things like $(i,j) \in \mA$ and $\mA \subseteq \mB$, where $\mB \in \sR^{m \times n}$.
\end{definition}

Recall that we would like to show that $\vtheta \eqobs \hat\vtheta$ implies $\vtheta \eqperm \hat\vtheta$, i.e., disentanglement (or partial disentanglement). Our approach will be to start from~\eqref{eq:density_v}, which is guaranteed by Proposition~\ref{prop:identDiffeo}, and perform a series of algebraic manipulations to gain mathematical insight into how regularizing $\hat\mG$ to be sparse (mechanism sparsity) can induce disentanglement. A key manipulation will be taking first and second order derivatives. For this to be possible, we require a certain level of smoothness for the transition models:
\begin{assumption}[Smoothness of transition model]\label{ass:smooth_trans}
    When $\va$ is continuous, the transition densities $p(\vz_i^t \mid \vz^{<t}, \va^{<t})$ are $C^2$ functions from $\sR \times \sR^{d_z \times (t-1)}\times \gA^t$ to $\sR$ and $\gA \subseteq \sR^\ell$ is regular closed\footnote{A set $\gA \subseteq \sR^\ell$ is regular closed when it is equal to the closure of its interior, i.e. $\overline{\gA^\circ} = \gA$. %This is useful since a $C^k$ function defined on a regular closed set has well defined derivatives everywhere on its domain in the sense that they do not depend on the choice of $C^k$ extension. Lemma~\ref{lem:equal_derivatives} in appendix shows this fact.
    }. When $\va$ is discrete (e.g. Section~\ref{sec:unknown-target}), for all $\va^{<t}$, $p(\vz_i^t \mid \vz^{<t}, \va^{<t})$ are $C^2$ functions from $\sR \times \sR^{d_z \times (t-1)}$ to $\sR$.
\end{assumption}

We start by taking the $\log$ on both sides of \eqref{eq:density_v} and let $q := \log p$ and $\hat q := \log \hat p$:
\begin{align*}
    \hat q(\vz^{t} \mid \vz^{<{t}}, \va^{<{t}}) = {q}(\bfv(\vz^{t}) \mid \bfv(\vz^{<{t}}), \va^{<{t}}) + \log|\det D\bfv(\vz^{{t}}) |\, .
\end{align*}
We then take the derivative w.r.t. $\vz^t$ on both sides:
\begin{align}
    D^t_z \hat q(\vz^{t} \mid \vz^{<{t}}, \va^{<{t}}) = D^t_z {q}(\bfv(\vz^{t}) \mid \bfv(\vz^{<{t}}), \va^{<{t}})D\vv(\vz^t) + \eta(\vz^t) \in \sR^{1 \times d_z}\, , \label{eq:first_deriv}
\end{align}
where $D^t_z q$ denotes the Jacobian of $q(\vz^{t} \mid \vz^{<{t}}, \va^{<{t}})$ w.r.t. $\vz^t$ and analogously for $D^t_z \hat q$. The term $\eta(\vz^t)$ is the derivative of $\log|\det D\bfv(\vz^{{t}}) |$ w.r.t. $\vz^t$. 

We differentiate\footnote{This derivative is well defined on $\gA$ (in the sense that it does not depend on its $C^k$ extension) since $\gA$ is regular close. We prove this general fact in Lemma~\ref{lem:equal_derivatives} in the appendix.} yet once more w.r.t. $\va^\tau$ for some $\tau < t$ (assuming $\va^t$ is continuous for now) and obtain
\begin{align}
    H^{t, \tau}_{z, a} \hat q(\vz^{t} \mid \vz^{<{t}}, \va^{<{t}}) = D\vv(\vz^t)^\top H^{t,\tau}_{z, a} {q}(\bfv(\vz^{t}) \mid \bfv(\vz^{<{t}}), \va^{<{t}}) \in \sR^{d_z \times d_a}\, , \label{eq:hessians_intuitive}
\end{align}
where $H^{t, \tau}_{z, a} q \in \sR^{d_z \times d_a}$ is the Hessian matrix of second derivatives w.r.t. $\vz^t$ and $\va^\tau$ and similarly for $H^{t, \tau}_{z, a} \hat q$.

We now look more closely at some specific entry $(i,\ell)$ of the Hessian $H^{t, \tau}_{z, a} q$. We first see that
\begin{align}
    \frac{\partial^2 }{\partial \va_\ell^\tau\partial \vz_i^t}q(\vz^t \mid \vz^{<t}, \va^{<t}) &= \frac{\partial^2 }{\partial \va_\ell^\tau\partial \vz_i^t}\sum_{j=1}^{d_z} q(\vz_j^t \mid \vz_{\Pa^z_j}^{<t}, \va_{\Pa^a_j}^{<t}) \label{eq:crucial_argument} \\
    &= \frac{\partial }{\partial \va_\ell^\tau} \sum_{j=1}^{d_z} \frac{\partial }{\partial \vz_i^t}q(\vz_j^t \mid \vz_{\Pa^z_j}^{<t}, \va_{\Pa^a_j}^{<t}) \\
    &= \frac{\partial }{\partial \va_\ell^\tau} \frac{\partial }{\partial \vz_i^t}q(\vz_i^t \mid \vz_{\Pa^z_i}^{<t}, \va_{\Pa^a_i}^{<t}) \,, \label{eq:zero_derivative}
\end{align}
where the first equality holds by~\eqref{eq:cond_indep}~\&~\eqref{eq:parents} and a basic property of logarithms. It is clear that \eqref{eq:zero_derivative} equals zero when $\ell \not\in \Pa_i^a$. This is a crucial observation, since it implies that whenever $\mG^a_{i, \ell} = 0$, we also have $(H^{t,\tau}_{z, a} {q})_{i, \ell} = 0$. In other words, $H^{t,\tau}_{z, a} {q} \subseteq \mG^a$. The same argument can also be applied to get $H^{t,\tau}_{z, a} \hat{q} \subseteq {\hat\mG^a}$. Notice how this argument would fail without the conditional independence assumption, since $\vz^t_i$ could appear in terms other than $q(\vz_i^t \mid \vz_{\Pa^z_i}^{<t}, \va_{\Pa^a_i}^{<t})$, as it could be the parent of some other $\vz_j^t$, for $j \not= i$.

\textbf{Intuitive argument.} We can start to see why regularizing $\hat\mG$ to be sparse might induce disentanglement. Intuitively, a sparse $\hat\mG^a$ forces $D\vv(\vz^t)$ to be sparse, since otherwise the l.h.s. of ~\eqref{eq:hessians_intuitive_2} will not be sparse:
\begin{align}
    \underbrace{H^{t, \tau}_{z, a} \hat q(\vz^{t} \mid \vz^{<{t}}, \va^{<{t}})}_{\subseteq \hat\mG^a} = \underbrace{D\vv(\vz^t)^\top}_\text{forced to be sparse} \underbrace{H^{t,\tau}_{z, a} {q}(\bfv(\vz^{t}) \mid \bfv(\vz^{<{t}}), \va^{<{t}})}_{\subseteq \mG^a} \, , \label{eq:hessians_intuitive_2}
\end{align}

And, of course, the sparser $D\vv(\vz^t)$ is, the more disentangled $\hat\vf$ is, since $D\vv_{i,j} = 0$ everywhere implies $\mV_{i,j}=0$ under weak assumptions (Proposition~\ref{prop:V_and_Dv}). The above argument is not rigorous and is provided only to build intuition. It will be formalized later. 

\textbf{Sparse temporal dependencies.} In what precedes, we made use of the sparsity of the graph $\hat\mG^a$ to argue that $D\vv$ must also be sparse. We now show a similar intuition based on the sparsity of $\hat\mG^z$. Starting from \eqref{eq:first_deriv}, instead of differentiating w.r.t. $\va^\tau$, we will differentiate w.r.t. $\vz^\tau$, for some $\tau < t$, which yields:
\begin{align}
    H^{t, \tau}_{z, z} \hat q(\vz^{t} \mid \vz^{<{t}}, \va^{<{t}}) = D\vv(\vz^t)^\top H^{t,\tau}_{z, z} {q}(\bfv(\vz^{t}) \mid \bfv(\vz^{<{t}}), \va^{<{t}})D\vv(\vz^\tau) \in \sR^{d_z \times d_z}\, , \label{eq:hessians_intuitive_time}
\end{align}
where $H^{t,\tau}_{z, z} {q}$ is the Hessian matrix of second derivatives of $q$ w.r.t. $\vz^t$ and $\vz^\tau$, and analogously for $H^{t,\tau}_{z, z} \hat{q}$. Using an argument perfectly analogous to Equations~\eqref{eq:crucial_argument}~to~\eqref{eq:zero_derivative}, we can show that whenever $\mG^z_{i,j} = 0$, we also have $(H^{t,\tau}_{z,z}q)_{i,j} = 0$, and similarly for $\hat\mG^z$ and $H^{t,\tau}_{z,z}\hat q$. In other words, $H^{t,\tau}_{z,z}q \subseteq \mG^z$ and $H^{t,\tau}_{z,z}\hat q \subseteq \hat\mG^z$. Therefore, analogously to~\eqref{eq:hessians_intuitive_2}, regularizing $\hat\mG^z$ to be sparse intuitively should force $D\vv$ to be sparse as well, i.e., bringing us closer to disentanglement:
\begin{align}
    \underbrace{H^{t, \tau}_{z, z} \hat q(\vz^{t} \mid \vz^{<{t}}, \va^{<{t}})}_{\subseteq \hat\mG^z} = \underbrace{D\vv(\vz^t)^\top}_\text{forced to be sparse} \underbrace{H^{t,\tau}_{z, z} {q}(\bfv(\vz^{t}) \mid \bfv(\vz^{<{t}}), \va^{<{t}})}_{\subseteq \mG^z}\underbrace{D\vv(\vz^\tau)}_\text{forced to be sparse} \, .\label{eq:hessians_intuitive_time_2}
\end{align}

The crux of our technical contribution in this work is to make the above arguments formal and precisely characterize what the sparsity structure of $D\vv(\vz)$ (hence, $\mV$) will be based on the ground-truth graph $\mG$ (Theorems~\ref{thm:nonparam_dis_cont_a}, \ref{thm:nonparam_dis_disc_a} \& \ref{thm:nonparam_dis_z}). We also provide conditions on $\mG$ to guarantee complete disentanglement (Proposition~\ref{def:disentanglement}). 

\subsection{Graph Preserving Maps}\label{sec:consistency_equ}
Theorems~\ref{thm:nonparam_dis_cont_a}, \ref{thm:nonparam_dis_disc_a}, \ref{thm:nonparam_dis_z} \& \ref{thm:expfam_dis_z} will show how regularizing $\hat\mG$ to be sparse can force the dependency graph of the entanglement map $\vv$ to be sparse as well. These results characterize the functional dependency structure of the entanglement map $\vv$ as a function of the ground-truth graph $\mG$. This link will be made precise thanks to the notion of graph preserving maps, which we define next. Before going further, we need to set up the following notation.

\begin{restatable}[Aligned subspaces of $\sR^m$ and $\sR^{m\times n}$]{definition}{AlignSubspaces} Given a binary vector $\vb \in \{0, 1\}^{m}$, let
\begin{align*}
    \sR^m_\vb := \{\vx \in \sR^m \mid \vb_i = 0 \implies \vx_i = 0\} \, 
\end{align*}
Given a binary matrix $\mB \in \{0,1\}^{m\times n}$, let
\begin{align*}
    \sR^{m\times n}_\mB := \{\mM \in \sR^{m \times n} \mid \mB_{i,j}=0 \implies \mM_{i,j} = 0\} \, .
\end{align*}
\end{restatable}
Note that $\sR^m_\vb$ and $\sR^{m\times n}_\mB$ are vector spaces under addition. This means that given $\va^{(1)}, \dots, \va^{(k)} \in \sR^m_\vb$, we have that $\textnormal{span} \{\va^{(1)}, \dots, \va^{(k)}\} \subset \sR^m_\vb$, where $\vecspan$ denotes the subspace of all linear combinations. Similarly, given $\mA^{(1)}, \dots, \mA^{(k)} \in \sR^{m\times n}_\mB$, we have that $\textnormal{span} \{\mA^{(1)}, \dots, \mA^{(k)}\} \subseteq \sR^{m\times n}_\mB$. 

To start formally reasoning about what will happen if $\hat\mG^a$ is regularized sparse, we temporarily assume that $\hat\mG^a = \mG^a$. With this assumption, we can interpret \eqref{eq:hessians_intuitive_2} as implying that $D\vv(\vz^t)^\top$ must \textit{preserve} the ``sparsity structure'' of the matrix $H^{t,\tau}_{z,a}q$. This observation motivates the following definitions, which will be central to our contribution.  

\begin{definition}[$\mG$-preserving matrix]\label{def:g_preserving_mat}
    Given $\mG \in \{0,1\}^{m \times n}$, a matrix $\mC \in \sR^{m\times m}$ is \textbf{$\mG$-preserving} when
    $$\mC^\top\sR^{m \times n}_\mG \subseteq \sR^{m \times n}_\mG\,.$$
\end{definition}

\begin{definition}[$\mG$-preserving functions]\label{def:g_preserving_map}
    Given $\mG \in \{0,1\}^{m \times n}$, a function $\vc:\sR^m \rightarrow \sR^m$ is \textbf{$\mG$-preserving} when its dependency graph $\mC$ (Definition~\ref{def:dep_graph}) is $\mG$-preserving.
\end{definition}

Without surprise, a linear map $\vc(\vz) := \mC\vz$ where $\mC \in \sR^{m\times m}$ is $\mG$-preserving (Definition~\ref{def:g_preserving_map}) if and only if the matrix $\mC$ is $\mG$-preserving (Definition~\ref{def:g_preserving_mat}).

We now show that $\mG$-preserving functions can be defined alternatively in terms of a simple condition on their dependency graph. This characterization of $\mG$-preserving functions is key to understanding how (partial) disentanglement results from sparsity regularization. 
\begin{proposition}\label{prop:charac_G_preserving}
    A function $\vc$ with dependency graph $\mC$ (Definition~\ref{def:dep_graph}) is $\mG$-preserving if and only
    $$\mG_{i,\cdot} \not\subseteq \mG_{j,\cdot} \implies \mC_{i,j} = 0,\ \text{for all}\ i,j\,.$$
\end{proposition}
\begin{proof}
    We start by showing the ``only if'' statement. We suppose $\mG_{i,\cdot} \not\subseteq \mG_{j, \cdot}$ and must now show that $\mC_{i,j} =0$. We know that there exists $k$ such that $\mG_{i,k} = 1$ but $\mG_{j,k} = 0$. Since $\mC^\top\sR^{m\times n}_\mG \subseteq \sR^{m\times n}_\mG$ and $\ve_i\ve_k^\top \in \sR^{m\times n}_\mG$, we must have $\mC^\top\ve_i\ve_k^\top \in \sR^{m\times n}_\mG$. Since $\mG_{j,k} = 0$, we must have $0 = (\mC^\top\ve_i\ve_k^\top)_{j,k} = \mC_{i,j}$.
    
    We now show the ``if'' statement. Let $\mA \in \sR^{m\times n}_\mG$. Take some $(i,j)$ such that $\mG_{i,j} = 0$. We must now show that $(\mC^\top\mA)_{i,j} = 0$. We have $(\mC^\top\mA)_{i,j} = \sum_{k}\mC_{k,i}\mA_{k,j}$. We now check that each term in this sum must be zero. If $\mA_{k,j} = 0$, of course the corresponding term is zero. If $\mA_{k,j} \not= 0$, it implies that $\mG_{k,j} = 1$ and thus $\mG_{k,\cdot} \not\subseteq \mG_{i,\cdot}$. By assumption, this implies that $\mC_{k, i} = 0$ and thus $\mC_{k,i}\mA_{k,j} = 0$. Hence $(\mC^\top\mA)_{i,j} = 0$ as desired. 
\end{proof}

We now characterize differentiable $\mG$-preserving functions in terms of their Jacobian matrices.

\begin{lemma}\label{lem:s_consistent_jac}
    A differentiable function $\vc:\sR^m \rightarrow \sR^m$ is $\mG$-preserving if and only if, for all $\vz \in \sR^m$, $D\vc(\vz)$ is $\mG$-preserving. 
\end{lemma}
\begin{proof}
    Assume $\vc$ is $\mG$-preserving with dependency graph $\mC$. By Proposition~\ref{prop:charac_G_preserving}, this is equivalent to having that, for all $i,j \in [n]$, 
    \begin{align*}
        \mG_{i, \cdot} \not\subseteq \mG_{j,\cdot} \implies \mC_{i,j} = 0 \,.
    \end{align*}
    But by Proposition~\ref{prop:V_and_Dv}, this statement is equivalent to
    \begin{align*}
        \mG_{i, \cdot} \not\subseteq \mG_{j,\cdot} \implies \forall \vz \in \sR^m,\ D\vc(\vz)_{i,j} = 0 \,,
    \end{align*}
    which is equivalent to saying that $D\vc(\vz)$ is $\mG$-preserving for all $\vz \in \sR^m$ (again by Proposition~\ref{prop:charac_G_preserving}).
\end{proof}

We will now show that $\mG$-preserving diffeomorphisms form a group under composition. To do so, we start by showing that invertible $\mG$-preserving matrices form a group under matrix multiplication (Proposition~\ref{prop:group_of_S_consistent}) and extend the result to diffeomorphisms in Proposition~\ref{prop:group_of_S_consistent_diffeo}.

\begin{proposition}\label{prop:group_of_S_consistent}
    Invertible $\mG$-preserving matrices form a group under matrix multiplication.
\end{proposition}
\begin{proof}
    We must show that the set of invertible $\mG$-preserving matrices contains the identity, is closed under matrix multiplication, and is closed under inversion.

    Clearly, $\mI$ is $\mG$-preserving since $\mI^\top \sR^{m\times n}_\mG = \sR^{m\times n}_\mG$.

    Let $\mC_1$ and $\mC_2$ be $\mG$-preserving. Then, $\mC_1\mC_2$ is $\mG$-preserving because 
    $$(\mC_1\mC_2)^\top\sR^{m\times n}_\mG = \mC_2^\top\mC_1^\top\sR^{m\times n}_\mG \subset \mC_2^\top\sR^{m\times n}_\mG \subset \sR^{m\times n}_\mG\,.$$

    Let $\mC$ be $\mG$-preserving and invertible. Since $\mC^\top$ is invertible as a map from $\sR^{m\times n}$ to $\sR^{m\times n}$, the dimensionality of the subspace $\sR^{m\times n}_\mG$ must be equal to the dimensionality of $\mC^\top\sR^{m\times n}_\mG$. This fact combined with $\mC^\top\sR^{m\times n}_\mG \subseteq \sR^{m\times n}_\mG$ implies that $\mC^\top\sR^{m\times n}_\mG = \sR^{m\times n}_\mG$. Hence $\sR^{m\times n}_\mG = (\mC^{-1})^\top\sR^{m\times n}_\mG$, i.e. $\mC^{-1}$ is $\mG$-preserving. 
\end{proof}

We now extend the above results to diffeomorphisms using Proposition~\ref{prop:V_and_Dv}.

\begin{restatable}{proposition}{GroupSConsistentDiffeo}\label{prop:group_of_S_consistent_diffeo}
    The set of $\mG$-preserving diffeomorphims forms a group under composition.
\end{restatable}
\begin{proof}
    We must show that the set of $\mG$-preserving diffeomorphisms contains the identity, is closed under matrix multiplication, and is closed under inversion.
    
    The first statement is trivial since the entanglement graph of the identity diffeomorphism is the identity graph $\mC := \mI$, and of course it is $\mG$-preserving.
    
   We now prove the second statement. Let $\vc$ and $\vc'$ be two diffeomorphisms with dependency graph $\mC$ and $\mC'$ respectively. By the chain rule, we have that 
    \begin{align*}
        D(\vc \circ \vc')(\vz) = D\vc(\vc'(\vz))D\vc'(\vz) \, .
    \end{align*}
    By Lemma~\ref{lem:s_consistent_jac}, we have that $D\vc(\vc'(\vz))$ and $D\vc'(\vz)$ and $\mG$-preserving matrices and, by Proposition~\ref{prop:group_of_S_consistent} their product must also be $\mG$-preserving. Hence $D(\vc \circ \vc')(\vz)$ is $\mG$-preserving for all $\vz$ and thus, by Lemma~\ref{lem:s_consistent_jac}, $\vc \circ \vc'$ is $\mG$-preserving. 
    %and by Lemma~\ref{lem:s_consistent_jac}, since the maps $\vc$ and $\vc'$ are $\mS$-preserving, their Jacobian matrices $D\vc(\vc'(\vz))$ and $D\vc'(\vz)$ also are everywhere. By closure of $\mS$-preserving matrices under multiplication (Proposition~\ref{prop:group_of_S_consistent}) we have that $D(\vc \circ \vc')(\vz)$ is $\mS$-preserving, for all $\vz$. Because $\vc \circ \vc'$ is $C^1$, we can use Lemma~\ref{lem:s_consistent_jac} to conclude that $\vc \circ \vc'$ is $\mS$-preserving.

    The proof of the third statement has a similar flavor. By the inverse function theorem, we have
    \begin{align*}
        D\vc^{-1}(\vz) = D\vc(\vc^{-1}(\vz))^{-1} \, .
    \end{align*}
    Moreover, by Lemma~\ref{lem:s_consistent_jac}, $D\vc(\vc^{-1}(\vz))$ is $\mG$ preserve. Furthermore, its inverse is also $\mG$-preserving by Proposition~\ref{prop:group_of_S_consistent}. Similarly to the previous step, because $\vc^{-1}$ is $C^1$, we can use Lemma~\ref{lem:s_consistent_jac} to conclude that $\vc^{-1}$ is also $\mG$-preserving.
\end{proof}

\subsection{Nonparametric Identifiability via Auxiliary Variables with Sparse Influence} \label{sec:nonparam_ident_action}

In this section, we introduce our first identifiability results based on the sparsity of the graph $\mG^a$ which describes the structure of the dependencies between $\va^{<t}$ and $\vz^t$. We will see that, under some assumptions, regularizing the learned graph $\hat\mG^a$ to be sparse will allow identifiability up to the following equivalence class:

%\seb{How should we call this equivalence relation now that we changed to ``$\mG$-preserving'' terminology?}

%\seb{I think I found a very short proof that this is an equivalence relation. We can use the fact that, for two subgroups $T$ and $S$, the set $TS$ is a group iff $TS=ST$ (See Exercise 11.9 of Chapter 2 in \cite{artin2013algebra}). Then this equivalence relation is induced by the action of this group on the set of models. This should make a few proofs simpler... Actually I'm not sure... since somehow the model (which includes $\mG^a$) and the group of $\vc \circ \mP^\top$ are "entangled"... That's odd... Ok I think that works if we take $T$ to be all permutations, and $S$ to be all maps wich are $\mP\mG^a$-preserving for some $\mP$. Is $S$ a group here?}

%\seb{After quite a bit of thinking... I couldn't write this equivalence relation as being induced by a group action. I feel extremely close to being there, but not sure it's worth it...}
\begin{definition}[$\va$-consistency equivalence]\label{def:a_consistent_models}
We say two models $\vtheta := (\bff, p, \mG)$ and $\tilde{\vtheta}:= (\tilde \bff,\tilde p, \tilde \mG)$ satisfying Assumptions~\ref{ass:diffeomorphism}, \ref{ass:cond_indep} \& \ref{ass:graph} are \textbf{$\va$-consistent}, denoted $\vtheta \eqcon^\va \tilde\vtheta$, if and only if there exists a permutation matrix $\mP$ such that
\begin{enumerate}
    \item $\vtheta \eqdiff \tilde{\vtheta}$ (Def.~\ref{def:eqdiff}), and $\tilde\mG^a = \mP\mG^a$\, ; and
    \item the entanglement map $\vv := \vf^{-1} \circ \tilde\vf$ can be written as $\vv = \vc \circ \mP^\top$ where $\vc$ is a $\mG^a$-preserving diffeomorphism (Def.~\ref{def:g_preserving_map}). % with complete diagonal \seb{Do I really want to include "complete diagonal" in the equivalence relation definition? In a sense, it's not really fundammental since the inverse of a complete diagonal map is does not necessarily have complete diagonal...}.
\end{enumerate}
\end{definition}

The main difference between $\va$-consistency (above definition) and permutation equivalence (Definition~\ref{def:perm_equ}), is that instead of having $\vv = \vd \circ \mP^\top$ where $\vd$ is element-wise, we have $\vv = \vc \circ \mP^\top$ where $\vc$ is $\mG^a$-preserving, which allows some mixing between the latent factors. Importantly, a $\mG^a$-preserving map typically has missing edges in its dependency graph, as Proposition~\ref{prop:charac_G_preserving} shows. This means that this equivalence relation imposes structure on the entanglement map $\vv$. %This $\va$-consistency relates the structure of the entanglement map $\vv$ to the graph $\mG^a$. 
Depending on the structure of $\mG^a$, this can mean either complete, partial, or no disentanglement whatsoever. Note that the equivalence $\eqperm$ is stronger than $\eqcon^\va$, in the sense that $\vtheta \eqperm \hat\vtheta \implies \vtheta \eqcon^\va \hat\vtheta$. This is because element-wise transformations $\vd$ are always $\mG$-preserving, for any $\mG$. %This is because, for any $\mS$, $\mS_{i,\cdot} \subseteq \mS_{i, \cdot}$ for all $i$.

We demonstrate in Appendix~\ref{sec:proof_consistence_equivalence} that the $\va$-consistency relation is indeed an equivalence relation, as claimed in the above definition. This follows from the fact that the set of $\mG^a$-preserving diffeomorphisms forms a \textit{group} under composition (Proposition~\ref{prop:group_of_S_consistent_diffeo}).

The first result provides conditions under which regularizing the learned graph $\hat\mG^a$ to be as sparse as the ground-truth graph $\mG^a$ will induce the learned model to be $\va$-consistent with the ground-truth one. 

\begin{restatable}[Nonparametric disentanglement from continuous $\va$ with sparse influence]{theorem}{NonparamDisContA}\label{thm:nonparam_dis_cont_a}
    Let the parameters $\vtheta := (\bff, p, \mG)$ and $\hat{\vtheta}:= (\hat \bff,\hat p, \hat \mG)$ correspond to two models satisfying Assumptions~\ref{ass:diffeomorphism}, \ref{ass:cond_indep}, \ref{ass:graph}, \& \ref{ass:smooth_trans}. Further assume that 
    \begin{enumerate}
        \item \textbf{[Observational equivalence]} $\vtheta \eqobs \hat\vtheta$ (Def.~\ref{def:eqobs}); 
    %\end{enumerate}
    %Then, by Proposition~\ref{prop:identDiffeo}, $\vtheta \eqdiff \hat\vtheta$ and, in particular, $\vf(\gZ) = \tilde\vf(\gZ)$. This means we can consider the entanglement graph $\mV$ of the pair $(\vf, \hat\vf)$.
    %\begin{enumerate}[resume]
        \item \textbf{[Sufficient influence of $\va$]} The Hessian matrix $H^{t,\tau}_{z, a} \log p (\vz^{t} \mid \vz^{<{t}}, \va^{<{t}})$ varies ``sufficiently'', as formalized in Assumption~\ref{ass:nonparam_suff_var_a_cont};
    \end{enumerate}
    Then, there exists a permutation matrix $\mP$ such that $\mP\mG^a \subseteq \hat\mG^a$. Further assume that
    \begin{enumerate}[resume]
        \item \textbf{[Sparsity regularization]} $||\hat\mG^a||_0 \leq ||\mG^a||_0$;
    \end{enumerate}
    Then, $\vtheta \eqcon^\va \hat\vtheta$ (Def.~\ref{def:a_consistent_models}).
\end{restatable}

The second assumption as well as a proof of this result is delayed to Section~\ref{sec:proofs} for pedagogical reasons. We now describe and provide intuition about each assumption one by one.

\paragraph{Observational equivalence.} The first assumption simply requires that both models agree on the observational model. In practice, this is achieved by fitting the model to data. 

\paragraph{Sufficient influence.} The second assumption requires that the ``effect'' of $\va^{<t}$ on $\vz^t$ is ``sufficiently strong''. The assumption will be formalized and discussed in more detail later in Sections~\ref{sec:proofs} \& \ref{sec:examples}, but we can already see that it concerns the Hessian matrix $H^{t,\tau}_{z, a} \log p$ that we saw earlier in Eq. \eqref{eq:hessians_intuitive_2} of Sect.~\ref{sec:insight_main}. 

\paragraph{Sparsity regularization.} The first two assumptions imply that the learned graph $\hat{\mG}^a$ is a supergraph of some permutation of the ground-truth graph $\mG^a$. By adding the \textit{sparsity regularization} assumption, we find that the learned graph $\hat{\mG}^a$ \textit{is exactly} a permutation of the ground-truth graph $\mG^a$ and that, more precisely, the learned model is $\eqcon^\va$-equivalent to the ground-truth. This assumption is satisfied if $\hat{\mG}^a$ is a minimal graph among all graphs that allow the model to exactly match the ground-truth generative distribution. In Section~\ref{sec:estimation}, we suggest achieving this in practice by adding a sparsity penalty in the training objective or by constraining the optimization problem. 

\paragraph{$\va$-consistency.} The final conclusion of the result states that the learned model is $\eqcon^\va$-equivalent to the ground-truth, which means the entanglement map $\vv := \vf^{-1} \circ \hat\vf$ can be written as $\vv = \vc \circ \mP^\top$ where $\vc$ is $\mG^a$-preserving. This is important since the $\mG^a$-preserving condition imposes structure on the entanglement graph $\mV$ (Definition~\ref{def:entanglement_graphs}), as implied by Proposition~\ref{prop:charac_G_preserving}. In other words, the result predicts precisely which latent factors are expected to remain entangled. 

\begin{remark}[Inverse of $\vv$]\label{rem:v_inverse}
    We defined $\vv$ to be the mapping from the learned to the ground-truth representation, but in some context, it might be more telling to look at $\vv^{-1}$, which maps from the ground-truth to the learned representation. If $\vv = \vc \circ \mP^\top$ where $\vc$ is $\mG^a$-preserving (as predicted by Theorem~\ref{thm:nonparam_dis_cont_a}), we know that its inverse is given by $\vv^{-1} = \mP \circ \vc^{-1}$ where $\vc^{-1}$ is $\mG^a$-preserving, by closure under inversion (Proposition~\ref{prop:group_of_S_consistent_diffeo}).
\end{remark}

The following result is the same as above, but for \textit{discrete} auxiliary variables $\va$. This case is very important to cover the case where $\va$ indexes sparse interventions targeting latent factors, which we discuss in more detail in Section~\ref{sec:unknown-target}. Note that the only difference from the above theorem is the ``sufficient influence'' assumption, which we will present formally in Section~\ref{sec:proofs} together with a proof of the result.

\begin{restatable}[Nonparametric disentanglement via discrete $\va$ with sparse influence]{theorem}{NonparamDisDiscA}\label{thm:nonparam_dis_disc_a}
    Let the parameters $\vtheta := (\bff, p, \mG)$ and $\hat{\vtheta}:= (\hat \bff,\hat p, \hat \mG)$ correspond to two models satisfying Assumptions~\ref{ass:diffeomorphism}, \ref{ass:cond_indep}, \ref{ass:graph} \& \ref{ass:smooth_trans}. Further assume that
    \begin{enumerate}
        \item \textbf{[Observational equivalence]} $\vtheta \eqobs \hat\vtheta$ (Def.~\ref{def:eqobs});  
        \item \textbf{[Sufficient influence of $\va$]} The vector of derivatives $D^{t}_{z} \log p (\vz^{t} \mid \vz^{<{t}}, \va^{<{t}})$ depends ``sufficiently strongly'' on each component $\va_\ell$, as formalized in Assumption~\ref{ass:nonparam_suff_var_a};
    \end{enumerate}
    Then, there exists a permutation matrix $\mP$ such that $\mP\mG^a \subseteq \hat\mG^a$. Further assume that
    \begin{enumerate}[resume]
        \item \textbf{[Sparsity regularization]} $||\hat\mG^a||_0 \leq ||\mG^a||_0$;
    \end{enumerate}
    Then, $\vtheta \eqcon^\va \hat\vtheta$ (Def.~\ref{def:a_consistent_models}).
\end{restatable}

We now provide a few examples to illustrate how Theorems~\ref{thm:nonparam_dis_cont_a}~\&~\ref{thm:nonparam_dis_disc_a} can be applied. Here, we concentrate on the relationship between the graph $\mG^a$ and the entanglement graph $\mV$ (Definition~\ref{def:entanglement_graphs}). The question of whether or not the sufficient influence assumption is satisfied will be delayed to Section~\ref{sec:examples}, where the examples will be made more concrete by specifying latent models more explicitly.

\begin{example}[$\mG^a = \mI$ implies complete disentanglement]\label{ex:single_node_complete_dis}
Assume that $d_a = d_z$ and $\mG^a = \mI$, i.e., each latent variable is affected by only one auxiliary variable, and each auxiliary variable affects only one latent variable. The graph $\mG^a$ is depicted in Figure~\ref{fig:Ga_a} and $\mG^z$ could be anything (see the remark below). Assuming that the ground-truth transition model satisfies the sufficient influence assumption of Theorem~\ref{thm:nonparam_dis_cont_a} or \ref{thm:nonparam_dis_disc_a}, we have $\vtheta \eqobs \hat\vtheta\ \&\ ||\hat\mG^a||_0 \leq ||\mG^a||_0 \implies \vtheta \eqcon^\va \hat\vtheta$. This means that there exists a permutation matrix $\mP$ such that $\hat\mG^a = \mP\mG^a$ and such that the entanglement map is given by $\vv = \vc\circ\mP^\top$ where $\vc$ is a $\mG^a$-preserving diffeomorphism~(Definition~\ref{def:g_preserving_mat}). But since $\mG^a = \mI$, Proposition~\ref{prop:charac_G_preserving} tells us that the dependency graph of $\vc$ is simply $\mC := \mI$ and thus the entanglement graph is $\mV = \mP^\top$, i.e., complete disentanglement holds. In fact, one could add more columns to $\mG^a$ (i.e. adding auxiliary variables) without changing the conclusion. Example~\ref{ex:single_node_complete_dis_2} will provide a concrete example satisfying the sufficient influence assumption of Theorem~\ref{thm:nonparam_dis_disc_a}.
\end{example}

\begin{remark}[Temporal dependencies are not necessary]\label{rem:no_temp}
    The above example did not mention anything about the temporal graph $\mG^z$. That is because this graph could be anything, in fact, we could be in the special case where there are no temporal dependencies whatsoever, i.e., $T=1$ and the latent model is simply $p(\vz \mid \va) = \prod_{i=1}^{d_z}p(\vz_i \mid \va)$. In that case, Theorems~\ref{thm:nonparam_dis_cont_a}~\&~\ref{thm:nonparam_dis_disc_a} could still be applied to prove the identifiability of the representation, as long as their assumptions hold. This remark also applies to the next two examples.%This setting is thus a special case of the ``causal representation learning'' problem where the latent variables are mutually independent (i.e. no causal links between them). 
\end{remark}

\begin{example}[Action targeting a single latent variable identifies it]\label{ex:a_target_one_z}
    Consider the situation depicted in Figure~\ref{fig:working_example} where $\vz_1$ is the tree position, $\vz_2$ is the robot position and $\vz_3$ is the ball position ($d_z = 3$). Assume $\va \in \sR$ corresponds to the torque applied to the robot wheels ($d_a = 1$). We therefore have $\mG^a = [0, 1, 0]^\top$, i.e., $\va$ affects only $\vz_2$. For the sake of this example, $\mG^z$ can be anything, i.e., it does not have to be lower triangular like in Figure~\ref{fig:working_example} (see remark above).
    
    If the sufficient influence assumption of Theorem~\ref{thm:nonparam_dis_cont_a} or \ref{thm:nonparam_dis_disc_a} is satisfied, we see that $\vtheta \eqobs \hat\vtheta\ \&\ ||\hat\mG^a||_0 \leq ||\mG^a||_0$ implies $\vv = \vc\circ\mP^\top$ where $\mP$ is a permutation and $\vc$ is a $\mG^a$-preserving diffeomorphism. Using Proposition~\ref{prop:charac_G_preserving}, this means that the dependency graph of $\vc$ is given by
    \begin{align*}
        \mC = \begin{bmatrix}
            * & * & * \\
            0 & * & 0 \\
            * & * & * 
        \end{bmatrix} \text{, since $\mG^a_{2, \cdot} \not\subseteq \mG^a_{1, \cdot}$ and $\mG^a_{2, \cdot} \not\subseteq \mG^a_{3, \cdot}$\, ,}
    \end{align*}
    where ``$*$'' indicates a potentially nonzero value. This means that one of the components of the learned representation will be an invertible transformation of the ground-truth variable $\vz_2$ (robot position), while the other components could be a mixture of $\vz_1$, $\vz_2$ and $\vz_3$. Figure~\ref{fig:Ga_b} shows both the graph $\mG^a$ and the corresponding entanglement graph $\mV$ assuming $\mP = \mI$. Example~\ref{ex:a_target_one_z_cont_a} will make this example more concrete by explicitly specifying a latent model that satisfies the sufficient influence assumption of Theorem~\ref{thm:nonparam_dis_cont_a}. 
\end{example}

\begin{example}[Complete disentanglement from multi-target actions] \label{ex:multi_target_a}
    Assume $d_z = 3$ and $d_a = 3$ where $\mG^a \in \sR^{d_z \times d_a}$ is given by Figure~\ref{fig:Ga_c} and the temporal graph $\mG^z$ could be anything (see Remark~\ref{rem:no_temp} above). If the sufficient influence assumption of Theorem~\ref{thm:nonparam_dis_cont_a} or \ref{thm:nonparam_dis_disc_a} is satisfied, then we have that $\vtheta \eqobs \hat\vtheta\ \&\ ||\hat\mG^a||_0 \leq ||\mG^a||_0$ implies $\vv = \vc\circ\mP^\top$ where $\mP$ is a permutation and $\vc$ is a $\mG^a$-preserving diffeomorphism. Proposition~\ref{prop:charac_G_preserving} implies that the dependency graph of $\vc$ is simply $\mC := \mI$ because $\mG^a_{i, \cdot} \not\subseteq \mG^a_{j, \cdot}$ for all distinct $i, j$. This means that we have \textit{complete} disentanglement (Definition~\ref{def:disentanglement}). Examples~\ref{ex:multi_target_a_cont_a}, \ref{ex:multinode_linear_gauss} and \ref{ex:group_interv} will explore more concrete instantiations of this example by specifying concrete latent models satisfying the sufficient influence assumptions of Theorems~\ref{thm:nonparam_dis_cont_a} and \ref{thm:nonparam_dis_disc_a}.
\end{example}

%\def\HorizNodeSpace{1cm}
%\def\VertNodeSpace{0.3cm}
%\def\GraphSpace{1.5cm}
%\def\VertLabelSpace{1.2cm}
%\def\Padding{2cm}
%\usetikzlibrary{positioning}
%\usetikzlibrary{calc}
%\usetikzlibrary{decorations.pathreplacing}
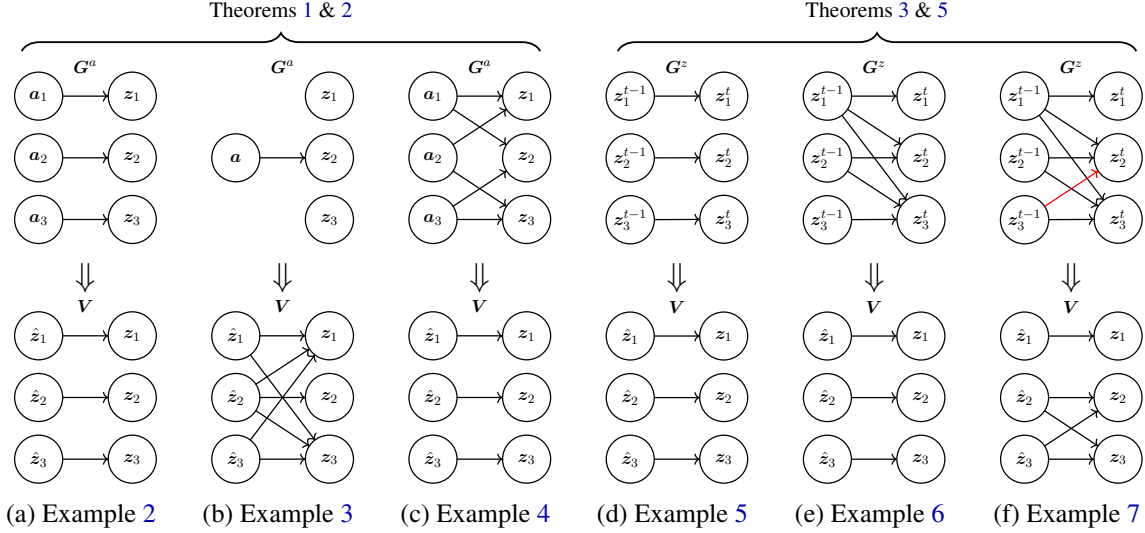
\begin{figure}
    \centering
    \resizebox{0.48\linewidth}{!}{
    %\usetikzlibrary{snakes}
\usetikzlibrary{decorations.pathreplacing}
\begin{tikzpicture}
    % Draw the brace
    \draw [decorate, decoration={brace,amplitude=5pt},line width=0.5pt] (0,0) -- (5,0);
    
    % Add the text above the brace
    \node at (2.5,0.4) {\tiny Theorems~\ref{thm:nonparam_dis_cont_a}~\&~\ref{thm:nonparam_dis_disc_a}};
\end{tikzpicture}}
    \hfill
    \resizebox{0.48\linewidth}{!}{
    \usetikzlibrary{decorations.pathreplacing}
\begin{tikzpicture}
    % Draw the brace
    \draw [decorate, decoration={brace,amplitude=5pt},line width=0.5pt] (0,0) -- (5,0);
    
    % Add the text above the brace
    \node at (2.5,0.4) {\tiny Theorems~\ref{thm:nonparam_dis_z}~\&~\ref{thm:expfam_dis_z}};
\end{tikzpicture}}
     \begin{subfigure}[b]{0.14\textwidth}
         \centering
         \resizebox{1.0\linewidth}{!}{
    \usetikzlibrary{positioning}
\usetikzlibrary{calc}
\begin{tikzpicture}[node distance=1cm, thick, x=\Padding, y=\Padding]
  \tikzset{main/.style = {
    shape          = circle,
    draw,
    text           = black,
    inner sep      = 2pt,
    outer sep      = 0pt,
    minimum size   = 30 pt
}}
  
  % ---- LEFT GRAPH ---- %
  % Nodes
  \node[main] (a1) {$\va_1$};
  \node[main] (a2) [below=\VertNodeSpace of a1]{$\va_2$};
  \node[main] (a3) [below=\VertNodeSpace of a2]{$\va_3$};
  \node[main] (z1) [right=\HorizNodeSpace of a1]{$\vz_1$};
  \node[main] (z2) [below=\VertNodeSpace of z1]{$\vz_2$};
  \node[main] (z3) [below=\VertNodeSpace of z2]{$\vz_3$};
  % Label
  \node[above=\VertLabelSpace] at ($(a3.north)!0.5!(z1.north)$) {$\mG^a$};
  
  % Edges
  \draw[->] (a1) -- (z1);
  \draw[->] (a2) -- (z2);
  \draw[->] (a3) -- (z3);

  % ---- RIGHT GRAPH ---- %
  % Nodes
  %\node[main] (z1_hat) [below=\GraphSpace of a3]{$\hat\vz_1$};
  %\node[main] (z2_hat) [below=\VertNodeSpace of z1_hat]{$\hat\vz_2$};
  %\node[main] (z3_hat) [below=\VertNodeSpace of z2_hat]{$\hat\vz_3$};
  %\node[main] (z1_) [right=\HorizNodeSpace of z1_hat]{$\vz_1$};
  %\node[main] (z2_) [below=\VertNodeSpace of z1_]{$\vz_2$};
  %\node[main] (z3_) [below=\VertNodeSpace of z2_]{$\vz_3$};

  \node[main] (z1_hat) [below=\GraphSpace of a3]{$\hat\vz_{B_1}$};
  \node[main] (z2_hat) [below=\VertNodeSpace of z1_hat]{$\hat\vz_{B_2}$};
  \node[main] (z3_hat) [below=\VertNodeSpace of z2_hat]{$\hat\vz_{B_3}$};
  \node[main] (z1_) [right=\HorizNodeSpace of z1_hat]{$\vz_{B_1}$};
  \node[main] (z2_) [below=\VertNodeSpace of z1_]{$\vz_{B_2}$};
  \node[main] (z3_) [below=\VertNodeSpace of z2_]{$\vz_{B_3}$};
  % Label
  \node[above=\VertLabelSpace] at ($(z1_hat.north)!0.5!(z3_.north)$) {$\mV$};
  
  % Edges
  %\draw[->] (z1_hat) -- (z1_);
  %\draw[->] (z2_hat) -- (z2_);
  %\draw[->] (z3_hat) -- (z3_);
  \draw[->] (z1_hat) -- (z2_);
  \draw[->] (z2_hat) -- (z3_);
  \draw[->] (z3_hat) -- (z1_);

  \node (arrow) at ($(a3)!0.5!(z1_)$) {{\huge$\Downarrow$}};
  %\node [right=0 of arrow] {\hspace{-0.3cm}\footnotesize$\begin{array}{l}
  %     \text{Thm.~\ref{thm:nonparam_dis_cont_a}}\\
  %     \text{Thm.~\ref{thm:nonparam_dis_disc_a}}
  %\end{array}$};
\end{tikzpicture}}
    \caption{Example~\ref{ex:single_node_complete_dis}}
         \label{fig:Ga_a}
     \end{subfigure}
     \hfill
     \begin{subfigure}[b]{0.14\textwidth}
         \centering
         \resizebox{1.0\linewidth}{!}{
    \usetikzlibrary{positioning}
\usetikzlibrary{calc}
\begin{tikzpicture}[node distance=1cm, thick, x=\Padding, y=\Padding]
  \tikzset{main/.style = {
    shape          = circle,
    draw,
    text           = black,
    inner sep      = 2pt,
    outer sep      = 0pt,
    minimum size   = 30 pt
}}
\tikzset{invisible/.style = {
    shape          = circle,
    draw,
    text           = black,
    inner sep      = 2pt,
    outer sep      = 0pt,
    minimum size   = 30 pt,
    opacity        = 0 pt
}}
  
  % ---- LEFT GRAPH ---- %
  % Nodes
  \node[main] (a) {$\va$};
  \node[invisible] (a_mock) [below=\VertNodeSpace of a]{};
  \node[main] (z2) [right=\HorizNodeSpace of a]{$\vz_2$};
  \node[main] (z1) [above=\VertNodeSpace of z2]{$\vz_1$};
  \node[main] (z3) [below=\VertNodeSpace of z2]{$\vz_3$};
  % Label
  \node[above=\VertLabelSpace] at ($(a_mock.north)!0.5!(z1.north)$) {$\mG^a$};
  
  % Edges
  \draw[->] (a) -- (z2);

  % ---- RIGHT GRAPH ---- %
  % Nodes
  \node[main] (z1_) [below=\GraphSpace of z3]{$\vz_1$};
  \node[main] (z2_) [below=\VertNodeSpace of z1_]{$\vz_2$};
  \node[main] (z3_) [below=\VertNodeSpace of z2_]{$\vz_3$};
  \node[main] (z1_hat) [left=\HorizNodeSpace of z1_]{$\hat\vz_1$};
  \node[main] (z2_hat) [left=\HorizNodeSpace of z2_]{$\hat\vz_2$};
  \node[main] (z3_hat) [left=\HorizNodeSpace of z3_]{$\hat\vz_3$};
  % Label
  \node[above=\VertLabelSpace] at ($(z1_hat.north)!0.5!(z3_.north)$) {$\mV$};
  
  % Edges
  \draw[->] (z1_hat) -- (z1_);
  \draw[->] (z1_hat) -- (z3_);
  \draw[->] (z2_hat) -- (z1_);
  \draw[->] (z2_hat) -- (z2_);
  \draw[->] (z2_hat) -- (z3_);
  \draw[->] (z3_hat) -- (z3_);
  \draw[->] (z3_hat) -- (z1_);

  \node (arrow) at ($(z3)!0.5!(z1_hat)$) {\huge$\Downarrow$};
  %\node [right=0 of arrow] {\hspace{-0.3cm}\footnotesize$\begin{array}{l}
  %     \text{Thm.~\ref{thm:nonparam_dis_cont_a}}\\
  %     \text{Thm.~\ref{thm:nonparam_dis_disc_a}}
  %\end{array}$};
\end{tikzpicture}}
    \caption{Example~\ref{ex:a_target_one_z}}
         \label{fig:Ga_b}
     \end{subfigure}
     \hfill
     \begin{subfigure}[b]{0.14\textwidth}
         \centering
         \resizebox{1.0\linewidth}{!}{
    \usetikzlibrary{positioning}
\usetikzlibrary{calc}
\begin{tikzpicture}[node distance=1cm, thick, x=\Padding, y=\Padding]
  \tikzset{main/.style = {
    shape          = circle,
    draw,
    text           = black,
    inner sep      = 2pt,
    outer sep      = 0pt,
    minimum size   = 30 pt
}}
  
  % ---- LEFT GRAPH ---- %
  % Nodes
  \node[main] (a1) {$\va_1$};
  \node[main] (a2) [below=\VertNodeSpace of a1]{$\va_2$};
  \node[main] (a3) [below=\VertNodeSpace of a2]{$\va_3$};
  \node[main] (z1) [right=\HorizNodeSpace of a1]{$\vz_1$};
  \node[main] (z2) [below=\VertNodeSpace of z1]{$\vz_2$};
  \node[main] (z3) [below=\VertNodeSpace of z2]{$\vz_3$};
  % Label
  \node[above=\VertLabelSpace] at ($(a1.north)!0.5!(z3.north)$) {$\mG^a$};
  
  % Edges
  \draw[->] (a1) -- (z1);
  \draw[->] (a1) -- (z2);
  \draw[->] (a2) -- (z1);
  \draw[->] (a2) -- (z3);
  \draw[->] (a3) -- (z2);
  \draw[->] (a3) -- (z3);

  % ---- RIGHT GRAPH ---- %
  % Nodes
  \node[main] (z1_hat) [below=\GraphSpace of a3]{$\hat\vz_1$};
  \node[main] (z2_hat) [below=\VertNodeSpace of z1_hat]{$\hat\vz_2$};
  \node[main] (z3_hat) [below=\VertNodeSpace of z2_hat]{$\hat\vz_3$};
  \node[main] (z1_) [right=\HorizNodeSpace of z1_hat]{$\vz_1$};
  \node[main] (z2_) [below=\VertNodeSpace of z1_]{$\vz_2$};
  \node[main] (z3_) [below=\VertNodeSpace of z2_]{$\vz_3$};
  % Label
  \node[above=\VertLabelSpace] at ($(z1_hat.north)!0.5!(z3_.north)$) {$\mV$};
  
  % Edges
  \draw[->] (z1_hat) -- (z1_);
  \draw[->] (z2_hat) -- (z2_);
  \draw[->] (z3_hat) -- (z3_);

  \node (arrow) at ($(a3)!0.5!(z1_)$) {\huge$\Downarrow$};
  %\node [right=0 of arrow] {\hspace{-0.3cm}\footnotesize$\begin{array}{l}
  %     \text{Thm.~\ref{thm:nonparam_dis_cont_a}}\\
  %     \text{Thm.~\ref{thm:nonparam_dis_disc_a}}
  %\end{array}$};
\end{tikzpicture}}
        \caption{Example~\ref{ex:multi_target_a}}
         \label{fig:Ga_c}
     \end{subfigure}
     \hfill
     \begin{subfigure}[b]{0.14\textwidth}
         \centering
         \resizebox{1.0\linewidth}{!}{
    \usetikzlibrary{positioning}
\usetikzlibrary{calc}
\begin{tikzpicture}[node distance=1cm, thick, x=\Padding, y=\Padding]
  \tikzset{main/.style = {
    shape          = circle,
    draw,
    text           = black,
    inner sep      = 2pt,
    outer sep      = 0pt,
    minimum size   = 30 pt
}}
  
  % ---- LEFT GRAPH ---- %
  % Nodes
  \node[main] (a1) {$\vz^{t-1}_1$};
  \node[main] (a2) [below=\VertNodeSpace of a1]{$\vz^{t-1}_2$};
  \node[main] (a3) [below=\VertNodeSpace of a2]{$\vz^{t-1}_3$};
  \node[main] (z1) [right=\HorizNodeSpace of a1]{$\vz^t_1$};
  \node[main] (z2) [below=\VertNodeSpace of z1]{$\vz^t_2$};
  \node[main] (z3) [below=\VertNodeSpace of z2]{$\vz^t_3$};
  % Label
  \node[above=\VertLabelSpace] at ($(a3.north)!0.5!(z1.north)$) {$\mG^z$};
  
  % Edges
  \draw[->] (a1) -- (z1);
  \draw[->] (a2) -- (z2);
  \draw[->] (a3) -- (z3);

  % ---- RIGHT GRAPH ---- %
  % Nodes
  \node[main] (z1_hat) [below=\GraphSpace of a3]{$\hat\vz_1$};
  \node[main] (z2_hat) [below=\VertNodeSpace of z1_hat]{$\hat\vz_2$};
  \node[main] (z3_hat) [below=\VertNodeSpace of z2_hat]{$\hat\vz_3$};
  \node[main] (z1_) [right=\HorizNodeSpace of z1_hat]{$\vz_1$};
  \node[main] (z2_) [below=\VertNodeSpace of z1_]{$\vz_2$};
  \node[main] (z3_) [below=\VertNodeSpace of z2_]{$\vz_3$};
  % Label
  \node[above=\VertLabelSpace] at ($(z1_hat.north)!0.5!(z3_.north)$) {$\mV$};
  
  % Edges
  \draw[->] (z1_hat) -- (z1_);
  \draw[->] (z2_hat) -- (z2_);
  \draw[->] (z3_hat) -- (z3_);

  \node (arrow) at ($(a3)!0.5!(z1_)$) {\huge$\Downarrow$};
  %\node [right=0 of arrow] {\hspace{-0.3cm}\footnotesize$\begin{array}{l}
  %     \text{Thm.~\ref{thm:nonparam_dis_z}}\\
  %     \text{Thm.~\ref{thm:expfam_dis_z}}
  %\end{array}$};
\end{tikzpicture}}
        \caption{Example~\ref{ex:diagonal_deps}}
         \label{fig:Gz_cd}
     \end{subfigure}
     \hfill
     \begin{subfigure}[b]{0.14\textwidth}
         \centering
         \resizebox{1.0\linewidth}{!}{
    \usetikzlibrary{positioning}
\usetikzlibrary{calc}
\begin{tikzpicture}[node distance=1cm, thick, x=\Padding, y=\Padding]
  \tikzset{main/.style = {
    shape          = circle,
    draw,
    text           = black,
    inner sep      = 2pt,
    outer sep      = 0pt,
    minimum size   = 30 pt
}}
  
  % ---- LEFT GRAPH ---- %
  % Nodes
  \node[main] (a1) {$\vz^{t-1}_1$};
  \node[main] (a2) [below=\VertNodeSpace of a1]{$\vz^{t-1}_2$};
  \node[main] (a3) [below=\VertNodeSpace of a2]{$\vz^{t-1}_3$};
  \node[main] (z1) [right=\HorizNodeSpace of a1]{$\vz^t_1$};
  \node[main] (z2) [below=\VertNodeSpace of z1]{$\vz^t_2$};
  \node[main] (z3) [below=\VertNodeSpace of z2]{$\vz^t_3$};
  % Label
  \node[above=\VertLabelSpace] at ($(a3.north)!0.5!(z1.north)$) {$\mG^z$};
  
  % Edges
  \draw[->] (a1) -- (z1);
  \draw[->] (a1) -- (z2);
  \draw[->] (a1) -- (z3);
  \draw[->] (a2) -- (z2);
  \draw[->] (a2) -- (z3);
  \draw[->] (a3) -- (z3);

  % ---- RIGHT GRAPH ---- %
  % Nodes
  \node[main] (z1_hat) [below=\GraphSpace of a3]{$\hat\vz_1$};
  \node[main] (z2_hat) [below=\VertNodeSpace of z1_hat]{$\hat\vz_2$};
  \node[main] (z3_hat) [below=\VertNodeSpace of z2_hat]{$\hat\vz_3$};
  \node[main] (z1_) [right=\HorizNodeSpace of z1_hat]{$\vz_1$};
  \node[main] (z2_) [below=\VertNodeSpace of z1_]{$\vz_2$};
  \node[main] (z3_) [below=\VertNodeSpace of z2_]{$\vz_3$};
  % Label
  \node[above=\VertLabelSpace] at ($(z1_hat.north)!0.5!(z3_.north)$) {$\mV$};
  
  % Edges
  \draw[->] (z1_hat) -- (z1_);
  \draw[->] (z2_hat) -- (z2_);
  \draw[->] (z3_hat) -- (z3_);

  \node (arrow) at ($(a3)!0.5!(z1_)$) {\huge$\Downarrow$};
  %\node [right=0 of arrow] {\hspace{-0.3cm}\footnotesize$\begin{array}{l}
  %     \text{Thm.~\ref{thm:nonparam_dis_z}}\\
  %     \text{Thm.~\ref{thm:expfam_dis_z}}
  %\end{array}$};
\end{tikzpicture}}
        \caption{Example~\ref{ex:lower_triangular_no_action}}
         \label{fig:Gz_d}
     \end{subfigure}
     \hfill
     \begin{subfigure}[b]{0.14\textwidth}
         \centering
         \resizebox{1.0\linewidth}{!}{
    \usetikzlibrary{positioning}
\usetikzlibrary{calc}
\begin{tikzpicture}[node distance=1cm, thick, x=\Padding, y=\Padding]
  \tikzset{main/.style = {
    shape          = circle,
    draw,
    text           = black,
    inner sep      = 2pt,
    outer sep      = 0pt,
    minimum size   = 30 pt
}}
  
  % ---- LEFT GRAPH ---- %
  % Nodes
  \node[main] (a1) {$\vz^{t-1}_1$};
  \node[main] (a2) [below=\VertNodeSpace of a1]{$\vz^{t-1}_2$};
  \node[main] (a3) [below=\VertNodeSpace of a2]{$\vz^{t-1}_3$};
  \node[main] (z1) [right=\HorizNodeSpace of a1]{$\vz^t_1$};
  \node[main] (z2) [below=\VertNodeSpace of z1]{$\vz^t_2$};
  \node[main] (z3) [below=\VertNodeSpace of z2]{$\vz^t_3$};
  % Label
  \node[above=\VertLabelSpace] at ($(a3.north)!0.5!(z1.north)$) {$\mG^z$};
  
  % Edges
  \draw[->] (a1) -- (z1);
  \draw[->] (a1) -- (z2);
  \draw[->] (a1) -- (z3);
  \draw[->] (a2) -- (z2);
  \draw[->] (a2) -- (z3);
  \draw[->] (a3) -- (z3);
  \draw[->, red] (a3) -- (z2);

  % ---- RIGHT GRAPH ---- %
  % Nodes
  \node[main] (z1_hat) [below=\GraphSpace of a3]{$\hat\vz_1$};
  \node[main] (z2_hat) [below=\VertNodeSpace of z1_hat]{$\hat\vz_2$};
  \node[main] (z3_hat) [below=\VertNodeSpace of z2_hat]{$\hat\vz_3$};
  \node[main] (z1_) [right=\HorizNodeSpace of z1_hat]{$\vz_1$};
  \node[main] (z2_) [below=\VertNodeSpace of z1_]{$\vz_2$};
  \node[main] (z3_) [below=\VertNodeSpace of z2_]{$\vz_3$};
  % Label
  \node[above=\VertLabelSpace] at ($(z1_hat.north)!0.5!(z3_.north)$) {$\mV$};
  
  % Edges
  \draw[->] (z1_hat) -- (z1_);
  \draw[->] (z2_hat) -- (z2_);
  \draw[->] (z2_hat) -- (z3_);
  \draw[->] (z3_hat) -- (z3_);
  \draw[->] (z3_hat) -- (z2_);

  \node (arrow) at ($(a3)!0.5!(z1_)$) {\huge$\Downarrow$};
  %\node [right=0 of arrow] {\hspace{-0.3cm}\footnotesize$\begin{array}{l}
  %     \text{Thm.~\ref{thm:nonparam_dis_z}}\\
  %     \text{Thm.~\ref{thm:expfam_dis_z}}
  %\end{array}$};
\end{tikzpicture}}
        \caption{Example~\ref{ex:temporal_partial}}
         \label{fig:Gz_e}
     \end{subfigure}
        \caption{Graphs $\mG^a$ and $\mG^z$ from Examples~\ref{ex:single_node_complete_dis}, \ref{ex:a_target_one_z}, \ref{ex:multi_target_a}, \ref{ex:diagonal_deps}, \ref{ex:lower_triangular_no_action} \& \ref{ex:temporal_partial} with their respective entanglement graphs $\mV$ (Definition~\ref{def:entanglement_graphs}) guaranteed by Theorems~\ref{thm:nonparam_dis_cont_a}, \ref{thm:nonparam_dis_disc_a}, \ref{thm:nonparam_dis_z} \& \ref{thm:expfam_dis_z} (assuming $\mP = \mI$ for simplicity). Recall, that $\mV$ describes the dependency structure of $\vv = \vf^{-1} \circ \hat\vf$, which maps $\hat\vz$ to $\vz$. By Remark~\ref{rem:v_inverse}, the functional dependency graph of $\vv^{-1}$ is exactly the same except for $\vz$ and $\hat\vz$ being interchanged.}
        \label{fig:Ga}
\end{figure}

\subsubsection{Unknown-Target Interventions on the Latent Factors}\label{sec:unknown-target}
An important special case of Theorem~\ref{thm:nonparam_dis_disc_a} is when $\va^{t-1}$ corresponds to a one-hot vector indexing an \textit{intervention with unknown targets} on the latent variables $\vz^t$. This specific kind of intervention has been explored previously in the context of causal discovery where the intervention occurs on \textit{observed} variables instead of \textit{latent} variables~\citep{eaton2007exact, JCI_jmlr, ut_igsp, NEURIPS2020_6cd9313e, dcdi, ke2019learning}. Recently, multiple works in causal representation learning have considered interventions on latent variables~\citep{lachapelle2022disentanglement,lippe2022icitris,ahuja2023interventional,squires2023linear,buchholz2023learning,vonkugelgen2023nonparametric,zhang2023identifiability,jiang2023learning} (see Section~\ref{sec:lit_review} for more). Here is how our framework can accommodate such interventions: Assume $\va^{t-1} \in \{\vec{0}, \ve_1, ..., \ve_{d_a}\}$, where each $\ve_\ell$ is a one-hot vector. The action $\va^{t-1} = \vec{0}$ corresponds to the \textit{observational setting}, i.e., when no intervention occurred, while $\va^{t-1} = \ve_\ell$ corresponds to the $\ell$th intervention. In that context, the unknown graph $\mG^a$ describes which latents are targeted by the intervention, i.e. $\ell \in \Pa^a_{i}$ if and only if $\vz_i$ is targeted by the $\ell$th intervention. To see this, recall that, under Assumption~\ref{ass:graph}, we have 
\begin{align*}
    p(\vz_i^t \mid \vz^{<t}, \va^{<t}) = p(\vz_i^t \mid \vz^{<t}_{\Pa_i^z}, \va^{t-1}_{\Pa_i^a}) \,,
\end{align*}
where we implicitly assumed that $p(\vz_i^t \mid \vz^{<t}, \va^{<t})$ does not depend on $\va^{<t-1}$. 
In the observational setting, i.e., when $\va^{t-1} = \vec{0}$, the conditional becomes $p(\vz_i^t \mid \vz^{<t}, \vec{0})$. Now suppose that we are in the $\ell$th intervention, i.e., $\va^{t-1} = \ve_\ell$. Then, if $\ell \not\in \Pa^a_i$, we have that $\va^{t-1}_{\Pa^a_i} = \vec{0}$, which means that the conditional is also $p(\vz_i^t \mid \vz^{<t}, \vec{0})$, meaning that the variable $\vz_i^t$ \textit{is not} targeted by the $\ell$th intervention. When $\ell \in \Pa^a_i$, we have $\va^{t-1}_{\Pa^a_i} \not= \vec{0}$ and thus the conditional is allowed to change freely, i.e., $\vz^t_i$ \textit{is} targeted by the $\ell$th intervention.

Importantly, the assumption that $\mG^a$ is sparse corresponds precisely to the \textit{sparse mechanism shift} hypothesis from~\cite{scholkopf2021causal}, i.e. that \textit{only a few mechanisms change at a time.} Thm.~\ref{thm:nonparam_dis_disc_a} thus provides precise conditions for when sparse mechanism shifts induce disentanglement. Interestingly, our theory covers both hard and soft interventions, as long as the sufficient influence assumption is satisfied.

\begin{remark}[Examples revisited]
    Examples~\ref{ex:single_node_complete_dis}, \ref{ex:a_target_one_z} and \ref{ex:multi_target_a} can be revisited while keeping in mind the ``unknown-target intervention interpretation'' in which $\mG^a$ describes which latent variable is targeted by each intervention. For example, Example~\ref{ex:single_node_complete_dis} tells us that if each latent variable is targeted by a single-node intervention, then complete disentanglement is guaranteed. Examples~\ref{ex:single_node_complete_dis_2}, \ref{ex:multinode_linear_gauss} and \ref{ex:group_interv} provide mathematically concrete latent models where $\va$ is interpreted to be an intervention.
\end{remark}

\begin{remark}[Causal representation learning without temporal dependencies (static)]\label{rem:static_CRL_and_us}
    The special case where $T=1$, i.e., no temporal dependencies, is of special interest. In that case, the latent variable model is simply $p(\vz \mid \va) = \prod_{i=1}^{d_z} p(\vz_i \mid \va)$. In other words, the causal graph that relates the latent factors $\vz_i$ is empty. In contrast, recent work on learning causal representation showed how to obtain disentanglement in general latent causal graphical models without temporal dependencies, but are limited to single-node interventions~\citep{ahuja2023interventional,squires2023linear,buchholz2023learning,vonkugelgen2023nonparametric,zhang2023identifiability,jiang2023learning}. \textcolor{black}{These works can be roughly thought of as fitting our interventional setting described in this section where (i) $T=1$; (ii) conditional independence (Assumption~\ref{eq:cond_indep}) is dropped (allowing instantaneous causal effects); and (iii) only single-node interventions are allowed, corresponding to a very sparse $\mG^a$ with columns containing at most one non-zero element, like in Example~\ref{ex:single_node_complete_dis}.} Although our framework with $T=1$ assumes that the causal graph between latent variables is empty, it allows for multi-node interventions which are sometimes sufficient to disentangle (Example~\ref{ex:group_interv}). See Section~\ref{sec:suff_inf_disc_a} for more on this.
\end{remark}

\subsection{Nonparametric Identifiability via Sparse Temporal Dependencies} \label{sec:nonparam_ident_time}
This section is analogous to the previous one, but instead of leveraging the sparsity of $\mG^a$ to show identifiability, it leverages the sparsity of $\mG^z$, which describes the structure of the dependencies between the latents from one time step to another. We will see that, under some assumptions, regularizing the learned graph $\hat\mG^z$ to be sparse will allow identifiability up to the following equivalence class:

\begin{definition}[$\vz$-consistency equivalence]\label{def:z_consistent_models}
We say two models $\vtheta := (\bff, p, \mG)$ and $\tilde{\vtheta}:= (\tilde \bff,\tilde p, \tilde \mG)$ satisfying Assumptions~\ref{ass:diffeomorphism}, \ref{ass:cond_indep} \& \ref{ass:graph} are \textbf{$\vz$-consistent}, denoted $\vtheta \eqcon^\vz \tilde\vtheta$, if and only if there exists a permutation matrix $\mP$ such that
\begin{enumerate}
    \item $\vtheta \eqdiff \tilde{\vtheta}$ (Def.~\ref{def:eqdiff}) and $\tilde\mG^z =  \mP \mG^z \mP^\top$; and
    \item the entanglement map $\vv := \vf^{-1} \circ \tilde\vf$ can be written as $\vv = \vc \circ \mP^\top$ where $\vc$ is a $\mG^z$-preserving and $(\mG^z)^\top$-preserving diffeomorphism (Definition~\ref{def:g_preserving_map}).
\end{enumerate}
\end{definition}

This relation can be shown to be an \textit{equivalence} relation, as was the case for $\eqcon^{\va}$. This is shown in Appendix~\ref{sec:proof_consistence_equivalence}. Analogously to $\eqcon^\va$, the equivalence relation $\eqcon^\vz$ relates the structure of the entanglement map $\vv$ to the graph $\mG^z$ via the notion of $\mG$-preserving maps. It is also true that $\vtheta \eqperm \hat\vtheta \implies \vtheta \eqcon^\vz \hat\vtheta$.

\begin{comment}
    \begin{definition}[Consistency equivalence]\label{def:consistent_models}
We say two models $\vtheta := (\bff, p, \mG)$ and $\tilde{\vtheta}:= (\tilde \bff,\tilde p, \tilde \mG)$ are \textbf{consistent}, denoted $\vtheta \eqcon \tilde\vtheta$, if and only if there exists a permutation matrix $\mP$ such that
\begin{enumerate}
    \item $\mG^z =  \mP^\top \tilde \mG^z \mP$ and $\mG^a = \mP^\top \tilde \mG^a$\, , and
    \item $\vtheta \eqdiff \tilde{\vtheta}$ (Def.~\ref{def:eqdiff}) with entanglement graph given by $\mV = \mC\mP^\top$, where $\mP$ is a permutation matrix and $\mC$ is a binary matrix that is $\mG^z\text{-consistent}$, $(\mG^z)^\top$-consistent and $\mG^a$-consistent (Def.~\ref{def:S-consistent}).
\end{enumerate}
\end{definition}
\end{comment}

The following result is analogous to Theorems~\ref{thm:nonparam_dis_cont_a} and \ref{thm:nonparam_dis_disc_a} where, instead of regularizing $\hat\mG^a$ to be sparse, we regularize $\hat\mG^z$. The next theorem shows how this type of sparsity regularization can induce the learned model to be $\vz$-consistent with the ground-truth one.

\begin{restatable}[Nonparametric disentanglement via sparse temporal dependencies]{theorem}{NonparamDisZ}\label{thm:nonparam_dis_z}
    Let the parameters $\vtheta := (\bff, p, \mG)$ and $\hat{\vtheta}:= (\hat \bff,\hat p, \hat \mG)$ correspond to two models satisfying Assumptions~\ref{ass:diffeomorphism}, \ref{ass:cond_indep}, \ref{ass:graph} \& \ref{ass:smooth_trans}. Further assume that
    \begin{enumerate}
        \item \textbf{[Observational equivalence]} $\vtheta \eqobs \hat\vtheta$ (Def.~\ref{def:eqobs});  
        \item \textbf{[Sufficient influence of $\vz$]} The Hessian matrix $H^{t,\tau}_{z, z} \log p (\vz^{t} \mid \vz^{<{t}}, \va^{<{t}})$ varies ``sufficiently'', as formalized in Assumption~\ref{ass:nonparam_suff_var_z};
    \end{enumerate}
    Then, there exists a permutation matrix $\mP$ such that $\mP\mG^z\mP^\top \subseteq \hat\mG^z$. Further assume that
    \begin{enumerate}[resume]
        \item \textbf{[Sparsity regularization]} $||\hat\mG^z||_0 \leq ||\mG^z||_0$;
    \end{enumerate}
    Then, $\vtheta \eqcon^\vz \hat\vtheta$ (Def.~\ref{def:z_consistent_models}).
\end{restatable}
The structure of the above theorem is very similar to Theorem~\ref{thm:nonparam_dis_cont_a}~\&~\ref{thm:nonparam_dis_disc_a}. For example, we still have a ``sufficient influence" condition, but this time it concerns the Hessian matrix $H^{t,\tau}_{z, z} \log p$ which we saw in Section~\ref{sec:insight_main}, Equation~\eqref{eq:hessians_intuitive_time_2}. The conclusion is that both models will be $\vz$-consistent, which means that we recover the graph $\mG^z$ up to permutation and have that the entanglement map $\vv$ has a dependency graph given by $\mV = \mC\mP^\top$ where $\mC$ is $\mG^z$- and $(\mG^z)^\top$-preserving. Section~\ref{sec:proofs} introduces the sufficient influence assumption formally as well as a proof of Theorem~\ref{thm:nonparam_dis_z}. 

We now build intuition through some minimal examples that show how one can apply the above theorem to draw links between the graph $\mG^z$ and the resulting entanglement graph $\mV$ (Definition~\ref{def:entanglement_graphs}). For now we simply assume that the assumption of sufficient influence (Assumption~\ref{ass:nonparam_suff_var_z}) is satisfied and wait until Section~\ref{sec:suff_inf_z} to present more concrete transition models that satisfy it.

\begin{example}[Disentanglement via independent factors with temporal dependencies]\label{ex:diagonal_deps}
    Consider the situation depicted in Figure \ref{fig:Gz_cd} where the graph $\mG^z = \mI$, i.e., the latents $\vz^t_i$ are dependent in time but independent across dimensions. For this example, actions are unnecessary. Assuming the sufficient influence assumption of Theorem~\ref{thm:nonparam_dis_z} is satisfied, we have $\vtheta \eqobs \hat\vtheta\ \&\ ||\hat\mG^z||_0 \leq ||\mG^z||_0 \implies \vtheta \eqcon^\vz \hat\vtheta$, which means there exists a permutation $\mP$ such that $\hat\mG^z = \mP\mG^z\mP^\top$ and such that the entanglement map is given by $\vv = \vc \circ \mP^\top$ where $\vc$ is $\mG^z$- and $(\mG^z)^\top$-preserving. Using Proposition~\ref{prop:charac_G_preserving}, one can verify  that the dependency graph of $\vc$ is $\mC = \mI$ and thus $\mV = \mP^\top$, i.e. the learned representation is completely disentangled. Example~\ref{ex:lower_triangular_no_action_exp} will provide a concrete transition model where the sufficient influence assumption of Theorem~\ref{thm:nonparam_dis_z} holds for this simple graph $\mG^z$.
\end{example}

\begin{example}[Disentanglement via sparsely dependent factors with temporal dependencies]\label{ex:lower_triangular_no_action}
    The previous examples assumed independent latents, i.e., $\mG^z = \mI$. Instead, we now consider a more interesting ``lower triangular'' graph $\mG^z$, as depicted in Figures~\ref{fig:Gz_d} (This is the same graph as in the tree-robot-ball example in Figure~\ref{fig:working_example}). Again using Proposition~\ref{prop:charac_G_preserving}, one can verify that $\mC = \mI$ and thus $\mV = \mP^\top$, i.e. the learned representation is completely disentangled. Example~\ref{ex:lower_triangular_no_action_exp} will provide a concrete transition model where the sufficient influence assumption of Theorem~\ref{thm:nonparam_dis_z} holds.
\end{example}

\begin{example}[Partial disentanglement via temporal sparsity]\label{ex:temporal_partial}
    Assume the same situation as previously, but add an additional edge from $\vz_B^{t-1}$ to $\vz_R^{t}$ (see Figure~\ref{fig:Gz_e}). This could occur, for example, if the robot tries to follow the ball and is thus influenced by it. Using Proposition~\ref{prop:charac_G_preserving}, one can show that $\vc$ being $\mG^z$- and $(\mG^z)^\top$-preserving means that its dependency graph is given by
    \begin{align*}
        \mC = \begin{bmatrix}
            * & 0 & 0 \\
            0 & * & * \\
            0 & * & * 
            \end{bmatrix} \, .
    \end{align*}
    This means the robot and the ball remain \textit{entangled} in the learned representation. 
\end{example}

\subsection{Combining Sparsity Regularization on $\hat\mG^a$ \& $\hat\mG^z$}
\label{sec:combine_Ga_Gz}
A natural question at this point is whether Theorem \ref{thm:nonparam_dis_cont_a} (or Theorem~\ref{thm:nonparam_dis_disc_a}) can be combined with Theorem~\ref{thm:nonparam_dis_z} to obtain stronger guarantees. The answer is yes. In this section, we explain how this can be done. We would like to show how the combination of the assumptions of Theorem~\ref{thm:nonparam_dis_cont_a} and Theorem~\ref{thm:nonparam_dis_z} can yield identifiability up to the following stronger equivalence relation.

\begin{definition}[$(\va,\vz)$-consistency equivalence]\label{def:az_consistent_models}
We say two models $\vtheta := (\bff, p, \mG)$ and $\tilde{\vtheta}:= (\tilde \bff,\tilde p, \tilde \mG)$ satisfying Assumptions~\ref{ass:diffeomorphism}, \ref{ass:cond_indep} \& \ref{ass:graph} are \textbf{$(\va, \vz)$-consistent}, denoted $\vtheta \eqcon^{\vz,\va} \tilde\vtheta$, if and only if there exists a permutation matrix $\mP$ such that
\begin{enumerate}
    \item $\vtheta \eqdiff \tilde{\vtheta}$ (Def.~\ref{def:eqdiff}) and $\tilde\mG^a = \mP^\top \mG^a$ and $\tilde\mG^z =  \mP^\top \mG^z \mP$; and
    \item the entanglement map $\vv := \vf^{-1} \circ \tilde\vf$ can be written as $\vv = \vc \circ \mP^\top$ where $\vc$ is a $\mG^a$-, $\mG^z$- and $(\mG^z)^\top$-preserving diffeomorphism (Def.~\ref{def:g_preserving_map}).
\end{enumerate}
\end{definition}

Of course, if the assumptions of both theorems hold, we must have $\vtheta \eqcon^\va \hat\vtheta$ and $\vtheta \eqcon^\vz \hat\vtheta$. As one might guess, this implies $\vtheta \eqcon^{\va,\vz} \hat\vtheta$, as the following proposition shows. The reason why this result is not completely trivial is that the permutations $\mP$ given by $\eqcon^\va$ and $\eqcon^\vz$ might not be the same. Its proof can be found in Appendix~\ref{app:combining_eq_rel}.

\begin{restatable}{proposition}{CombineEquivalences}\label{prop:combine_equivalences} Let $\vtheta := (\bff, p, \mG)$ and $\tilde{\vtheta}:= (\tilde \bff,\tilde p, \tilde \mG)$ be two models satisfying Assumptions~\ref{ass:diffeomorphism}, \ref{ass:cond_indep} \& \ref{ass:graph}. We have $\vtheta \eqcon^{\vz,\va} \tilde\vtheta$ if and only if $\vtheta \eqcon^{\va} \tilde\vtheta$ and $\vtheta \eqcon^{\vz} \tilde\vtheta$.
\end{restatable}

We can thus combine both Theorems \ref{thm:nonparam_dis_cont_a} (or Theorem~\ref{thm:nonparam_dis_disc_a}) with Theorem~\ref{thm:nonparam_dis_z} to obtain stronger guarantees. Practically, this means that regularizing both $\hat\mG^a$ and $\hat\mG^z$ to be sparse will lead to a more disentangled representation, i.e. a sparser entanglement graph $\mV$, than if regularization was applied only on $\hat\mG^a$ or only on $\hat\mG^z$. 

\subsection{Graphical Criterion for Complete Disentanglement}\label{sec:graph_crit}
The previous sections introduced results guaranteeing identifiability up to $\eqcon^\va$, $\eqcon^\vz$ and $\eqcon^{\vz,\va}$, which all correspond to potentially \textit{partial} disentanglement. This section provides an additional assumption to guarantee identifiability up to $\eqperm$, i.e., \textit{complete} disentanglement. 

One can easily see from the definitions that $\vtheta \eqperm \hat\vtheta$ holds precisely when $\vtheta \eqcon^{\va,\vz} \hat\vtheta$ with $\mC = \mI$. This condition can be achieved by making an additional assumption on $\mG$. This assumption is taken directly from \cite{lachapelle2022disentanglement}.
\begin{comment}
Theorem~5 of \citet{lachapelle2022disentanglement} is very similar to the result one would obtain by combining Theorems \ref{thm:nonparam_dis_disc_a} \& \ref{thm:nonparam_dis_z}, and Proposition~\ref{prop:combine_equivalences}. The main differences are that \citet{lachapelle2022disentanglement} concentrates on the exponential family case, and has an extra ``graphical criterion'' that allows them to guarantee identifiability up to $\eqperm$ instead of only $\eqcon^{\vz,\va}$. The exponential family case and how it relates to the more general nonparametric setting considered here will be covered in detail in Section~\ref{sec:exponential_family}. We recall the graphical criterion imposed on the ground-truth graph $\mG$ in \citet{lachapelle2022disentanglement}\footnote{This graphical criterion is a slight simplification of the one of \citet{lachapelle2022disentanglement}. Prop.~\ref{prop:graph_crit_simplified} shows they are equivalent.}: $\forall1 \leq i \leq d_z$,
\end{comment}
\begin{restatable}[Graphical criterion,  \citet{lachapelle2022disentanglement}]{assumption}{GraphCrit}\label{def:graph_crit} Let $\mG = [\mG^z\ \mG^a]$ be a  graph. For all $i \in \{1, ..., d_z\}$,
\begin{align}
    \left( \bigcap_{j \in {\bf Ch}_i^z} {\bf Pa}^z_j \right) \cap \left(\bigcap_{j \in {\bf Pa}_i^z} {\bf Ch}^z_j \right) \cap \left(\bigcap_{\ell \in {\bf Pa}^a_i} {\bf Ch}^a_\ell \right)  = \{i\} \, , \nonumber
\end{align}
where ${\bf Pa}^z_i$ and ${\bf Ch}^z_i$ are the sets of parents and children of node $\vz_i$ in $\mG^{z}$, respectively, while ${\bf Ch}^a_\ell$ is the set of children of $\va_\ell$ in $\mG^a$.
\end{restatable}

The following proposition shows that when $\mG$ satisfies the above criterion, the set of models that are $\eqcon^{\va,\vz}$-equivalent to $\vtheta$ is equal to the set of models that are $\eqperm$-equivalent to $\vtheta$, thus allowing complete disentanglement. See Appendix~\ref{sec:connect_lachapelle2022} for a proof.

\begin{restatable}[Complete disentanglement as a special case]{proposition}{CompleteDisentanglement} \label{prop:complete_disentanglement}
Let $\vtheta := (\bff, p, \mG)$ and $\hat{\vtheta}:= (\hat \bff,\hat p, \hat \mG)$ be two models satisfying Assumptions~\ref{ass:diffeomorphism}, \ref{ass:cond_indep} \& \ref{ass:graph}. If $\vtheta \eqcon^{\vz,\va} \hat\vtheta$ and $\mG$ satisfies Assumption~\ref{def:graph_crit}, then $\vtheta \eqperm \hat\vtheta$. 
\end{restatable}

The above result shows that our general theory can guarantee complete disentanglement as a special case. This is one way in which our work generalizes the work of \citet{lachapelle2022disentanglement}, in addition to relaxing the exponential family assumption. The following section explores how the exponential family assumption fits into our nonparametric theory and how it allows one to simplify the ``sufficient influence assumptions''. But before, we provide some example to illustrate when Assumption~\ref{def:graph_crit} holds.

\begin{wrapfigure}[14]{r}{.31\textwidth}
\centering
\vspace{-0.8cm}
%\vspace{-0.1cm}
\includegraphics[width=.70\linewidth]{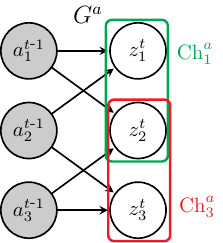}
\captionsetup{margin={2mm, 0mm}, oneside}
\caption{
An example satisfying Assumption~\ref{def:graph_crit}. Indeed, $\{\vz_1\} = {\bf Ch}_1^a \cap {\bf Ch}_2^a$, 
$\{\vz_2\} = \textcolor{Green}{{\bf Ch}_1^a} \cap \textcolor{Red}{{\bf Ch}_3^a}$ 
and $\{\vz_3\} = {\bf Ch}_2^a \cap {\bf Ch}_3^a$.
}
\label{fig:sparsity}
\end{wrapfigure}

For example, the graphical criterion of Assumption~\ref{def:graph_crit} is trivially satisfied when $\mG^z$ is diagonal, since $\{i\} = {\bf Pa}^\vz_i$ for all $i$ (actions are not necessary here). This simple case amounts to having mutual independence between the sequences $\vz^{\leq T}_i$, which is a standard assumption in the ICA literature~\citep{AmuseICA90, PCL17, slowVAE}. 
The illustrative example we introduced in Fig.~\ref{fig:working_example} has a more interesting ``non-diagonal'' graph satisfying our criterion. 
Indeed, we have $\{T\} = {\bf Pa}^\vz_T$, $\{R\} = {\bf Ch}^\vz_R \cap {\bf Pa}^\vz_R$, and $\{B\} = {\bf Ch}^\vz_B$. This example is actually part of an interesting family of graphs that satisfy our criterion: 

\begin{proposition}[Sufficient condition for the graphical criterion] \label{prop:2_cycles}

\noindent If $\mG^\vz_{i,i} = 1$ for all $i$ (all nodes have a self-loop) and $\mG^z$ has no 2-cycles, then $\mG$ satisfies Assumption~\ref{def:graph_crit}.
\end{proposition}
\begin{proof}
Self-loops guarantee $i \in {\bf Pa}^\vz_i \cap {\bf Ch}^\vz_i$ for all $i$. Suppose $j \in {\bf Pa}^\vz_i \cap {\bf Ch}^\vz_i$ for some $i \not= j$. This implies that $i$ and $j$ form a 2-cycle, which is a contradiction. Thus, $\{i\} = {\bf Pa}^\vz_i \cap {\bf Ch}^\vz_i$ for all $i$. 
\end{proof}

\subsection{Proofs of Theorems~\ref{thm:nonparam_dis_cont_a}, \ref{thm:nonparam_dis_disc_a} \& \ref{thm:nonparam_dis_z} and their Sufficient Influence Assumptions}\label{sec:proofs}
In this section, we introduce the sufficient influence assumptions and use them to prove Theorems~\ref{thm:nonparam_dis_cont_a}, \ref{thm:nonparam_dis_disc_a} \& \ref{thm:nonparam_dis_z}. In the next section (Section~\ref{sec:examples}), we provide multiple examples to gain intuition about the sufficient influence assumptions. Throughout, the following lemma will come in handy.

\begin{lemma}[Invertible matrix contains a permutation] \label{lemma:L_perm}
Let $\mL \in \sR^{m\times m}$ be an invertible matrix. Then, there exists a permutation $\sigma$ such that $\mL_{i, \sigma(i)} \not=0$ for all $i$, or, in other words, $\mP^\top \subseteq \mL$ where $\mP$ is the permutation matrix associated with $\sigma$, i.e., $\mP\ve_i = \ve_{\sigma(i)}$. Note that this implies that $\mP\mL$ and $\mL\mP$ have no zero on their diagonals.
\end{lemma}
\begin{proof}
Since the matrix $\mL$ is invertible, its determinant is non-zero, i.e.
\begin{align*}
    \det(\mL) := \sum_{\sigma\in \mathfrak{S}_m} \text{sign}(\sigma) \prod_{i=1}^m \mL_{i, \sigma(i)} \neq 0 \, , 
\end{align*}
where $\mathfrak{S}_m$ is the set of $m$-permutations. This equation implies that at least one term of the sum is non-zero, meaning there exists a permutation $\sigma$ such that, for all $i$, $\mL_{i, \sigma(i)} \neq 0$. 
\end{proof}

\subsubsection{Sufficient influence assumption of Theorem~\ref{thm:nonparam_dis_cont_a} and its proof}
We start by introducing the sufficient influence assumption of Theorem~\ref{thm:nonparam_dis_cont_a}. Although it may seem intimidating at first, the reason why it is necessary will become clear when we prove the theorem.

\begin{assumption}[Sufficient influence of $\va$ (nonparametric/continuous)]\label{ass:nonparam_suff_var_a_cont}
    For almost all $\vz \in \sR^{d_z}$ (i.e. except on a set with zero Lebesgue measure) %(i.e. maybe except on a set of zero probability under the transition model)
    and all $\ell \in [d_a]$, there exists 
    $$\{(t_{(r)}, \tau_{(r)}, \vz_{(r)}, \va_{(r)})\}_{r=1}^{|{\bf Ch}^\va_\ell|} \, ,$$
    such that $t_{(r)} \in [T]$, $\tau_{(r)} < t_{(r)}$, $\vz_{(r)} \in \sR^{d_z\times(t_{(r)} - 1)}$, $\va_{(r)} \in \gA^{t_{(r)}}$ and
    \begin{align}
        \vecspan \left\{H^{t_{(r)},\tau_{(r)}}_{z,a} \log p (\vz \mid \vz_{(r)}, \va_{(r)})_{\cdot, \ell}\right\}_{r=1}^{|{\bf Ch}^\va_\ell|} 
        =\sR^{d_z}_{{\bf Ch}^\va_\ell} \, . \nonumber
    \end{align}
\end{assumption}

\begin{proof}[of Theorem~\ref{thm:nonparam_dis_cont_a}]
Recall equation~\eqref{eq:hessians_intuitive_2}, which we derived in Section~\ref{sec:insight_main}:
\begin{align}
    \underbrace{H^{t, \tau}_{z, a} \hat q(\vz^{t} \mid \vz^{<{t}}, \va^{<{t}})}_{\subseteq \hat \mG^a} = D\vv(\vz^t)^\top \underbrace{H^{t,\tau}_{z, a} {q}(\bfv(\vz^{t}) \mid \bfv(\vz^{<{t}}), \va^{<{t}})}_{\subseteq \mG^a} \label{eq:23453656} \,.
\end{align}

Notice that Assumption~\ref{ass:nonparam_suff_var_a_cont} holds only ``almost everywhere'', i.e. on a set $\sR^{d_z} \setminus E_0$ where $E_0$ has zero Lebesgue measure. Fix an arbitrary $\vz \in \sR^{d_z} \setminus E_0$. For notational convenience, define 
\begin{align}
    \Lambda(\vz, \gamma) := H^{t,\tau}_{z, a} {q}(\bfv(\vz) \mid \bfv(\vz^{<{t}}), \va^{<{t}})\quad\quad \hat\Lambda(\vz, \gamma) := H^{t, \tau}_{z, a} \hat q(\vz \mid \vz^{<{t}}, \va^{<{t}}) \,,\nonumber
\end{align}
where $\gamma := (t, \tau, \vz^{<t}, \va^{<t})$. This allows us to rewrite \eqref{eq:23453656} with a much lighter notation:
\begin{align}
    \hat\Lambda(\vz, \gamma) = D\vv(\vz)^\top\Lambda(\vz, \gamma) \,. \label{eq:9rew89}
\end{align}
Now, notice that the sufficient influence assumption (Assumption~\ref{ass:nonparam_suff_var_a_cont}) requires that, for all $\ell \in [d_a]$ there exists $\{\gamma_{(r)}\}_{r=1}^{|\Ch^a_\ell|}$ such that $\vecspan \{\Lambda(\vz, \gamma_{(r)})_{\cdot, \ell}\}_{r=1}^{|\Ch^a_\ell|} = \sR^{d_z}_{\Ch^a_\ell}$. We can thus write
\begin{align}
    &D\vv(\vz)^\top\sR^{d_z}_{\mG^a_{\cdot, \ell}} = D\vv(\vz)^\top\vecspan \{\Lambda(\vz, \gamma_{(r)})_{\cdot, \ell}\}_{r=1}^{|\Ch^a_\ell|}=  \vecspan \{\hat\Lambda(\vz, \gamma_{(r)})_{\cdot, \ell}\}_{r=1}^{|\Ch^a_\ell|} \subseteq \sR^{d_z}_{\hat\mG^a_{\cdot, \ell}} \label{eq:5645342u}
    %\implies &D\vv(\vz)^\top\sR^{d_z \times d_a}_{\mG^a} \subseteq \sR^{d_z \times d_a}_{\hat\mG^a} \,.
\end{align}
Since $D\vv(\vz)$ is invertible, there exists a permutation $\mP(\vz)$ such that $D\vv(\vz)\mP(\vz)$ has no zero on its diagonal (Lemma~\ref{lemma:L_perm}). Let $\mC(\vz) := D\vv(\vz)\mP(\vz)$. By left-multiplying \eqref{eq:5645342u} by $\mP(\vz)^\top$, we get
\begin{align}
    \mC(\vz)^\top\sR^{d_z}_{\mG^a_{\cdot,\ell}} \subseteq \sR^{d_z}_{\mP(\vz)^\top\hat\mG^a_{\cdot,\ell}} \,. \label{eq:6543245}
\end{align}
We would like to show that $\mC(\vz)$ is $\mG^a$-preserving. Notice how the above equation is almost exactly the definition of $\mG^a$-preserving. All that is left to prove is that $\mP(\vz)^\top\hat\mG^a = \mG^a$.

We start by showing $\mP(\vz)^\top\hat\mG^a \supseteq \mG^a$. Take $(i,\ell) \in \mG^a$. Since $\ve_i \in \sR^{d_z}_{\mG^a_{\cdot, \ell}}$, equation \eqref{eq:6543245} implies 
$$\mC(\vz)^\top\ve_i = \mC(\vz)_{i,\cdot} \in \sR^{d_z}_{\mP(\vz)^\top\hat\mG^a_{\cdot,\ell}}\,.$$
Since $\mC(\vz)_{i,i} \not=0$ (all elements on its diagonal are nonzero), we must have that $(i,\ell) \in \mP(\vz)^\top\hat\mG^a$.

Now, since $||\mP(\vz)^\top\hat\mG^a||_0 = ||\hat\mG^a||_0 \leq ||\mG^a||_0$, we have ${\mP(\vz)^\top\hat\mG^a = \mG^a}$. This implies
\begin{align*}
    \mC(\vz)^\top\sR^{d_z\times d_a}_{\mG^a} \subseteq \sR^{d_z\times d_a}_{\mG^a}\,,
\end{align*}
i.e. $\mC(\vz)$ is a $\mG^a$-preserving matrix, as desired.

To recap, we now have that, for all $\vz \in \sR^{d_z} \setminus E_0$, there exists a permutation $\mP(\vz)$ s.t. $D\vv(\vz)\mP(\vz)$ is $\mG^a$-preserving. We are not done yet, since, a priori, the permutation $\mP(\vz)$ can be different for different values of $\vz$, and we do not know what happens on the measure-zero set $E_0$. What we need to show is that there exists a permutation $\mP$ such that, for all $\vz$, $D\vv(\vz)\mP$ is $\mG^a$-preserving. Lemma~\ref{lem:same_permutation_almost} in Appendix~\ref{app:technical_lemmas_ident} shows precisely this, by leveraging the continuity of $D\vv(\vz)$ ($\vv$ is a diffeomorphism and thus $C^1$).

Notice that $D(\vv\circ\mP)(\vz) = D\vv(\mP\vz)\mP$, which is $\mG^a$-preserving everywhere. Using Lemma~\ref{lem:s_consistent_jac}, we conclude that the function $\vc:=\vv\circ\mP$ is $\mG^a$-preserving. This concludes the proof.
\end{proof}

\begin{remark}[Alternative view on sufficient influence assumptions]\label{rem:suff_lin_indep_func}
    Assumption~\ref{ass:nonparam_suff_var_a_cont}, and all sufficient influence assumptions we present later on, can be thought of in terms of linear independence of functions. By definition, a family of functions $(f^{(i)}: X \rightarrow \sR)_{i=1}^n$ is \textit{linearly independent} when $\sum_{i} \alpha_i f^{(i)}(x) = 0$ for all $x \in X$ implies $\alpha_i = 0$ for all $i$. It turns out that Assumption~\ref{ass:nonparam_suff_var_a_cont} is equivalent to requiring that, for all $\vz \in \sR^{d_z}$ and $\ell \in [d_a]$, the family of functions $(H^{t,\tau}_{z,a} \log p (\vz \mid \vz^{<t}, \va^{<t})_{i,\ell})_{i \in \Ch^a_\ell}$ (seen as functions of $t, \tau, \vz^{<t}$ and $\va^{<t}$) is linearly independent. To see this, note that, in general, $(f^{(i)}: X \rightarrow \sR)_{i=1}^n$ is linearly independent iff there exist $x_1, ..., x_n \in X$ s.t. the vectors $((f^{(1)}(x_i), ..., f^{(n)}(x_i)))_{i=1}^n$ are linearly independent (see Appendix~\ref{app:lin_indep_func} for a proof). 
    
\end{remark}

\subsubsection{Sufficient Influence Assumption of Theorem~\ref{thm:nonparam_dis_disc_a} and its Proof}
%\subsubsection{Proof of Theorem~\ref{thm:nonparam_dis_disc_a} and its sufficient influence assumption}
One can see that, if $\va$ is discrete, Theorem~\ref{thm:nonparam_dis_cont_a} cannot be applied because its sufficient influence assumption (Assumption~\ref{ass:nonparam_suff_var_a_cont}) refers to the cross derivative of $\log p$ w.r.t. $\vz^t$ and $\va^\tau$, which, of course, is not well defined when $\va$ is discrete. The discrete case is important to discuss interventions with unknown-targets as we did in Section~\ref{sec:unknown-target}, which is why we have a specialized result (Theorem~\ref{thm:nonparam_dis_disc_a}) which has an analogous sufficient influence assumption based on \textit{partial differences}.
\begin{definition}[Partial difference]\label{def:partial_difference} Let us define
\begin{align}
\Delta_{a,\ell}^{\tau, \epsilon} D^t_z \log p(\vz^t \mid \vz^{<t}, \va^{<t}) := D^{t}_{z} \log p (\vz^t \mid \vz^{<t}, \va^{<t} + \epsilon\mE^{(\ell, \tau)}) - D^{t}_{z} \log p (\vz^t \mid \vz^{<t}, \va^{<t}) \,, \nonumber
\end{align}
where $\epsilon \in \sR$ and $\mE^{(\ell, \tau)}$ is a matrix with a one at entry $(\ell,\tau)$ and zeros everywhere else. 
\end{definition}

One can see that $\Delta_{a,\ell}^{\tau, \epsilon} D^t_z \log p$ is essentially the discrete analog of $(H^{t,\tau}_{z,a} \log p)_{\cdot, \ell}$. Apart from this difference, the sufficient influence assumption for discrete $\va$ is the same as for continuous $\va$.

\begin{assumption}[Sufficient influence of $\va$ (nonparametric/discrete)]\label{ass:nonparam_suff_var_a}
    For almost all $\vz \in \sR^{d_z}$ (i.e. except on a set with zero Lebesgue measure) %(i.e. except on a set of zero probability under the transition model) 
    and all $\ell \in [d_a]$, there exists 
    $$\{(t_{(r)}, \tau_{(r)}, \vz_{(r)}, \va^{< t}_{(r)}, \epsilon_{(r)})\}_{r=1}^{|{\bf Ch}^\va_\ell|}\, ,$$
    such that $t_{(r)} \in [T]$, $\tau_{(r)} < t_{(r)}$,  $\vz_{(r)} \in \sR^{d_z\times (t_{(r)} - 1)}$, $\va_{(r)}  \in \gA^{t_{(r)}}$, $\epsilon_{(r)} \in \sR$, $(\va_{(r)})_{\cdot, \tau_{(r)}} + \epsilon_{(r)}\ve_\ell \in \gA$ and
    \begin{align}
        \vecspan \left\{\Delta^{\tau_{(r)}, \epsilon_{(r)}}_{a, \ell}D^{t_{(r)}}_{z} \log p (\vz \mid \vz_{(r)}, \va_{(r)})\right\}_{r=1}^{|{\bf Ch}^\va_\ell|} 
        =\sR^{d_z}_{{\bf Ch}^\va_\ell} \, . \nonumber
    \end{align}
\end{assumption}

We can now provide a proof of Theorem~\ref{thm:nonparam_dis_disc_a}. Note that it is almost identical to the proof of Theorem~\ref{thm:nonparam_dis_cont_a} except for the very first steps where we take a partial difference instead of a partial derivative.
\begin{proof}[of Theorem \ref{thm:nonparam_dis_disc_a}]
    We recall equation \eqref{eq:first_deriv} derived in Section~\ref{sec:insight_main}:
    \begin{align}
    D^t_z \hat q(\vz^{t} \mid \vz^{<{t}}, \va^{<{t}}) = D^t_z {q}(\bfv(\vz^{t}) \mid \bfv(\vz^{<{t}}), \va^{<{t}})D\vv(\vz^t) + \eta(\vz^t) \in \sR^{1 \times d_z}\, . \label{eq:first_deriv_app}
\end{align}
Now, instead of differentiating w.r.t. $\va^\tau_\ell$ for some $\tau < t$ and $\ell \in [d_a]$, we are going to take a partial difference. That is, we evaluate the above equation on at $\va^{<t}$ and $\va^{<t} + \epsilon\mE^{(\ell, \tau)}$ and $\epsilon \in \sR$, where $\mE^{(\ell, \tau)}$ is a ``one-hot matrix'', while keeping everything else constant, and take the difference. This yields:
\begin{align*}
    [D^t_{z} \hat q (\vz^{t} \mid \vz^{<{t}}, \va^{<{t}} + \epsilon\mE^{(\ell, \tau)}) - D^{t}_{z} \hat q (\vz^{t} \mid \vz^{<{t}}, \va^{<{t}})]^\top \quad\quad\quad\quad\quad\quad\quad\quad\quad\quad\quad\quad\quad\quad\\
    = D\vv(\vz^t)^\top [D^t_{z} q (\vv(\vz^{t}) \mid \vv(\vz^{<{t}}), \va^{<{t}} + \epsilon\mE^{(\ell, \tau)}) - D^{t}_{z} q (\vv(\vz^{t}) \mid \vv(\vz^{<{t}}), \va^{<{t}})]^\top  \nonumber 
\end{align*}
\begin{align}
    \Delta^{\tau, \epsilon}_{a,\ell}D^t_{z} \hat q (\vz^{t} \mid \vz^{<{t}}, \va^{<{t}})^\top = D\vv(\vz^t)^\top \Delta^{\tau, \epsilon}_{a,\ell}D^t_{z} q (\vv(\vz^{t}) \mid \vv(\vz^{<{t}}), \va^{<{t}})^\top \, , \label{eq:48858573}
\end{align}
where we used the notation for partial differences introduced in Definition~\ref{def:partial_difference}. Notice that the difference on the left is $\subseteq \hat\mG^a_{\cdot, \ell}$ and the difference on the right is $\subseteq \mG^a_{\cdot, \ell}$. This equation is thus analogous to \eqref{eq:23453656} from the continuous case. For that reason, we can employ a completely analogous strategy. Hence, we define
\begin{align}
        \hat\Lambda(\vz^t, \gamma)_{\cdot, \ell} :=  \Delta^{\tau, \epsilon}_{a,\ell}D^t_{z} \hat q (\vz^{t} \mid \vz^{<{t}}, \va^{<{t}})^\top  \nonumber \quad\quad \Lambda(\vz^t, \gamma)_{\cdot, \ell} :=\Delta^{\tau, \epsilon}_{a,\ell}D^t_{z} q (\vv(\vz^{t}) \mid \vv(\vz^{<{t}}), \va^{<{t}})^\top \nonumber\,,
\end{align}
where $\gamma = (t, \tau, \vz^{<t}, \va^{<t}, \vec{\epsilon})$. This notation allows us to rewrite \eqref{eq:48858573} more compactly as
    \begin{align*}
        \underbrace{\hat\Lambda(\vz^t, \gamma)}_{\subseteq \hat\mG^a} = \mL(\vz^t)^\top\underbrace{\Lambda(\vz^t, \gamma)}_{\subseteq \mG^a} \, .
\end{align*} 
From here, the rest of the argument is exactly analogous to the proof of Theorem~\ref{thm:nonparam_dis_cont_a}.
\end{proof}

%\subsubsection{Sufficient influence of lagged latent variables (Theorem~\ref{thm:nonparam_dis_z})}

\subsubsection{Sufficient Influence Assumption of Theorem~\ref{thm:nonparam_dis_z} and its Proof}
%\subsubsection{Proof of Theorem~\ref{thm:nonparam_dis_z} and its sufficient influence assumption}

We now introduce the sufficient influence assumption of Theorem~\ref{thm:nonparam_dis_z}, which showed how regularizing the temporal dependency graph $\hat\mG^z$ to be sparse can result in disentanglement. Again, it is very similar to other sufficient influence assumptions we have seen so far.

\begin{assumption}[Sufficient influence of $\vz$ (nonparametric)]\label{ass:nonparam_suff_var_z}
    For almost all $\vz \in \sR^{d_z}$ (i.e. except on a set with zero Lebesgue measure), there exists 
    $$\{ (t_{(r)}, \tau_{(r)}, \vz_{(r)}, \va_{(r)})\}_{r=1}^{||\mG^z||_0}\,,$$
    such that $t_{(r)} \in [T]$, $\tau_{(r)} < t_{(r)}$,  $\vz_{(r)} \in \sR^{d_z \times (t_{(r)} - 1)}$, $\va_{(r)} \in \gA^{t_{(r)}}$, $\vz = \vz^{\tau_{(r)}}_{(r)}$ and 
    \begin{align}
        \vecspan \left\{ H^{t_{(r)}, \tau_{(r)}}_{z, z} q(\vz \mid \vz_{(r)}, \va_{(r)}) \right\}_{r=1}^{||\mG^z||_0} 
        =\sR^{d_z}_{\mG^z} \, . \nonumber
    \end{align}
\end{assumption}
\begin{proof}[of Theorem~\ref{thm:nonparam_dis_z}]
    We recall equation~\eqref{eq:hessians_intuitive_time_2} derived in Section~\ref{sec:insight_main}:
    \begin{align*}
    \underbrace{H^{t, \tau}_{z, z} \hat q(\vz^{t} \mid \vz^{<{t}}, \va^{<{t}})}_{\subseteq \hat\mG^z} = D\vv(\vz^t)^\top \underbrace{H^{t,\tau}_{z, z} {q}(\bfv(\vz^{t}) \mid \bfv(\vz^{<{t}}), \va^{<{t}})}_{\subseteq \mG^z}D\vv(\vz^\tau) \, .%\label{eq:hessians_intuitive_time_2_app}
\end{align*}
This equation holds for all pairs of $\vz^t$ and $\vz^\tau$ in $\sR^{d_z}$. We can thus evaluate it at a point such that $\vz^t = \vz^\tau$, which yields
\begin{align}
    H^{t, \tau}_{z, z} \hat q(\vz^{t} \mid \vz^{<{t}}, \va^{<{t}}) = D\vv(\vz^t)^\top {H^{t,\tau}_{z, z} {q}(\bfv(\vz^{t}) \mid \bfv(\vz^{<{t}}), \va^{<{t}})}D\vv(\textcolor{blue}{\vz^t}) \, . \label{eq:hessians_intuitive_time_2_app}
\end{align}
Recall that Assumption~\ref{ass:nonparam_suff_var_z} holds for all $\vz^t \in \sR^{d_z} \setminus E_0$ where $E_0$ has Lebesgue measure zero. Fix an arbitrary $\vz^t \in \sR^{d_z} \setminus E_0$ and set $\vz^\tau = \vz^t$. Let us define
\begin{align*}
    \Lambda(\vz^t, \gamma) := H^{t,\tau}_{z, z} {q}(\bfv(\vz^t) \mid \bfv(\vz^{<{t}}), \va^{<{t}})\quad\quad \hat\Lambda(\vz^t, \gamma) := H^{t, \tau}_{z, z} \hat q(\vz^t \mid \vz^{<{t}}, \va^{<{t}}) \,, \nonumber
    \end{align*}
    where $\gamma := (t, \tau, \vz^{<t}_{-\tau}, \va^{<t})$ and $\vz^{<t}_{-\tau}$ is $\vz^{<t}$ but without $\vz^\tau$. We can now rewrite \eqref{eq:hessians_intuitive_time_2_app} compactly as
\begin{align*}
    {\hat\Lambda(\vz^t, \gamma)}= D\vv(\vz^t)^\top {\Lambda(\vz^t, \gamma)} D\vv(\vz^t) \,.
\end{align*}

Now, notice that the sufficient influence assumption (Assumption~\ref{ass:nonparam_suff_var_z}) requires that, there exists $\{\gamma_{(r)}\}_{r=1}^{||\mG^z||_0}$ such that $\vecspan \{\Lambda(\vz^t, \gamma_{(r)})\}_{r=1}^{||\mG^z||_0} = \sR^{d_z \times d_z}_{\mG^z}$. We can thus write
\begin{align}
    &D\vv(\vz^t)^\top\vecspan \{\Lambda(\vz^t, \gamma_{(r)})\}_{r=1}^{||\mG^z||_0}D\vv(\vz^t) =  \vecspan \{\hat\Lambda(\vz^t, \gamma_{(r)})\}_{r=1}^{||\mG^z||_0} \subseteq \sR^{d_z \times d_z}_{\hat\mG^z} \nonumber\\
    \implies &D\vv(\vz^t)^\top\sR^{d_z\times d_z}_{\mG^z}D\vv(\vz^t) \subseteq \sR^{d_z \times d_z}_{\hat\mG^z} \label{eq:589g3r}
\end{align}
Since $D\vv(\vz)$ is invertible, there exists a permutation $\mP(\vz)$ such that $D\vv(\vz)\mP(\vz)$ has no zero on its diagonal (Lemma~\ref{lemma:L_perm}). Let $\mC(\vz) := D\vv(\vz)\mP(\vz)$. If we left and right-multiply \eqref{eq:589g3r} by $\mP(\vz)^\top$ and $\mP(\vz)$, respectively, we obtain
\begin{align}
    \mC(\vz^t)^\top\sR^{d_z\times d_z}_{\mG^z}\mC(\vz^t) \subseteq \sR^{d_z \times d_z}_{\mP(\vz)^\top\hat\mG^z\mP(\vz)}\, . \label{eq:dfhgfe}
\end{align}

We now show that $\mG^z \subseteq \mP(\vz)^\top\hat\mG^z\mP(\vz)$. Take $(i,j) \in \mG^z$. Since $\ve_i\ve_j^\top \in \sR^{d_z\times d_z}_{\mG^z}$, equation \eqref{eq:dfhgfe} implies 
\begin{align}
    \mC(\vz^t)^\top\ve_i\ve_j^\top\mC(\vz^t) = (\mC(\vz^t)_{i,\cdot})^\top\mC(\vz^t)_{j,\cdot} \subseteq \sR^{d_z \times d_z}_{\mP(\vz)^\top\hat\mG^z\mP(\vz)} \label{eq:4ujedi3}
\end{align}
Since $\mC(\vz^t)_{i,i}\mC(\vz^t)_{j,j} \not= 0$ (recall the diagonal of $\mC(\vz^t)$ has no zero), we must have $(i,j) \in \mP(\vz)^\top\hat\mG^z\mP(\vz)$. This shows that $\mG^z \subseteq \mP(\vz)^\top\hat\mG^z\mP(\vz)$.

Since $||\mP(\vz)^\top\hat\mG^z\mP(\vz)||_0 = ||\hat\mG^z||_0 \leq ||\mG^z||_0$, we must have $\mG^z = \mP(\vz)^\top\hat\mG^z\mP(\vz)$, which yields
\begin{align}
    \mC(\vz^t)^\top\sR^{d_z\times d_z}_{\mG^z}\mC(\vz^t) \subseteq \sR^{d_z \times d_z}_{\mG^z}\, . \label{eq:cij394ji29}
\end{align}
We are now going to show that the above implies that $\mC(\vz^t)$ is both $\mG^z$-preserving and $(\mG^z)^\top$-preserving. Start by rewriting \eqref{eq:4ujedi3} as follows:
\begin{align}
    \text{for all}\ (i,j) \in \mG^z,\ (\mC(\vz^t)_{i,\cdot})^\top\mC(\vz^t)_{j,\cdot} \subseteq \sR^{d_z \times d_z}_{\mG^z} \, . \label{eq:4989875839id}
\end{align}
We start by showing $\mG^z$-preservation. To do so, we leverage the characterization of Proposition~\ref{prop:charac_G_preserving}. We must show that $\mG^z_{i,\cdot} \not\subset \mG^z_{j,\cdot}$ implies $\mC(\vz^t)_{i,j} = 0$. Because $\mG^z_{i,\cdot} \not\subset \mG^z_{j,\cdot}$, there must exists $k$ s.t. $\mG^z_{i,k} = 1$ and $\mG^z_{j,k}=0$. We thus have, by \eqref{eq:4989875839id}, that $(\mC(\vz^t)_{i,\cdot})^\top\mC(\vz^t)_{k,\cdot} \subseteq \sR^{d_z \times d_z}_{\mG^z}$. Because $\mG^z_{j,k}=0$, we have $\mC(\vz^t)_{i,j}\mC(\vz^t)_{k,k} = 0$. But since $\mC(\vz^t)_{k,k} \not= 0$, we must have that $\mC(\vz^t)_{i,j} = 0$, as desired. To show $(\mG^z)^\top$-preservation, one can use a completely analogous argument.

We showed that $\mC(\vz^t)$ is  $\mG^z$-preserving and $(\mG^z)^\top$-preserving. It is easy to verify that this is equivalent to being $[\mG^z\ (\mG^z)^\top]$-preserving (where $[ \cdot \ \cdot]$ stands for column concatenation). This remark will be useful below.

Similarly to the proof of Theorem~\ref{thm:nonparam_dis_cont_a}, we must now show that there exists a single permutation that works for all $\vz^t \in \sR^{d_z}$. To achieve this, we use Lemma~\ref{lem:same_permutation_almost} with $\mG := [\mG^z\ (\mG^z)^\top]$ and $\mL(\vz) := D\vv(\vz)$. This allows us to say that there exists a permutation $\mP$ such that $D\vv(\vz)\mP$ is $[\mG^z\ (\mG^z)^\top]$-preserving for all $\vz$ (not ``almost all'').

Notice that $D(\vv\circ\mP)(\vz) = D\vv(\mP\vz)\mP$, which is $[\mG^z\ (\mG^z)^\top]$-preserving everywhere. Using Lemma~\ref{lem:s_consistent_jac}, we conclude that the function $\vc:=\vv\circ\mP$ is $[\mG^z\ (\mG^z)^\top]$-preserving. 
\end{proof}

\subsection{Examples to Illustrate the Scope of the Theory} \label{sec:examples}
In this section, we provide several examples to gain better intuition as to when our results apply. Specifically, we will provide mathematically concrete examples of latent models $p(\vz^t \mid \vz^{<t}, \va^{<t})$ illustrating the various sufficient influence assumptions we introduced. All these examples are summarized in Table~\ref{tab:examples}.

Even though our results are nonparametric, we will concentrate on the special case of Gaussian models which are useful to get a good intuition of what the sufficient influence assumptions mean. The following simple lemma will be useful in the following examples. We present it without proof, as it can be derived from simple computations.
\begin{lemma}\label{lem:gauss_deriv}
    Let $p(\vz) = \gN(\vz ; \vmu, \bm\Sigma)$ where $\vmu \in \sR^{d_z}$ and $\bm\Sigma := \textnormal{diag}(\sigma^2_1, ..., \sigma_{d_z}^2)$. Then
    \begin{align*}
        D_z\log p(\vz) = - \left[(\vz_1 -\vmu_1) / {\sigma_1^2}, \dots, (\vz_{d_z} -\vmu_{d_z})/{\sigma_{d_z}^2}\right] \in \sR^{1 \times d_z}\,.
    \end{align*}
\end{lemma}

\subsubsection{Continuous Auxiliary Variables (Theorem~\ref{thm:nonparam_dis_cont_a})}

We start by illustrating Assumption~\ref{ass:nonparam_suff_var_a_cont} from Theorem~\ref{thm:nonparam_dis_cont_a}. Example~\ref{ex:a_target_one_z_cont_a} assumes that we observe continuous actions that target each latent factor individually, while Example~\ref{ex:multi_target_a_cont_a} gives a multi-target example.

\begin{example}[Sufficient influence for continuous single-target actions]\label{ex:a_target_one_z_cont_a}
We make Example~\ref{ex:a_target_one_z} more concrete by explicitly specifying a latent transition model. Recall the situation depicted in Figure~\ref{fig:working_example} where $\vz_1$ is the tree position, $\vz_2$ is the robot position, and $\vz_3$ is the ball position ($d_z = 3$). Assume $\va \in [-1, 1]$ corresponds to the amount of torque applied to the wheels of the robot. We thus have $\mG^a = [0, 1, 0]^\top$, i.e., $\va$ affects only the robot position $\vz_2$. For this example, $\mG^z$ can be anything. Let $p(\vz^t \mid \vz^{t-1}, \va) = \gN(\vz^t ; \vmu(\vz^{t-1}, \va), \sigma^2\mI)$ where
    \begin{align}
        \vmu(\vz^{t-1}, \va) := \vz^{t-1} + \vg(\vz^{t-1}) + \va \cdot \mG^a \,. \nonumber
    \end{align}
    where $\vg: \sR^{d_z} \rightarrow \sR^{d_z}$ is some function that satisfies the dependency graph $\mG^z$ (e.g. $\vg(\vz) := \mW\vz$ where $\mW \in \sR^{d_z \times d_z}_{\mG^z}$). If no torque is applied ($\va = 0$), then the position of the robots is determined by the dynamics of the system. However, adding a positive or negative torque ($\va \not= 0$) nudges the robot to the right or to the left. Using Lemma~\ref{lem:gauss_deriv}, we can compute
    \begin{align*}
        H^t_{z, a}\log p(\vz^t \mid \vz^{t-1}, \va) = [0, 1/\sigma^2, 0]^\top \,,
    \end{align*}
 which spans $\sR^3_{\{2\}}$ and thus Assumption~\ref{ass:nonparam_suff_var_a_cont} holds.
\end{example}

\begin{example}[Sufficient influence for continuous multi-target actions]\label{ex:multi_target_a_cont_a}
    We make Example~\ref{ex:multi_target_a} more concrete by specifying an explicit latent model. Recall that $\mG^a$ is given by Figure~\ref{fig:Ga_c} with $d_z = d_a = 3$. Assume there are no temporal dependencies ($T=1$), that $\va \in \sR^3$ and that the latent model is given by $p(\vz \mid \va) = \gN(\vz ; \vmu(\va), \sigma^2\mI)$ where
    \begin{align}
        \vmu(\va) := \begin{bmatrix}
            \va_1 \\
            \va_1^2 \\
            0
        \end{bmatrix} + 
        \begin{bmatrix}
            \va_2 \\
            0 \\
            \va_2^2
        \end{bmatrix} + 
        \begin{bmatrix}
            0 \\
            \va_3 \\
            \va_3^2
        \end{bmatrix}\, .\label{eq:9949392928482}
    \end{align}
    Using Lemma~\ref{lem:gauss_deriv}, we can compute
    \begin{align*}
        H_{z,a}q(\vz \mid a) = \frac{1}{\sigma^2}\begin{bmatrix}
            1 & 1 & 0 \\
            2\va_1 & 0 & 1 \\
            0 & 2\va_2 & 2\va_3
        \end{bmatrix} \, .
    \end{align*}
    Consider $\ell = 1$ so that $\Ch^a_1 = \{1,2\}$. We can see that $H_{z,a}q(\vz \mid a = 0)_{\cdot, 1} = [1, 0, 0]^\top$ and $H_{z,a}q(\vz \mid a = \ve_1)_{\cdot, 1} = [1, 2, 0]^\top$ span $\sR^{3}_{\{1,2\}}$. Analogous conclusions can also be reached for $\ell = 2, 3$, which shows that Assumption~\ref{ass:nonparam_suff_var_a_cont} holds. 

    Now suppose that instead $\vmu(\va)$ is a linear map, i.e., $\vmu(\va) := \mW\va$ where $\mW \in \sR^{d_z \times d_a}_{\mG^a}$. This would imply that $H_{z,a}q(\vz \mid a) \propto \mW$, which means it cannot satisfy the sufficient influence assumption (unless $||\mG^a_{\cdot, \ell}||_0 \leq 1$ for all $\ell$).
    
\end{example}
\subsubsection{Discrete Auxiliary Variables or Interventions (Theorem~\ref{thm:nonparam_dis_disc_a})} \label{sec:suff_inf_disc_a}

We now provide three concrete examples of latent models $p(\vz^t \mid \vz^{<t}, \va^{<t})$ that satisfy Assumption~\ref{ass:nonparam_suff_var_a}, from Theorem~\ref{thm:nonparam_dis_disc_a}. Here, we interpret the discrete auxiliary variable $\va$ as an \textit{intervention index}, as discussed in Section~\ref{sec:unknown-target}, but note that other interpretations are possible (such as $\va$ as an action). Recall that our identifiability result does not require the knowledge of the targets of the interventions, these can be learned.

Example~\ref{ex:single_node_complete_dis_2} shows how single target interventions can be used to obtain complete disentanglement without temporal dependencies, Example~\ref{ex:multinode_linear_gauss} shows how multi-target interventions can be leveraged for disentanglement if temporal dependencies are present, and Example~\ref{ex:group_interv} shows how \textit{grouped} multi-target interventions allow disentanglement even when there are no time dependencies (Remark~\ref{rem:issues_multi_target}).

\begin{example}[Single-target interventions for complete disentanglement without time]\label{ex:single_node_complete_dis_2}
    We make\linebreak Example~\ref{ex:single_node_complete_dis} more concrete by specifying an explicit latent model. Assume $d_a = d_z$ and that $\va \in \gA := \{\bm0, \ve_1, \dots, \ve_{d_a}\}$ is interpreted as an intervention index (see Section~\ref{sec:unknown-target}). Furthermore, Example~\ref{ex:single_node_complete_dis} assumed $\mG^a = \mI$, i.e. each latent factor is targeted once by an intervention that targets only this factor (the example actually allowed to add arbitrary columns to $\mG^a$, i.e. adding more interventions, without compromising complete disentanglement). Assume there are no temporal dependencies, i.e., $T=1$, and that $p(\vz \mid \va) := \gN(\vz; \vmu(\va), \textnormal{diag}(\vsigma^2(\va)))$ with 
    \begin{align*}
        \vmu(\va) := \vmu \odot \va \quad \text{and}\quad \vsigma^2(\va) := \mathbbm{1} + \vdelta \odot \va\, ,
    \end{align*}
    where $\odot$ denotes the Hadamard product (a.k.a. the element-wise product), $\vmu \in \sR^{d_z}$ is the vector of means for each intervention and $\vdelta \in \sR^{d_z}$ is the vector of shifts in variance for all interventions. Thus, in the observational setting ($\va = \bm0$), we have $\mu(\va) = \bm0$ and $\vsigma(\va) = \mathbbm{1}$ while in the $\ell$th intervention ($\va = \ve_\ell$), the mean and variance of the targeted latent shift are the same, i.e. $\vmu(\va) = \vmu_\ell \ve_\ell$ and $\vsigma^2(\va) = \mathbbm{1} + \vdelta \ve_\ell$ (assuming that the shifted variance is $>0$). Using Lemma~\ref{lem:gauss_deriv}, we can compute
    \begin{align}
        \Delta^{\epsilon = 1}_{a,\ell}D_z\log p(\vz \mid \va = \bm0) := D_z\log p(\vz \mid \va = \ve_\ell) - D_z\log p(\vz \mid \va = \bm0) =  \frac{\vmu_\ell + \vdelta_\ell \vz_\ell}{1 + \vdelta_\ell}\ve_\ell \,,\nonumber
    \end{align}
    which must span $\sR^{d_z}_{\{\ell\}}$ unless ${\vmu_\ell + \vdelta_\ell \vz_\ell} = 0$. However, note that when, for all $\ell$, $\vmu_\ell \not=0$ or $\vdelta_\ell \not= 0$ (i.e. all interventions truly have an effect), the set $\{\vz \in \sR^{d_z} \mid {\vmu_\ell + \vdelta_\ell \vz_\ell} = 0\ \text{for some}\ \ell\}$ has zero Lebesgue measure in $\sR^{d_z}$, which is allowed by Assumption~\ref{ass:nonparam_suff_var_a}.
\end{example}

\begin{remark}[Potential issues with multi-target interventions without time]\label{rem:issues_multi_target} What if an intervention targets more than one latent at a time? Can it still satisfy the sufficient influence assumption? We will now see that, without time-dependencies ($T=1$), it is impossible. Consider the simple situation where $d_z = 3$, $d_a = 1$, $\va \in \{0, 1\}$ and $\mG^a = [1, 1, 0]^\top$, i.e., there is a single intervention targeting $\vz_1$ and $\vz_2$. In that case, there is a single possible difference vector, which is
\begin{align}
    \Delta^{\epsilon=1}_{a}D_z \log p(\vz \mid \va = 0) = D_z \log p(\vz \mid \va = 1) - D_z \log p(\vz \mid \va = 0) \in \sR^{d_z}_{\{1,2\}}\, . \nonumber
\end{align}
Since this is the only difference vector, we can see that we cannot span the 2-dimensional space $\sR^{d_z}_{\{1,2\}}$.
Therefore, to leverage multi-target interventions in our framework, more ``variability'' is required. Example~\ref{ex:multinode_linear_gauss} below shows how temporal dependencies can provide this additional variability, while Example~\ref{ex:group_interv} shows how having ``groups'' of interventions known to have the same (unknown) targets can also provide the required variability.
\end{remark}

\begin{example}[Multi-target interventions for complete disentanglement with time]\label{ex:multinode_linear_gauss}
    We make Example~\ref{ex:multi_target_a} more concrete by specifying an explicit latent model that satisfies Assumption~\ref{ass:nonparam_suff_var_a}. Recall $d_z = 3$, $d_a = 3$ and $\mG^a$ is depicted in Figure~\ref{fig:Ga_c}. This time, we assume there are temporal dependencies, i.e. $T > 1$ and $\mG^z$ is non-trivial. Suppose $\va \in \gA := \{\bm0, \ve_1, \ve_2, \ve_3\}$ where $\ve_\ell$ is the $\ell$th one-hot and we interpret $\bm0$ to correspond to the observational setting and $\ve_\ell$ to correspond to the $\ell$th intervention. Recall that in this interpretation $\mG^a$ describes which latent variable is targeted by each intervention. Let $p(\vz^t \mid \vz^{t-1}, \va) = \gN(\vz^t ; \vmu(\vz^{t-1}, \va), \sigma^2\mI)$ where
    \begin{align}
        \vmu(\vz^{t-1}, \va) := \vz^{t-1} + (\mathbbm{1} - \mG^a\va) \odot \vg(\vz^{t-1})\,, \nonumber
    \end{align}
    where $\vg: \sR^{d_z} \rightarrow \sR^{d_z}$ is some function that respects the graph $\mG^z$ (e.g. $\vg(\vz) = \mW\vz$ where $\mW \in \sR^{d_z \times d_z}_{\mG^z}$). The observational dynamics is then $\vmu(\vz^t, \va = \bm0) = \vz^{t-1} + \vg(\vz^{t-1})$ and the interventional settings correspond to zeroing out the elements of $\vg(\vz^{t-1})$ targeted by the intervention. 
\begin{comment}
    Since the variance is fixed, we can represent this distribution with an exponential family with natural parameter given by
    \begin{align}
        \bflambda(\vz^{t-1}, \va) := [\vz^{t-1} + (\mathbbm{1} - \mG^a\va) \odot \mW\vz^{t-1}] / \sigma \,,
    \end{align}
    where the sufficient statistic is given by $\vs(\vz) := \vz / \sigma$ (here, $k=1$). By using Lemma~\ref{lemma:derivatives_of_exponential}, we can compute
\end{comment}
    Using Lemma~\ref{lem:gauss_deriv}, we can compute 
    \begin{align}
        &\ \Delta^{\epsilon = 1}_{\va, \ell}D_z q(\vz^t \mid \vz^{t-1}, \va = \bm0) \nonumber \\
        =&\ D_z q(\vz^t \mid \vz^{t-1}, \va = \ve_\ell) - D_z q(\vz^t \mid\vz^{t-1}, \va = \bm0) = -\frac{1}{\sigma^2}\mG^a_{\cdot, \ell} \odot \vg(\vz^{t-1})\,. \nonumber
    \end{align}
    One can see that, as soon as the image of $\vg$ spans $\sR^{d_z}$, Assumption~\ref{ass:nonparam_suff_var_a} is satisfied since we can choose values $\vz_{(1)}, \dots, \vz_{(d_z)} \in \sR^{d_z}$ such that $\vecspan\{\vg(\vz_{(1)}), \dots, \vg(\vz_{(d_z)})\} = \sR^{d_z}$, which implies $\vecspan\{\mG^a_{\cdot, \ell} \odot \vg(\vz_{(1)}), \dots, \mG^a_{\cdot, \ell} \odot\vg(\vz_{(d_z)})\} = \sR^{d_z}_{\Ch^a_\ell}$. An example of a transition function $\vg$ that satisfies this property is $\vg(\vz) := \mW\vz$ where $\mW \in \sR^{d_z\times d_z}_{\mG^z}$ is invertible.

    Note that even if the temporal dependencies are not sparse, they are still helpful for identifiability as they make it more likely to satisfy the sufficient influence assumption (Assumption~\ref{ass:nonparam_suff_var_a}).
\end{example}

\begin{example}[Grouped multi-target interventions for disentanglement without time]\label{ex:group_interv}
    In this example, we assume that there are no temporal dependencies ($T=1$) and that the learner has access to $d_a$ groups of interventions where the interventions belonging to the $\ell$th group are known to target the same latent variables given by $\mG^a_{\cdot, \ell}$ (these targets are unknown). Here is how our framework can accommodate this setting: Given that we have $d_a$ groups of interventions where the $\ell$th group contains $k_\ell$ interventions, we set $\gA := \{\bm0, 1\ve_1, ..., k_1\ve_1, 1\ve_2, ..., k_2\ve_2, ..., k_{d_a}\ve_{d_a}\}$. In this setting, $\va = j\ve_\ell$ corresponds to the $j$th intervention of the $\ell$th group. Moreover, the sufficient influence assumption requires that interventions within a group $\ell$ span $\sR^{d_z}_{\Ch^a_\ell}$. More precisely, we need $\vecspan\{\Delta^{\epsilon}_{a,\ell}D_z \log p(\vz \mid \va = 0)\}_{\epsilon=1}^{k_\ell} = \sR^{d_z}_{\Ch^a_\ell}$.
    %Pursuing\linebreak with our example, if instead of having $\gA := \{0,1\}$ we had $\gA := \{0,1,2\}$, then there is hope that the sufficient influence assumption will be satisfied since we now have two partial differences. The implicit assumption here is that the interventions $\va = 1$ and $\va = 2$ belong to the same group in the sense that they target the same latent variables. Indeed this is enforce by the model since both interventions share the same targets defined by $\mG^a = [1,1,0]^\top$. In this interpretation, $d_a$ corresponds to the number of groups of interventions. In this example, there is a single group of intervention. More generally, 
\end{example}

\begin{comment}
\textcolor{blue}{
\begin{assumption}[(Alternative) Sufficient influence of $\va$ (nonparametric/discrete)]\label{ass:nonparam_suff_var_a_alt}
    For all $\vz \in \gZ$ and all $\ell \in [d_a]$, the following implication holds for all $\vb \in \sR^{d_z}$:
    \begin{center}
        $\vb^\top\Delta_{a_\ell^\tau}^{\epsilon} D^t_z \log p(\vz \mid \vz^{<t}, \va^{<t}) = 0$ for all $t, \tau, \vz^{<t}, \va^{<t}$ and $\epsilon \in \sR$ $\implies \vb \subseteq (\Ch^a_\ell)^c$.
    \end{center}
\end{assumption}
Analogously to Assumption~\ref{ass:nonparam_suff_var_a_cont_alt}, this condition exactly requires that, for all $\vz \in \gZ$, the functions $\{\Delta_{a_\ell^\tau}^{\epsilon} D^t_z \log p(\vz \mid \vz^{<t}, \va^{<t})_i\}_{i\in \Ch^a_\ell}$ are linearly independent (as functions of $t, \tau, \vz^{<t}$ and $\va^{<t}$).
}
\end{comment}

\subsubsection{Temporal Dependencies (Theorem~\ref{thm:nonparam_dis_z})}\label{sec:suff_inf_z}
Finally, we provide an example (Example~\ref{ex:lower_triangular_no_action_exp}) where temporal dependencies alone (no auxiliary variable $\va$) are enough to disentangle. We start with an important remark about the sufficient influence assumption of Theorem~\ref{thm:nonparam_dis_z}.

\begin{remark}[Auxiliary variables or non-Markovianity are required]\label{rem:non_markovian}
An important observation is that, if the transition model does not have an auxiliary variable $\va$ and is Markovian, i.e., $p(\vz^t \mid \vz^{<t}, \va^{<t}) = p(\vz^t \mid \vz^{t-1})$, then Assumption~\ref{ass:nonparam_suff_var_z} cannot be satisfied (except in trivial circumstances). To see this, simply note that, in that case, $H^{t, t-1}_{z, z} q(\vz^t \mid \vz^{t-1})$ depends only on $\vz^{t-1}$, which is forced to be equal to $\vz^t$. This means that the span of the Hessian must be at most one-dimensional, which means that the assumption cannot hold as soon as $||\mG^z||_0 > 1$. Therefore, when no auxiliary variable $\va$ is observed, Assumption~\ref{ass:nonparam_suff_var_z} requires that the transition model be non-Markovian. In Example~\ref{ex:lower_triangular_no_action_exp}, we provide a concrete example of a transition model without auxiliary variable $\va$ that satisfies this assumption. We will also see in Section~\ref{sec:exponential_family} that if the transition model $p(\vz^{t} \mid \vz^{t-1})$ is in the exponential family, this assumption can be relaxed so that non-Markovianity is no longer required.  
\end{remark}

\begin{example}[Sparse temporal dependencies for disentanglement without auxiliary variables]\label{ex:lower_triangular_no_action_exp}
    We continue with Examples~\ref{ex:diagonal_deps} \& \ref{ex:lower_triangular_no_action} which were based on the graphs $\mG^z$ depicted in Figures \ref{fig:Gz_cd}~\&~\ref{fig:Gz_d}, respectively. Assume that no action is observed, i.e., we can only take advantage of the sparsity of $\mG^z$ to disentangle. Examples~\ref{ex:diagonal_deps} \& \ref{ex:lower_triangular_no_action} already showed that these graph structures allow complete disentanglement, as long as the sufficient influence assumption for $\vz$ (Assumption~\ref{ass:nonparam_suff_var_z}) is satisfied. We now provide concrete transition models $p(\vz^t \mid \vz^{< t})$ that satisfy this requirement. Similarly to previous examples, assume $p(\vz^t \mid \vz^{< t}) = \gN(\vz^t \mid \vmu(\vz^{t-1}, \vz^{t-2}), \sigma^2\mI)$ where
    \begin{align*}
        \vmu(\vz^{t-1}, \vz^{t-2}) := \vz^{t-1} + \mW(\vz^{t-2})\vz^{t-1}\, ,
    \end{align*}
    where $\mW:\sR^{d_z} \rightarrow \sR_{\mG^z}^{d_z \times d_z}$ is some function of $\vz^{t-2}$. Using Lemma~\ref{lem:gauss_deriv}, we can derive
    \begin{align*}
        H^{t, t-1}_{z,z}q(\vz^t \mid \vz^{<t}) = \frac{1}{\sigma^2}[\mI + \mW(\vz^{t-2})] \, .
    \end{align*}
 Thus, Assumption~\ref{ass:nonparam_suff_var_z} holds when there exists $\{\vz_{(r)}^{t-2}\}_{r=1}^{||\mG^z||_0}$ such that 
    \begin{align}
        \vecspan\{\mI + \mW(\vz_{(r)}^{t-2})\}_{r=1}^{||\mG^z||_0} = \sR_{\mG^z}^{d_z \times d_z}\, . \label{eq:simple_sufficient_var}
    \end{align}
    One can directly see that, if $\mW(\vz^{t-2})$ was actually constant in $\vz^{t-2}$, the assumption could not hold (unless $||\mG^z||_0 \leq 1$). This case would correspond to a simple linear model of the form $\vmu(\vz^{t-1}) := \vz^{t-1} + \mW\vz^{t-1}$. Our theory suggests that this transition function is ``too simple'' to allow disentanglement.
    
    Nevertheless, we can find examples satisfying \eqref{eq:simple_sufficient_var}. For example, if $\mG^z = \mI$, we can take
    \begin{align*}
        \mW(\vz) = \begin{bmatrix}
            \vz_1 & 0     & 0\\
            0     & \vz_2 & 0 \\
            0     & 0     & \vz_3 
        \end{bmatrix}\, 
    \end{align*}
    and see that the family of functions $(1+\vz_1, 1+\vz_2, 1+\vz_3)$ is linearly independent (when seen as functions from $\sR^3$ to $\sR$). By Lemma~\ref{lem:lin_indep_charac} in the appendix, this is equivalent to the existence of $\vz_{(1)}, \vz_{(2)}, \vz_{(3)} \in \sR^{d_z}$ such that \eqref{eq:simple_sufficient_var} holds (see also Remark~\ref{rem:suff_lin_indep_func}). In other words, the sufficient influence assumption holds. In the case where $\mG^z$ is lower triangular like in Figure~\ref{fig:Gz_d}, one can take
    \begin{align*}
        \mW(\vz) = \begin{bmatrix}
            \vz_1 & 0 & 0\\
            \vz_2^2 & \vz_2 & 0 \\
            \vz_3^3 & \vz_3^2 & \vz_3 
        \end{bmatrix}\,
    \end{align*}
    and see that the family of functions $(1+\vz_1, 1 + \vz_2, 1+ \vz_3, \vz_2^2, \vz_3^2, \vz_3^3)$ is linearly independent, which similarly implies the existence of $\vz_{(1)}, \dots, \vz_{(6)} \in \sR^{d_z}$ such that \eqref{eq:simple_sufficient_var} holds. 
    %\seb{Just argue with Remark~\ref{rem:suff_lin_indep_func} instead of numerically?} 
    %One can verify numerically that \eqref{eq:simple_sufficient_var} holds when taking $\{\vz^{t-2}_{(r)}\}_{r=1}^{||\mG^z||_0} = \{(1,0,0)$, $(0,1,0)$, $(0,2,0)$, $(0, 0, 1)$, $(0, 0, 2)$, $(0, 0, 3)\}$. To achieve this, for each value $\vz^{t-2}_{(r)}$, collect the entries of $\mI + \mW(\vz^{t-2})$ corresponding to the six edges of $\mG^z$ into a 6-dimensional vector to form a $6\times 6$ matrix (each column corresponds to a specific value of $\vz^{t-2}_{(r)}$), and compute its determinant to observe that it is nonzero (it is actually -128). This means the matrix is invertible and thus $\mI + \mW(\vz^{t-2})$ does span the 6-dimensional space $\sR^{d_z\times d_z}_{\mG^z}$.
    % My own test: https://colab.research.google.com/drive/1V81QTza9i4Z3T8ibGtWtxtfrbdXmqrth
    \begin{comment}
     Similarly to Example~\ref{ex:multinode_linear_gauss}, since the variance is fixed, we can represent this distribution with an exponential family with natural parameter given by
    \begin{align}
        \bflambda(\vz^{t-1}, \vz^{t-2}) = [\vz^{t-1} + \mW(\vz^{t-2})\vz^{t-1}]/\sigma \, ,
    \end{align}
    where the sufficient statistic is given by $\vs(\vz):=\vz / \sigma$ (here, $k=1$). By using Lemma~\ref{lemma:derivatives_of_exponential} we can compute
\end{comment}
\end{example}

Example~\ref{ex:lower_triangular_no_action_exp_expfam} will show how one can leverage the exponential family assumption to allow Markovianity even without auxiliary variables.

\begin{comment}
    \item \ [Figured out] Action result will enforce $||\hat\mG^a||_0 \leq ||\mG^a||_0$ while Temporal result will enforce $||\hat\mG^z||_0 \leq ||\mG^z||_0$. The combined result will enforce $||\hat\mG||_0 \leq ||\mG||_0$. Can we have a small lemma showing 
    \begin{align}
        ||\hat\mG||_0 \leq ||\mG||_0 \implies ||\hat\mG^a||_0 \leq ||\mG^a||_0\ \&\ ||\hat\mG^z||_0 \leq ||\mG^z||_0
    \end{align}
    I actually think it's going to be a bit harder to modularize. We also need
    \begin{align}
        ||\hat\mG^a||_0 \geq ||\mG^a||_0\ \&\ ||\hat\mG^z||_0 \geq ||\mG^z||_0 \,.
    \end{align}
    which we get from sufficient variability (action and time). But the modularized result already shows that, with action/time sufficient variability we have $ G^a \subset P\hat G^a$  and similarly for $G^a$. So we're good!
\end{comment}

\section{Partial Disentanglement via Mechanism Sparsity in Exponential Families} \label{sec:exponential_family}
The goal of this section is to understand how restricting the transition model to be in the \textit{exponential family} allows us to weaken the sufficient influence assumption of Theorem~\ref{thm:nonparam_dis_z}. Section~\ref{sec:expfam_model} introduces the exponential family assumption. Section~\ref{sec:linear} follows \citet{iVAEkhemakhem20a} and shows that this additional assumption guarantees that the entanglement map $\vv$ is ``quasi-linear'', which means $\vv(\vz) := \vs^{-1}(\mL\vs(\vz) + \vb)$, where $\mL$ is a matrix and $\vs$ is an element-wise invertible function. Section~\ref{sec:specialized_identi_expfam} will introduce an identifiability result analogous to Theorem~\ref{thm:nonparam_dis_z} for sparse $\hat\mG^z$ that takes advantage of the quasi-linearity of $\vv$ to weaken Assumption~\ref{ass:nonparam_suff_var_z} (sufficient influence of $\vz$). We also briefly discuss an additional result from Appendix~\ref{app:connecting_ass} that shows connections between the nonparametric sufficient influence assumptions of this work (Assumptions~\ref{ass:nonparam_suff_var_a}~\&~\ref{ass:nonparam_suff_var_z}) and their counterparts in~\citet{lachapelle2022disentanglement} (Assumptions~\ref{ass:temporal_suff_var_main}~\&~\ref{ass:action_suff_var_main}).
%will dive deeper to compare the sufficient influence assumptions of this work to those of \citet{lachapelle2022disentanglement} and will uncover the following: (i) the nonparametric result for sparse $\mG^a$ (Theorems~\ref{thm:nonparam_dis_cont_a} \& \ref{thm:nonparam_dis_disc_a}) subsume the analogous result of \citet{lachapelle2022disentanglement}, and (ii) the sufficient influence assumption of the nonparametric result for sparse $\mG^z$ (Theorem~\ref{thm:nonparam_dis_z}) is actually strictly stronger than the analogous assumption of \citet{lachapelle2022disentanglement}, thus justifying a specialized identifiability result sparse $\mG^z$ in exponential families (Theorem~\ref{thm:expfam_dis_z}). 

\subsection{Exponential Family Latent Transition Models}\label{sec:expfam_model}

We will assume that the conditional densities $p(\vz_i^t \mid \vz^{<t}, \va^{<t})$ are from an \textit{exponential family}~\citep{WainwrightJordan08}:

\begin{assumption}[Exponential family transition model]\label{ass:expo}
    For all $i \in [d_z]$, we have
    \begin{align} \label{eq:z_transition}
    p(\vz_i^{t} \mid \vz^{<t}, \va^{<t}) = h_i(\vz^t_i)\exp\{\bfT_i(\vz^t_i)^\top \bflambda_i(\vz^{<t}, \va^{<t}) - \psi_i(\vz^{<t}, \va^{<t})\} \, .
\end{align}
\end{assumption}

Well-known distributions which belong to this family include the Gaussian and beta distributions. In the Gaussian case, the \textit{sufficient statistic} is $\bfT_i(z) := (z, z^2)$ and the \textit{base measure} is $h_i(z) := \frac{1}{\sqrt{2\pi}}$. The function $\bflambda_i(\vz^{<t}, \va^{<t})$ outputs the \textit{vectors of natural parameters} for the conditional distribution and can be itself parametrized, for instance, by a multi-layer perceptron (MLP) or a recurrent neural network (RNN). We will refer to the functions $\bflambda_i$ as the \textit{mechanisms} or the \textit{transition functions}. In the Gaussian case, the natural parameter is two-dimensional and is related to the usual parameters $\mu$ and $\sigma^2$ via the equation $(\lambda_1, \lambda_2) = (\frac{\mu}{\sigma^2}, -\frac{1}{2\sigma^2})$. We will denote by $k$ the dimensionality of the natural parameter and that of the sufficient statistic (which are equal). Thus, $k=2$ in the Gaussian case. The remaining term $\psi_i(\vz^{<t}, \va^{<t})$ acts as a normalization constant.

We define $\bflambda(\vz^{<t}, \va^{<t}) \in \sR^{kd_z}$ as the concatenation of all $\bflambda_i(\vz^{<t}, \va^{<t})$ and similarly for $\bfT(\vz^t) \in \sR^{kd_z}$. Similarly to the nonparameteric case, the learnable parameters are $\vtheta := (\bff, \bflambda, \mG)$. Note that throughout, we assume that the sufficient statistic $\vs$ is not learned and known in advance. With this notation, we can write the full transition model as
\begin{align*}
    p(\vz^t \mid \vz^{<t}, \va^{<t}) = h(\vz^t)\exp\{\vs(\vz^t)^\top\bflambda(\vz^{<t}, \va^{<t}) - \psi(\vz^{<t}, \va^{<t})\}\, ,
\end{align*}
where $h := \prod_{i=1}^{d_z} h_i$ and $\psi = \sum_{i=1}^{d_z} \psi_i$.

\begin{remark}[Applying nonparametric identifiability results to exponential families]
    One can apply the nonparametric results (Theorems~\ref{thm:nonparam_dis_cont_a}, \ref{thm:nonparam_dis_disc_a} \& \ref{thm:nonparam_dis_z}) to models satisfying the exponential family assumption. In fact, all examples of Section~\ref{sec:examples} were Gaussians and thus are in the exponential family.  
\end{remark}

%\subsection{Applying nonparametric theory to the exponential family case}
%\label{sec:nonparam_aplied_to_exp_fam}

%A natural idea at this point is to directly apply the general nonparameteric identifiability results of Theorems~\ref{thm:nonparam_dis_cont_a}, \ref{thm:nonparam_dis_disc_a}, \ref{thm:nonparam_dis_z} \& \ref{thm:combined} to transition models in an exponential family. This is interesting as it shows how the nonparametric theory introduced in Section~\ref{sec:nonparam_ident} covers cases that were not taken into account by the theory of \citet{iVAEkhemakhem20a} \& \citet{lachapelle2022disentanglement} (see Example~\ref{ex:k2_gaussian}). 

\subsection{Conditions for Quasi-Linear Identifiability}\label{sec:linear}
In this section, we follow \cite{iVAEkhemakhem20a} and show that the exponential family assumption combined with an additional sufficient variability assumption allows us to go from identifiability up to diffeomorphism (Definition~\ref{def:eqdiff}) to identifiability up to quasi-linearity, which we define next:

\begin{definition}[Quasi-linear equivalence] \label{def:linear_eq} We say that two models $\vtheta := (\bff, \bflambda, \mG)$ and $\tilde{\vtheta}:= (\tilde \bff,\tilde \bflambda,\tilde \mG)$ satisfying Assumptions~\ref{ass:diffeomorphism},~\ref{ass:cond_indep}~\&~\ref{ass:expo} are \textbf{equivalent up to quasi-linearity}, denoted $\vtheta \eqlin \tilde\vtheta$, if and only if $\vtheta \eqdiff \tilde\vtheta$ and there exist an invertible matrix $\mL \in \sR^{kd_z \times kd_z}$ and a vector $\vb \in \sR^{kd_z}$ such that the map $\vv := \vf^{-1} \circ \tilde\vf$  satisfies
\begin{align*}
    \vs(\vv(\vz)) = \mL\vs(\vz) + \vb,\ \forall \vz \in \sR^{d_z}\,.
\end{align*}
If the sufficient statistic $\vs$ is invertible, one obtains
\begin{align}
    \vv(\vz) = \vs^{-1}(\mL\vs(\vz) + \vb),\ \forall \vz \in \sR^{d_z}\, . \label{eq:v_quasi_lin}
\end{align}
\end{definition}

Equation~\eqref{eq:v_quasi_lin} is particularly interesting, as it says that the mapping relating both representations is ``almost'' linear in the following sense: although the map is not necessarily linear because the sufficient statistic $\vs$ might not be, the  ``mixing'' between components is linear. Indeed, notice that the sufficient statistic $\vs$ and its inverse operate ``element-wise''. The mixing between components is only due to the matrix $\mL$. This specific form simplifies a few steps in the identifiability proof, which might explain the popularity of this assumption in the literature on nonlinear ICA~\citep{TCL2016,iVAEkhemakhem20a,ice-beem20,Halva2020HMM,morioka2021innovation,Yang_2021_CVPR,lachapelle2022disentanglement,liu2023identifying,xi2023indeterminacy}. 

The following theorem provides conditions to guarantee identifiability up to quasi-linearity. This is an adaptation and minor extension of Theorem~1 from~\cite{iVAEkhemakhem20a}. For completeness, we provide a proof in Appendix~\ref{app:exp_linear_ident}.

\begin{restatable}[Conditions for linear identifiability - Adapted from~\cite{iVAEkhemakhem20a}]{theorem}{linearExpo}\label{thm:linear}
    Let $\vtheta := (\bff, \bflambda, \mG)$ and $\hat{\vtheta}:= (\hat \bff,\hat\bflambda, \hat \mG)$ be two models satisfying Assumptions~\ref{ass:diffeomorphism}, \ref{ass:cond_indep} \& \ref{ass:expo}. Further assume that 
    \begin{enumerate}
        \item \textbf{[Observational equivalence]} $\vtheta \eqobs \hat\vtheta$ (Definition~\ref{def:eqobs}); 
        \item \textbf{[Minimal sufficient statistics]} For all $i$, the sufficient statistic $\bfT_i$ is minimal (see below).
        \item \textbf{[Sufficient variability]} The natural parameter $\bflambda$ varies ``sufficiently" as formalized by Assumption~\ref{ass:suff_var_expfam} (see below).
    \end{enumerate}
    Then, $\vtheta \eqlin \hat\vtheta$ (Def.~\ref{def:linear_eq}).
\end{restatable}

The ``minimal sufficient statistics'' assumption is a standard one saying that $\bfT_i$ is defined appropriately  to ensure that the parameters of the exponential family are identifiable (see e.g.\ \citet[p. 40]{WainwrightJordan08}). See Definition~\ref{def:minimal_statistic} for a formal definition of minimality. The last assumption is sometimes 
called the \textit{assumption of variability}~\citep{HyvarinenST19}, and requires that the conditional distribution of $\vz^{t}$ depends ``sufficiently strongly'' on $\vz^{<t}$ and/or $\va^{<t}$. We stress the fact that this assumption concerns the ground-truth data generating model $\vtheta$. 
\begin{assumption}[Sufficient variability in exponential families]\label{ass:suff_var_expfam}
    There exist $(\vz_{(r)}, \va_{(r)})_{p=0}^{kd_z}$ in their respective supports such that the $kd_z$-dimensional vectors ${(\bflambda(\vz_{(r)}, \va_{(r)}) - \bflambda(\vz_{(0)}, \va_{(0)})})_{r=1}^{kd_z}$ are linearly independent.
\end{assumption}
Notice that $\vz_{(r)}$ represents values of $\vz^{<t}$ for potentially different values of $t$ and therefore can vary in shape. 

The following example builds on Example~\ref{ex:single_node_complete_dis_2} and shows that the sufficient variability of the above theorem might hold or not. The first case is interesting since it guarantees that $\vv$ is linear while the second is interesting because it showcases a situation where the theory of~\citet{iVAEkhemakhem20a} and \citet{lachapelle2022disentanglement} do not apply (since they both rely on the above theorem), thus highlighting the importance of our nonparametric extension. 

\begin{example}[Satisfying or not the sufficient variability assumption of Theorem~\ref{thm:linear}]\label{ex:linear}
    We recall\linebreak Example~\ref{ex:single_node_complete_dis_2} in which $d_a = d_z$, $\va \in \gA := \{\bm0, \ve_1, \dots, \ve_{d_a}\}$ and $\mG^a = \mI$ without temporal dependencies: For all $i\in [d_z]$, $p(\vz_i \mid \va) = \gN(\vz ; \vmu_i\va_i, 1 + \vdelta_i\va_i)$ where $\vmu_i \in \sR$ and $\vdelta_i > -1$. We consider the cases where $\forall i,\ \vdelta_i = 0$ (variances do not change) and $\forall i,\ \vdelta_i \not=0$ (variances change).
    
    {If $\forall i,\ \vdelta_i = 0$}, we can represent $p(\vz_i \mid \va)$ in its exponential form with a one-dimensional sufficient statistic given by $\vs_i(\vz_i) = \vz_i$ and a natural parameter given by $\bflambda_i(\va) = \vmu_i\va_i$. It can be easily seen that if $\forall i,\ \vmu_i \not= 0$ (i.e. the mean changes after the intervention), then the sufficient variability assumption of Theorem~\ref{thm:linear} holds since the vectors $\bflambda(\ve_i) - \bflambda(\bm0) = \vmu_i \ve_i$ span $\sR^{d_z}$.

    {If $\forall i,\ \vdelta_i \not= 0$}, we can represent $p(\vz_i \mid \va)$ in its exponential form with a two-dimensional sufficient statistics given by $\vs_i(\vz_i) = (\vz_i, \vz_i^2)$ and natural parameter given by $\bflambda_i(\va) = \left(\frac{\vmu_i\va_i}{1 + \vdelta_i\va_i}, \frac{-1}{2(1 + \vdelta_i\va_i)}\right)$. Note that, because we only have $d_z$ interventions, for any choice of $\va_{(0)} \in \gA$, the vectors $\{\bflambda(\va) - \bflambda(\va_{(0)})\}_{\va \in \gA}$ can span at most a $d_z$-dimensional subspace, which is insufficient variability according to Theorem~\ref{thm:linear} since it requires spanning $\sR^{2d_z}$. 
\end{example}

\subsection{Partial Disentanglement via Sparse Time Dependencies in Exponential Families}\label{sec:specialized_identi_expfam}

We now provide a (partial) disentanglement guarantee that takes advantage of sparsity regularization of $\hat\mG^z$ and is specialized for exponential families with a one-dimensional sufficient statistic ($k=1$). We will see that this extra parametric assumption on the transition model allows us to weaken the sufficient influence assumption of Theorem~\ref{thm:nonparam_dis_z} (Assumption~\ref{ass:nonparam_suff_var_z}). In particular, this is going to allow for Markovian transitions without auxiliary variables, which was not allowed by the nonparametric result (Remark~\ref{rem:non_markovian}). 

The sufficient influence assumption for $\vz$ specialized to exponential families with $k=1$ is directly taken from~\citet{lachapelle2022disentanglement}:

\begin{assumption}[Sufficient influence of  $\vz$~\citep{lachapelle2022disentanglement}]\label{ass:temporal_suff_var_main} Assume $k = 1$ and $D\vs(\vz)$ is invertible everywhere. There exist $\{(\vz_{(r)}, \va_{(r)}, \tau_{(r)})\}_{r=1}^{||\mG^z||_0}$ belonging to their respective support such that
    \begin{align}
        \vecspan \left\{D^{\tau_{(r)}}_{z}\bflambda(\vz_{(r)}, \va_{(r)}) D\bfT(\vz^{\tau_{(r)}}_{(r)})^{-1}\right\}_{r=1}^{||\mG^z||_0} = \sR^{d_z\times d_z}_{\mG^z} \, , \nonumber
    \end{align}
     where $D^{\tau_{(r)}}_{z}\bflambda$ and $D \bfT$ are Jacobians with respect to $\vz^{\tau_{(r)}}$ and $\vz$, respectively.
\end{assumption}

In Appendix~\ref{app:connecting_ass}, we show that the above assumption is implied by its nonparametric version (Assumption~\ref{ass:nonparam_suff_var_z}) when the transition model is in an exponential family with $k=1$. However, Assumption~\ref{ass:temporal_suff_var_main} is \textit{strictly weaker} than its nonparametric counterpart, Assumption~\ref{ass:nonparam_suff_var_z}. The reason is that in the former, $\vz^{\tau_{(r)}}$ can vary for different $p$, whereas this is not allowed in the latter, since we require $\vz = \vz^{\tau_{(r)}}$ for all $r$. 

The following theorem, extended from~\citet{lachapelle2022disentanglement}, shows that making stronger parametric assumptions on the transition model allows to weaken the sufficient influence assumption. Note that its structure is nearly identical to Theorem~\ref{thm:nonparam_dis_z}. Its proof can be found in Appendix~\ref{app:expfam_dis_z}.

\begin{restatable}[Disentanglement via sparse temporal dependencies in exponential families]{theorem}{TemporalExp}\label{thm:expfam_dis_z}
    Let\newline $\vtheta := (\bff, \bflambda, \mG)$ and $\hat{\vtheta}:= (\hat \bff,\hat \bflambda, \hat \mG)$ be two models satisfying Assumptions~\ref{ass:diffeomorphism}, \ref{ass:cond_indep}, \ref{ass:graph}, \ref{ass:smooth_trans}, \ref{ass:expo} as well as all assumptions of Theorem~\ref{thm:linear}. Further suppose that
    \begin{enumerate}
        \item The sufficient statistic $\vs$ is $d_z$-dimensional ($k=1$) and is a diffeomorphism from $\sR^{d_z}$ to $\bfT(\sR^{d_z})$; 
        \item \textbf{[Sufficient influence of $\vz$]} The Jacobian of the ground-truth transition function~$\bflambda$ with respect to $\vz$ varies ``sufficiently'', as formalized in Assumption~\ref{ass:temporal_suff_var_main};
    \end{enumerate}
    Then, there exists a permutation matrix $\mP$ such that $\mP\mG^z\mP^\top \subseteq \hat\mG^z$. Further assume that 
    \begin{enumerate}[resume]
        \item \textbf{[Sparsity regularization]} $||\hat\mG^z||_0 \leq ||\mG^z||_0$;
    \end{enumerate}
    Then, $\vtheta \eqcon^\vz \hat\vtheta$ (Def.~\ref{def:z_consistent_models}) \& $\vtheta \eqlin \hat\vtheta$ (Def.~\ref{def:linear_eq}), which together implies that 
    $$\vv(\vz) = \vs^{-1}(\mC\mP^\top\vs(\vz) + \vb)\, ,$$ 
    where $\vb \in \sR^{d_z}$ and $\mC \in \sR^{d_z \times d_z}$ is invertible, $\mG^\vz$- and $(\mG^\vz)^\top$-preserving (Definition~\ref{def:g_preserving_mat}). 
\end{restatable}

The reason we can simplify the sufficient influence assumption in the exponential family case has to do with the quasi-linear form of $\vv$. Indeed, in that case, one can compute that the Jacobian of $\vv$ takes a special form: $D\vv(\vz) = D\bfT(\bfv(\vz))^{-1}\mL D\bfT(\vz)$. Since $D\bfT$ is diagonal everywhere, one can see that the ``non-diagonal part'' of $D\vv(\vz)$, i.e. $\mL$, does not depend on $\vz$, which simplifies the proof. See Appendix~\ref{app:expfam_dis_z} for more details.

In Appendix~\ref{sec:implausible_ivae}, we discuss how \cite{iVAEkhemakhem20a} \& \cite{yao2022iclr} obtain disentanglement guarantees and how their assumptions differ from ours.

\begin{example}[Markovian sparse temporal dependencies without auxiliary variables]\label{ex:lower_triangular_no_action_exp_expfam}
    Recall Remark~\ref{rem:non_markovian} which pointed out that, without auxiliary variables, non-Markovianity was necessary to satisfy the nonparametric Assumption~\ref{ass:nonparam_suff_var_z}. We now illustrate that the analogous assumption specialized for exponential families with $k=1$ (Assumption~\ref{ass:temporal_suff_var_main}) is not as restrictive, i.e. it allows for Markovianity even when there are no auxiliary variables.
    
    We start from Example~\ref{ex:lower_triangular_no_action} which was based on the situation depicted in Figures~\ref{fig:working_example} \& \ref{fig:Gz_d} where the temporal graph $\mG^z$ is lower triangular. Assume that no action is observed, i.e. we can only leverage the sparsity of $\mG^z$ to disentangle. We now provide a concrete Markovian transition model $p(\vz^t \mid \vz^{t-1})$ that satisfies Assumption~\ref{ass:temporal_suff_var_main}. Similarly to previous examples, assume $p(\vz^t \mid \vz^{t-1}) = \gN(\vz^t ; \vmu(\vz^{t-1}), \sigma^2\mI)$ where
    \begin{align*}
        \vmu(\vz) := \vz + \begin{bmatrix}
            \vz_1^2/2 \\
            \vz_1^3/3 \\
            \vz_1^4/4
        \end{bmatrix} +
        \begin{bmatrix}
            0         \\
            \vz_2^2/2 \\
            \vz_2^3/3 
        \end{bmatrix} + 
        \begin{bmatrix}
            0\\
            0 \\
            \vz_3^2/2 
        \end{bmatrix}\, .
    \end{align*}
    Because the variance $\sigma^2$ is not influenced by $\vz^{t-1}$, we can represent this transition model in an exponential family with $k=1$ where the natural parameter is given by
    \begin{align*}
        \bflambda(\vz^{t-1}) = \mu(\vz^{t-1}) / \sigma
    \end{align*}
    and the sufficient statistic is given by $\vs(\vz) = \vz / \sigma$. We can thus compute
    \begin{align*}
        D\bflambda(\vz)D\vs(\vz)^{-1} = \mI + \begin{bmatrix}
            \vz_1 & 0 & 0\\
            \vz_1^2 & \vz_2 & 0 \\
            \vz_1^3 & \vz_2^2 & \vz_3 
        \end{bmatrix}
    \end{align*}
    which spans the 6-dimensional space $\sR^{3 \times 3}_{\mG^z}$, as shown in Example~\ref{ex:lower_triangular_no_action_exp}.
\end{example}

\paragraph{Connecting with the sufficient influence assumption of $\va$ in \citet{lachapelle2022disentanglement}} The previous work of \citet{lachapelle2022disentanglement} could also leverage the sparse influence of $\va$ to disentangle and was based on exponential family and sufficient influence assumptions. In Appendix~\ref{app:connecting_ass}, Proposition~\ref{prop:suff_var_nonparam_vs_exp} shows that their sufficient influence assumption for $\va$ is actually equivalent to our nonparametric version (Assumption~\ref{ass:nonparam_suff_var_a}) in the exponential family case with $k=1$. An important conclusion of this observation is that the identifiability result via sparse $\mG^\va$ from \citet{lachapelle2022disentanglement}, which was limited to the exponential family case with $k=1$, can be derived from the more general nonparametric result of Theorem~\ref{thm:nonparam_dis_disc_a} we introduced earlier.

\section{Model Estimation with a Sparsity Constraint} \label{sec:estimation}

The identifiability results presented in this work are based on two crucial postulates: (i) the distribution over observations of both the learned and ground-truth models must match, i.e. $\hat\vtheta \eqobs \vtheta$ (Definition~\ref{def:eqobs}), and (ii) the learned graphs $\hat\mG^a$ and $\hat\mG^z$ must be as sparse as their ground-truth counterparts, respectively $\mG^a$ and $\mG^z$. The theory suggests that in order to learn a (partially) disentangled representation, one should learn a model that satisfies these two requirements. In this section, we present one particular practical approach to achieve this. Appendix~\ref{sec:ours_details} provides further details.

\paragraph{Data fitting.} The first condition can be achieved by fitting a model to the data. Since the models discussed in this work present latent variable models, it is a natural idea to use a maximum likelihood approach based on the well-known framework of variational autoencoders (VAEs) \citep{Kingma2014} in which the decoder neural network corresponds to the mixing function $\hat\bff$. We consider an approximate posterior of the form
\begin{align}
    q(\vz^{\leq T} \mid \vx^{\leq T}, \va^{<T}) := \prod_{t=1}^{T}q(\vz^t \mid \vx^t) \, , \label{eq:approx_post}
\end{align}
where $q(\vz^t \mid \vx^t)$ is a Gaussian distribution with mean and diagonal covariance outputted by a neural network $\texttt{encoder}(\vx^t)$. In our experiments, the latent model $\hat p(\vz_i^t \mid \vz^{<t}, \va^{<t})$ is a Gaussian distribution with mean $\hat\vmu_i(\vz^{<t}, \va^{<t})$ parameterized as a fully connected neural network that ``looks" only at a fixed window of $s$ lagged latent variables.\footnote{The theory we developed would allow for a $\vmu$ function that depends on all previous time steps, not only the $s$ previous ones. This could be achieved with a recurrent neural network or transformer, but we leave this to future work.} Furthermore, the variances are learned but does not depend on $(\vz^{<t}, \va^{<t})$ (see Appendix~\ref{sec:ours_details} for details). This variational inference model induces the following evidence lower bound (ELBO) on $\log \hat{p}(\vx^{\leq T}|\va^{<T})$:
\begin{align}
    &\log \hat{p}(\vx^{\leq T}|\va^{<T}) \geq \text{ELBO}(\hat\bff, \hat\vmu, \hat \mG, q; \vx^{\leq T}, \va^{< T}) := \nonumber\\
    &\quad\quad\quad\quad\quad\quad\quad \sum_{t=1}^T\mathop{\sE}_{q(\vz^t| \vx^t)} [\log \hat{p}(\vx^t \mid \vz^t) ] - \mathop{\sE}_{q(\vz^{<t} \mid \vx^{<t})} KL(q(\vz^t \mid \vx^t) || \hat{p}(\vz^t \mid \vz^{<t}, \va^{<t})) \, .\label{eq:elbo}
\end{align}
We derive this fact in Appendix~\ref{app:elbo}. This lower bound can then be maximized using some variant of stochastic gradient ascent such as Adam~\citep{Adam}. We note that many works have proposed learning dynamical models with latent variables using VAEs~\citep{girin2020dynamical}, with various choice of architectures and approximate posteriors. Our specific choices were made out of a desire for simplicity, but the reader should be aware of other possibilities. 

The learned distribution will exactly match the ground truth distribution if (i) the model has enough capacity to express the ground-truth generative process, (ii) the approximate posterior has enough capacity to express the ground-truth posterior $p(\vz^t | \vx^{\leq T}, \va^{<T})$, (iii) the dataset is sufficiently large, and (iv) the optimization finds the global optimum. If, in addition, the ground truth generative process satisfies the assumptions of Proposition~\ref{prop:identDiffeo}, we can guarantee that the learned model $\hat{\vtheta}$ will be equivalent to the ground truth model $\vtheta$ up to diffeomorphism (Definition~\ref{def:eqdiff}).

\paragraph{Learning $\hat\mG$ with sparsity constraints.} To go from equivalence up to diffeomorphism to actual disentanglement (partial or not), Theorems~\ref{thm:nonparam_dis_cont_a}, \ref{thm:nonparam_dis_disc_a}, \ref{thm:nonparam_dis_z} \& \ref{thm:expfam_dis_z} suggest that we should not only fit the data, but also choose the learned graph $\hat\mG$ such that $||\hat{\mG}^a||_0 \leq ||{\mG}^a||_0$ and/or $||\hat{\mG}^z||_0 \leq ||{\mG}^z||_0$. In order to allow gradient-based optimization, our strategy consists in treating each edge $\hat\mG_{i,j}$ as an independent Bernoulli random variable with probability of success $\sigma(\vgamma_{i,j})$, where $\sigma$ is the sigmoid function and $\vgamma_{i,j}$ is a parameter learned using the Gumbel-Softmax trick~\citep{jang2016categorical, maddison2016concrete}. Let $\text{ELBO}(\hat\bff, \hat\vmu, \hat \mG, q)$ be the ELBO objective of \eqref{eq:elbo} averaged over the whole dataset. We tackle the following constrained optimization problem:
\begin{align*}
    \max_{\hat\bff, \hat\vmu, \vgamma, q} \sE_{\hat \mG \sim \sigma(\vgamma)} \text{ELBO}(\hat\bff, \hat\vmu, \hat \mG, q)\ \ \text{subject to}\ \ \sE_{\hat \mG \sim \sigma(\vgamma)} ||\hat \mG||_0 \leq \beta \,.
\end{align*}
where $\beta$ is a hyperparameter (which should ideally be set to $\beta^* := ||\mG||_0$, i.e., the number of edges in the ground-truth graph) and $\hat \mG \sim \sigma(\vgamma)$ means that $\hat \mG_{i,j}$ are independent and distributed according to $\sigma(\vgamma_{i,j})$. Because $\sE_{\hat \mG \sim \sigma(\vgamma)} ||\hat \mG||_0 = ||\sigma(\vgamma)||_1$ where $\sigma(\vgamma)$ is a matrix, the constraint becomes $||\sigma(\vgamma)||_1 \leq \beta$. To solve this problem, we perform gradient descent-ascent on the Lagrangian function given by
\begin{align*}
    \sE_{\hat \mG \sim \sigma(\vgamma)}\text{ELBO}(\hat\bff, \hat\vmu, \hat \mG, q) - \alpha(||\sigma(\vgamma)||_1 - \beta)
\end{align*}
where the ascent step is performed w.r.t. $\hat\bff, \hat\vmu, \hat \mG$ and $q$; and the descent step is performed w.r.t. Lagrangian multiplier $\alpha$, which is forced to remain greater or equal to zero via a simple projection step. As suggested by~\citet{gallego2021flexible}, we perform \textit{dual restarts} which simply means that, as soon as the constraint is satisfied, the Lagrangian multiplier is reset to $0$. We used the \texttt{Cooper} library~\citep{gallegoPosada2022cooper}, which implements many constrained optimization procedures in Python, including the one described above. Note that we use Adam~\citep{Adam} for the ascent steps and standard gradient descent for the descent step on the Lagrangian multiplier $\alpha$.

We also found empirically that the following schedule for $\beta$ is helpful: We start training with $\beta = \max_{\hat \mG} ||\hat\mG||_0$ and linearly decrease its value until the desired number of edges is reached. This avoids getting a sparse graph too quickly while training, thus giving enough time to the model parameters to adjust. In each experiment, we trained for 300K iterations, and $\beta$ takes 150K to go from its initial value to its desired value. We discuss how to select the hyperparameter $\beta$ in Section~\ref{sec:exp}.

\section{Evaluation with $R_\textnormal{con}$ and SHD} \label{sec:evaluation}
In this section, we tackle the problem of quantitatively evaluating whether a learned representation $\hat\vz$ is completely or partially disentangled with respect to the ground-truth representation $\vz$, given a dataset of paired representations $\{(\vz^i, \hat\vz^i)\}_{i \in [N]}$. More precisely, we want to evaluate whether two models are $\va$-consistent or $\vz$-consistent (Definitions~\ref{def:a_consistent_models} \& \ref{def:z_consistent_models}). To achieve this, we must evaluate whether there exists a graph preserving map $\vc$ (Definition~\ref{def:g_preserving_map}) and a permutation matrix $\mP$ such that for all $i \in [N]$, $\vz^i = \vc(\mP^\top\hat\vz^i)$. %This is exactly what is required to evaluate, for instance, whether a learned model is $\va$-consistent to the ground-truth model (Definition~\ref{def:a_consistent_models}). 
For evaluation purposes, we assume we observe the ground-truth latent representation for each observation, i.e., we have $\{(\vx^i, \vz^i)\}_{i \in N}$ sampled i.i.d. from the ground-truth data generating process. We will take $\hat\vz^i := \texttt{encoder}(\vx^i)$ where $\texttt{encoder}$ is from the learned VAE model introduced in Section~\ref{sec:estimation}. For simplicity, we assume that $\vc$ is affine.\footnote{This is not a simplification when the latent factors in the model and in the data-generating process are Gaussian with fixed variance and the assumptions of Theorem~\ref{thm:linear} hold. That is because the latent model is in the exponential family with sufficient statistic $\vs(\vz) = \vz$ and, by Theorem~\ref{thm:linear}, we must have $\vz = \vs^{-1}(\mL\vs(\hat\vz) + \vb) = \mL\hat\vz + \vb$.}

We start with evaluating complete disentanglement. A popular choice for this is the \textit{mean correlation coefficient}~(MCC), which is obtained by first computing the Pearson correlation matrix $\mK \in \sR^{d_z \times d_z}$ between the ground-truth representation and the learned representation ($\mK_{i,j}$ is the correlation between $\vz_i$ and $\hat \vz_j$). Then $\text{MCC} := \max_{\mP \in \text{permutations}} \tfrac{1}{d_z}\sum_{i=1}^{d_z} |(\mK\mP)_{i,i}|$. We denote by $\hat{\mP}$ the optimal permutation found by MCC.

To evaluate whether the learned representation is identified up to linear transformation (Definition~\ref{def:linear_eq}), we perform linear regression to predict the ground-truth latent factors from the learned ones, and report the mean of the Pearson correlations between the predicted ground-truth latents and the actual ones. This metric is sometimes called the \textit{coefficient of multiple correlation}, and happens to be the square root of the better known \textit{coefficient of determination}, usually denoted by $R^2$. The advantage of using $R$ instead of $R^2$ is that the former is comparable to MCC, and we always have $\text{MCC} \leq R$. Let us denote by $\hat \mL$ the matrix of estimated coefficients, which should be thought of as an estimation of $\mL$ in Definition~\ref{def:linear_eq} (assuming $\vs(\vz) = \vz$, as is the case with Gaussian latents with fixed variance). Note that $\hat\mL$ was fitted to the standardized $\vz$ and $\hat\vz$ (shifted and scaled to have mean 0 and variance 1). This yields coefficients $\hat\mL_{i,j}$ that are directly comparable without changing the value of the $R$ score. We visualize $\hat\mL$ in Figures~\ref{fig:graph_vis_graph_crit}~\&~\ref{fig:graph_vis_no_graph_crit}.

To evaluate whether the learned representation is $\va$-consistent or $\vz$-consistent with the ground-truth~(Definitions~\ref{def:a_consistent_models}~\&~\ref{def:z_consistent_models}), as predicted by Theorems~\ref{thm:nonparam_dis_cont_a} \& \ref{thm:nonparam_dis_z}, we introduce a new metric, denoted by $R_\text{con}$. The idea behind $R_\text{con}$ is to predict the ground-truth factors $\vz$ from only the inferred factors $\hat\vz$ that are allowed by the equivalence relations. For example, for $\va$-consistency (Definition~\ref{def:a_consistent_models}), the relation between $\vz$ and $\hat\vz$ is given by $\vz = \vc(\mP^\top\hat\vz)$ where $\vc$ is a $\mG^a$-preserving diffeomorphism. Since we assume for simplicity that $\vc$ is affine, we have $\vz = \mC\mP^\top\hat\vz + \vb$ where $\mC$ is a $\mG^a$-preserving matrix. The idea is then to estimate both $\mP$ and $\mC$ using samples $(\vz, \hat\vz)$. The permutation $\mP$ is estimated by $\hat \mP$, which was found when computing MCC~(Section~\ref{sec:exp}). To estimate $\mC$, we compute $\hat \vz_\text{perm} := \hat \mP^\top \hat \vz$ and then compute the mask $\mM \in \{0,1\}^{d_z \times d_z}$ specifying which entries of $\mC$ are allowed to be nonzero, as required by the $\mG^a$-preservation property (Proposition~\ref{prop:charac_G_preserving}). Then, for every $i$, we predict the ground-truth $\vz_i$ by performing linear regression only on the \textit{allowed} factors, i.e., $\mM_{i, \cdot} \odot \hat \vz_\text{perm}$, and compute the associated coefficient of multiple correlations $R_{\text{con}, i}$ and report the mean, i.e., $R_\text{con} := \frac{1}{d_z} \sum_{i=1}^{d_z} R_{\text{con}, i}$. It is easy to see that we must have $R_\text{con} \leq R$, since $R_\text{con}$ was computed with fewer features than $R$. Moreover, $\text{MCC} \leq R_\text{con}$, because MCC can be thought of as computing exactly the same thing as for $R_\text{con}$, but by predicting $\vz_i$ only from $\hat \vz_{\text{perm},i}$, i.e. with fewer features than $R_\text{con}$. This means that we always have $0 \leq \text{MCC} \leq R_\text{con} \leq R \leq 1$. This is a nice property allowing us to compare all three metrics together and reflects the hierarchy between equivalence relations. Note that $R_\text{con}$ depends implicitly on the ground-truth graph, since the matrix $\mM$ indicating which entries of $\mC$ are forced to be zero by the equivalence relation depends on $\mG$.

To compare the learned graph $\hat\mG$ with the ground-truth $\mG$, we report the (normalized) \textit{structural Hamming distance} (SHD) between the ground-truth graph and the estimated graph \textit{permuted by $\hat\mP$}. More precisely, we report $\text{SHD}=(||\mG^a - \hat\mP^\top\hat\mG^a||_0 + ||\mG^z - \hat\mP^\top\hat\mG^z\hat\mP||_0)/ (d_a d_z + d_z^2)$, where $\hat\mP$ is the permutation found by MCC and $(d_a d_z + d_z^2)$ is the maximal number of edges $\mG$ can have.

\section{Related Work}\label{sec:lit_review}
\paragraph{Causal discovery.} Causal discovery is the problem of learning a causal graph from observational and possibly interventional data when all variables are observed, unlike in the present work where the representation $\vz$ is latent. In this simpler setting, graph identifiability is still challenging since, without interventions, one can only identify the Markov equivalence class of the causal graph (thus leaving some edge orientation ambiguous), provided the distribution entailed by the causal graphical model is faithful to its causal graph, i.e. its conditional independences are precisely the ones induced by the causal graph. The faithfulness condition implies that we can infer precisely which $d$-separation statements hold in the causal graph by performing conditional independence tests on the observed distribution. This class of methods is known as \textit{constrained-based}~\citep{pearl2009causality,peters2017elements}. Another approach is to define a score function on the space of graphs, usually a sparsity-regularized maximum likelihood score, and search the space of graphs to maximize it. These are known as \textit{score-based} methods~\citep{Chickering2003OptimalSearch,Shimizu_lingam,Peters2014jmlr,zheng2018dags}. These approaches can be specialized for the interventional setting~\citep{hauser2012characterization,dcdi} and the temporal setting~\citep{pamfil2020dynotears,Runge2019PCMCI}.

\paragraph{Linear and nonlinear ICA.} The first results showing that latent variables can be identified up to permutation and rescaling  at least date back to classical linear ICA which assumes a linear mixing function $\vf$ and mutually independent and non-Gaussian latent variables~\citep{JuttenHerault1991,Tong1993,COMON1994}. \citet{HYVARINEN1999429} showed that when allowing $\vf$ to be a general nonlinear transformation, a setting known as nonlinear ICA, mutual independence and non-Gaussianity alone are insufficient to identify the latent variables. This inspired multiple variations of nonlinear ICA that enabled identifiability by leveraging, e.g., nonstationarity~\citep{TCL2016} and temporal dependencies~\citep{PCL17}. \citet{HyvarinenST19} generalized these works by introducing a data generating process in which latent variables are conditionally mutually independent given an observed \textit{auxiliary variable} (corresponding to $\va$ in our work). These last three works rely on some form of \textit{noise contrastive estimation} (NCE)~\citep{Gutmann12JMLR}, but similar identifiability results have also been shown for VAEs~\citep{iVAEkhemakhem20a,pmlr-v119-locatello20a,slowVAE}, normalizing flows~\citep{Sorrenson2020Disentanglement} and energy-based models~\citep{ice-beem20}. \citet{kivva2022identifiability} showed that it is not necessary to observe the auxiliary variable to obtain disentanglement when the mixing function is piecewise affine and the latent factors are distributed according to a mixture of Gaussians with diagonal covariance. 

\paragraph{Causal representation learning (static).} Since the publication of the first iteration of this work in CLeaR 2022, the field now known as \textit{causal representation learning} (CRL)~\citep{scholkopf2021causal} gained significant traction. The prototypical problem of CRL is similar to nonlinear ICA in that the goal is to identify latent factors of variations but differs in that the latent variables are assumed to be related via a causal graphical model (CGM) and interventions on the latents are typically observed. Although some works assumed that the structure of the causal graph is known~\citep{kocaoglu2018causalgan,shen2021disentangled,nair2019causal,liang2023causal}, significant progress has been achieved recently in the setting where the latent causal graph is unknown and must be inferred from single-node interventions targeting latent variables 
\citep{ahuja2023interventional,squires2023linear,buchholz2023learning,vonkugelgen2023nonparametric,zhang2023identifiability,jiang2023learning, varici2023general,varici2023scorebased}. \textcolor{black}{See Remark~\ref{rem:static_CRL_and_us} for a discussion contrasting these works with our interventional framework.} In a similar spirit, \citet{liu2023identifying,Yang_2021_CVPR} leverage a form of nonstationarity that does not necessarily correspond to interventions, and \citet{Bengio2020A} suggests using adaptation speed as a heuristic objective to disentangle latent factors in the bivariate case, although without identifiability guarantees. The above works do not support temporal dependencies, unlike the framework presented in this work. Although we do focus on temporal dependencies, the special case where $T=1$ discussed in Remark~\ref{rem:static_CRL_and_us} and fleshed out in Examples~\ref{ex:multi_target_a_cont_a}, \ref{ex:single_node_complete_dis_2} \&  \ref{ex:group_interv} can be categorized as static CRL where the latent factors are independent \textcolor{black}{(Assumption~\ref{ass:cond_indep}).} This approach has been applied to single-cell data with gene perturbations~\citep{lopez2022learninglong,bereket2023modelling}. Importantly, Example~\ref{ex:group_interv} illustrates how \textit{multi-node} interventions on the latent factors can yield (partial) disentanglement in the independent factors regime. To the best of our knowledge, this constitutes the first identifiability guarantee from multi-node interventions with nonlinear mixing and should form an important step towards generalizing to arbitrary latent graphs. Note that \citet{bing2023identifying} recently proposed a disentanglement guarantee from multi-node interventions in the linear mixing setting. 

CRL is closely related to methods that assume access to \textbf{paired observations} $(\vx, \vx')$ that are generated from a common decoder $\vf$. These are in contrast with the works discussed above, which assume the samples from observational and interventional distributions are \textbf{unpaired}. In the \textit{paired} data regime, \citet{pmlr-v119-locatello20a} and \citet{ahuja2022sparse} assume that only  a small set of latent factors $S \subset [d_z]$ changes between $\vx$ and $\vx'$. Interestingly, \citet{pmlr-v119-locatello20a} assume that for all $i$, $P(S \cap S' = \{i\}) > 0$ (for i.i.d $S$ and $S'$), which resembles our graphical criterion for complete disentanglement~(Definition~\ref{def:graph_crit}). \citet{karaletsos2016bayesian} proposed a related strategy based on triplets of observations and weak labels indicating which observation is closer to the reference in the (masked) latent space. \citet{vonkugelgen2021selfsupervised} modeled the self-supervised setting with data augmentation using a similar idea and showed the block-identifiability of the latent variables shared among $\vx$ and $\vx'$. \citet{brehmer2022weakly} assumes that the latent variables are sampled from a structural causal model (SCM)~\citep{peters2017elements} and that $\vx'$ is \textit{counterfactual} in the sense that it is generated using the same SCM and exogenous noise values as $\vx$ except for some noises which are modified randomly. Similar approaches can also provide identifiability guarantees in the \textit{multi-view} setting where the decoders for the different views $\vx$ and $\vx'$ are allowed to be different~\citep{gresele20rosetta,daunhawer2023identifiability}. The paired observation setting bears some similarity with the temporal setting covered in this work since the pairs $(\vx^t, \vx^{t-1})$ are observed jointly. However, contrary to the above works, Theorems~\ref{thm:nonparam_dis_z}~\&~\ref{thm:expfam_dis_z} allow all latent variables to change between $t-1$ and $t$, only the temporal dependencies between them are assumed to be sparse. \citet{pmlr-v206-morioka23a} can also be seen as paired CRL in which the latents of different views can interact causally in a restricted manner. Recently, \citet{yao2023multiview} generalized previous work by allowing more than two views.

\paragraph{Leveraging temporal dependencies or non-stationarity.} \citet{AmuseICA90} proved the identifiability of linear ICA when latent factors $\vz_i^t$ are correlated across time steps $t$ but remain independent across components $i$, an idea that has been extended to nonlinear mixing \citep{PCL17,slowVAE,Schell_continuous_timeICA}. Using our notation, these works assume a diagonal adjacency matrix $\mG^z$ which contrasts with Theorems~\ref{thm:nonparam_dis_z}~\&~\ref{thm:expfam_dis_z} which allow for general $\mG^z$ (although some graphs might not yield complete disentanglement). \citet[Theorem 1]{yao2022neurips} also allows for general $\mG^z$, but does not rely on the sparsity of $\mG^z$ or sparse interventions on latent factors for identification. Instead, it relies on the conditional independence of $\vz_i^t$ given $\vz^{t-1}$ and on a ``sufficient variability'' condition involving the third cross-derivatives $\frac{\partial^3}{(\partial \vz^t_i)^2\partial \vz^{t-1}_j}\log p(\vz_i^t \mid \vz^{t-1})$ which excludes simple Gaussian models with homoscedastic variance like the ones we considered in Examples~\ref{ex:a_target_one_z_cont_a}, \ref{ex:multi_target_a_cont_a}, \ref{ex:multinode_linear_gauss} and in our experiments of Section~\ref{sec:exp}. %In that regard, it bears more resemblance to \citet{HyvarinenST19} and \citet{iVAEkhemakhem20a}.
General non-stationarity of the latent distribution, i.e., that are not necessarily sparse like what is considered in this work, can also be used to identify the latent factors \citep{TCL2016,HyvarinenST19,iVAEkhemakhem20a,Halva2020HMM,morioka2021innovation,yao2022iclr,yao2022neurips}, but these results require sufficient variability of higher-order derivatives/differences of the log-densities, which again typically exclude simple homoscedastic Gaussian models (see Appendix~\ref{sec:implausible_ivae} for more). \citet{ahuja2022properties} characterized the indeterminacies of the representation in latent dynamical models as the set of equivariances of the transition mechanism. In addition to temporal dependencies, one can also consider latent factors structured according to a spatial topology~\citep{halva2021disentangling}.

\paragraph{Dynamical causal representation learning:} The previous iteration of this work~\citep{lachapelle2022disentanglement} concurrently with \citet{CITRIS} introduced latent variables identifiability guarantees for dynamical latent models based on sparse interventions. \citet{lippe2022icitris} later proposed a generalization in which instantaneous causal connections are allowed. The key differences from the present work are (i) \citet{lippe2022icitris} considers interventions with \textit{known targets}, while the present work (as well as \citet{lachapelle2022disentanglement}) considers interventions with \textit{unknown targets}; (ii) \citet[Theorem~5]{lachapelle2022disentanglement} and Theorems~\ref{thm:nonparam_dis_z} \& \ref{thm:expfam_dis_z} do not necessarily need interventions to disentangle since they leverage sparsity of the temporal dependencies, contrary to \citet{lippe2022icitris}; (iii) \citet{lippe2022icitris} allows instantaneous causal connections, unlike the present work; and (iv) \citet{lippe2022icitris} demonstrates their approach on image data. \textcolor{black}{We hypothesize that the \textit{partially-perfect interventions} from \citet{lippe2022icitris} in which instantaneous causal effects are disabled might be enough to get identifiability in our framework. We leave this for future work.} The concurrent work of \citet{causeOccam2021} independently proposed an approach very similar to ours, which also learns a sparse latent causal graph over latent variables and actions using binary masks, but focuses on testing various algorithmic variants and empirically verifies that the approach works on interactive environments rather than formal identifiability guarantees. \citet{lopez2022learninglong,lei2022variational} found that such models adapt to sparse interventions more quickly than their entangled counterparts. \citet{keurti2023homomorphism} discusses disentanglement in temporal regimes through the lens of group theory but does not provide identifiability guarantees. Recently, \citet{BISCUIT} proposed a model similar to ours with disentanglement guarantees based on the constraint that the effect of the variable $\va^{t-1}$ (analogous to $R$ in their work) on each $\vz^t_i$ is mediated by a deterministic binary variable.

\paragraph{Constraining the decoder function $\vf$.} It is worth noting that one can also obtain disentanglement guarantees by constraining the decoder function $\vf$ in some way~\citep{TalebJutten1999,gresele2021independent,buchholz2022function,leemann2023posthoc,lachapelle2023additive,horan2021when}. In particular, this can be achieved by imposing some form of sparsity on $\vf$~\citep{moran2022identifiable,zheng2022SparsityAndBeyond,brady2023provably,xi2023indeterminacy}. In contrast, the present work assumes only that $\vf$ is a general diffeomorphism onto its image. Note that \citet{zheng2022SparsityAndBeyond} reused many proof strategies from the shorter version of this work \citep{lachapelle2022disentanglement}.

\paragraph{Disentanglement with explicit supervision.} Some works leverage more explicit supervision to disentangle. For example, \citet{ahuja2022towards} assumes that the labels are given by a linear transformation of mutually independent and non-Gaussian latent factors. Instead of relying on independence, \citet{lachapelle2023synergies, fumero2023leveraging} leverage the sparsity of the linear map to disentangle.

\paragraph{Other relevant works on sparsity.} The assumption that high-level variables are sparsely related to each other and/or to actions was discussed by \citet{consciousBengio,InducGoyal2021,ke2021systematic}. These ideas have also been leveraged by~\citet{goyal2021recurrent, goyal2021neural,madan2021fast} via attention mechanisms. Although these works are, in part, motivated by the same core assumption as ours, their focus is more on empirically verifying out-of-distribution generalization than it is on disentanglement~(Definition~\ref{def:disentanglement}) and formal identifiability results. The assumption that individual actions often affect only one factor of variation has been leveraged to disentangle~\citet{IndependentlyThomas}. Loosely speaking, the theory we developed in the present work can be seen as a formal justification for such approaches.

\begin{comment}
\cite{iVAEkhemakhem20a}, which introduced iVAE, is likely the closest to the present work. Thm.~\ref{thm:linear} is quite similar to Thm.~1 from \cite{iVAEkhemakhem20a}, but iVAE's notion of linear equivalence is different in that it does not characterize the relationship between $\bflambda$ and $\hat{\bflambda}$, which is crucial for our proof of Thm.~\ref{thm:combined}. The most significant distinction between the theory of~\citep{iVAEkhemakhem20a} and ours is how \textit{permutation-identifiability} is obtained: Thm.~2~\&~3 from iVAE shows that if the assumptions of their Thm.~1 are satisfied and $\bfT_i$ has dimension $k>1$ or is non-monotonic, then the model is not just linearly, but permutation-identifiable. In contrast, our theory covers the case where $k=1$ and $\bfT_i$ is monotonic, like in the Gaussian case with fixed variance. Interestingly, \citet{iVAEkhemakhem20a} mentioned this specific case as a counterexample to their theory in their Prop.~3. The extra power of our theory comes from the extra \textit{structure} in the dependencies of the latent factors coupled with sparsity regularization. In App.~\ref{sec:implausible_ivae}, we argue that the assumptions of iVAE for disentanglement are less plausible in an environment like the one depicted in Fig.~\ref{fig:working_example}, thus highlighting the importance of the case $k=1$ with monotonic $\bfT_i$ of Thm.~\ref{thm:combined}.
\end{comment}

\section{Experiments} \label{sec:exp}

To illustrate our identifiability results and the benefit of mechanism sparsity regularization for disentanglement, we apply the sparsity regularized VAE method of Section~\ref{sec:estimation} on various synthetic datasets. Section~\ref{sec:exp_graph_crit} focuses on graphs satisfying the criterion of Assumption~\ref{def:graph_crit} which, as we saw, guarantees complete disentanglement. We also verify experimentally that the sufficient influence assumptions are indeed important for disentanglement and explore latent model with both homoscedastic and heteroscedastic variance. Section~\ref{sec:exp_no_graph_crit} explores graphs that do not satisfy the criterion. Details about our implementation are provided in Appendix~\ref{sec:ours_details} and the code used to run these experiments can be found here: 
\url{https://github.com/slachapelle/disentanglement_via_mechanism_sparsity}. %\seb{Just need to accept pull request on github to merge with main branch.}

%that both satisfy~(Fig.~\ref{fig:diag_non_diag}) and violate~(Fig.~\ref{fig:no_suff_no_crit}, in App.~\ref{sec:violating_assumptions}) the assumptions of Thm.~\ref{thm:combined}. 

\paragraph{Synthetic datasets.} The datasets we considered are separated into two groups: \textit{Action} \& \textit{Time} datasets. The former group has only auxiliary variables, which we interpret as actions, without temporal dependence, so we fix $\hat{\mG}^z = \bm{0}$. The latter group has only temporal dependence without actions, so we fix $\hat{\mG}^a = \bm{0}$. In each dataset, the ground-truth mixing function $\bff$ is a randomly initialized neural network. The dimensionalities of $\vz$ and $\vx$ are $d_z = 10$ and $d_x = 20$, respectively. In the action datasets, the dimensionality of $\va$ is $d_a = 10$, unless otherwise specified. The ground-truth transition model $p(\vz^t \mid \vz^{<t}, \va^{<t})$ is always a Gaussian with a mean outputted by some function $\vmu_\mG(\vz^{t-1}, a^{t-1})$ (the data is \textit{Markovian}). For all datasets considered, the covariance matrix is given by $\sigma_z^2 \mI$, i.e., the variance if \textit{homoscedastic}, except for the datasets $\textnormal{ActionNonDiag}_{k=2}$ and $\textnormal{TimeNonDiag}_{k=2}$, which have \textit{heteroscedastic} variance.
%Hence, each dataset has a 1d sufficient statistic ($k=1$) that is also monotonic and, thus, is not covered by the theory of~\citet{iVAEkhemakhem20a}.
Appendix~\ref{sec:syn_data} provides more detailed descriptions of the datasets including the explicit form of $\vmu$ and $\mG$ in each case. Note that the learned transition model $\hat{p}(\vz^t \mid \vz^{t-1}, \va^{t-1})$ is also a homoscedastic Gaussian where the mean function $\hat\vmu$ is an MLP.

\paragraph{Baselines.} On the action datasets, we compare with TCVAE~\citep{tcvae}, iVAE~\citep{iVAEkhemakhem20a}. Only iVAE leverages the action. On the temporal datasets, we compare our approach with TCVAE, PCL~\citep{PCL17} and SlowVAE~\citep{slowVAE}. Only PCL and SlowVAE leverage the temporal dependencies. We also report the performance of a randomly initialized encoder (Random) and one trained via least-square regression directly on the ground-truth latent factors (Supervised). See Appendix~\ref{sec:baselines} for details on the baselines.

\paragraph{Unsupervised hyperparameter selection.} In practice, the hyperparameters cannot be selected so as to optimize MCC, since this metric requires access to the ground-truth latent factors. \citet{Duan2020UDR} introduced \textit{unsupervised disentanglement ranking}~(UDR) as a solution to unsupervised hyperparameter selection for disentanglement.  Figures~\ref{fig:clear_exp}~\&~\ref{fig:jmlr_exp} show the performance of all approaches using UDR to select the hyperparameter (when it has one). For our approach, we show a range of sparsity bounds $\beta$ and indicate the hyperparameter selected by UDR with a black star. Note that for our approach, we excluded from the UDR selection hyperparameters that yielded graphs with fewer edges than latent factors as a heuristic to prevent UDR from selecting overly sparse graphs. Figures~\ref{fig:clear_exp}~\&~\ref{fig:jmlr_exp} show that this \textit{unsupervised} procedure selects a reasonable regularization level (as indicated by the black star), although not always the optimal one. See Appendix~\ref{sec:udr} for more details.

\subsection{Graphs Allowing Complete Disentanglement (Satisfying Assumption~\ref{def:graph_crit})}
\label{sec:exp_graph_crit}

\begin{figure}[]
    \centering
    \includegraphics[width=\linewidth]{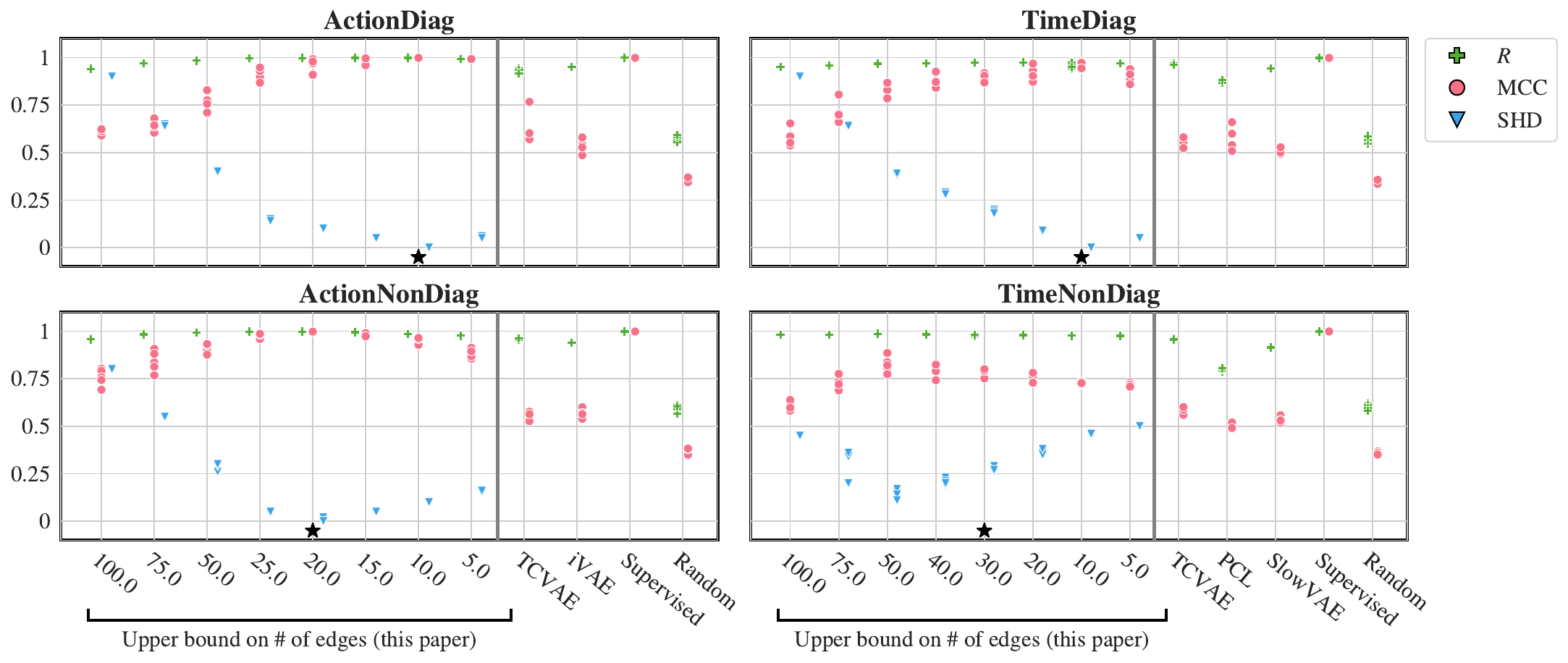}
    \caption{\textbf{Graphical criterion holds:} Datasets ActionDiag and TimeDiag have diagonal graphs while ActionNonDiag and TimeNonDiag have non-diagonal graphs. Sufficient influence is always satisfied. For our regularized VAE approach, we report performance for multiple sparsity levels $\beta$. In the left column, only $\hat{\mG}^a$ is learned while in the right column, only $\hat{\mG}^z$ is learned. For more details on the synthetic datasets, see Appendix~\ref{sec:syn_data}. The black star indicates which regularization parameter is selected by the filtered UDR procedure (see Appendix~\ref{sec:udr}). For $R$ and MCC, higher is better. For SHD, lower is better. Performance is reported on 5 random seeds.}
    \label{fig:clear_exp}
\end{figure}

\begin{figure}[h]
     \begin{subfigure}[b]{0.49\textwidth}
         \centering
         \includegraphics[width=\linewidth]{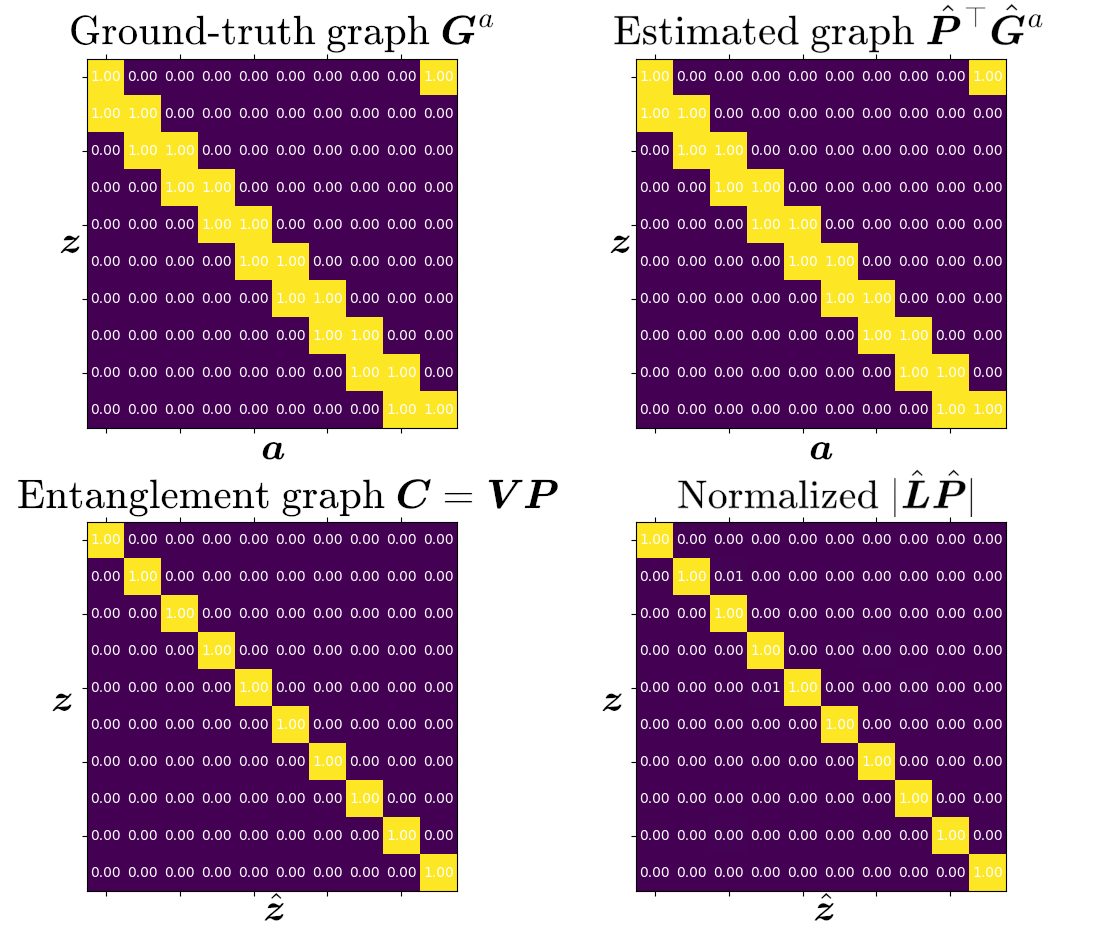}
    \caption{ActionNonDiag dataset, $\beta = 20$}
         \label{fig:graph_vis_ActionNonDiag}
     \end{subfigure}
     \hfill
     \begin{subfigure}[b]{0.49\textwidth}
         \centering
         \includegraphics[width=\linewidth]{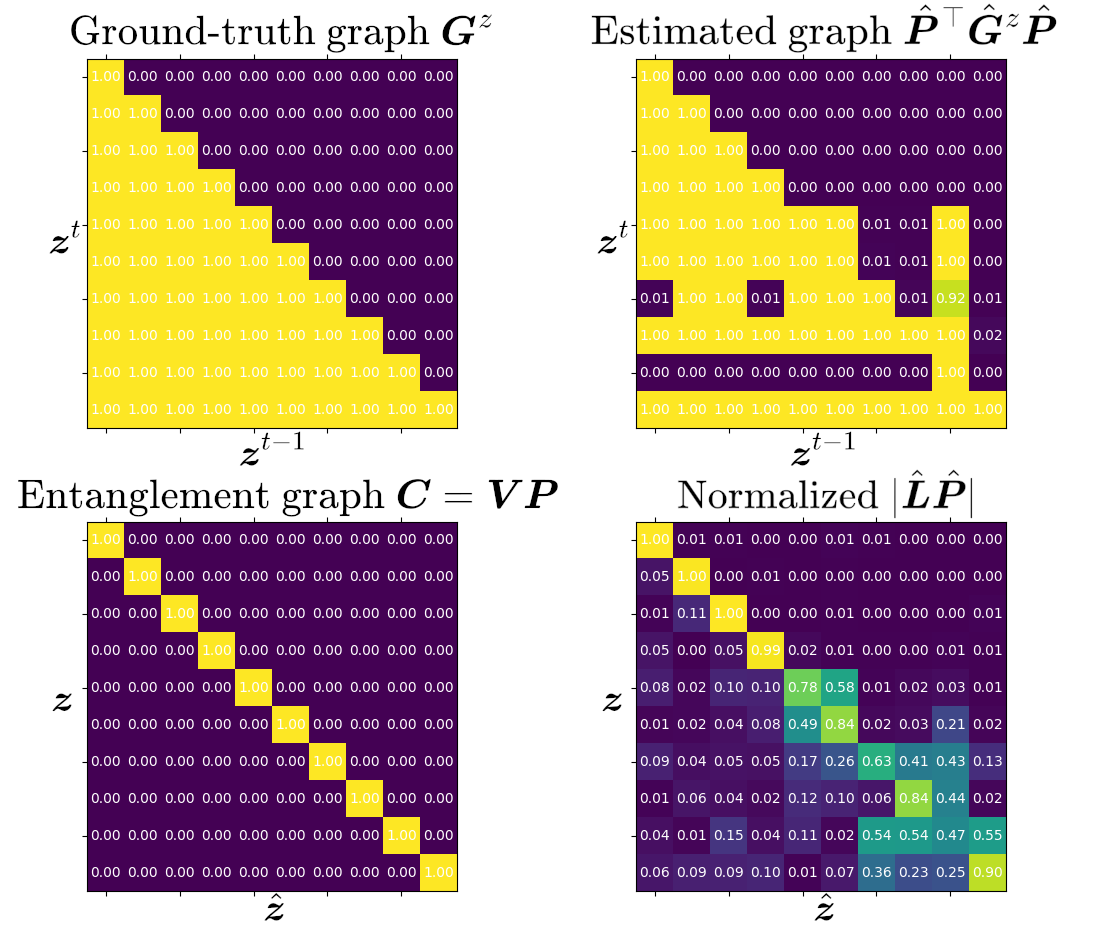}
    \caption{TimeNonDiag dataset, $\beta = 50$}
         \label{fig:graph_vis_TimeNonDiag}
     \end{subfigure}
     \hfill
     
        \caption{For each dataset, we visualize the median SHD run among the five randomly initialized runs of Figure~\ref{fig:clear_exp} with the sparsity level $\beta$ that is the closest to the ground-truth sparsity level $||\mG||_0$. For each dataset, we visualize (i) the ground-truth graph, (ii) the permuted estimated graph, (iii) the entanglement graph predicted by our theory, and (iv) the permuted matrix of regression coefficients in absolute value normalized by the maximum coefficient i.e. ${|\hat\mL\hat\mP|}/{\max_{i,j}|\hat\mL_{i,j}|}$. In Figure~\ref{fig:graph_vis_ActionNonDiag}, the estimated graph is exactly the ground-truth and $|\hat\mL\hat\mP|$ is perfectly diagonal, indicating complete disentanglement. In Figure~\ref{fig:graph_vis_TimeNonDiag}, the learned graph is close but not equal to the ground-truth. We can see that the off-diagonal nonzero values in $|\hat\mL\hat\mP|$ align with the poorly estimated parts of the graph.}
        \label{fig:graph_vis_graph_crit}
\end{figure}

\paragraph{Satisfying sufficient influence assumptions.} Figure~\ref{fig:clear_exp} reports the MCC and $R$ scores of all methods on four datasets that satisfy both the graphical criterion and the sufficient influence assumption: the datasets ActionDiag and TimeDiag have ``diagonal'' graphs, i.e., $\mG^a = \mI$ and $\mG^a = \mI$, while ActionNonDiag and TimeNonDiag present more involved graphs (depicted in Figure~\ref{fig:graph_vis_graph_crit}). ${\textbf{Observations:}}$ We see that the sparsity constraint improves MCC and SHD on all datasets. Although most baselines obtain good $R$ scores, which indicates that their representation encodes all the information about the factors of variations, they obtain poor MCC in comparison to our approach with a properly selected sparsity level, which indicates that they fail to disentangle. Moreover, the sparsity level selected by UDR (indicated by a black star) corresponds to the lowest SHD value for three out of four datasets and when it does not, it is still better than no sparsity at all. Figure~\ref{fig:graph_vis_graph_crit} shows examples of estimated graphs. More details can be found in the caption.  

\paragraph{Violating sufficient influence assumptions.} The left column of Table~\ref{tab:violation} reports the performance of all methods on the $\text{ActionNonDiag}_\text{NoSuffInf}$ and $\text{TimeNonDiag}_\text{NoSuffInf}$ datasets, which are essentially the same as ActionNonDiag and TimeNonDiag but do not satisfy the sufficient influence assumptions (see Appendix~\ref{sec:syn_data} for details). \noindent\textbf{Observations:} For the $\text{ActionNonDiag}_\text{NoSuffInf}$ dataset, we still see an improvement in MCC and SHD when regularizing for sparsity, but not as important as for ActionNonDiag, which got $\text{MCC}\approx 1$ and $\text{SHD}\approx 0$. Nevertheless, our approach outperforms the baselines. For the $\text{TimeNonDiag}_\text{NoSuffInf}$ dataset, there is simply no improvement in MCC from sparsity regularization. In that case, SlowVAE (with the hyperparameter selected to maximize MCC) and PCL have a higher MCC. These observations confirm the importance of the sufficient influence assumptions.

\paragraph{Heteroscedastic variance ($k=2$).} The right column of Table~\ref{tab:violation} reports the performance of all methods on the $\textnormal{ActionNonDiag}_{k=2}$ and $\textnormal{TimeNonDiag}_{k=2}$ datasets, which are essentially the same as ActionNonDiag and TimeNonDiag but presents heteroscedastic variance, i.e., $\text{var}(\vz^{t} \mid \vz^{t-1}, \va^{t-1})$ is not a constant function of $(\vz^{t-1}, \va^{t-1})$. This setting is interesting since it is not covered by the exponential family theory of \citet{lachapelle2022disentanglement} which assumed a one-dimensional sufficient statistic $\vs$ ($k=1$) whereas here we have $k=2$. Both datasets fall under the umbrella of our nonparametric theory. However, $\textnormal{TimeNonDiag}_{k=2}$ cannot satisfy the sufficient influence assumption because the data is Markovian and do not present an auxiliary variable (see Remark~\ref{rem:non_markovian}). ${\textbf{Observations:}}$ Both datasets benefit from sparsity and outperform baselines. On $\textnormal{ActionNonDiag}_{k=2}$ we obtain near perfect MCC and SHD while on $\textnormal{TimeNonDiag}_{k=2}$ we obtain a performance similar to TimeNonDiag. We hypothesize that the performance bottleneck in both TimeNonDiag and $\textnormal{TimeNonDiag}_{k=2}$ is the estimation of the graph, as in both cases SHD is always greater than $\approx 20\%$.

\begin{table}[]
    \centering
\scalebox{0.9}{
\begin{tabular}{llll|llll}
\toprule
Datasets & \multicolumn{3}{c}{$\text{ActionNonDiag}_\text{NoSuffInf}$} & \multicolumn{3}{c}{$\textnormal{ActionNonDiag}_{k=2}$} \\
\toprule
{Metrics} & SHD & MCC & $R$ & SHD & MCC & $R$ \\
\midrule
iVAE               &                                            -- &                 .61\scriptsize$\pm$.02 &    .97\scriptsize$\pm$.00&                     -- &                 .59\scriptsize$\pm$.03 &    .94\scriptsize$\pm$.00  \\
%TC-VAE            &                                            -- &                 .56\scriptsize$\pm$.04 &    .84\scriptsize$\pm$.08&                     -- &                 .49\scriptsize$\pm$.04 &    .80\scriptsize$\pm$.09  \\
TCVAE  (UDR)       &                                            -- &                 .58\scriptsize$\pm$.03 &    .88\scriptsize$\pm$.01&                     -- &                 .55\scriptsize$\pm$.02 &    .96\scriptsize$\pm$.00  \\
TCVAE  (MCC)       &                                            -- &                 .61\scriptsize$\pm$.02 &    .96\scriptsize$\pm$.00&                     -- &                 .55\scriptsize$\pm$.02 &    .96\scriptsize$\pm$.00  \\
Ours (no sparsity) &                        .80\scriptsize$\pm$.00 &                 .62\scriptsize$\pm$.02 &    .93\scriptsize$\pm$.00& .80\scriptsize$\pm$.00 &                 .70\scriptsize$\pm$.03 &    .97\scriptsize$\pm$.00  \\
Ours (sparsity)    &                        .13\scriptsize$\pm$.03 &                 .86\scriptsize$\pm$.04 &    1.0\scriptsize$\pm$.00& .03\scriptsize$\pm$.01 &                 .98\scriptsize$\pm$.02 &    1.0\scriptsize$\pm$.00  \\
\midrule
Random             &                                            -- &                 .37\scriptsize$\pm$.02 &    .63\scriptsize$\pm$.02&                    --  &                 .37\scriptsize$\pm$.02 &    .60\scriptsize$\pm$.02  \\
Supervised         &                                            -- &                 1.0\scriptsize$\pm$.00 &    1.0\scriptsize$\pm$.00&                     -- &                 1.0\scriptsize$\pm$.00 &    1.0\scriptsize$\pm$.00  \\
\bottomrule
\toprule
Datasets & \multicolumn{3}{c}{$\text{TimeNonDiag}_\text{NoSuffInf}$} & \multicolumn{3}{c}{$\textnormal{TimeNonDiag}_{k=2}$} \\
\toprule
{Metrics} & SHD & MCC & $R$ & SHD & MCC & $R$ \\
\midrule
PCL                 &                                              -- &                 .66\scriptsize$\pm$.04 &    .96\scriptsize$\pm$.00 &                      -- &                 .58\scriptsize$\pm$.04 &    .83\scriptsize$\pm$.01\\
%SlowVAE            &                                              -- &                 .64\scriptsize$\pm$.04 &    .98\scriptsize$\pm$.01 &                      -- &                 .54\scriptsize$\pm$.03 &    .93\scriptsize$\pm$.05\\
SlowVAE (UDR)       &                                              -- &                 .59\scriptsize$\pm$.02 &    .98\scriptsize$\pm$.00 &                      -- &                 .57\scriptsize$\pm$.01 &    .93\scriptsize$\pm$.00\\
SlowVAE (MCC)       &                                              -- &                 .71\scriptsize$\pm$.02 &    .98\scriptsize$\pm$.00 &                      -- &                 .58\scriptsize$\pm$.02 &    .95\scriptsize$\pm$.00\\
%TC-VAE             &                                              -- &                 .56\scriptsize$\pm$.02 &    .92\scriptsize$\pm$.04 &                     --  &                 .48\scriptsize$\pm$.05 &    .79\scriptsize$\pm$.10\\
TCVAE (UDR)        &                                              -- &                 .58\scriptsize$\pm$.03 &    .98\scriptsize$\pm$.00 &                      -- &                 .57\scriptsize$\pm$.01 &    .96\scriptsize$\pm$.00\\
TCVAE (MCC)        &                                              -- &                 .58\scriptsize$\pm$.03 &    .98\scriptsize$\pm$.00 &                      -- &                 .57\scriptsize$\pm$.01 &    .96\scriptsize$\pm$.00\\
Ours (no sparsity)  &                          .45\scriptsize$\pm$.00 &                 .62\scriptsize$\pm$.04 &    .98\scriptsize$\pm$.00 &  .45\scriptsize$\pm$.00 &                 .62\scriptsize$\pm$.01 &    .98\scriptsize$\pm$.00\\
Ours (sparsity)     &                          .32\scriptsize$\pm$.05 &                 .63\scriptsize$\pm$.03 &    .99\scriptsize$\pm$.00 &  .20\scriptsize$\pm$.07 &                 .74\scriptsize$\pm$.04 &    .98\scriptsize$\pm$.00\\
\midrule
Random              &                                              -- &                 .40\scriptsize$\pm$.04 &    .67\scriptsize$\pm$.02 &                      -- &                 .36\scriptsize$\pm$.01 &    .59\scriptsize$\pm$.02\\
Supervised          &                                             --  &                 1.0\scriptsize$\pm$.00 &    1.0\scriptsize$\pm$.00 &                     --  &                 1.0\scriptsize$\pm$.00 &    1.0\scriptsize$\pm$.00\\
\bottomrule
\end{tabular}
}
    \caption{Datasets $\text{ActionNonDiag}_\text{NoSuffInf}$ and $\text{TimeNonDiag}_\text{NoSuffInf}$ do not satisfy their respective sufficient influence assumptions (Assumptions~\ref{ass:nonparam_suff_var_a_cont}~\&~\ref{ass:temporal_suff_var_main}). Datasets $\textnormal{ActionNonDiag}_{k=2}$ and $\textnormal{TimeNonDiag}_{k=2}$ are such that $\text{var}(\vz^t \mid \vz^{t-1}, \va^{t-1})$ depends on $\vz^{t-1}$ or $\va^{t-1}$ (which means the sufficient statistic has dimension $k=2$, contrarily to all other datasets). For our method, we show performance both with and without the sparsity constraint. In the former case, the constraint is set to the number of edges in the ground-truth graph. For baselines that have hyperparameters, we report their performance with the hyperparameter configurations that maximize UDR and MCC.}
    \label{tab:violation}
\end{table}

\subsection{Graphs Allowing only Partial Disentanglement (Violating Assumption~\ref{def:graph_crit})}
\label{sec:exp_no_graph_crit}

In this section, we explore datasets with graphs that do not satisfy the criterion of Assumption~\ref{def:graph_crit}. This means that our theory can only guarantee a form of partial disentanglement. For this reason, we will report the $R_\text{con}$ metric introduced in Section~\ref{sec:evaluation} which measures whether two representations are $\va$-consistent (Definition~\ref{def:a_consistent_models}) or $\vz$-consistent (Definition~\ref{def:z_consistent_models}).

\paragraph{Satisfying sufficient influence assumptions.} Figure~\ref{fig:jmlr_exp} reports the MCC, $R_\text{con}$ and $R$ scores of all methods on four datasets that satisfy the sufficient influence assumption but not the graphical criterion: the datasets ActionBlockDiag and TimeBlockDiag have ``block diagonal'' graphs, while ActionBlockNonDiag and TimeBlockNonDiag have more intricate graphs (depicted in Figure~\ref{fig:graph_vis_no_graph_crit}). See Appendix~\ref{sec:syn_data} for details on the datasets. \textbf{Observations:} In all four datasets, some sparsity level yields near perfect $R_\text{con}$, indicating the learned models are approximately $\va$-consistent or $\vz$-consistent to the ground-truth. Moreover, SHD is correlated with $R_\text{con}$.  Without surprise, MCC never comes close to one, since complete disentanglement is not guaranteed by our theory. Similarly to Figure~\ref{fig:clear_exp}, the baselines have decent $R$ values but very low MCC and $R_\text{con}$, indicating that they cannot achieve partial disentanglement. Figure~\ref{fig:graph_vis_no_graph_crit} shows examples of estimated graphs. When it comes to hyperparameter selection, UDR selects the hyperparameter with the lowest SHD on three out of four datasets, which indicates that UDR does reasonably well.

\begin{figure}
    \centering
    \includegraphics[width=\linewidth]{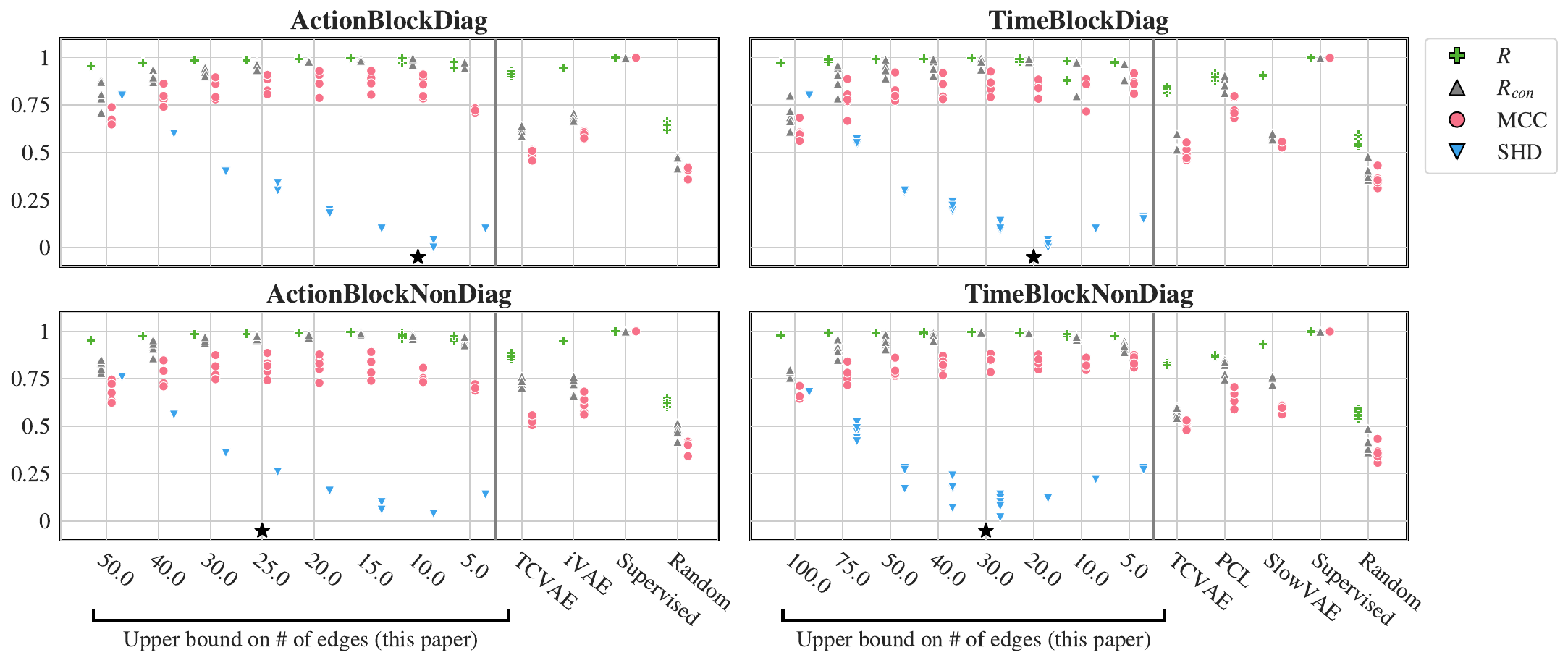}
    \caption{\textbf{Graphical criterion does not hold:} Datasets ActionBlockDiag and TimeBlockDiag have block-diagonal graphs while ActionBlockNonDiag and TimeBlockNonDiag have non-diagonal graphs. Sufficient influence is always satisfied. In the left column, only $\hat{\mG}^a$ is learned and we vary $\beta_a$, and in the right column, only $\hat{\mG}^z$ is learned and we vary $\beta_z$. For more details on the synthetic datasets, see Appendix~\ref{sec:syn_data}. The black star indicates which regularization parameter is selected by the filtered UDR procedure (see Appendix~\ref{sec:udr}). For $R$ and MCC, higher is better. For SHD, lower is better. Performance is reported on 5 random seeds.}
    \label{fig:jmlr_exp}
\end{figure}

\begin{figure}
     \begin{subfigure}[b]{0.49\textwidth}
         \centering
         \includegraphics[width=\linewidth]{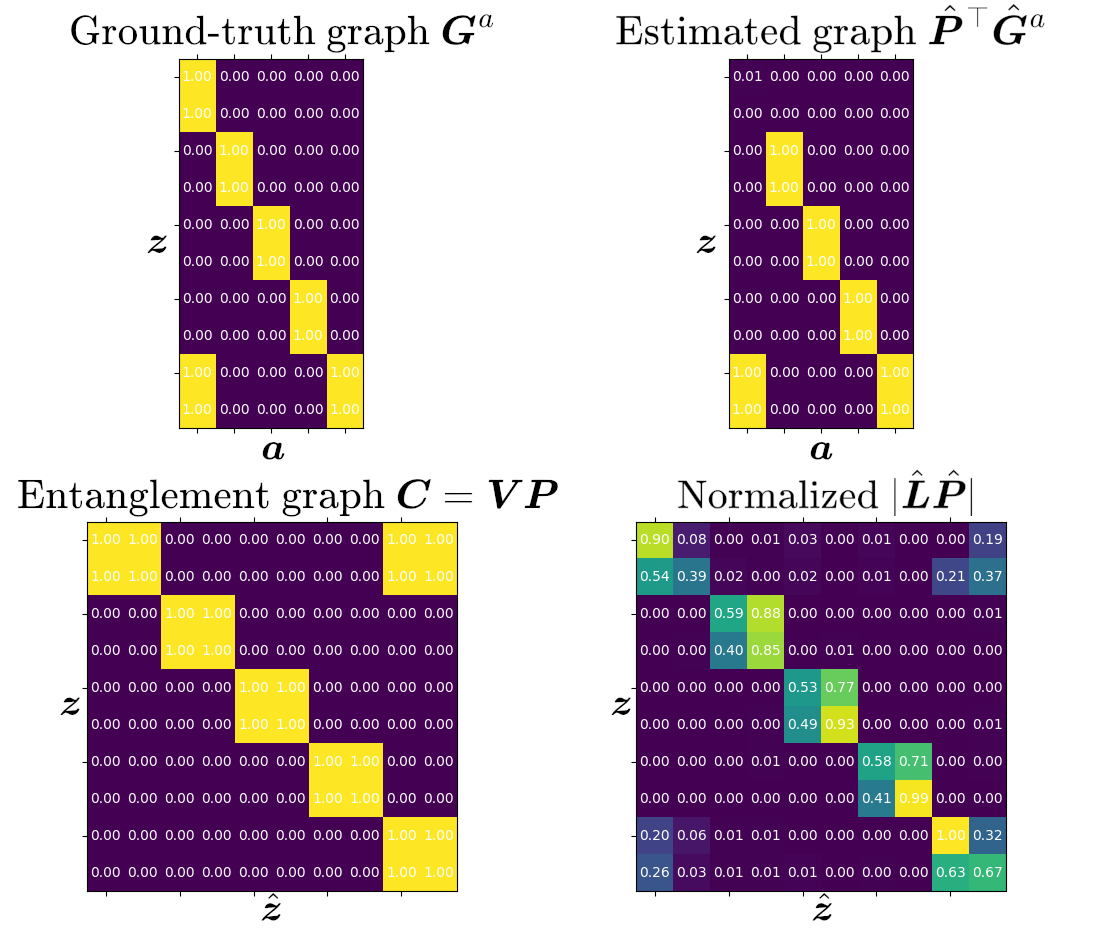}
    \caption{ActionBlockNonDiag dataset, $\beta = 10$}
         \label{fig:graph_vis_ActionBlockNonDiag}
     \end{subfigure}
     \hfill
     \begin{subfigure}[b]{0.49\textwidth}
         \centering
         \includegraphics[width=\linewidth]{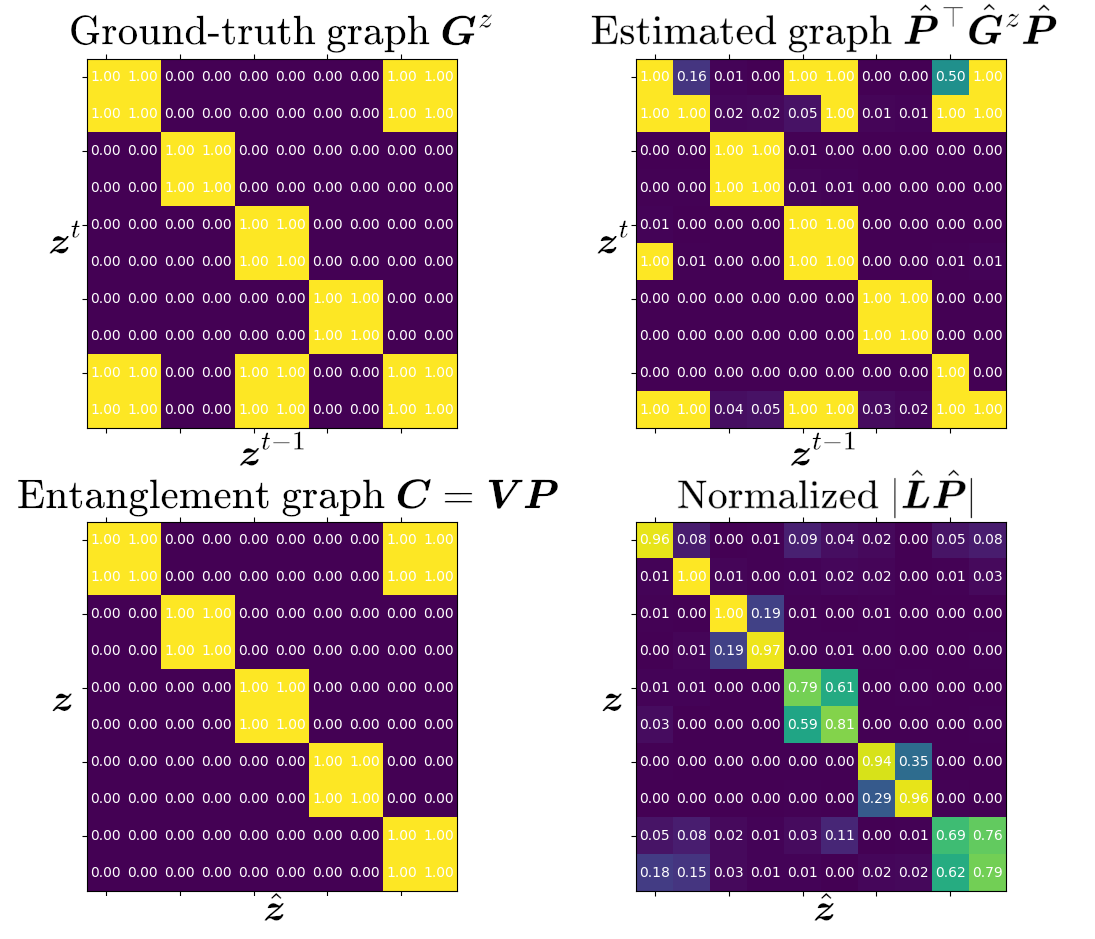}
    \caption{TimeBlockNonDiag dataset, $\beta = 30$}
         \label{fig:graph_vis_TimeBlockNonDiag}
     \end{subfigure}
     \hfill
     
        \caption{For each dataset, we visualize the median SHD run among the five randomly initialized runs of Figure~\ref{fig:jmlr_exp} with the sparsity level $\beta$ that is the closest to the ground-truth sparsity level $||\mG||_0$. For each dataset, we visualize (i) the ground-truth graph, (ii) the permuted estimated graph, (iii) the entanglement graph predicted by our theory, and (iv) the permuted matrix of regression coefficients in absolute value normalized by the maximum coefficient i.e. ${|\hat\mL\hat\mP|}/{\max_{i,j}|\hat\mL_{i,j}|}$. For both datasets, the learned graph is very close to the ground-truth. Furthermore, the match between the zero entries of $\hat\mL\hat\mP$ and those of the theoretical entanglement graph $\mC$ is very good, although not perfect. Notice how certain blocks of latent factors remain entangled, as predicted by the theory.}
        \label{fig:graph_vis_no_graph_crit}
\end{figure}

\begin{table}[]
    \centering
\scalebox{0.9}{
\begin{tabular}{c|lll|llll|c}
\toprule
\multicolumn{9}{c}{ActionRandomGraphs} \\
\toprule
 & \multicolumn{3}{c}{\textbf{Without sparsity}} & \multicolumn{4}{|c|}{\textbf{With sparsity}} & \\
 $p(\text{edge})$ & MCC & $R_\text{con}$ & $R$ & MCC & $R_\text{con}$ & $R$ & SHD & $\sE||\mV||_0$\\
\midrule
$10\%$&                 .58\scriptsize$\pm$.13 &               .61\scriptsize$\pm$.11 &    .68\scriptsize$\pm$.13  &                 .69\scriptsize$\pm$.14 &               .70\scriptsize$\pm$.12  &   .70\scriptsize$\pm$.12 & .00\scriptsize$\pm$.00 & 45.0 \\
$20\%$&                 .67\scriptsize$\pm$.06 &               .69\scriptsize$\pm$.05 &    .83\scriptsize$\pm$.08  &                 .85\scriptsize$\pm$.08 &               .86\scriptsize$\pm$.08  &   .86\scriptsize$\pm$.09 & .01\scriptsize$\pm$.01 & 25.8\\
$40\%$&                 .67\scriptsize$\pm$.03 &               .70\scriptsize$\pm$.03 &    .93\scriptsize$\pm$.04  &                 .94\scriptsize$\pm$.05 &               .95\scriptsize$\pm$.05  &   .98\scriptsize$\pm$.04 & .06\scriptsize$\pm$.05 & 15.8\\
$60\%$&                 .69\scriptsize$\pm$.06 &               .73\scriptsize$\pm$.05 &    .96\scriptsize$\pm$.00  &                 .88\scriptsize$\pm$.07 &               .91\scriptsize$\pm$.05  &   .99\scriptsize$\pm$.01 & .14\scriptsize$\pm$.08 & 15.8\\
$90\%$&                 .63\scriptsize$\pm$.04 &               .81\scriptsize$\pm$.08 &    .97\scriptsize$\pm$.00  &                 .60\scriptsize$\pm$.01 &               .78\scriptsize$\pm$.07  &   .97\scriptsize$\pm$.00 & .21\scriptsize$\pm$.07 & 45.0\\
\bottomrule
\toprule
\multicolumn{9}{c}{TimeRandomGraphs} \\
\toprule
 & \multicolumn{3}{c}{\textbf{Without sparsity}} & \multicolumn{4}{|c|}{\textbf{With sparsity}} \\
 $p(\text{edge})$ &  MCC & $R_\text{con}$ & $R$ & MCC & $R_\text{con}$ & $R$ & SHD & $\sE||\mV||_0$ \\
\midrule
$10\%$&                 .66\scriptsize$\pm$.03 &               .66\scriptsize$\pm$.03 &    .98\scriptsize$\pm$.00 &                  1.0\scriptsize$\pm$.00 &               1.0\scriptsize$\pm$.00 & 1.0\scriptsize$\pm$.00 &  .01\scriptsize$\pm$.02 & 10.2 \\
$20\%$&                 .63\scriptsize$\pm$.05 &               .63\scriptsize$\pm$.05 &    .98\scriptsize$\pm$.00 &                  .99\scriptsize$\pm$.01 &               .99\scriptsize$\pm$.01 & .99\scriptsize$\pm$.00 &  .08\scriptsize$\pm$.09 & 10.2\\
$40\%$&                 .61\scriptsize$\pm$.02 &               .61\scriptsize$\pm$.02 &    .98\scriptsize$\pm$.00 &                  .82\scriptsize$\pm$.16 &               .82\scriptsize$\pm$.16 & .98\scriptsize$\pm$.01 &  .27\scriptsize$\pm$.13 & 10.2 \\
$60\%$&                 .58\scriptsize$\pm$.02 &               .58\scriptsize$\pm$.02 &    .98\scriptsize$\pm$.00 &                  .71\scriptsize$\pm$.12 &               .71\scriptsize$\pm$.12 & .98\scriptsize$\pm$.00 &  .33\scriptsize$\pm$.06 & 10.4 \\
$90\%$&                 .58\scriptsize$\pm$.03 &               .63\scriptsize$\pm$.08 &    .98\scriptsize$\pm$.00 &                  .58\scriptsize$\pm$.02 &               .63\scriptsize$\pm$.08 & .98\scriptsize$\pm$.00 &  .20\scriptsize$\pm$.07 & 26.1\\
\bottomrule
\end{tabular}
}
    \caption{Experiments with randomly generated graphs. The probability of sampling an edge is $p(\text{edge})$. We report an estimation of $\sE||\mV||_0$ which is the average number of edges in the entanglement graph $\mV$ entailed by the random ground-truth graph $\mG$.}
    \label{tab:random_graphs}
\end{table}

\paragraph{Random graphs of varying sparsity levels.} In Table~\ref{tab:random_graphs}, we consider the same $\vmu$ functions as in datasets ActionNonDiag and TimeNonDiag, but explore more diverse randomly generated ground-truth graphs with various degrees of sparsity. Edges are sampled i.i.d. with some probability $p(\text{edge})$. However, note that for the TimeRandomGraphs dataset, the self-loops are present with probability one. We report the performance of our approach both with and without sparsity regularization. When using sparsity, we set $\beta$ equal to the ground-truth number of edges, $||\mG||_0$. ${\textbf{Observations:}}$ First, all datasets obtain an improvement in MCC and $R_\text{con}$ from sparsity regularization, except for very dense graphs  with $p(\text{edge}) = 90\%$, in which case regularization does nothing or slightly degrades performance. Secondly, we can see that the SHD tends to be higher for larger graphs, suggesting that these are harder to learn. Third, in the ActionRandomGraphs datasets, we can see a negative correlation between MCC and $\sE||\mV||_0$, which is expected since $||\mV||_0$ close to 10 means complete disentanglement is possible (assuming the graph is learned properly). This pattern also appears to some extent in the TimeRandomGraphs datasets. Notice how, among the TimeRandomGraphs datasets, all datasets sparser than $p(\text{edge}) = 90\%$ always have $\sE||\mV||_0 \approx 10$, indicating complete disentanglement should be possible.\footnote{We suspect this occurs because the self-loops which are present with probability one, unlike the Action dataset.} This is confirmed by very high MCC, at least for sparser graphs which are learned properly. Finally, we note that the $R$ score is low for the very sparse action datasets. We suspect this is because very sparse graphs are less likely to satisfy the assumption of sufficient variability (Theorem~\ref{thm:linear}), which guarantees quasi-linear equivalence (here it is actually \textit{linear} equivalence, because of Gaussianity). Indeed, for very sparse graphs, some latent factors might end up without parents. This is not the case in the time datasets because of the self-loops which are always present.

\section{Conclusion}
This work proposed a novel principle for disentanglement based on \textit{mechanism sparsity regularization}. The idea is based on the assumption that the mechanisms that govern the dynamics of high-level concepts are often sparse: actions usually affect only a few entities, and objects usually interact sparsely with each other. We provided novel nonparametric identifiability guarantees for this setting which give sufficient conditions for disentanglement, whether complete or partial. Given the dependency structure between latent factors and auxiliary variables, our theory predicts the entanglement graph that describes which of the estimated latent factors are expected to remain entangled. This constitutes a significant extension of the shorter version of this work \citep{lachapelle2022disentanglement}. We further provide various examples to illustrate the consequences of our guarantees as well as the assumptions they rely on. For instance, we show that multi-node interventions with unknown targets fall under the umbrella of our framework. Finally, we demonstrate the theory experimentally by training a sparsity-constrained variational autoencoder on synthetic data, which allows us to explore various settings. Our work establishes a solid theoretical foundation for further empirical investigations in more realistic scenarios, such as single-cell data with gene perturbations~\citep{lopez2022learninglong} and video~\citep{lei2022variational}. Future works include relaxing assumptions such as conditional independence or considering more permissive settings such as ``contextual sparsity'', i.e., the assumption that objects only interact with each other in particular situations. We believe that the latter could be formalized and leveraged for disentanglement using the tools developed in this work.

%Building on recent developments in nonlinear ICA, we constructed a rigorous theory which provides precise conditions, e.g. on the structure of the ground-truth dependency graph, for when regularizing the mechanisms to be sparse will result in disentanglement. A special case of our framework shows how one can leverage \textit{unknown-target interventions} on the latent factors, or \textit{sparse mechanism shifts}, for disentanglement. We proposed a regularized VAE-based approach and demonstrated that it can improve disentanglement in controlled synthetic settings, thereby preparing the stage for more realistic scenarios, e.g. interactive environments. We believe this work opens up new possibilities at the intersection of causality and disentanglement that leverage \textit{structural assumptions}. For instance, we posit that contextual sparsity, i.e., the assumption that objects only interact with each other in particular situations, could be formalized and leveraged for disentanglement using the tools developed in this work.

% Acknowledgements and Disclosure of Funding should go at the end, before appendices and references
\clearpage 
\acks{The authors would like to thank Aristide Baratin for important feedback on the manuscript. This research was partially supported by the Canada CIFAR AI Chair Program, by an IVADO excellence PhD scholarship, by a Google Focused Research award, the German Federal Ministry of Education and Research (BMBF): Tübingen AI Center, FKZ: 01IS18039A, and by Mitacs through the Mitacs Accelerate program. The experiments were in part enabled by computational resources provided by Calcul Quebec and the Digital Research Alliance of Canada. The authors would like to thank Yoshua Bengio for inspiring mechanism sparsity regularization through various talks and discussions. The authors would also like to thank the International Max Planck Research School for Intelligent Systems (IMPRS-IS) for supporting Yash Sharma. Simon Lacoste-Julien is a CIFAR Associate Fellow in the Learning in Machines \& Brains program.\newpage}
%}

\appendix

\begin{table}[H]\caption{{Table of Notation.}} %htbp
    {
    \begin{center}% used the environment to augment the vertical space
        % between the caption and the table
            \begin{tabular}{r c p{10cm} }
                \toprule
                \multicolumn{3}{c}{}\\
                \multicolumn{3}{c}{\underline{Calligraphic \& indexing conventions}}\\
                $[n]$ & $:=$ & $\{1, 2, \dots, n\}$ \\
                $x$ & & Scalar (random or not, depending on context)\\
                $\vx$& & Vector (random or not, depending on context) \\
                $\mX$ & & Matrix \\
                $\gX$ & & Set/Support \\
                $f$   & & Scalar-valued function \\
                $\vf$   & & Vector-valued function \\
                $Df$, $D\vf$ & & Jacobian of $f$ and $\vf$ \\
                $D^2f$ & & Hessian of $f$ \\
                $B \subseteq [n]$ & & Subset of indices \\
                $\vx_B$ & & Vector formed with the $i$th coordinates of $\vx$, for all $i \in B$\\
                $\mX_{B, B'}$ & & Matrix formed with the entries $(i, j) \in B \times B'$ of $\mX$. \\
                %\midrule
                \multicolumn{3}{c}{}\\
                \multicolumn{3}{c}{\underline{Recurrent notation}}\\
                $\vx^t \in \sR^{d_x}$ &  & Observation at time $t$\\
                $\vx^{\leq t} \in \sR^{d_x \times t}$ &  & Matrix of observations at times $1, \dots, t$\\
                $\vz^t \in \sR^{d_z}$ &  & Vector of latent factors of variations at time $t$ \\
                $\vz^{\leq t} \in \sR^{d_z \times t}$ &  & Matrix of latent vectors at times $1, \dots, t$\\
                %$\gZ \subseteq \sR^{d_z}$ & & Support of $\vz^t$ \\
                $\va^t \in \sR^{d_a}$ &  & Vector of auxiliary variables at time $t$ \\
                $\va^{< t} \in \sR^{d_a \times t}$ &  & Matrix of auxiliary vectors at times $0, 1 \dots, t - 1$\\
                $\gA \subseteq \sR^{d_a}$ &  & Support of $\va^t$\\
                $\vf: \sR^{d_z} \rightarrow \sR^{d_x}$ &  & Ground-truth decoder function \\
                $\hat\vf: \sR^{d_z} \rightarrow \sR^{d_x}$ &  & Learned decoder function \\
                $p(\vz^t \mid \vz^{<t}, \va^{<t})$ &  & Ground-truth latent transition model \\
                $\hat p(\vz^t \mid \vz^{<t}, \va^{<t})$ &  & Learned latent transition model \\
                $\mG^a \in \{0,1\}^{d_z \times d_a}$ & & Ground-truth adjacency matrix of graph connecting $\va^{<t}$ to $\vz^t$\\
                $\mG^z \in \{0,1\}^{d_z \times d_z}$ & & Ground-truth adjacency matrix of graph connecting $\vz^{<t}$ to $\vz^t$\\
                $\hat\mG^a, \hat\mG^z$ & & Learned adjacency matrices \\
                ${\bf Pa}_i^a \subseteq [d_a]$ & & Parents of $\vz^{t}_i$ in $\mG^a$ \\
                ${\bf Ch}_\ell^a \subseteq [d_z]$ & & Children of $\va^t_\ell$ in $\mG^z$ \\
                ${\bf Pa}_i^z \subseteq [d_z]$ & & Parents of $\vz^t_i$ in $\mG^z$ \\
                ${\bf Ch}_i^z \subseteq [d_z]$ & & Children of $\vz^{t-1}_i$ in $\mG^z$ \\
                $D^t_z\log p \in \sR^{1 \times d_z}$ & & Jacobian vector of $\log p(\vz^t \mid \vz^{<t}, \va^{<t})$ w.r.t. $\vz^t$\\
                $H^{t, \tau}_{z,a}\log p \in \sR^{d_z \times d_a}$ & & Hessian matrix of $\log p(\vz^t \mid \vz^{<t}, \va^{<t})$ w.r.t. $\vz^t$ and $\va^\tau$\\
                $H^{t, \tau}_{z,z}\log p \in \sR^{d_z \times d_z}$ & & Hessian matrix of $\log p(\vz^t \mid \vz^{<t}, \va^{<t})$ w.r.t. $\vz^t$ and $\vz^\tau$\\
                $\sigma: [d_z] \rightarrow [d_z]$ & & A permutation \\
                \multicolumn{3}{c}{}\\
                \multicolumn{3}{c}{\underline{Topology}}\\
                $\overline{\gX}$ & & Closure of the set $\gX \subset \sR^n$ \\
                $\gX^\circ$ & & Interior of the set $\gX \subset \sR^n$ \\
                \bottomrule
            \end{tabular}
        \end{center}
        \label{tab:TableOfNotation}
    }
\end{table}

\newpage

\section{Identifiability Theory - Nonparametric Case}

\subsection{Useful Lemmas}\label{app:lin_indep_func}

\begin{definition}[Regular closed set] A set $A \subseteq \sR^n$ is regular closed when it is equal to the closure of its interior, i.e. $\overline{A^\circ} = A$. %Let $\gX \subseteq \sR^n$ and $f:\sR^n \rightarrow \sR$ and $\tilde f:\sR^n \rightarrow \sR$ be two $C^2$ functions.
\end{definition}

\begin{lemma}\label{lem:equal_derivatives}
    Let $A \subset \sR^n$ and $\vf:A \rightarrow \sR^m$ be a $C^k$ function. Then, its $k$ first derivatives is uniquely defined on $\overline{A^\circ}$ in the sense that they do not depend on the specific choice of $C^k$ extension.
\end{lemma}
\begin{proof}
    Let $\vg:U \rightarrow \sR^n$ and $\vh:V \rightarrow \sR^n$ be two $C^k$ extensions of $\vf$ to $U \subset \sR^n$ and $V \subset\sR^n$ both open in $\sR^n$. By definition,
    \begin{align*}
        \vg(\vx) = \vf(\vx) = \vh(\vx),\ \forall \vx\in A\,. 
    \end{align*}
    The usual derivative is uniquely defined on the interior of the domain, so that
    \begin{align*}
        D\vg(\vx) = D\vf(\vx) = D\vh(\vx),\ \forall \vx\in A^\circ\,.
    \end{align*}
    Consider a point $\vx_0 \in \overline{A^\circ}$. By definition of closure, there exists a sequence $\{\vx_k\}_{k=1}^\infty \subset A^\circ$ s.t. $\lim_{k \rightarrow \infty}\vx_k = \vx_0$. We thus have that
    \begin{align*}
        \lim_{k\to\infty}D\vg(\vx_k) &= \lim_{k\to\infty}D\vh(\vx_k) \\
        D\vg(\vx_0) &= D\vh(\vx_0) \,,
    \end{align*}
    where we used the fact that the derivatives of $\vg$ and $\vh$ are continuous to go to the second line. Thus, all the $C^k$ extensions of $\vf$ must have equal derivatives on $\overline{A^\circ}$. This means we can unambiguously define the derivative of $\vf$ everywhere on $\overline{A^\circ}$ to be equal to the derivative of one of its $C^k$ extensions.

    Since $\vf$ is $C^k$, its derivative $D\vf$ is $C^{k-1}$, we can thus apply the same argument to get that the second derivative of $\vf$ is uniquely defined on $\overline{\overline{A^\circ}^\circ}$. It can be shown that $\overline{\overline{A^\circ}^\circ} = \overline{A^\circ}$. One can thus apply the same argument recursively to show that the first $k$ derivatives of $\vf$ are uniquely defined on $\overline{A^\circ}$.
\end{proof}

\begin{comment}
\begin{lemma}\label{lem:equal_derivatives}
    Let $\gX \subseteq \sR^n$ be regular closed and $f:\sR^n \rightarrow \sR$ and $\tilde f:\sR^n \rightarrow \sR$ be two $C^2$ functions. If, for all $\vx \in \gX$, $\vf(\vx) = \tilde \vf(\vx)$, then, for all $\vx \in \gX$, $Df(\vx) = D\tilde{f}(\vx)$ and $D^2 f(\vx)$ = $D^2\tilde f(\vx)$.
\end{lemma}
\begin{proof}
    The derivative is defined only on the interior of a domain, hence
    \begin{align}
        \forall \vx\in \gX^\circ, Df(\vx) &= D\tilde f(\vx)
    \end{align}

    Choose $\vx_0 \in \gX$. Because $\gX = \overline{\gX^\circ}$ (regular closed), we have that $\vx_0$ is in the closure of $\gX^\circ$, hence there exists a sequence $\{\vx_k\}_{k=1}^{\infty} \subseteq \gX^{\circ}$ such that $\lim_{k\to\infty} \vx_k = \vx_0$.  Of course we have
    \begin{align}
        \lim_{k\to\infty}Df(\vx_k) &= \lim_{k\to\infty}D\tilde f(\vx_k) \\
        Df(\vx_0) &= D\tilde f(\vx_0) \,,
    \end{align}
    where the last step holds because the derivative itself is continuous. Hence the derivatives are equal on $\gX$. Because $f$ and $\tilde f$ are $C^2$, their derivatives are $C^1$. We can thus apply a similar argument to show that their second derivatives are equal on $\gX$.
\end{proof}
\end{comment}

\begin{lemma}\label{lem:lin_indep_charac}
    Let $X$ be some set. A family of functions $(f_i: X \rightarrow \sR)_{i=1}^n$ is linearly independent if and only if there exists $x_1, ..., x_n \in X$ such that the family of vectors $((f_1(x_i), ..., f_n(x_i)))_{i=1}^n$ is linearly independent.
\end{lemma}
\begin{proof}
    We start by proving the ``if'' part. Assume the functions are linearly dependent. Then there exists $\bm\alpha \not=0$ such that, for all $x\in X$, $\sum_{i=1}^n \bm\alpha_i f_i(x)=0$. Choose distinct $x_1, ..., x_n \in X$. We thus have that for all $j \in [n]$, $\sum_{i=1}^n \bm\alpha_i f_i(x_j)=0$. This can be  written in matrix form:
    \begin{align*}
        \begin{bmatrix}
            f_1(x_1) & \cdots & f_1(x_n) \\
            \vdots & \ddots & \vdots \\
            f_n(x_1) & \cdots & f_n(x_n) 
        \end{bmatrix}\bm\alpha =\bm0 \,,
    \end{align*}
    which implies that the columns are linearly dependent.

    We now show the ``only if'' part. Suppose that for all $\{x_1, \dots, x_n\} \subseteq X$, the family of vectors $((f_1(x_i), ..., f_n(x_i)))_{i=1}^n$ is linearly dependent. This means that the set $U = \vecspan\{(f_1(x), ..., f_n(x)) \mid x \in X\}$ is a proper linear subspace of $\sR^n$. This means that there is a nonzero $u \in U^\perp$, the orthogonal complement of $U$. By definition, $u$ is orthogonal to all elements in $\{(f_1(x), ..., f_n(x)) \mid x \in X\}$. In other words, for all $x \in X$, $\sum_{i=1}^n u_i f_i(x) = 0$. Hence the $\vf_i$ are linearly dependent.
\end{proof}

\subsection{Proof of Proposition~\ref{prop:V_and_Dv}}\label{sec:V_and_Dv}

\JacobianEntanglement*

\begin{proof}
    The ``$\implies$'' direction holds since we can simply differentiate $\vh_i(\va) = \bar\vh_i(\va_{-j})$ w.r.t. $\va_j$ to get zero. 
    
    We now show the ``$\impliedby$'' direction. Suppose that for all $\va \in \sR^n$, $D\vh(\va)_{i,j} = 0$. We must now show that $\vh_i(\va)$ is constant in $\va_j$ for all $\va_{-j}$. Choose any $\va^0, \va^1 \in \sR^n$ such that $\va^0_{-j} = \va^1_{-j}$. %Because $A$ is convex, $(1 - \alpha)\va^0 + \alpha\va^1 \in A$ for all $\alpha \in [0,1]$.
    Thanks to the fundamental theorem of calculus, we can write
    \begin{align*}
        \vh_i(\va^1) - \vh_i(\va^0) &= \int_{[0,1]} \frac{d}{d\alpha} \vh_i((1 - \alpha)\va^0 + \alpha\va^1)) d\alpha \\
        &= \int_{[0,1]} \underbrace{D\vh((1 - \alpha)\va^0 + \alpha\va^1)_{i, \cdot}}_\text{zero at $j$} \cdot \underbrace{(\va^1 - \va^0)}_\text{zero except at $j$} d\alpha \\
        &= 0 \,.
    \end{align*}
    Since $\va^0$ and $\va^1$ were arbitrary points such that $\va^0_{-j} = \va^1_{-j}$, this means the function $\vh_i(\va)$ is constant in $\va_j$ for all values of $\va_{-j}$.
\end{proof}

\subsection{Proof of Proposition~\ref{prop:identDiffeo}}\label{app:identDiffeo}

In this section, we prove Proposition~\ref{prop:identDiffeo}. Before doing so, we first recall the definition of the support of a random variable (Definition~\ref{def:support}) and prove a useful lemma (Lemma~\ref{lemma:support_of_Y}).

\begin{definition}(Support of a random variable)\label{def:support}
Let $\vx$ be a random variable with values in $\mathbb{R}^n$ with distribution $\sP_\vx$. Let $\mathcal{O}_n$ be the standard topology of $\mathbb{R}^n$ (i.e. the set of open sets of $\mathbb{R}^n$). The support of $\vx$ is defined as
\begin{align*}
    \textnormal{supp($\vx$)} := \{\vx \in \mathbb{R}^n \mid \vx \in O \in \mathcal{O}_n \implies \sP_\vx(O) > 0\} \, .
\end{align*}
\end{definition}

\begin{comment}
\textcolor{red}{I believe I'm implicitly assuming throughout that all random variables considered have a range equal to their support, i.e. if $\vx: \Omega \rightarrow \sR^m$ is a random variable, we assume that $\vx(\Omega) = \supp(\vx)$. This is not always the case, take for example $\Omega = \sR$ with probability measure $P$ equal to the Lebesgue measure and $\vx(\omega) := \mathbbm{1}(\omega = 0)$. }

\textcolor{red}{Starting to think this theorem is false... Instead of defining maps over support, maybe it's better to define them over the ranges of the random variables, which I believe will simplify a few things...}

\textcolor{red}{Here's what I need: Let $\vz:\Omega \rightarrow \sR^{d_z}$ and $\hat\vz:\Omega \rightarrow \sR^{d_z}$ be random vectors with distribution $\sP_\vz$ and $\sP_{\hat\vz}$. Let $\vx := \vf(\vz)$ and $ \hat\vx := \hat\vf(\hat\vz)$ where $\vf: \vz(\Omega) \rightarrow \sR^{d_x}$ and $\hat\vf: \hat\vz(\Omega) \rightarrow \sR^{d_x}$ are diffeomorphisms onto their respective images. Find minimal assumption such that $\sP_{\vz}\vf^{-1} = \sP_{\hat\vz}\hat\vf^{-1} \implies \vf(\vz(\Omega)) = \hat\vf(\hat\vz(\Omega))$.}
\end{comment}

\begin{lemma}\label{lemma:support_of_Y}
Let $\vz$ be a random variable with values in $\sR^m$ with distribution $\sP_\vz$ and $\vy:= \vf(\vz)$ where ${\vf:\textnormal{supp}(\vz) \rightarrow \sR^n}$ is a homeomorphism onto its image. Then
\begin{align*}
    \vf(\textnormal{supp}(\vz)) \subseteq \textnormal{supp}(\vy) \subseteq \overline{\vf(\textnormal{supp}(\vz))} \, .
\end{align*}
where the closure is taken w.r.t. the topology of $\sR^n$.
\end{lemma}

\begin{proof}
We first prove that $f(\text{supp}(\vz)) \subseteq \text{supp}(\vy)$. Let $\vy^0 \in \vf(\text{supp}(\vz))$ and $N$ be an open neighborhood of $\vy^0$, i.e. $\vy^0 \in N \in \gO_n $. Note that there exists $\vz^0 \in \text{supp}(\vz)$ such that $\vf(\vz^0) = \vy^0$. Note that ${\vz^0 \in \vf^{-1}(\{\vy^0\}) \subseteq \vf^{-1}(N)}$ and that, by continuity of $\vf$, $\vf^{-1}(N)$ is an open neighborhood of $\vz^0$. Since $\vz^0 \in \text{supp}(\vz)$, we have 
\begin{align*}
    0 &< \sP_\vz(f^{-1}(N))\\
      &= \sP_\vz \circ f^{-1}(N) \\
      &= \sP_\vy(N)\,.
\end{align*}
Hence $\vy^0 \in \text{supp}(\vy)$, which concludes the ``$\subseteq$'' part. 

We now prove the other inclusion. Let $\vy^0 \in \supp(\vy)$ and suppose, by contradiction, that $\vy^0 \not\in \overline{\vf(\supp(\vz))}$. Since $\overline{\vf(\supp(\vz))}$ is closed in $\sR^n$, there exists $N$ s.t. $\vy^0 \in N \in \gO_n$ with $N \cap \overline{\vf(\supp(\vz))} = \emptyset$. Since $\vy^0 \in \supp(\vy)$, 
\begin{align*}
    0 &< \sP_\vy(N) \\
      &= \sP_\vz (\vf^{-1}(N)) \\
      &= \sP_\vz(\emptyset) = 0 \, .
\end{align*}
The above contradiction implies that $\vy^0 \in \overline{\vf(\supp(\vz))}$.
% We can thus apply the same argument as above to show ${\vf^{-1}(\text{supp}(\vy)) \subseteq \text{supp}}(\vz)$, which implies that ${\text{supp}(\vy) \subseteq \vf(\text{supp}(\vz))}$.
\end{proof}

\identDiffeo*

\begin{proof}

\textbf{Equality of Denoised Distributions.}
Given an arbitrary $\va^{< T} \in \gA^T$ and a parameter $\vtheta = (\bff, p, \mG)$, let $\sP_{\vx^{\leq T} \mid \va^{<T}; \vtheta}$ be the conditional probability distribution of $\vx^{\leq T}$, let $\sP_{\vz^{\leq T} \mid \va^{<T}; \vtheta}$ be the conditional probability distribution of $\vz^{\leq T}$ and let $\sP_{\vn^{\leq T}}$ be the probability distribution of $\vn^{\leq T}$ (the Gaussian noises added on $\bff(\vz^{\leq T})$, defined in Sec.~\ref{sec:model}). Let ${\vy^t := \bff(\vz^t)}$ and $\sP_{\vy^{\leq T} \mid \va^{<T}; \vtheta}$ be its conditional probability distribution. First, notice that
\begin{align*}
    \mathbb{P}_{\vx^{\leq T} \mid \va^{<T}; \vtheta} = \mathbb{P}_{\vy^{\leq T} \mid \va^{<T}; \vtheta} * \sP_{\vn^{\leq T}} \, ,
\end{align*}
where $*$ is the convolution operator between two measures. We now show that if two models agree on the observations, i.e. $\mathbb{P}_{\vx^{\leq T} \mid \va^{<T}; \vtheta} = \mathbb{P}_{\vx^{\leq T} \mid \va^{<T};\hat{\vtheta}}$, then $\mathbb{P}_{\vy^{\leq T} \mid \va^{<T}; \vtheta} = \mathbb{P}_{\vy^{\leq T} \mid \va^{<T}; \hat{\vtheta}}$. The following argument makes use of the Fourier transform $\mathcal{F}$ generalized to arbitrary probability measures. This tool is necessary to deal with measures which do not have a density w.r.t either the Lebesgue or the counting measure, as is the case of $\mathbb{P}_{\vy^{\leq T} \mid \va^{<T}; \vtheta}$ (all its mass is concentrated on the set $\bff(\sR^{d_z})$). See \citet[Chapter 8]{pollard_2001} for an introduction and useful properties. 
\begin{align}
    \mathbb{P}_{\vx^{\leq T} \mid \va^{<T}; \vtheta} &= \mathbb{P}_{\vx^{\leq T} \mid \va^{<T};\hat{\vtheta}} \label{eq:first_denoise}\\
    \mathbb{P}_{\vy^{\leq T} \mid \va^{<T}; \vtheta} * \sP_{\vn^{\leq T}} &= \mathbb{P}_{\vy^{\leq T} \mid \va^{<T};\hat{\vtheta}} * \sP_{\vn^{\leq T}} \\
    \mathcal{F}(\mathbb{P}_{\vy^{\leq T} \mid \va^{<T}; \vtheta} * \sP_{\vn^{\leq T}}) &= \mathcal{F}(\mathbb{P}_{\vy^{\leq T} \mid \va^{<T};\hat{\vtheta}} * \sP_{\vn^{\leq T}}) \label{eq:inversion_1} \\
    \mathcal{F}(\mathbb{P}_{\vy^{\leq T} \mid \va^{<T}; \vtheta})  \mathcal{F}(\sP_{\vn^{\leq T}}) &= \mathcal{F}(\mathbb{P}_{\vy^{\leq T} \mid \va^{<T};\hat{\vtheta}}) \mathcal{F}(\sP_{\vn^{\leq T}}) \label{eq:convolution_product}\\
    \mathcal{F}(\mathbb{P}_{\vy^{\leq T} \mid \va^{<T}; \vtheta}) &= \mathcal{F}(\mathbb{P}_{\vy^{\leq T} \mid \va^{<T};\hat{\vtheta}}) \label{eq:non_zero_transform}\\
    \mathbb{P}_{\vy^{\leq T} \mid \va^{<T}; \vtheta} &= \mathbb{P}_{\vy^{\leq T} \mid \va^{<T};\hat{\vtheta}} \label{eq:inversion_2}\, ,
\end{align}
where~\eqref{eq:inversion_1}~\&~\eqref{eq:inversion_2} use the fact that the Fourier transform is invertible, \eqref{eq:convolution_product} is an application of the fact that the Fourier transform of a convolution is the product of their Fourier transforms and~\eqref{eq:non_zero_transform} holds because the Fourier transform of a Normal distribution is nonzero everywhere. Note that the latter argument holds because we assume $\sigma^2$, the variance of the Gaussian noise added to $\vy^t$, is the same for both models. %For an argument that takes into account the fact that $\sigma^2$ is learned and assumes $d_z < d_x$, see Appendix~\ref{sec:learned_var}.  
Notice that, since $\vf(\sR^{d_z})$ and $\hat\vf(\sR^{d_z})$ are closed in $\sR^{d_x}$ and the support of $\vz^{\leq T}$ is $\sR^{d_z \times T}$, Lemma~\ref{lemma:support_of_Y} implies that 
\begin{align*}
    \bff(\sR^{d_z \times T}) = \text{supp}(\mathbb{P}_{\vy^{\leq T} \mid \va^{<T}; \vtheta}) \quad\quad \& \quad\quad \text{supp}(\mathbb{P}_{\vy^{\leq T} \mid \va^{<T}; \hat{\vtheta}}) = \hat{\bff}(\sR^{d_z \times T})
\end{align*}
where we overloaded the notation by defining $\bff(\vz^{\leq T}) := ( \bff(\vz^1), ..., \bff(\vz^T))$ and analogously for $\hat{\bff}(\vz^{\leq T})$. Since both measure in \eqref{eq:inversion_2} are equal, their supports must also be. This implies that $\vf(\sR^{d_z}) = \hat\vf(\sR^{d_z})$, which is part of the definition of equivalence up to diffeomorphism (Definition~\ref{def:eqdiff}). 

\textbf{Equality of densities.}
Continuing with~\eqref{eq:inversion_2},
\begin{align}
\sP_{\vy^{\leq T} \mid \va^{<T} ; \vtheta} &= \sP_{\vy^{\leq T} \mid \va^{<T} ;\hat{\vtheta}} \nonumber\\
\sP_{\vz^{\leq T} \mid \va^{<T} ; \vtheta} \circ \bff^{-1} &= \sP_{\vz^{\leq T} \mid \va^{<T} ; \hat{\vtheta}} \circ \hat{\bff}^{-1} \nonumber\\
\sP_{\vz^{\leq T} \mid \va^{<T} ; \vtheta} \circ {\bff}^{-1} \circ \hat\bff &= \sP_{\vz^{\leq T} \mid \va^{<T} ; \hat{\vtheta}} \nonumber\\
\sP_{\vz^{\leq T} \mid \va^{<T} ; \vtheta} \circ \vv &= \sP_{\vz^{\leq T} \mid \va^{<T} ; \hat{\vtheta}} \, ,
\label{eq:measure_equality}
\end{align}
where $\bfv := {\bff}^{-1} \circ \hat\bff$ is a composition of diffeomorphisms and thus a diffeomorphism from $\sR^{d_z}$ to itself. Note that this composition is well defined because $\bff{(\sR^{d_z})}=\hat{\bff}(\sR^{d_z})$. 
We chose to work directly with measures (functions on sets), as opposed to manifold integrals in \citet{iVAEkhemakhem20a}, because it simplifies the derivation of~\eqref{eq:measure_equality} and avoids having to define densities w.r.t. measures concentrated on a manifold.

\begin{comment}
Let $E \subset \gZ^{T}$ be an event, we then have
\begin{align}
    &\sP_{\vz^{\leq T} \mid \va^{<T} ; \hat{\vtheta}} \circ \bfv (E) \nonumber\\
    &= \int_{\bfv (E)} d\sP_{\vz^{\leq T} \mid \va^{<T} ; \hat{\vtheta}}\\
    &= \int_{\bfv (E)} \prod_{t=1}^T \hat{p}(\vz^t \mid \vz^{< t}, \va^{< t}) dm(\vz^{\leq T}) \\
    &= \int_{E} \prod_{t=1}^T \left[\hat{p}(\bfv(\vz^t) \mid \bfv(\vz^{<t}), \va^{<t}) \right] |\det D\bfv(\vz^{\leq T}) | dm(\vz^{\leq T}) \\
    &= \int_{E} \prod_{t=1}^T \left[\hat{p}(\bfv(\vz^t) \mid \bfv(\vz^{<t}), \va^{<t}) |\det D\bfv(\vz^{t}) |\right] dm(\vz^{\leq T}) \,,
\end{align}
where $D\bfv$ is the Jacobian matrix of $\bfv$ (which is differentiable because both $\bff$ and $\tilde{\bff}$ are diffeomorphisms) and $\hat{p}$ refers to the conditional density of the model with parameter $\hat{\vtheta}$. Since $\sP_{\vz^{\leq T} \mid \va^{<T} ; \vtheta}$ and $\sP_{\vz^{\leq T} \mid \va^{<T} ; \hat{\vtheta}} \circ \bfv$ are equal, they must have the same density:
\begin{align}
    \prod_{t=1}^T p(\vz^t \mid \vz^{<t}, \va^{<t}) = \prod_{t=1}^T \hat{p}(\bfv(\vz^t) \mid \bfv(\vz^{<t}), \va^{<t}) |\det D\bfv(\vz^{t}) | \, ,
\end{align}
where $p$ refers to the conditional density of the model with parameter $\vtheta$. 
\end{comment}

The density of $\sP_{\vz^{\leq T} \mid \va^{<T} ; {\vtheta}} \circ \bfv$ w.r.t. to the Lebesgue measure is given by the change-of-variable rule for random vectors (which can be applied because $\bfv$ is a diffeomorphism) and is given by $\prod_{t=1}^T {p}(\bfv(\vz^t) \mid \bfv(\vz^{<t}), \va^{<t}) |\det D\bfv(\vz^{t})|$, where $p$ refers to the density model with parameter $\vtheta$ and $D\bfv(\vz^{t})$ is the Jacobian matrix of $\bfv$. Since $\sP_{\vz^{\leq T} \mid \va^{<T} ; \vtheta} \circ \bfv = \sP_{\vz^{\leq T} \mid \va^{<T} ; \hat{\vtheta}}$, their respective densities w.r.t. Lebesgue must also agree:
\begin{align*}
    \prod_{t=1}^T \hat p(\vz^t \mid \vz^{<t}, \va^{<t}) = \prod_{t=1}^T {p}(\bfv(\vz^t) \mid \bfv(\vz^{<t}), \va^{<t}) |\det D\bfv(\vz^{t}) | \, ,
\end{align*}
where $\hat p$ refers to the conditional density of the model with parameter $\hat\vtheta$.

For a given $t_0$, we have
\begin{align}
    \prod_{t=1}^{t_0}\hat p(\vz^t \mid \vz^{<t}, \va^{<t}) = \prod_{t=1}^{t_0}{p}(\bfv(\vz^t) \mid \bfv(\vz^{<t}), \va^{<t}) |\det D\bfv(\vz^{t}) | \, ,  \label{eq:t0}
\end{align}
by integrating first $\vz^{T}$, then $\vz^{t\shortminus 1}$, then ..., up to $\vz^{t_0 + 1}$. Note that we can integrate $\vz^{t_0}$ and get
\begin{align}
     \prod_{t=1}^{t_0 - 1}\hat p(\vz^t \mid \vz^{<t}, \va^{<t}) =\prod_{t=1}^{t_0 - 1}{p}(\bfv(\vz^t) \mid \bfv(\vz^{<t}), \va^{<t}) |\det D\bfv(\vz^{t}) | \, . \label{eq:t0-1}
\end{align}
By dividing~\eqref{eq:t0} by~\eqref{eq:t0-1}, we get
\begin{align}
    \hat p(\vz^{t_0} \mid &\vz^{<{t_0}}, \va^{<{t_0}}) = {p}(\bfv(\vz^{t_0}) \mid \bfv(\vz^{<{t_0}}), \va^{<{t_0}}) |\det D\bfv(\vz^{{t_0}}) |  \, , \label{eq:densities_link}
\end{align}
which completes the proof.
\end{proof}

\subsection{The Consistency Relations (Definitions~\ref{def:a_consistent_models}~\&~\ref{def:z_consistent_models}) are Equivalence Relations} \label{sec:proof_consistence_equivalence}

In this section, we demonstrate that the relations $\eqcon^\va$ and $\eqcon^\vz$ are equivalence relations by leveraging the fact that the set of $\mG$-preserving diffeomorphisms form a group under composition (Proposition~\ref{prop:group_of_S_consistent_diffeo}). We start by showing a fact that will be useful in the following.
\begin{lemma}\label{lemma:pre_eq}
Let $\mG \in \{0,1\}^{m \times n}$.
\begin{enumerate}
    \item A map $\vc: \sR^m \rightarrow \sR^m$ is $\mG$-preserving if and only if $\vc$ is $\mG\mP$-preserving, where $\mP$ is an $n \times n$ permutation matrix.
    \item A map $\vc: \sR^m \rightarrow \sR^m$ is $\mG$-preserving if and only if $\mP\circ\vc \circ \mP^\top$ is $\mP\mG$-preserving, where $\mP$ is a $m \times m$ permutation matrix.
    \item When m = n, a map $\vc: \sR^m \rightarrow \sR^m$ is $\mG$-preserving if and only if $\mP\circ\vc\circ\mP^\top$ is $\mP\mG\mP^\top$-preserving, where $\mP$ is a $m \times m$ permutation matrix.
\end{enumerate}
\end{lemma}
\begin{proof}
Let $\mC$ be the dependency graph of $\vc$.
    \begin{enumerate}
        \item $\mC^\top\sR^{m\times n}_\mG \subset \sR^{m\times n}_\mG \iff \mC^\top\sR^{m\times n}_\mG\mP \subset \sR^{m\times n}_\mG\mP \iff \mC^\top\sR^{m\times n}_{\mG\mP} \subset \sR^{m\times n}_{\mG\mP}$
        \item First, notice that the dependency graph of $\mP\circ\vc \circ \mP^\top$ is $\mP\mC\mP^\top$. 
        $$(\mP\mC\mP^\top)^\top\sR^{m\times n}_{\mP\mG} \subset \sR^{m\times n}_{\mP\mG} \iff  \mP\mC^\top\mP^\top\mP\sR^{m\times n}_{\mG} \subset \mP\sR^{m\times n}_{\mG} \iff \mC^\top\sR^{m\times n}_\mG \subset \sR^{m\times n}_\mG$$
        \item This is a consequence of the first two statements.
    \end{enumerate}
\end{proof}
\begin{comment}
\begin{proof}
To show the first statement, we simply have to notice that
\begin{align}
    \mS_{i,\cdot} \subseteq \mS_{j,\cdot} \iff \mS_{i,\cdot}\mP \subseteq \mS_{j,\cdot}\mP 
\end{align}

To show the second statement, we use Lemma~\ref{lemma:charac_S-consistent}. Let $\mC$ be the dependency graph of $\vc$. By Lemma~\ref{lemma:charac_S-consistent}, we know that $\vc$ being $\mS$-consistent is equivalent to having: we start with
\begin{align}
    \mC &\in \sR^{m\times m}_{[\mathbbm{1} - \mS(\mathbbm{1} - \mS)^\top]^+} \\
    \iff\ \mP\mC\mP^\top &\in \mP\sR^{m\times m}_{[\mathbbm{1} - \mS(\mathbbm{1} - \mS)^\top]^+}\mP^\top \\
    &= \sR^{m\times m}_{\mP[\mathbbm{1} - \mS(\mathbbm{1} - \mS)^\top]^+\mP^\top}\\
    &= \sR^{m\times m}_{[\mathbbm{1} - \mP\mS(\mathbbm{1} - \mP\mS)^\top]^+} \, ,
\end{align}
which is itself equivalent to $\mP \circ \vc \circ \mP^\top$ being $\mP\mS$-consistent since $\mP\mC\mP^\top$ is the dependency graph of $\mP \circ \vc \circ \mP^\top$. 

The third statement, is a combination of the first two.
\end{proof}
\end{comment}

We are now ready to show that the relation $\eqcon^\va$ (Definition~\ref{def:a_consistent_models}) is an equivalence relation.
\begin{proposition}\label{prop:a_consistence_equivalence}
The consistency relation, $\eqcon^\va$ (Def.~\ref{def:a_consistent_models}), is an equivalence relation.
\end{proposition}
\begin{proof}

\textbf{Reflexivity.} It is easy to see that $\vtheta \eqcon^\va \vtheta$, simply by setting $\vv(\vz) := \vz$ with $\mP := \mI$.

\textbf{Symmetry.} Assume $\vtheta \eqcon^\va \tilde \vtheta$. Hence, we have $\mP\mG^a = \tilde \mG^a$ as well as
\begin{align}
    \vf(\sR^{d_z}) &= \tilde\vf(\sR^{d_z}) , \text{and}\\
    \tilde p(\vz^{t} \mid \vz^{<{t}}, \va^{<{t}}) &= p(\bfv(\vz^{t}) \mid \bfv(\vz^{<{t}}), \va^{<{t}}) |\det D\bfv(\vz^{{t}}) |\, , \label{eq:wt3345478}
\end{align}
where $\vv := \vf^{-1} \circ \tilde\vf$ can be written as $\vv := \vc \circ \mP^\top$, where $\vc$ is a $\mG^a$-preserving diffeomorphism and $\mP$ is a permutation. We can massage \eqref{eq:wt3345478} to get
\begin{align*}
    p(\vz^{t} \mid \vz^{<{t}}, \va^{<{t}}) = \tilde p(\vv^{-1}(\vz^t) \mid \vv^{-1}(\vz^{<t}), \va^{<t})|\det D \bfv^{-1} (\vz^t)| \, .
\end{align*}
Of course, we also have that $\tilde\mP\tilde\mG^a = \mG^a$, where $\tilde\mP := \mP^\top$. Now the only thing left to prove is that $\vv^{-1}$ can be written as $\tilde\vc\circ \tilde\mP^\top$ where $\tilde\vc$ is $\tilde\mG^a$-preserving. We know that 
\begin{align*}
    \vv^{-1} = \mP \circ \vc^{-1} = \underbrace{\mP \circ \vc^{-1} \circ \mP^\top}_{\tilde{\vc}:=} \circ \mP = \tilde\vc \circ \tilde\mP^\top \,.
\end{align*}
Note that $\vc^{-1}$ is $\mG^a$-preserving and thus, by Lemma~\ref{lemma:pre_eq}, $\tilde\vc$ is $\mP\mG^a$-preserving, i.e. $\tilde\mG^a$-preserving.
Hence, $\eqcon^\va$ is symmetric.

\textbf{Transitivity.} Suppose $\vtheta \eqcon^\va \tilde \vtheta$ and $\tilde\vtheta \eqcon^\va \hat\vtheta$. This means
\begin{align}
    \mP_1\mG^a &= \tilde \mG^a\,, \label{eq:G_rel_1} \\
    \vf(\sR^{d_z}) &= \tilde\vf(\sR^{d_z}),  \text{and}\\
    \tilde p(\vz^{t} \mid \vz^{<{t}}, \va^{<{t}}) &= p(\bfv_1(\vz^{t}) \mid \bfv_1(\vz^{<{t}}), \va^{<{t}}) |\det D\bfv_1(\vz^{{t}}) |\, , \label{eq:26548362}
\end{align}
where $\vv_1 := \vc_1 \circ \mP^\top_1$ with $\vc_1$ being $\mG^a$-preserving; and
\begin{align}
    \mP_2\tilde \mG^a &= \hat \mG^a\, \label{eq:G_rel_2},\\
    \tilde\vf(\sR^{d_z}) &= \hat\vf(\sR^{d_z}), \text{and}\\
    \hat p(\vz^{t} \mid \vz^{<{t}}, \va^{<{t}}) &= \tilde p(\bfv_2(\vz^{t}) \mid \bfv_2(\vz^{<{t}}), \va^{<{t}}) |\det D\bfv_2(\vz^{{t}}) |\, , \label{eq:43623534}
\end{align}
where $\vv_2 := \vc_2 \circ \mP^\top_2$ with $\vc_2$ being $\tilde\mG^a$-preserving.

To show that $\vtheta \eqcon^\va \hat\vtheta$, we first combine~\eqref{eq:G_rel_1} with~\eqref{eq:G_rel_2} to get
\begin{align}
    \underbrace{\mP_2\mP_1}_{\mP :=}\mG^a = \hat \mG^a \, . \label{eq:G_hatG}
\end{align}
Of course we also have that $\vf(\sR^{d_z}) = \tilde\vf(\sR^{d_z}) = \hat\vf(\sR^{d_z})$. By massaging both \eqref{eq:26548362} and \eqref{eq:43623534}, we get:
\begin{align*}
    \hat p(\vz^{t} \mid \vz^{<{t}}, \va^{<{t}}) &= p(\vv_1\circ\bfv_2(\vz^{t}) \mid \vv_1\circ\bfv_2(\vz^{<{t}}), \va^{<{t}}) |\det D(\vv_1\circ\bfv_2)(\vz^{{t}}) | \, .
\end{align*}
Define $\vv := \vv_1\circ\bfv_2$. We now want to show that $\vv$ can be written as $\vv = \vc \circ \mP^\top$ where $\vc$ is $\mG^a$-preserving. We have that
\begin{align*}
    \vv_1 \circ \vv_2 &= \vc_1 \circ \mP_1^\top \circ \vc_2 \circ \mP_2^\top \\
    &= \vc_1 \circ \underbrace{\mP_1^\top \circ \mP_2^\top}_{\mP^\top = } \circ \underbrace{\mP_2 \circ \vc_2 \circ \mP_2^\top}_{\hat\vc := } \\
    &= \vc_1 \circ \mP^\top \circ \hat\vc
\end{align*}
where, by Lemma~\ref{lemma:pre_eq}, $\hat\vc$ is $\mP_2\tilde\mG^a$-preserving, i.e. $\hat\mG^a$-preserving. We continue and get
\begin{align*}
    \vv_1 \circ \vv_2 &= \vc_1 \circ \mP^\top \circ \hat\vc \\
    &= \vc_1 \circ \underbrace{\mP^\top \circ \hat\vc \circ \mP}_{\vc' :=} \circ \mP^\top \\
    &= \vc_1 \circ \vc' \circ \mP^\top \, ,
\end{align*}
where, by Lemma~\ref{lemma:pre_eq}, $\vc'$ is $\mP^\top\hat\mG^a$-preserving, i.e. $\mG^a$-preserving (by \eqref{eq:G_hatG}). Since both $\vc_1$ and $\vc'$ are $\mG^a$-preserving, $\vc := \vc_1 \circ \vc'$ is $\mG^a$-preserving, which concludes the proof.
\begin{comment}
If $\vc$ happens to have an incomplete diagonal (Definition~\ref{def:complete_diag}), we can use the same argument as above to find a $\mG^a$-consistent $\tilde\vc_0$ with complete diagonal and a permutation $\mP_0$ such that $\vc_0 \circ \mP^\top_0 = \vc \circ \mP^\top$ and $\mP_0\mG^a = \hat\mG^a$.
\end{comment}
\end{proof}

The same can be shown for $\eqcon^\vz$ (Definition~\ref{def:z_consistent_models}).

\begin{proposition}\label{prop:z_consistence_equivalence}
The consistency relation, $\eqcon^\vz$ (Def.~\ref{def:z_consistent_models}), is an equivalence relation.
\end{proposition}
\begin{proof}
    The proof is exactly analogous to the proof that $\eqcon^\va$ is an equivalence relation. Essentially, every statement of the form ``$\mP\mG^a = \tilde\mG^a$'' becomes ``$\mP\mG^z\mP^\top = \tilde\mG^z$'' and statements of the form ``$\vc$ is $\mG^a$-preserving'' becomes ``$\vc$ is $\mG^z$-preserving and $(\mG^z)^\top$-preserving''. The full proof is left as an exercise to the reader.
\end{proof}

\subsubsection{Combining Equivalence Relations}\label{app:combining_eq_rel}

\CombineEquivalences*
\begin{proof}
    The ``only if'' part of the statement is trivial. We now show the ``if" part.
    
    Let $\vv := \vf^{-1} \circ \tilde\vf$. Since $\vtheta \eqcon^\va \tilde\vtheta$, we have that $\tilde\mG^a = \mP\mG^a$ and $\vv = \vc \circ \mP^\top$ where $\mP$ a permutation matrix and $\vc$ is $\mG^a$-preserving. Since $\vtheta \eqcon^\vz \tilde\vtheta$, we have that $\tilde\mG^z = \bar\mP\mG^z\bar\mP^\top$ and $\vv = \bar\vc \circ \bar\mP^\top$ where $\bar\mP$ is a permutation matrix and $\bar\vc$ is $\mG^z$-preserving and $(\mG^z)^\top$-preserving. Let $\mC$ and $\bar\mC$ be the dependency graphs $\vc$ and $\bar\vc$, respectively.

    Choose an arbitrary $\vz$. Since $D\vc(\vz)$ is invertible, Lemma~\ref{lemma:L_perm} implies that there exists a permutation $\mP_0$ such that $\mP_0^\top \subseteq D\vv(\vz)$, which in turn implies that $\mP_0^\top \subseteq \bar\mC$. Because $\mP_0^\top \subseteq \bar\mC$, we have that $\mP^\top_0$ is $\mG^z$-preserving and $(\mG^z)^\top$-preserving (Proposition~\ref{prop:charac_G_preserving}). By closure under composition and inversion, $\bar\vc\circ\mP_0$ is $\mG^z$- and $(\mG^z)^\top$-preserving. 

    Note that 
    \begin{align*}
        \vc \circ \mP^\top &= \bar\vc \circ \bar\mP^\top \\
        \implies \mC\mP^\top &= \bar\mC\bar\mP^\top \\
        \mC\mP^\top\bar\mP &= \bar\mC \supseteq \mP_0^\top \\
        \implies \mC &\supseteq \mP_0^\top\bar\mP^\top\mP\,.
    \end{align*}
    This means the permutation $\mP_0^\top\bar\mP^\top\mP$ must be $\mG^a$-preserving since $\mC$ is (Proposition~\ref{prop:charac_G_preserving}).
    
    This further implies that $\mC\mP^\top\bar\mP \mP_0 = \bar\mC\mP_0$ is $\mG^a$-preserving by closure under multiplication (recall $\mC$ is $\mG^a$-preserving too). Hence $\bar\vc \circ \mP_0$ is $\mG^a$-preserving.

    We thus have that $\vv = (\bar \vc \circ \mP_0)(\bar\mP\mP_0)^\top$ where $(\bar \vc \circ \mP_0)$ is $\mG^a$-, $\mG^z$- and $(\mG^z)^\top$-preserving. The only thing left to show is that $(\bar\mP\mP_0)\mG^a = \tilde\mG^a$ and that $(\bar\mP\mP_0)\mG^z(\bar\mP\mP_0)^\top = \tilde\mG^z$. The former holds since
    $$(\bar\mP\mP_0)\mG^a = \mP(\mP^\top\bar\mP\mP_0)\mG^a = \mP\mG^a = \tilde\mG^a \, ,$$
    where the second equality leverages the fact that $\mP^\top\bar\mP \mP_0$ is $\mG^a$-preserving. Furthermore,
    $$(\bar\mP\mP_0)\mG^z(\bar\mP\mP_0)^\top = \bar\mP\mG^z(\bar\mP\mP_0)^\top = \bar\mP\mG^z\mP_0^\top\bar\mP^\top = \bar\mP(\mP_0(\mG^z)^\top)^\top\bar\mP^\top = \bar\mP\mG^z\bar\mP^\top = \tilde\mG^z\, ,$$
    where the first and fourth equalities leveraged the fact that $\mP_0$ is $\mG^z$- and $(\mG^z)^\top$-preserving.
\end{proof}

\subsection{Technical Lemmas Leading to Theorems~\ref{thm:nonparam_dis_cont_a}, \ref{thm:nonparam_dis_disc_a} \& \ref{thm:nonparam_dis_z}}\label{app:technical_lemmas_ident}
The goal of this section is to introduce and prove Lemma~\ref{lem:same_permutation_almost} which was crucial in proofs of Theorems~\ref{thm:nonparam_dis_cont_a}, \ref{thm:nonparam_dis_disc_a} \& \ref{thm:nonparam_dis_z}. To prove it, we need a few more results, which we present next.

The following two lemmas are standard, but we provide them with proofs for completeness.
\begin{lemma}\label{lemma:extend_continuity_to_closure}
    Let $f: \sR^{n} \rightarrow \sR$ be continuous and $A \subseteq \sR^n$. If, for all $\vx \in A$, $\vf(\vx) = 0$, then the equality holds on $\overline{A}$.
\end{lemma}
\begin{proof}
    We have that $A \subseteq \vf^{-1}(\{0\})$. Since $\{0\}$ is closed, $\vf^{-1}(\{0\})$ is also closed by continuity of $f$. This means $\overline{A} \subseteq \vf^{-1}(\{0\})$ (since the closure of $\mA$ is the smallest closed set containing $\mA$).  
\end{proof}
\begin{lemma}\label{lemma:closure_almost}
    Let $\mu$ be the Lebesgue measure on $\sR^{d_z}$ and let $E_0 \subseteq \sR^{d_z}$ be a zero measure set, i.e. $\mu(E_0) = 0$. Then, $\overline{\sR^{d_z} \setminus E_0} = {\sR^{d_z}}$.
\end{lemma}
\begin{proof}
    Clearly, $\overline{\sR^{d_z} \setminus E_0} \subseteq \sR^{d_z}$.

    We now show that $\sR^{d_z} \subseteq \overline{\sR^{d_z} \setminus E_0}$. Take $\vz_0 \in \sR^{d_z}$ and let $U$ be an open set of $\sR^{d_z}$ containing $\vz_0$. Every nonempty open sets have nonzero Lebesgue measure, so 
    \begin{align*}
        0 \not= \mu(U) = \mu(U \cap \sR^{d_z}) = \mu(U \cap (\sR^{d_z} \setminus E_0)) \implies U \cap (\sR^{d_z} \setminus E_0) \not= \emptyset \,.
    \end{align*}
    Since $U$ was arbitrary, this means $\vz_0 \in \overline{\sR^{d_z} \setminus E_0}$.
\end{proof}

This simple lemma will come in handy when proving Lemma~\ref{lem:same_permutation}.
\begin{lemma}\label{lemma:zero_on_diagonal}
    If a permutation $\mP$ is not $\mG$-preserving and $\mC$ is a $\mG$-preserving matrix, we have that $\mC\mP^\top$ and $\mP^\top\mC$ have a zero on their diagonal. 
    \end{lemma}
    \begin{proof}
       Assume $\mP$ is not $\mG$-preserving, hence there exists $i$, $j$ such that $\mG_{i,\cdot} \not\subseteq \mG_{j,\cdot}$, but $\mP_{i,j} = 1$. Now note that
       \begin{align*}
           (\mC\mP^\top)_{i,i} = \mC_{i, \cdot}(\mP_{i\cdot})^\top = \mC_{i, \cdot}\ve_j = \mC_{i,j} \,,
       \end{align*}
       which is equal to zero because $\mC$ is $\mG$-preserving and $\mG_{i,\cdot} \not\subseteq \mG_{j,\cdot}$. Similarly,
       \begin{align*}
           (\mP^\top\mC)_{j,j} = (\mP_{\cdot, j})^\top\mC_{\cdot, j} = \ve_i^\top \mC_{\cdot, j} = \mC_{i,j} = 0 \, .
       \end{align*}
       which concludes the proof.
    \end{proof}

%\seb{Continue removing $\gZ$ here... Could some of the above be simplified?}
The following lemma is the same as Lemma~\ref{lem:same_permutation_almost} which is used to prove Theorems~\ref{thm:nonparam_dis_cont_a}, \ref{thm:nonparam_dis_disc_a} \& \ref{thm:nonparam_dis_z}, except it does not take into account the ``almost everywhere'' subtlety. Lemma~\ref{lem:same_permutation_almost} will extend it to deal with this difficulty.

\begin{lemma}\label{lem:same_permutation}
    Let $\mG \in \{0,1\}^{m \times n}$, let $\gZ$ be a connected subset of some topological space and let $\mL: \gZ \rightarrow \sR^{m\times m}$ be a continuous function such that $\mL(\vz)$ is invertible for all $\vz \in \gZ$. Suppose that, for all $\vz \in \gZ$, there exists a permutation matrix $\mP(\vz)$ such that $\mL(\vz)\mP(\vz)$ is $\mG$-preserving. Then, there exists a permutation matrix $\mP$ such that, for all $\vz \in \gZ$, $\mL(\vz)\mP$ is $\mG$-preserving.
\end{lemma}
\begin{proof}
    The goal of this lemma is to show that in the statement above, one can change the order of the ``for all $\vz \in \gZ$'' and ``there exists a permutation''. To do that, we show that if $\gZ$ is connected and the map $\mL(\cdot)$ is continuous, then one can find a single permutation that works for all $\vz \in \gZ$. 
    
    Let $\gG$ be the set of $\mG^a$-preserving matrices. Recall that, by Proposition~\ref{prop:charac_G_preserving}, $\gG$ corresponds to all matrices that have some set of entries equal to zero. 

    First, since $\gZ$ is connected and $\mL$ is continuous, its image, $\mL(\gZ)$, must be connected (by \cite[Theorem 23.5]{Munkres2000Topology}). 

    Second, from the hypothesis of the lemma, we know that 
    \begin{align*}
        \mL(\gZ) \subseteq \gL := \left(\bigcup_{\pi \in \mathfrak{S}_m} \gG\mP_\pi \right) \setminus \{\text{singular matrices}\} \,,
    \end{align*}
    where $\mathfrak{S}_m$ is the set of permutations and $\gG\mP_\pi = \{\mL\mP_\pi \mid \mL \in \gG\}$. We can rewrite the set $\gL$ above as 
    \begin{align*}
       \gL = \left(\bigcup_{\pi \in \mathfrak{S}_m} \gG\mP_\pi \setminus \{\text{singular matrices} \}\right) \,.
    \end{align*}

    We now define an equivalence relation $\sim$ over permutations: $\pi \sim \pi'$ iff $\mP_{\pi}\mP_{\pi'}^\top$ is $\mG$-preserving. One can verify that the relation $\sim$ is indeed an equivalence relation by using the fact that invertible $\mG$-preserving matrices form a group (Proposition~\ref{prop:group_of_S_consistent}). We notice that
    \begin{align}
        \pi \sim \pi' \implies \gG = \gG\mP_{\pi}\mP_{\pi'}^\top \implies \gG\mP_{\pi'} = \gG\mP_{\pi} \, , \label{eq:same_set}
    \end{align}
    where the first implication holds because $\mG$-preserving matrices are closed under matrix multiplication (Proposition~\ref{prop:group_of_S_consistent}). Let $\mathfrak{S}_m / \sim$ be the set of equivalence classes induce by $\sim$ and let $\Pi$ stand for one such equivalence class. Thanks to~\eqref{eq:same_set}, we can define, for all $\Pi \in \mathfrak{S}_m / \sim$, the following set:
    \begin{align*}
        V_\Pi := \gG\mP_{\pi} \setminus \{\text{singular matrices}\},\ \text{for some $\pi \in \Pi$}\,,
    \end{align*}
    where the specific choice of $\pi \in \Pi$ is arbitrary (any $\pi' \in \Pi$ would yield the same definition, by~\eqref{eq:same_set}). 
    This construction allows us to write 
    \begin{align*}
        \gL = \bigcup_{\Pi \in \mathfrak{S}_m / \sim} V_\Pi \,.
    \end{align*}
    We now show that $\{V_\Pi\}_{\Pi \in \mathfrak{S}_m / \sim}$ forms a partition of $\gL$. Choose two distinct equivalence classes of permutations $\Pi$ and $\Pi'$ and let $\pi \in \Pi$ and $\pi' \in \Pi'$ be representatives. We will now prove that
    \begin{align}
        \gG\mP_\pi \cap \gG\mP_{\pi'} \subseteq \{\text{singular matrices}\}\, . \label{eq:w4554325}
    \end{align}
    To achieve this, we proceed by contradiction: Suppose there exists an invertible matrix $\mA \in \gG\mP_\pi \cap \gG\mP_{\pi'}$. By Lemma~\ref{lemma:L_perm}, there exists a permutation $\mP$ s.t. $\mA\mP^\top$ has no zero on its diagonal. Of course, the permutation $\mP$ belongs to only one equivalence class and thus either $\mP \not\in \Pi$ or $\mP \not\in \Pi'$. Without loss of generality, assume the former. We thus have that $\mP \not\sim \pi$ and thus $\mP_{\pi}\mP^\top$ is not $\mG$-preserving. We can thus write
    \begin{align*}
        \mA \in \gG\mP_\pi \cap \gG\mP_{\pi'} \implies \mA\mP^\top \in \gG\mP_\pi\mP^\top \cap \gG\mP_{\pi'}\mP^\top\, . 
    \end{align*}
    By Lemma~\ref{lemma:zero_on_diagonal}, all matrices in $\gG\mP_\pi\mP^\top$ have a zero on their diagonal. This is a contradiction with $\mA\mP^\top \in \gG\mP_\pi\mP^\top$, since, as we said, $\mA\mP^\top$ has no zero on its diagonal. We thus conclude that no invertible matrix is in the intersection $\gG\mP_\pi \cap \gG\mP_{\pi'}$ and thus \eqref{eq:w4554325} holds. 
    
    We thus have that 
    \begin{align*}
        V_\Pi \cap V_{\Pi'} = \emptyset \,,
    \end{align*}
    which shows that $\{V_\Pi\}_{\Pi \in \mathfrak{S}_m / \sim}$ is indeed a partition of $\gL$.

    Each $V_\Pi$ is closed in $\gL$ (w.r.t. the subset topology inherited from $\sR^{m\times m}$) since
    \begin{align*}
        V_\Pi = \gG\mP_{\pi} \setminus \{\text{singular matrices}\} = \gL \cap \underbrace{\gG\mP_{\pi}}_{\text{closed in $\sR^{m\times m}$}}.
    \end{align*}

    Moreover, $V_\Pi$ is open in $\gL$, since
    \begin{align*}
        V_\Pi = \gL \setminus \underbrace{\bigcup_{\Pi' \not= \Pi} V_{\Pi'}}_{\text{closed in $\gL$}}\,.
    \end{align*}
    Thus, for any $\Pi \in \mathfrak{S}(\gB) / \sim$, the sets $V_\Pi$ and $\bigcup_{\Pi' \not= \Pi} V_{\Pi'}$ forms a \textit{separation} (see \cite[Section 23]{Munkres2000Topology}). Since $\mL(\gZ)$ is a connected subset of $\gL$, it must lie completely in $V_\Pi$ or $\bigcup_{\Pi' \not= \Pi} V_{\Pi'}$, by \cite[Lemma 23.2]{Munkres2000Topology}. Since this is true for all $\Pi$, it must follow that there exists a $\Pi^*$ such that $\mL(\gZ) \subseteq V_{\Pi^*}$. Choose any representative $\mP_* \in \Pi^*$. We thus have that, for all $\vz \in \gZ$, $\mL(\vz) = \mC(\vz)\mP_*^\top$, where $\mC(\vz)$ is $\mG$-preserving, which completes the proof.
\end{proof}

The goal of the next result is to relax the conditions of Lemma~\ref{lem:same_permutation} so that $\mL(\vz)\mP(\vz)$ is $\mG$-preserving for \textit{almost all} $\vz \in \sR^{d_z}$. %This happens to be enough when the measure zero set $E_0$ is such that $\gZ \setminus E_0$ has finitely many connected components.

\begin{lemma}\label{lem:same_permutation_almost}
    Let $\mG \in \{0,1\}^{m \times n}$ and let $\mL: \sR^{d_z} \rightarrow \sR^{m\times m}$ be a continuous function such that $\mL(\vz)$ is invertible for all $\vz \in \sR^{d_z}$. Suppose that, for almost all $\vz \in \sR^{d_z}$ (i.e. except on a set $E_0$ of Lebesgue measure zero), there exists a permutation matrix $\mP(\vz)$ such that $\mL(\vz)\mP(\vz)$ is $\mG$-preserving. Then, there exists a permutation matrix $\mP$ such that, for all $\vz \in \sR^{d_z}$, $\mL(\vz)\mP$ is $\mG$-preserving.
\end{lemma}

%\begin{lemma}\label{lem:same_permutation_almost}
%     Let $\mS \in \{0,1\}^{m \times n}$ and $\mC := [\mathbbm{1} - \mS(\mathbbm{1} - \mS)^\top]^+$ so that $\sR^{m\times m}_\mC$ is the set of $\mS$-consistent matrices (Lemma~\ref{lemma:charac_S-consistent}).
%    Let $\mu$ be a measure over $\sR^{d_z}$ with a support $\gZ$ that is connected and let $\mL: \gZ \rightarrow \sR^{d\times d}$ be a continuous function. Suppose there exists a zero measure set $E_0$ (i.e. $\mu(E_0) = 0$) s.t. $\gZ \setminus E_0$ has finitely many components and the following holds: for all $\vz \in \gZ \setminus E_0$, there exists a permutation matrix $\mP(\vz)$ such that $\mL(\vz) \subseteq \mC\mP(\vz)^\top$. Then, there exists a permutation matrix $\mP$ such that, for all $\vz \in \gZ$, $\mL(\vz) \subseteq \mC\mP^\top$.
%\end{lemma}

\begin{proof}
    We know that for all $\vz \in \sR^{d_z} \setminus E_0$, where $\mu(E_0) = 0$ (Lebesgue measure zero), there exists a permutation matrix $\mP(\vz)$ such that $\mL(\vz)\mP(\vz)$ is $\mG$-preserving. For all permutations $\mP$, define $\gZ^{(\mP)} := \{\vz\in \sR^{d_z} \setminus E_0 \mid \mP(\vz) = \mP\}$. The collection of all sets $\gZ^{(\mP)}$ is finite (since there are finitely many permutations) and forms a partition of $\sR^{d_z} \setminus E_0$. Of course, for all $\mP$, $\mL(\vz)\mP$ is $\mG$-preserving for all $\vz \in \gZ^{(\mP)}$. By Lemma~\ref{lemma:extend_continuity_to_closure}, we can extend this statement to the closure, i.e. for all $\vz \in \overline{\gZ^{(\mP)}}$, $\mL(\vz)\mP$ is $\mG$-preserving. 
    
    Furthermore, we have that $\bigcup_{\mP}\overline{\gZ^{(\mP)}} = \overline{\bigcup_{\mP}\gZ^{(\mP)}} = \overline{\sR^{d_z} \setminus E_0} = \sR^{d_z}$, where the first equality is a standard property of closure (which holds only for finite unions), and the last equality holds by Lemma~\ref{lemma:closure_almost}. We thus have that, for all $\vz \in \sR^{d_z}$, there exists a permutation $\mP(\vz)$ such that $\mL(\vz)\mP(\vz)$ is $\mG$-preserving. Since $\sR^{d_z}$ is connected we can apply Lemma~\ref{lem:same_permutation} to get the desired conclusion.
\end{proof}
\begin{comment}
\begin{proof}\seb{old proof...}
    We know that for all $\vz \in \gZ \setminus E_0$, where $\mu(E_0) = 0$ (measure zero), there exists a permutation matrix $\mP(\vz)$ such that $\mL(\vz) \subseteq \mC\mP(\vz)^\top$. We note that $\gZ \setminus E_0$ is not necessarily connected and thus we cannot apply Lemma~\ref{lem:same_permutation} directly on this set. Let $\gZ^{(i)}$ for $i = 1, ..., k$ be the connected components of $\gZ \setminus E_0$ (there are finitely many by assumption). We can apply Lemma~\ref{lem:same_permutation} on each connected component to get that, for all $i \in [k]$, there exists a permutation $\mP_i$ such that for all $\vz \in \gZ^{(i)}$, $\mL(\vz) \subseteq \mC\mP^\top_i$. By Lemma~\ref{lemma:extend_continuity_to_closure}, we can extend this statement to the closure of each component, i.e. for all $\vz \in \overline{\gZ^{(i)}}$, $\mL(\vz) \subseteq \mC\mP^\top_i$. We have that $\bigcup_{i=1}^k\overline{\gZ^{(i)}} = \overline{\bigcup_{i=1}^k\gZ^{(i)}} = \overline{\gZ \setminus E_0} = \gZ$, where the first equality is a standard property of closure (which holds only for finite unions), and the last equality holds by Lemma~\ref{lemma:closure_almost}. We thus have that, for all $\vz \in \gZ$, there exists a permutation $\mP(\vz)$ such that $\mL(\vz) \subseteq \mC \mP(\vz)^\top$. Since $\gZ$ is connected we can apply Lemma~\ref{lem:same_permutation} to get the desired conclusion.
\end{proof}
\end{comment}

\subsection{Connecting to the Graphical Criterion of~\citet{lachapelle2022disentanglement}} \label{sec:connect_lachapelle2022}
The goal of this section is to prove Proposition~\ref{prop:complete_disentanglement} which states if some graphical criterion holds (Assumption~\ref{def:graph_crit}), then $\vtheta \eqcon^{\vz,\va} \hat\vtheta$ implies $\vtheta \eqperm \hat\vtheta$, i.e. complete disentanglement. We recall Assumption~\ref{def:graph_crit}. 

\GraphCrit*

We note that the above assumption is slightly different from the original one from~\citet{lachapelle2022disentanglement}, since the intersections run over ${\bf Ch}_i^z$, ${\bf Pa}_i^z$ and ${\bf Pa}^\va_i$ instead of over some sets of indexes $\gI, \gJ \subset \{1, ..., d_z\}$ and $\gL \subset \{1, ..., d_a\}$. This slightly simplified criterion is equivalent to the original one, which we now demonstrate for the interested reader.

\begin{proposition}
Let $\mG = [\mG^z\ \mG^a] \in \{0,1\}^{d_z \times (d_z + d_a)}$. The criterion of Assumption~\ref{def:graph_crit} holds for $\mG$ if and only if the following holds for $\mG$: For all $i \in \{1, ..., d_z\}$, there exist sets $\gI, \gJ \subset \{1, ..., d_z\}$ and $\gL \subset \{1, ..., d_a\}$ such that
    \begin{align}
        \left( \bigcap_{j \in \gI} {\bf Pa}^\vz_j \right) \cap \left(\bigcap_{j \in \gJ} {\bf Ch}^\vz_j \right) \cap \left(\bigcap_{\ell \in \gL} {\bf Ch}^\va_\ell \right)  = \{i\} \, , \nonumber
    \end{align}
\end{proposition}
\begin{proof}
The direction ``$\implies$'' is trivial, since we can simply choose $\gI := {\bf Ch}_i^z$, $\gJ:={\bf Pa}_i^z$ and $\gL:={\bf Pa}_i^a$.

To show the other direction, we notice that we must have $\gI \subset {\bf Ch}^\vz_i$, $\gJ \subset {\bf Pa}_i^z$ and $\gL \subset {\bf Pa}_i^a$, otherwise one of the sets in the intersection would not contain $i$, contradicting the criterion. Thus, the criterion of Def.~\ref{def:graph_crit} intersects the same sets or more sets. Moreover, these potential additional sets must contain $i$ because of the obvious facts that $j \in {\bf Ch}_i^z \iff i \in {\bf Pa}^\vz_j$ and $\ell \in {\bf Pa}_i^a \iff i \in {\bf Ch}^\va_\ell$, thus they do not change the result of the intersection.
\end{proof}

To prove Proposition~\ref{prop:complete_disentanglement}, we will need the following lemma.
\begin{lemma}\label{lemma:charac_S-consistent}
Let $\mG \in \{0,1\}^{m \times n}$ and $\vc$ be a diffeomorphism with a dependency graph given by $\mC \in \{0,1\}^{m \times m}$ (Definition~\ref{def:dep_graph}). The function $\vc$ is $\mG$-preserving (Definition~\ref{def:g_preserving_map}) if and only if 
$$\forall i, \mC_{i, \cdot} \subseteq \bigcap_{k \in \mG_{i, \cdot}}\mG_{\cdot, k} \, .$$
\end{lemma}
\begin{proof}
We leverage Proposition~\ref{prop:charac_G_preserving}.    
    \begin{align}
    \mG_{i, \cdot} \not\subseteq \mG_{j, \cdot} & \iff \exists k \text{ s.t. } \mG_{i,k} = 1 \text{ and } \mG_{j, k} = 0 \label{eq:red_star}\\
    & \iff \exists k \in \mG_{i, \cdot} \text{ s.t. } j \not\in \mG_{\cdot, k} \nonumber \\
    & \iff j \not\in \bigcap_{k \in \mG_{i, \cdot}}\mG_{\cdot, k} \, . \nonumber
    \end{align}
\end{proof}

\CompleteDisentanglement*
\begin{proof}
By definition of $\eqcon^{\vz,\va}$, we know that the entanglement graph for $(\vf, \hat\vf)$ is given by $\mV = \mC\mP^\top$ where $\mP$ is a permutation and $\mC$ is a binary matrix that is $\mG^a$-preserving, $\mG^z$-preserving and $(\mG^z)^\top$-preserving. Using Lemma~\ref{lemma:charac_S-consistent}, we have that, for all $i$,
\begin{align*}
    \mC_{i, \cdot} &\subseteq {\left(\bigcap_{j \in \mG^\vz_{i, \cdot}} \mG^\vz_{\cdot, j}\right) \cap \left(\bigcap_{j \in \mG^\vz_{\cdot, i}} \mG^\vz_{j, \cdot}\right) \cap \left(\bigcap_{j \in \mG^\va_{i, \cdot}} \mG^\va_{\cdot, j}\right)}\\
    &= {\left(\bigcap_{j \in {\bf Pa}_i^z} {\bf Ch}^\vz_j \right) \cap \left( \bigcap_{j \in {\bf Ch}_i^z} {\bf Pa}^\vz_j \right) \cap \left(\bigcap_{\ell \in {\bf Pa}^\va_i} {\bf Ch}^\va_\ell \right)} \\
    &= {\{i\}} \, .
\end{align*}
Thus, $\mC$ is in fact the identity matrix, and hence $\vtheta \eqperm \hat\vtheta$.
\end{proof}

\section{Identifiability Theory - Exponential Family Case} \label{app:theory}

\subsection{Technical Lemmas and Definitions}
We recall the definition of a minimal sufficient statistic in an exponential family, which can be found in~\citet[p. 40]{WainwrightJordan08}.

\begin{definition}[Minimal sufficient statistic]\label{def:minimal_statistic}
Given a parameterized distribution in the exponential family, as in~\eqref{eq:z_transition}, we say that its sufficient statistic $\bfT_i$ is minimal when there is no $v \not= 0$ such that $v^\top\bfT_i(z)$ is constant for all $z \in \mathcal{Z}$.
\end{definition}

The following Lemma gives a characterization of minimality, which will be useful in the proof of Theorem~\ref{thm:linear}.

\begin{lemma}[Characterization of minimal $\bfT$]\label{lemma:charac_minimal}
A sufficient statistic of an exponential family distribution $\bfT: \gZ \rightarrow \sR^k$ is minimal if and only if there exists $\vz_{(0)}$, $\vz_{(1)}$, ..., $\vz_{(k)}$ belonging to the support $\gZ$ such that the following $k$-dimensional vectors are linearly independent:
\begin{align}
    \bfT(\vz_{(1)}) - \bfT(\vz_{(0)}), ..., \bfT(\vz_{(k)}) - \bfT(\vz_{(0)}) \, .\label{eq:T_vectors}
\end{align}
\end{lemma}
\textit{Proof.} We start by showing the ``if'' part of the statement. Suppose there exist $\vz_{(0)}, ..., \vz_{(k)}$ in $\gZ$ such that the vectors of \eqref{eq:T_vectors} are linearly independent. By contradiction, suppose that $\bfT$ is not minimal, i.e. there exist a nonzero vector $v$ and a scalar $b$ such that $v^\top \bfT(z) = b$ for all $z \in \gZ$. Notice that $b = v^\top \bfT(\vz_{(0)})$. Hence, $v^\top (\bfT(\vz_{(i)}) - \bfT(\vz_{(0)})) = 0$ for all $i=1, ..., k$. This can be rewritten in matrix form as
\begin{align*}
    v^\top [\bfT(\vz_{(1)}) - \bfT(\vz_{(0)})\ ...\ \bfT(\vz_{(k)}) - \bfT(\vz_{(0)})]  = 0 \, ,
\end{align*}
 which implies that the matrix in the above equation is not invertible. This is a contradiction.

We now show the ``only if '' part of the statement. Suppose that there is no $\vz_{(0)}$, ..., $\vz_{(k)}$ such that the vectors of~\eqref{eq:T_vectors} are linearly independent. Choose an arbitrary $\vz_{(0)} \in \gZ$. We thus have that $U := \text{span}\{\bfT(z) - \bfT(\vz_{(0)}) \mid z \in \gZ\}$ is a proper subspace of $\sR^k$. This means the orthogonal complement of $U$, $U^{\bot}$, has dimension 1 or greater. We can thus pick a nonzero vector $v \in U^\bot$ such that $v^\top(\bfT(z) - \bfT(\vz_0)) = 0$ for all $z \in \gZ$, which is to say that $v^\top\bfT(z)$ is constant for all $z \in \gZ$, and thus, $\bfT$ is not minimal. $\blacksquare$

\begin{comment}
\textcolor{red}{TODO: Fixing technicality with diffeomorphism assumption on $\bff$.} In the paper, we make the assumption that $\bff: \gZ \rightarrow \gX$ $\hat\bff: \hat\gZ \rightarrow \gX$ are diffeomorphisms in order to ensure that $\bff^{-1} \circ \hat \bff $ is also a diffeomorphism. But when $d_z < d_x$, this assumption makes sense only if we assume $\gX$ is a differentiable $d_z$-manifold embbeded in $\sR^{d_x}$, otherwise it's not clear what it means for $\bff^{-1}$ to be differentiable. However we never make this assumption about $\gX$ explicitly, and I would like to avoid getting into these technicalities in the main text. I now present an alternative solution. 

Assume $\bff: \gZ \rightarrow \gX$ is a homeomorphism, is continuously differentiable (i.e. $\bff$ is $C^1$) and has a full rank Jacobian everywhere. This means $\gX$ is a $k$-manifold of class $C^1$ and $\bff$ is a coordinate patch, following the definition introduced in~\citet[p. 200]{munkres1991analysis}. If another function $\hat \bff: \hat\gZ \rightarrow \gX$ has exactly the same property (i.e. it is another coordinate patch of $\gX$), then, by Theorem 24.1 of \citet{munkres1991analysis}, the map $\bff^{-1} \circ \hat \bff$ is $C^1$, and so is $\hat\bff^{-1} \circ \bff$, which is what we need.
\end{comment}

\subsection{Proof of Linear Identifiability (Theorem~\ref{thm:linear})}\label{app:exp_linear_ident}
%The following theorem and its proof are a minor extension of that of \citet{iVAEkhemakhem20a}. The key differences are (i) the fact that the sufficient statistics $\bfT_i$ do not have to be differentiable, which allows us to cover discrete latent variables (even though this is not highlighted in the main text), (ii) the notion of linear equivalence does not say anything about the link between $\bflambda$ and $\hat{\bflambda}$, which is crucial for the proof of Thm.~\ref{thm:combined}, and (iii) allowing $\bflambda$ to depend on $\vz^{<t}$. Strictly speaking, point (iii) was not covered by previous nonlinear ICA frameworks since $\vz^{<t}$ is not observed and, thus, cannot be treated as an auxiliary variable (which must be observed).  

\linearExpo*

\begin{proof}
First, we apply Proposition~\ref{prop:identDiffeo} to get
\begin{align}\label{eq:density_v_e453}
\tilde p(\vz^{t} \mid \vz^{<{t}}, \va^{<{t}}) = p(\bfv(\vz^{t}) \mid \bfv(\vz^{<{t}}), \va^{<{t}}) |\det D\bfv(\vz^{{t}}) |\, .
\end{align}

\textbf{Linear relationship between $\bfT(\bff^{-1}(x))$ and $\bfT(\hat{\bff}^{-1}(x))$.}
By taking the logarithm on each side of~\eqref{eq:density_v_e453} and making explicit the exponential family form, we get
\begin{align}
    &\sum_{i=1}^{d_z} \log h_i(\vz_i^t) + \bfT_i(\vz_i^t)^\top \bflambda_i(\mG_i^z \odot \vz^{<t}, \mG_i^a \odot \va^{<t}) - \psi_i(\vz^{<t}, \va^{<t}) \label{eq:expo1}\\
    =&\sum_{i = 1}^{d_z} \log h_i(\bfv_i(\vz^t)) + \bfT_i(\bfv_i(\vz^t)))^\top \hat{\bflambda}_i(\hat{\mG}_i^z \odot \bfv(\vz^{<t}), \hat{\mG}_i^a \odot \va^{<t}) - \hat{\psi}_i(\bfv(\vz^{<t}), \va^{<t}) \nonumber\\
    &\ \ \ \ \ \ \ \ \ \ \ \ \ \ \ \ \ \ \ \ \ \ \ \ \ \ \ \ \ \ \ \ \ \ \ \ \ \ \ \ \ \ \ \ \ \ \ \ \ \ \ \ \ \ \ \ \ \ \ \ \ \ \ \ \ \ \ \ \ \ \ \ \ \ \ \ \ \ \ \ \ \ \ \ \ \ \ \ \ \  \ \ \ \ \ \ \ \ \ \ \ \ \ + \log |\det D\bfv(\vz^{{t}}) | \nonumber
\end{align}

Note that~\eqref{eq:expo1} holds for all $\vz^{<t}$ and $\va^{<t}$. In particular, we evaluate it at the points given in the assumption of sufficient variability of Theorem~\ref{thm:linear}. We evaluate the equation at $(\vz^t, \vz_{(r)}, \va_{(r)})$ and $(\vz^t, \vz_{(0)}, \va_{(0)})$ and take the difference which yields\footnote{Note that $\vz_{(0)}$ and $\vz_{(r)}$ can have different dimensionalities if they come from different time steps. It is not an issue to combine equations from different time steps, since~\eqref{eq:expo1} holds for all values of $t$, $\vz^t$, $\vz^{<t}$ and $\va^{<t}$.}
\begin{align}
    &\sum_{i = 1}^{d_z} \bfT_i(\vz_i^t)^\top [\bflambda_i({\mG}_i^z \odot \vz_{(r)}, {\mG}_i^a \odot \va_{(r)}) - \bflambda_i({\mG}_i^z \odot \vz_{(0)}, {\mG}_i^a \odot \va_{(0)})] - \psi_i(\vz_{(r)}, \va_{(r)}) + \psi_i(\vz_{(0)}, \va_{(0)}) \nonumber \\
    =&\sum_{i = 1}^{d_z} \bfT_i(\bfv_i(\vz^t))^\top [\hat{\bflambda}_i(\hat{\mG}_i^z \odot\bfv(\vz_{(r)}), \hat{\mG}_i^a \odot \va_{(r)}) - \hat{\bflambda}_i(\hat{\mG}_i^z \odot \bfv(\vz_{(0)}), \hat{\mG}_i^a \odot \va_{(0)})]  \label{eq:expo2}\\
    &\ \ \ \ \ \ \ \ \ \ \ \ \ \ \ \ \ \ \ \ \ \ \ \ \ \ \ \ \ \ \ \ \ \ \ \ \ \ \ \ \ \ \ \ \ \ \ \ \ \ \ \ \ \ \ \ \ \ \ \ \ \ \ \ \ \ \ \ \ \ \ \ \ \ - \hat{\psi}_i(\bfv(\vz_{(r)}), \va_{(r)}) + \hat{\psi}_i(\bfv(\vz_{(0)}), \va_{(0)}) \nonumber
\end{align}
We regroup all normalization constants $\psi$ into a term $d(\vz_{(r)}, \vz_{(0)}, \va_{(r)}, \va_{(0)})$ and write
\begin{align}
    &\bfT(\vz^t)^\top [\bflambda(\vz_{(r)}, \va_{(r)}) - \bflambda(\vz_{(0)}, \va_{(0)})] \nonumber\\
=&\bfT(\bfv(\vz^t))^\top [\hat{\bflambda}(\bfv(\vz_{(r)}), \va_{(r)}) - \hat{\bflambda}(\bfv(\vz_{(0)}), \va_{(0)})] + d(\vz_{(r)}, \vz_{(0)}, \va_{(r)}, \va_{(0)})\label{eq:expo3} \, .
\end{align}
Define
\begin{align*}
    \vw_{(r)} &:= \bflambda(\vz_{(r)}, \va_{(r)}) - \bflambda(\vz_{(0)}, \va_{(0)}) \\
    \hat{\vw}_{(r)} &:= \hat{\bflambda}(\bfv(\vz_{(r)}), \va_{(r)}) - \hat{\bflambda}(\bfv(\vz_{(0)}), \va_{(0)})\\
    d_{(r)} &:= d(\vz_{(r)}, \vz_{(0)}, \va_{(r)}, \va_{(0)}) \, , 
\end{align*}
which yields 
\begin{align}
    \bfT(\vz^t)^\top \vw_{(r)} = \bfT(\bfv(\vz^t))^\top \hat{\vw}_{(r)} + d_{(r)}\, . \label{eq:all_j}
\end{align}

We can regroup the $\vw_{(r)}$ into a matrix and the $d_{(r)}$ into a vector:
\begin{align*}
    \mW &:= [\vw_{(1)} ...\ \vw_{(kd_z)}] \in \sR^{kd_z\times kd_z} \\
    \hat{\mW} &:=  [\hat{\vw}_{(1)} ...\ \hat{\vw}_{(kd_z)}] \in \sR^{kd_z\times kd_z}\\
    \vd &:= [d_{(1)} ...\ d_{(kd_z)}] \in \sR^{1 \times kd_z} \, .
\end{align*}
Since~\eqref{eq:all_j} holds for all $1 \leq p \leq kd_z$, we can write
\begin{align*}
    \bfT(\vz^t)^\top \mW = \bfT(\bfv(\vz^t))^\top \hat{\mW} + \vd \, .
\end{align*}
Note that $\mW$ is invertible by the assumption of variability, hence
\begin{align*}
    \bfT(\vz^t)^\top = \bfT(\bfv(\vz^t))^\top \hat{\mW} \mW^{-1} + \vd\mW^{-1} \, .
\end{align*}
Let $\vb := (\vd\mW^{-1})^\top$ and $\mL:= (\hat{\mW} \mW^{-1})^\top$. We can thus rewrite this as
\begin{align}
    \bfT(\vz^t) = \mL \bfT(\bfv(\vz^t)) + \vb \, . \label{eq:lin_equivalence}
\end{align}

\textbf{Invertibility of $\mL$.}
We now show that $\mL$ is invertible. By Lemma~\ref{lemma:charac_minimal}, the fact that the $\bfT_i$ are minimal is equivalent to, for all $i \in \{1, ..., d_z\}$, having elements $\vz_i^{(0)}$, ..., $\vz_i^{(k)}$ in $\gZ$ such that the family of vectors
\begin{align*}
    \bfT_i(\vz_i^{(1)}) - \bfT_i(\vz_i^{(0)}),\ ...\ , \bfT_i(\vz_i^{(k)}) - \bfT_i(\vz_i^{(0)})
\end{align*}
is linearly independent. Define
\begin{align*}
    \vz^{(0)} := [\vz_1^{(0)} \hdots \vz_{d_z}^{(0)}]^\top \in \sR^{d_z}
\end{align*}
For all $i \in \{1, ..., d_z\}$ and all $p\in \{1, ..., k\}$, define the vectors 
\begin{align*}
    \vz^{(p, i)}:= [\vz_1^{(0)} \hdots \vz_{i-1}^{(0)}\ \vz_i^{(p)}\ \vz_{i+1}^{(0)} \hdots \vz_{d_z}^{(0)}]^\top \in \sR^{d_z} \, .
\end{align*}
For a specific $1 \leq p \leq k$ and $i \in \{1, ..., d_z\}$, we can take the following difference based on~\eqref{eq:lin_equivalence}
\begin{align}
    \bfT(\vz^{(p, i)}) - &\bfT(\vz^{(0)}) = \mL [\bfT(\bfv(\vz^{(p, i)})) - \bfT(\bfv(\vz^{(0)}))]\, , \label{eq:difference}
\end{align}
where the left hand side is a vector filled with zeros except for the block corresponding to $\bfT_i(\vz_i^{(p, i)}) - \bfT_i(\vz_i^{(0)})$. Let us define
\begin{align}
    \Delta \bfT^{(i)}  := [\bfT(&\vz^{(1, i)}) - \bfT(\vz^{(0)})\ \dots\ \bfT(\vz^{(k,i)}) - \bfT(\vz^{(0)})] \in \sR^{kd_z \times k} \nonumber\\
    \Delta \hat{\bfT}^{(i)}  := [{\bfT}(\bfv(&\vz^{(1, i)})) - {\bfT}(\bfv(\vz^{(0)}))\ \dots\ {\bfT}(\bfv(\vz^{(k,i)})) - {\bfT}(\bfv(\vz^{(0)}))] \in \sR^{kd_z \times k}\, . \nonumber
\end{align}
Note that the columns of $\Delta \bfT^{(i)}$ are linearly independent and all rows are filled with zeros except for the block of rows $\{(i-1)k + 1, ..., ik\}$. We can thus rewrite~\eqref{eq:difference} in matrix form
\begin{align*}
    \Delta \bfT^{(i)} = \mL \Delta \hat{\bfT}^{(i)}\, .
\end{align*}
We can regroup these equations for every $i$ by doing
\begin{align}
    [\Delta \bfT^{(1)}\ ...\ \Delta \bfT^{(d_z)}] = \mL [\Delta \hat{\bfT}^{(1)}\ ...\ \Delta \hat{\bfT}^{(d_z)}]\, . \label{eq:difference_matrix}
\end{align}
Notice that the newly formed matrix on the left hand side has size $kd_z \times kd_z$ and is block diagonal. Since every block is invertible, the left hand side of~\eqref{eq:difference_matrix} is an invertible matrix, which in turn implies that $\mL$ is invertible. This completes the proof.
\end{proof}

\subsection{Proof of Theorem~\ref{thm:expfam_dis_z}}\label{app:expfam_dis_z}
\begin{lemma}\label{lemma:derivatives_of_exponential}
    Let $\vtheta := (\vf, p, \mG)$ satisfy Assumptions~\ref{ass:diffeomorphism}, \ref{ass:cond_indep}, \ref{ass:graph}, \ref{ass:smooth_trans} \& \ref{ass:expo} and let $q := \log p$. Then
    \begin{align*}
    D^t_z q (\vz^t \mid \vz^{<t}, \va^{<t}) &= \bflambda(\vz^{<t}, \va^{<t})^\top D\vs(\vz^t) + D(\log h) (\vz^t)\\
    H^{t,\tau}_{z,a}q(\vz^t \mid \vz^{<t}, \va^{<t}) &= D\vs(\vz^t)^\top D^\tau_a\bflambda(\vz^{<t}, \va^{<t}) \\
    H^{t,\tau}_{z,z}q(\vz^t \mid \vz^{<t}, \va^{<t}) &= D\vs(\vz^t)^\top D^\tau_z\bflambda(\vz^{<t}, \va^{<t}) \, .
\end{align*}
\end{lemma}
\begin{proof}
    We have
    \begin{align*}
        \log p(\vz^t \mid \vz^{<t}, \va^{<t}) := \log h (\vz^t) + \vs(\vz^t)^\top\bflambda(\vz^{<t}, \va^{<t}) - \psi(\vz^{<t}, \va^{<t}) \\
        \log h (\vz^t) + \bflambda(\vz^{<t}, \va^{<t})^\top\vs(\vz^t) - \psi(\vz^{<t}, \va^{<t}) \, .
    \end{align*}
    We can differentiate the above w.r.t. $\vz^t$ to get
    \begin{align*}
        D^t_z q (\vz^t \mid \vz^{<t}, \va^{<t}) &= \bflambda(\vz^{<t}, \va^{<t})^\top D\vs(\vz^t) + D(\log h) (\vz^t)
    \end{align*}
    Differentiating the above w.r.t. $\vz^\tau$ or $\va^\tau$ yields the desired result.
\end{proof}

\TemporalExp*

\begin{proof}
Recall the equation we derived in Section~\ref{sec:insight_main}:
\begin{align*}
    H^{t, \tau}_{z, z} \hat q(\vz^{t} \mid \vz^{<{t}}, \va^{<{t}}) = D\vv(\vz^t)^\top H^{t,\tau}_{z, z} {q}(\bfv(\vz^{t}) \mid \bfv(\vz^{<{t}}), \va^{<{t}})D\vv(\vz^\tau)\, .
\end{align*}
    Using Lemma~\ref{lemma:derivatives_of_exponential}, we get
\begin{align}
    D\vs(\vz^t)^\top D^\tau_z\hat\bflambda(\vz^{<t}, \va^{<t}) = D\vv(\vz^t)^\top D\vs(\vv(\vz^t))^\top D^\tau_z\bflambda(\vv(\vz^{<t}), \va^{<t})D\vv(\vz^\tau) \, . \label{eq:1432048}
\end{align}    
    Note that Assumption~\ref{ass:graph} requires that $D^\tau_z\bflambda(\vz^{<t}, \va^{<t}) \subseteq \mG^z$ and $D^\tau_z\hat\bflambda(\vz^{<t}, \va^{<t}) \subseteq \hat\mG^z$. Theorem~\ref{thm:linear} implies that there exist an invertible matrix $\mL \in \sR^{d_z \times d_z}$ and a vector $\vb \in \sR^{d_z}$ such that
    \begin{align*}
        \vv(\vz) = \vs^{-1}(\mL\vs(\vz) + \vb) \, .
    \end{align*}
    Taking the derivative of the above w.r.t. $\vz$, we obtain
\begin{align}
    D\bfv(\vz) &= D\bfT^{-1}(\mL\bfT(\vz) + \vb)\mL D\bfT(\vz) \\
    &= D\bfT^{-1}(\bfT(\bfv(\vz)))\mL D\bfT(\vz) \\
    &= D\bfT(\bfv(\vz))^{-1}\mL D\bfT(\vz)\, , \label{eq:56346}
\end{align}
where we used $\vs(\vv(\vz)) = \mL\vs(\vz) + \vb$ to go from the first to the second line and used the inverse function theorem to go from the second to the third line. Plugging \eqref{eq:56346} into \eqref{eq:1432048} yields
\begin{align*}
    & D\vs(\vz^t)^\top D^\tau_z\hat\bflambda(\vz^{<t}, \va^{<t})  \\
    &= D\bfT(\vz^t)^\top\mL^\top D\bfT(\bfv(\vz^t))^{-\top}    D\vs(\vv(\vz^t))^\top D^\tau_z\bflambda(\vv(\vz^{<t}), \va^{<t})D\bfT(\bfv(\vz^\tau))^{-1}\mL D\bfT(\vz^\tau) \\
    &= D\bfT(\vz^t)^\top\mL^\top D^\tau_z\bflambda(\vv(\vz^{<t}), \va^{<t})D\bfT(\bfv(\vz^\tau))^{-1}\mL D\bfT(\vz^\tau) \,,
\end{align*}
which implies
\begin{align}
    D\vs(\vz^t)^\top D^\tau_z\hat\bflambda(\vz^{<t}, \va^{<t}) &= D\bfT(\vz^t)^\top\mL^\top D^\tau_z\bflambda(\vv(\vz^{<t}), \va^{<t})D\bfT(\bfv(\vz^\tau))^{-1}\mL D\bfT(\vz^\tau) \\
    D^\tau_z\hat\bflambda(\vz^{<t}, \va^{<t})D\bfT(\vz^\tau)^{-1} &= \mL^\top D^\tau_z\bflambda(\vv(\vz^{<t}), \va^{<t})D\bfT(\bfv(\vz^\tau))^{-1}\mL\,, \label{eq:05492}
\end{align}
where we right- and left-multiplied by $D\vs(\vz^t)^{-\top}$ and $D\bfT(\vz^\tau)$, respectively. Let us define
\begin{align*}
    \Lambda(\gamma) := D^\tau_z\bflambda(\vv(\vz^{<t}), \va^{<t})D\bfT(\bfv(\vz^\tau))^{-1} \quad 
    \hat\Lambda(\gamma) := D^\tau_z\hat\bflambda(\vv(\vz^{<t}), \va^{<t})D\bfT(\bfv(\vz^\tau))^{-1}\ \, ,\nonumber
\end{align*}
where $\gamma = (t, \tau, \vz^{<t}, \va^{<t})$. Note that because $D\bfT$ is diagonal, we have that $\Lambda(\gamma) \subseteq \mG^z$ and $\hat\Lambda(\gamma) \subseteq \hat\mG^z$. Using this notation, we can rewrite \eqref{eq:05492} as
\begin{align*}
    \underbrace{\hat\Lambda(\gamma)}_{\subseteq \hat\mG^z} = \mL^\top \underbrace{\Lambda(\gamma)}_{\subseteq\mG^z} \mL\, . 
\end{align*}
Thanks to Assumption~\ref{ass:temporal_suff_var_main}, we can apply the same argument as in Theorem~\ref{thm:expfam_dis_z} to show that $\mL = \mC\mP^\top$ where $\mC$ is a matrix that is both $\mG^z$-preserving and $(\mG^z)^\top$-preserving, as desired. 
\end{proof}

\subsection{Connecting to Sufficient Influence Assumptions of~\citet{lachapelle2022disentanglement}}\label{app:connecting_ass}

In this section, we relate the nonparametric sufficient influence assumptions of this work, i.e. Assumptions~\ref{ass:nonparam_suff_var_a}~\&~\ref{ass:nonparam_suff_var_z}, to the analogous assumptions of~\citet{lachapelle2022disentanglement} for exponential families, i.e. Assumptions~\ref{ass:temporal_suff_var_main}~\&~\ref{ass:action_suff_var_main}, the latter of which we recall below.

\begin{assumption}[Sufficient influence of $\va$~\citep{lachapelle2022disentanglement}]\label{ass:action_suff_var_main}
    Assume $k = 1$, i.e. the sufficient statistics $\vs_i$ are one-dimensional. For all $\ell \in \{1, ..., d_a\}$, there exist $\{(\vz_{(r)}, \va_{(r)}, \epsilon_{(r)}, \tau_{(r)})\}_{r=1}^{|{\bf Ch}^\va_\ell|}$ belonging to their respective support such that
    \begin{align}
        \vecspan \left\{\bflambda(\vz_{(r)}, \va_{(r)} + \epsilon_{(r)} \mE^{(\ell, \tau)}) - \bflambda(\vz_{(r)}, \va_{(r)})\right\}_{r=1}^{|{\bf Ch}^\va_\ell|} 
        =\sR^{d_z}_{{\bf Ch}^\va_\ell} \, , \nonumber
    \end{align}
    %where $\Delta_\ell^{\tau} \bflambda(\vz^{<t}, \va^{<t}, \epsilon)$ is a partial difference defined by
    %\begin{align}
    %\Delta_\ell^{\tau} \bflambda(\vz^{<t}, \va^{<t}, \epsilon) := \bflambda(\vz^{<t}, \va^{<t} + \epsilon \mE_{\ell, \tau}) - \bflambda(\vz^{<t}, \va^{<t})\, , \label{eq:partial_difference}
%\end{align}
where $\epsilon \in \sR$ and $\mE^{(\ell, \tau)} \in \sR^{d_a \times t}$ is the one-hot matrix with the entry $(\ell, \tau)$ set to one. %Thus,~\eqref{eq:partial_difference} is the discrete analog of a partial derivative w.r.t. $\va^{\tau}_\ell$.
\end{assumption}

The following proposition shows that, when the exponential family holds with $k=1$, we have that (i) for the ``sufficient influence of $\va$" assumptions, the nonparametric and exponential family versions are actually equivalent, and (ii) for the ``sufficient influence of $\vz$'' assumptions, the nonparametric version implies the exponential family version.

\begin{proposition}[Sufficient influence assumptions: nonparametric v.s. exponential] \label{prop:suff_var_nonparam_vs_exp}Let the parameter $\vtheta := (\vf, p, \mG)$ satisfy Assumptions~\ref{ass:diffeomorphism}, \ref{ass:cond_indep}, \ref{ass:graph} \& \ref{ass:expo}. Further, assume that $k=1$ and that $D\vs(\vz) \in \sR^{d_z \times d_z}$ is invertible everywhere. Then,
\begin{align}
    &\text{Sufficient influence of $\va$: Assumption~\ref{ass:nonparam_suff_var_a} (nonparametric)} \iff \text{Assumption~\ref{ass:action_suff_var_main} (exponential family)} \nonumber \\
    &\text{Sufficient influence of $\vz$: Assumption~\ref{ass:nonparam_suff_var_z} (nonparametric)} \implies \text{Assumption~\ref{ass:temporal_suff_var_main} (exponential family)}\nonumber
\end{align}
\end{proposition}
\begin{proof}
    We start by proving the first equivalence for the sufficient influence of $\va$ assumptions. By using Lemma~\ref{lemma:derivatives_of_exponential} we see that
    \begin{align}
        &\ \vecspan \left\{D^{t_{(r)}}_{z} \log p (\vz \mid \vz_{(r)}, \va_{(r)} + \epsilon_{(r)}\mE^{(\ell, \tau_{(r)})}) - D^{t_{(r)}}_{z} \log p (\vz \mid \vz_{(r)}, \va_{(r)})\right\}_{r=1}^{|{\bf Ch}^\va_\ell|} \nonumber\\
        =&\ \vecspan \left\{D\vs(\vz)^\top\bflambda(\vz_{(r)}, \va_{(r)} + \epsilon_{(r)}\mE^{(\ell, \tau_{(r)})}) - D\vs(\vz)^\top\bflambda(\vz_{(r)}, \va_{(r)})\right\}_{r=1}^{|{\bf Ch}^\va_\ell|} \nonumber\\
        =&\ D\vs(\vz)^\top\vecspan \left\{\bflambda(\vz_{(r)}, \va_{(r)} + \epsilon_{(r)}\mE^{(\ell, \tau_{(r)})}) - \bflambda(\vz_{(r)}, \va_{(r)})\right\}_{r=1}^{|{\bf Ch}^\va_\ell|}\, . \label{eq:rhs_4353}
    \end{align}
    We start by showing ``$\impliedby$''. Assumption~\ref{ass:action_suff_var_main} implies that \eqref{eq:rhs_4353} is equal to $D\vs(\vz^t)^\top \sR^{d_z}_{{\bf Ch}^\va_\ell}$ which is equal to $\sR^{d_z}_{{\bf Ch}^\va_\ell}$ since $D\vs(\vz^t)$ is invertible everywhere and is diagonal. To show ``$\implies$'', we can apply the same argument.

    We now show that Assumption~\ref{ass:nonparam_suff_var_z} implies Assumption~\ref{ass:temporal_suff_var_main}. we again use Lemma~\ref{lemma:derivatives_of_exponential} and see that
    \begin{align*}
        \sR^{d_z}_{\mG^z} &= \vecspan \left\{ H^{t_{(r)}, \tau_{(r)}}_{z, z} \log p(\vz \mid \vz_{(r)}, \va_{(r)}) \right\}_{r=1}^{||\mG^z||_0} \\
        &= \vecspan \left\{ D\vs(\vz)^\top D^{\tau_{(r)}}_z\bflambda(\vz_{(r)}, \va_{(r)}) \right\}_{r=1}^{||\mG^z||_0}\\
        &= D\vs(\vz)^\top \vecspan \left\{D^{\tau_{(r)}}_z\bflambda(\vz_{(r)}, \va_{(r)}) D\vs(\vz)^{-1} \right\}_{r=1}^{||\mG^z||_0} D\vs(\vz) \, .
    \end{align*}
    Now recall that, in Assumption~\ref{ass:nonparam_suff_var_z}, we had that $\vz = \vz^{\tau_{(r)}}$ for all $r=1, ..., ||\mG^z||_0$, which allows us to write
    \begin{align*}
        \sR^{d_z}_{\mG^z} &= D\vs(\vz)^\top \vecspan \left\{D^{\tau_{(r)}}_z\bflambda(\vz_{(r)}, \va_{(r)}) D\vs(\vz^{\tau_{(r)}})^{-1} \right\}_{r=1}^{||\mG^z||_0} D\vs(\vz) \, ,
    \end{align*}
    which implies 
    \begin{align*}
        \vecspan \left\{D^{\tau_{(r)}}_z\bflambda(\vz_{(r)}, \va_{(r)}) D\vs(\vz^{\tau_{(r)}})^{-1} \right\}_{r=1}^{||\mG^z||_0} = D\vs(\vz)^{-\top}\sR^{d_z}_{\mG^z}D\vs(\vz)^{-1} = \sR^{d_z}_{\mG^z}\, ,
    \end{align*}
    where the last equality holds because $D\vs(\vz)$ is diagonal and invertible everywhere.
%=\sR^{d_z}_{\mG^z}
\end{proof}

\section{Experiments}
\subsection{Synthetic Datasets} \label{sec:syn_data}
We now provide a detailed description of the synthetic datasets used in the experiments of Section~\ref{sec:exp}. 

For all experiments, the dimensionality of $\vx^t$ is $d_x = 20$ and the ground-truth $\bff$ is a random neural network with three hidden layers of $20$ units with Leaky-ReLU activations with a negative slope of 0.2. The weight matrices are sampled according to a 0-1 Gaussian distribution and, to ensure that $\bff$ is injective, as assumed in all theorems of this paper, we orthogonalize its columns. Inspired by typical weight initialization in NN~\citep{pmlr-v9-glorot10a}, we rescale the weight matrices by $\sqrt{\frac{2}{1 + 0.2^2}}\sqrt{\frac{2}{d_{in} + d_{out}}}$ . The standard deviation of the Gaussian noise added to $\bff(\vz^t)$ is set to $\sigma = 10^{-2}$ throughout. Since the goal of the experiments is to validate our identifiability results, which assume infinite data, all datasets considered here are very large: 1 million examples.

We now present the different choices of ground-truth ${p(\vz^t \mid \vz^{<t}, \va^{<t})}$ we explored in our experiments. In all cases considered (except for the experiment with $k=2$ of Table~\ref{tab:violation}), it is a Gaussian with covariance $0.0001I$ independent of $(\vz^{<t}, \va^{<t})$ and a mean given by some function $\mu(\vz^{t - 1}, \va^{t - 1})$. Notice that we hence are in the case where $k=1$ with monotonic sufficient statistics, which is not covered by the theory of~\citet{iVAEkhemakhem20a}. Throughout, we set $d_z = 10$ and, unless explicitly specified otherwise, we set $d_a = 10$. In all \textit{Time} datasets, sequences have length $T=2$. In \textit{Action} datasets, the value of $T$ has no consequence since we assume there is no time dependence.

\subsubsection{Datasets Satisfying the Graphical Criterion}
The datasets of this section satisfy the graphical criterion of Section~\ref{sec:graph_crit}. This means our theory predicts complete disentanglement (Definition~\ref{def:disentanglement}). Unless specified otherwise, all datasets satisfy their respective sufficient influence assumptions (Section~\ref{sec:proofs}). These can be checked using Remark~\ref{rem:suff_lin_indep_func} combined with standard facts about independence of the sine and cosine functions.\\

\noindent\textbf{ActionDiag (Figure~\ref{fig:clear_exp}).} In this dataset, $d_a = d_x$ and the connectivity matrix between $\va^{t\shortminus 1}$ and $\vz^t$ is diagonal, which trivially implies that the graphical criterion of Section~\ref{sec:graph_crit} is satisfied. The mean function is given by
\begin{align}
    \mu(\vz^{t\shortminus 1}, \va^{t\shortminus 1}) := \sin(\va^{t\shortminus 1}) \, , \nonumber
\end{align}
where $\sin$ is applied element-wise. Moreover, the components of the action vector $\va^{t\shortminus 1}$ are sampled independently and uniformly between $-2$ and $2$. The same sampling scheme is used for all following datasets. One can check that the sufficient influence assumption (Assumption~\ref{ass:nonparam_suff_var_a_cont}) holds.

\noindent\textbf{ActionNonDiag (Figure~\ref{fig:clear_exp}).} 
We consider a case where the graphical criterion of Section~\ref{sec:graph_crit} is satisfied non-trivially. Let 
\newcommand\bigzero{\makebox(0,0){\text{\huge0}}}
\begin{align}
    \mG^a := \left(
    \begin{array}{ccccc}
    1    &       &         &    & 1 \\ 
    1    &   1   &         &    &   \\ 
         &   1   & \ddots  &    &   \\ 
         &       & \ddots  &  1 &   \\ 
         &       &         &  1 & 1 \\ 
  \end{array}\right) \label{eq:double_diagonal}
\end{align}
be the adjacency matrix between $\va^{t\shortminus 1}$ and $\vz^t$. The $i$th row, denoted by $\mG^\va_i$, corresponds to the parents of $\vz_i^t$ in $\va^{t\shortminus 1}$. Note that it is analogous to the graph shown in Figure~\ref{fig:sparsity}, which satisfies the graphical criterion. The mean function is given by
\begin{align}
    \mu(\vz^{t\shortminus 1}, \va^{t\shortminus 1}) := \begin{bmatrix}
           \mG^\va_{1} \cdot \sin(\frac{3}{\pi}\va^{t\shortminus 1}) \\
           \mG^\va_{2}\cdot \sin(\frac{4}{\pi}\va^{t\shortminus 1} + 1)\\
           \vdots \\
           \mG^\va_{d_z}\cdot  \sin(\frac{d_z + 2}{\pi}\va^{t\shortminus 1} + d_z - 1)
         \end{bmatrix} \, . \label{eq:double_diagonal_mean_func}
\end{align}
One can check that the sufficient influence assumption (Assumption~\ref{ass:nonparam_suff_var_a_cont}) holds, due to the independence of sines with different frequencies.

\noindent\textbf{$\text{ActionNonDiag}_\text{NoSuffInf}$ (Table~\ref{tab:violation}).} This dataset has the same ground truth adjacency matrix as the above dataset~\eqref{eq:double_diagonal}, but a different transition function which does not satisfy the assumption of sufficient influence (Section~\ref{ass:nonparam_suff_var_a_cont}). We sampled a matrix $\mW$ with independent Normal 0-1 entries. The mean function is thus
\begin{align*}
    \mu(\vz^{t\shortminus 1}, \va^{t\shortminus 1}) := (\mG^a \odot \mW)\va^{t\shortminus 1} \, ,
\end{align*}
where $\odot$ is the Hadamard product (a.k.a. element-wise product).

\begin{comment}
\textbf{Action sparsity with graphical criterion violation (Fig~\ref{fig:no_suff_no_crit}).} This dataset does not satisfy the graphical criterion. The mean function is the same as~\eqref{eq:double_diagonal_mean_func}, but its ground-truth graph $\mG^a$ is given by
\begin{align}
    \mG^a := \left(
    \begin{array}{ccccc}
    \sI_{2\times 2}   &       &     &     &   \\ 
        &   \sI_{2\times 2}   &     &     &  \\ 
        &       &   \ddots &     &  \\ 
        &      &    &     &  \\ 
        &     &    &   & \sI_{2\times 2} \\ 
  \end{array}\right) \label{eq:block_adj_action} \, ,
\end{align}
where $\sI_{2 \times 2}$ is the $2\times 2$ matrix filled with ones.
\end{comment}

\noindent\textbf{$\text{ActionNonDiag}_{k=2}$ (Table~\ref{tab:violation}).} This dataset has the ``double diagonal'' adjacency matrix of~\eqref{eq:double_diagonal} and the same mean function of~\eqref{eq:double_diagonal_mean_func}, but the variance of $\vz^t$ (we assume diagonal covariance) depends on $\va^{t-1}$ via
\begin{align}
    &\sigma^2(\vz^{t\shortminus 1}, \va^{t\shortminus 1}) := \frac{1}{10d_a} \begin{bmatrix}
           \exp{(\mG^\va_{1} \cdot \cos(\frac{3}{\pi}\va^{t\shortminus 1}))} \\
           \exp(\mG^\va_{2}\cdot \cos(\frac{4}{\pi}\va^{t\shortminus 1} + 1))\\
           \vdots \\
           \exp(\mG^\va_{d_z}\cdot  \cos(\frac{d_z + 2}{\pi}\va^{t\shortminus 1} + d_z - 1))
         \end{bmatrix} \, . \label{eq:triangular_var_func_}
\end{align}

\noindent\textbf{TimeDiag (Figure~\ref{fig:clear_exp}).} In this dataset, each $\vz_i^t$ has only $\vz_i^{t\shortminus 1}$ as parent. This trivially satisfies the graphical criterion of Section~\ref{sec:graph_crit}. The mean function is given by
\begin{align}
    \mu(\vz^{t\shortminus 1}, \va^{t\shortminus 1}) := \vz^{t\shortminus 1} + 0.5 \sin(\vz^{t\shortminus 1}) \, , \nonumber
\end{align}
where the $\sin$ function is applied element-wise. Notice that no auxiliary variables are required. One can check that the sufficient variability assumption (Assumption~\ref{ass:suff_var_expfam}) and sufficient influence assumption (Assumption~\ref{ass:temporal_suff_var_main}) of Theorem~\ref{thm:expfam_dis_z} (exponential family) holds.

\noindent\textbf{TimeNonDiag (Figure~\ref{fig:clear_exp}).} We consider a case where the graphical criterion of Section~\ref{sec:graph_crit} is satisfied non-trivially. Let
\begin{align}
    \mG^z := \left(
    \begin{array}{ccccc}
    1    &        &     &     &   \\ 
    1    &   1    &     &     &  \\ 
    \vdots    &       &   \ddots &     &  \\ 
    1    &      &    &     1&  \\ 
    1    &   1   &  \hdots  & 1 & 1 \\ 
  \end{array}\right) \label{eq:triangular_adj}
\end{align}
be the adjacency matrix between $\vz^t$ and $\vz^{t\shortminus 1}$. The $i$th row of $\mG^z$, denoted by $\mG^z_i$, corresponds to the parents of $\vz_i^t$. Notice that this connectivity matrix has no 2-cycles and all self-loops are present. Thus, by Proposition~\ref{prop:2_cycles}, it satisfies the graphical criterion of Section~\ref{sec:graph_crit}. The mean function in this case is given by
\begin{align}
    &\mu(\vz^{t\shortminus 1}, \va^{t\shortminus 1}) := \vz^{t\shortminus 1} + 0.5 \begin{bmatrix}
           \mG^\vz_{1} \cdot \sin(\frac{3}{\pi}\vz^{t\shortminus 1}) \\
           \mG^\vz_{2}\cdot \sin(\frac{4}{\pi}\vz^{t\shortminus 1} + 1)\\
           \vdots \\
           \mG^\vz_{d_z}\cdot  \sin(\frac{d_z + 2}{\pi}\vz^{t\shortminus 1} + d_z - 1)
         \end{bmatrix} \, , \label{eq:triangular_mean_func}
\end{align}
which is analogous to \eqref{eq:double_diagonal_mean_func}. One can verify that this transition model satisfies the sufficient variability assumption (Assumption~\ref{ass:suff_var_expfam}) and sufficient influence assumption (Assumption~\ref{ass:suff_var_expfam}) of Theorem~\ref{thm:expfam_dis_z} (exponential family) holds.

\noindent\textbf{$\text{TimeNonDiag}_\text{NoSuffInf}$ (Table~\ref{tab:violation}).}
This dataset has the same ground truth adjacency matrix as in~\eqref{eq:triangular_adj}, but a different transition function that does not satisfy the assumption of sufficient influence. We sampled a transition matrix $W$ with independent Normal 0-1 entries. The transition function is thus 
\begin{align*}
    &\mu(\vz^{t\shortminus 1}, \va^{t\shortminus 1}) := \vz^{t\shortminus 1} + 0.5 (\mG^z \odot W)\vz^{t-1} \, .
\end{align*}
\begin{comment}
\textbf{Temporal sparsity with graphical criterion violation (Fig.~\ref{fig:no_suff_no_crit}).} In this dataset, the mean function is the same as the one given in~\eqref{eq:triangular_mean_func} except for the adjacency matrix which does not satisfy the graphical criterion and is given by
\begin{align}
    \mG^z := \left(
    \begin{array}{cc}
    \sI_{\frac{1}{2}d_z\times \frac{1}{2}d_z}   &         \\ 
        &   \sI_{\frac{1}{2}d_z\times \frac{1}{2}d_z}     \\ 
  \end{array}\right) \label{eq:block_adj} \, ,
\end{align}
where $\sI_{\frac{1}{2}d_z\times \frac{1}{2}d_z}$ is the $\frac{1}{2}d_z \times \frac{1}{2}d_z$ matrix filled with ones.
\end{comment}

\noindent\textbf{$\text{TimeNonDiag}_{k=2}$ (Table~\ref{tab:violation}).} This dataset has the lower triangular adjacency matrix of~\eqref{eq:triangular_adj} and the same mean function of~\eqref{eq:triangular_mean_func}, but the variance of $\vz^t$ (we assume diagonal covariance) depends on $\vz^{t-1}$ via
\begin{align}
    &\sigma^2(\vz^{t\shortminus 1}, \va^{t\shortminus 1}) := \frac{1}{10d_z} \begin{bmatrix}
           \exp{(\mG^\vz_{1} \cdot \cos(\frac{3}{\pi}\vz^{t\shortminus 1}))} \\
           \exp(\mG^\vz_{2}\cdot \cos(\frac{4}{\pi}\vz^{t\shortminus 1} + 1))\\
           \vdots \\
           \exp(\mG^\vz_{d_z}\cdot  \cos(\frac{d_z + 2}{\pi}\vz^{t\shortminus 1} + d_z - 1))
         \end{bmatrix} \, . \label{eq:triangular_var_func}
\end{align}

\subsubsection{Datasets that Violate the Graphical Criterion}
The transition mechanisms for the action and temporal datasets of Figure~\ref{fig:jmlr_exp} and Table~\ref{tab:random_graphs} are \eqref{eq:double_diagonal_mean_func} and \eqref{eq:triangular_mean_func}, respectively, except for the graphs which are different. 

\noindent\textbf{ActionBlockDiag and ActionBlockNonDiag (Figure~\ref{fig:jmlr_exp}).} The left graph corresponds to ActionBlockDiag while the right one corresponds to ActionBlockNonDiag.

\begin{align}
    \mG_{(1)}^a := \begin{bmatrix}
    1 &   &   &   &  \\
    1 &   &   &   &  \\
      & 1 &   &   &  \\
      & 1 &   &   &  \\
      &   & 1 &   &  \\
      &   & 1 &   &  \\
      &   &   & 1 &  \\
      &   &   & 1 &  \\
      &   &   &   & 1\\
      &   &   &   & 1
    \end{bmatrix}
    \ \ 
    \mG_{(2)}^a := \begin{bmatrix}
    1 &   &   &   &  \\
    1 &   &   &   &  \\
      & 1 &   &   &  \\
      & 1 &   &   &  \\
      &   & 1 &   &  \\
      &   & 1 &   &  \\
      &   &   & 1 &  \\
      &   &   & 1 &  \\
    1 &   &   &   & 1\\
    1 &   &   &   & 1
    \end{bmatrix} \nonumber
\end{align}

\noindent\textbf{TimeBlockDiag and TimeBlockNonDiag (Figure~\ref{fig:jmlr_exp}).} The left graph corresponds to TimeBlockDiag while the right one corresponds to TimeBlockNonDiag.
\begin{align}
    \mG_{(1)}^z := \begin{bmatrix}
    1 & 1 &   &   &   &   &   &   &   &   \\
    1 & 1 &   &   &   &   &   &   &   &   \\
      &   & 1 & 1 &   &   &   &   &   &   \\
      &   & 1 & 1 &   &   &   &   &   &   \\
      &   &   &   & 1 & 1 &   &   &   &   \\
      &   &   &   & 1 & 1 &   &   &   &   \\
      &   &   &   &   &   & 1 & 1 &   &   \\
      &   &   &   &   &   & 1 & 1 &   &   \\
      &   &   &   &   &   &   &   & 1 & 1 \\
      &   &   &   &   &   &   &   & 1 & 1
    \end{bmatrix}
    \ \ \ 
    \mG_{(2)}^z := \begin{bmatrix}
    1 & 1 &   &   &   &   &   &   & 1 & 1 \\
    1 & 1 &   &   &   &   &   &   & 1 & 1 \\
      &   & 1 & 1 &   &   &   &   &   &   \\
      &   & 1 & 1 &   &   &   &   &   &   \\
      &   &   &   & 1 & 1 &   &   &   &   \\
      &   &   &   & 1 & 1 &   &   &   &   \\
      &   &   &   &   &   & 1 & 1 &   &   \\
      &   &   &   &   &   & 1 & 1 &   &   \\
    1 & 1 &   &   & 1 & 1 &   &   & 1 & 1 \\
    1 & 1 &   &   & 1 & 1 &   &   & 1 & 1
    \end{bmatrix} \nonumber
\end{align}

\noindent\textbf{ActionRandomGraphs and TimeRandomGraphs (Table~\ref{tab:random_graphs}).} The transition mechanisms are the same as in the ActionNonDiag and TimeNonDiag datasets, i.e. they are given by \eqref{eq:double_diagonal_mean_func} \& \eqref{eq:triangular_mean_func}, respectively. However, the graphs are sampled randomly, with various levels of sparsity. For the ActionRandomGraphs dataset, we have $\mG^a_{i,j} \sim \text{Ber}(p)$ and independent. For the TimeRandomGraphs datasets, it is the same except for the diagonal elements, which are forced to be active, i.e. $\mG^z_{i,i} = 1$.

\subsection{Implementation Details of our Constrained VAE Method} \label{sec:ours_details}
All details of our implementation match those of~\citet{lachapelle2022disentanglement} (except for the constrained optimization introduced in Section~\ref{sec:estimation}). 

\textbf{Learned mechanisms.} Every coordinate $\vz_i$ of the latent vector has its own mechanism $\hat{p}(\vz_i^t \mid \vz^{<t}, \va^{<t})$ that is Gaussian with mean outputted by $\hat{\vmu}_i(\vz^{t-1}, \va^{t-1})$ (a multilayer perceptron with 5 layers of 512 units) and a learned variance which does not depend on the previous time steps. For learning, we use the typical parameterization of the Gaussian distribution with $\mu$ and $\sigma^2$ and not its exponential family parameterization. Throughout, the dimensionality of $\vz^t$ in the learned model always matches the dimensionality of the ground-truth (same for the baselines). Learning the dimensionality of $\vz^t$ is left for future work.

\textbf{Prior of $\vz^1$ in time-sparsity experiments.} In \textit{time-sparsity} experiments, the prior of the first latent $\hat{p}(\vz^1)$ (when $t=1$) is modeled separately as a Gaussian with learned mean and learned diagonal covariance. Note that this learned covariance at time $t=1$ is different from the subsequent learned conditional covariance at time $t>1$.

%\textbf{Learned graphs $\hat{\mG}^z$ and $\hat{\mG}^a$.} As explained in Section~\ref{sec:estimation}, to allow for gradient-based optimization, each edge $\hat{\mG}_{i,j}$ is viewed as a Bernoulli random variable with probability of success $\text{sigmoid}(\vgamma_{i,j})$, where $\vgamma_{i,j}$ is a learned parameter. The gradient of the loss with respect to the parameter $\vgamma_{i,j}$ is estimated using the Gumbel-Softmax Gradient estimator~\citep{jang2016categorical, maddison2016concrete}. We found that initializing the parameters $\vgamma_{i,j}$ to a large value such that the probability of sampling all edge is almost one improved performance. In \textit{time-sparsity} experiments, there is no action so $\hat{\mG}^a$ is fixed to $\bm0$, i.e. it is not learned. Analogously, in \textit{action-sparsity} experiments, there is no temporal dependence so $\hat{\mG}^z$ is fixed to $\bm0$.

\textbf{Encoder/Decoder.} In all experiments, including baselines, both the encoder and the decoder is modeled by a neural network with 6 fully connected hidden layers of 512 units with LeakyReLU activation with negative slope $0.2$. For all VAE-based methods, the encoder outputs the mean and a diagonal covariance. Moreover, $p(\vx|\vz)$ has a \textit{learned} isotropic covariance $\sigma^2 I$. Note that $\sigma^2 I$ corresponds to the covariance of the independent noise $\vn^t$ in the equation $\vx^t = \bff(\vz^t) + \vn^t$.

\subsection{Baselines} \label{sec:baselines}
In synthetic experiments of Sec.~\ref{sec:exp}, all methods used a minibatch size of 1024 and the same encoder and decoder architecture: A MLP with 6 layers of 512 units with LeakyReLU activations (negative slope of 0.2). We manually tuned the learning rate of each method to ensure proper convergence. For VAE-based methods, i.e. TCVAE, SlowVAE and iVAE, we are always choosing $p(x|z)$ Gaussian with a covariance $\sigma^2 I$ and learn $\sigma^2$.

\textbf{$\beta$-TCVAE.} We used the implementation provided in the original paper by~\cite{tcvae}, which is available at \url{https://github.com/rtqichen/beta-tcvae}. We used a learning rate of 1e-4.

\textbf{iVAE.} We used the implementation available at \url{https://github.com/ilkhem/icebeem} from~\cite{iVAEkhemakhem20a}. In it, the mean of the prior $p(z|a)$ is fixed to zero, while its diagonal covariance is allowed to depend on $a$ through an MLP. We change this to allow the mean to also depend on $a$ through the neural network (with 5 layers and width 512). We also lower bounded its variance as well as the variance of $q(z\mid x, a)$ to improve the stability of learning. In the original implementation, the covariance of $p(x | z)$ was not learned. We found that learning it (analogously to what we do in our method) improved performance. We used a learning rate of 1e-4.

\textbf{SlowVAE.} We used the implementation provided in \url{https://github.com/bethgelab/slow_disentanglement}~\citep{slowVAE}. Like for other VAE-based methods, we modeled $p(x|z)$ as a Gaussian with covariance $\sigma^2 I$ and learned $\sigma^2$. 

\textbf{PCL.} We used the implementation provided here: \url{https://github.com/bethgelab/slow_disentanglement/tree/baselines}. PCL~\citep{PCL17} stands for ``permutation contrastive learning'' and works as follows: Given sequential data $\{\vx^t\}_{t=1}^T$, PCL trains a regression function $r((x',x))$ to discriminate between pairs of adjacent observations (positive pairs) and randomly matched pairs (negative pairs). The regression function has the form
\begin{align*}
    r((x, x')) = \sum_{i=1}^{d_z} B_i(h_i(x), h_i(x'))\, ,
\end{align*}
where $h: \sR^{d_x} \rightarrow \sR^{d_z}$ is the encoder and $B_i:\sR^2 \rightarrow \sR$ are learned functions. In our implementation, the functions $B_i$ are fully connected neural networks with 5 layers and 512 hidden units. We experimented with the less expressive function suggested in the original work, but found that the extra capacity improved performance across all datasets we considered.

\subsection{Unsupervised Hyperparameter Selection}
\label{sec:udr}

\begin{figure}
    \centering
    \includegraphics[width=\linewidth]{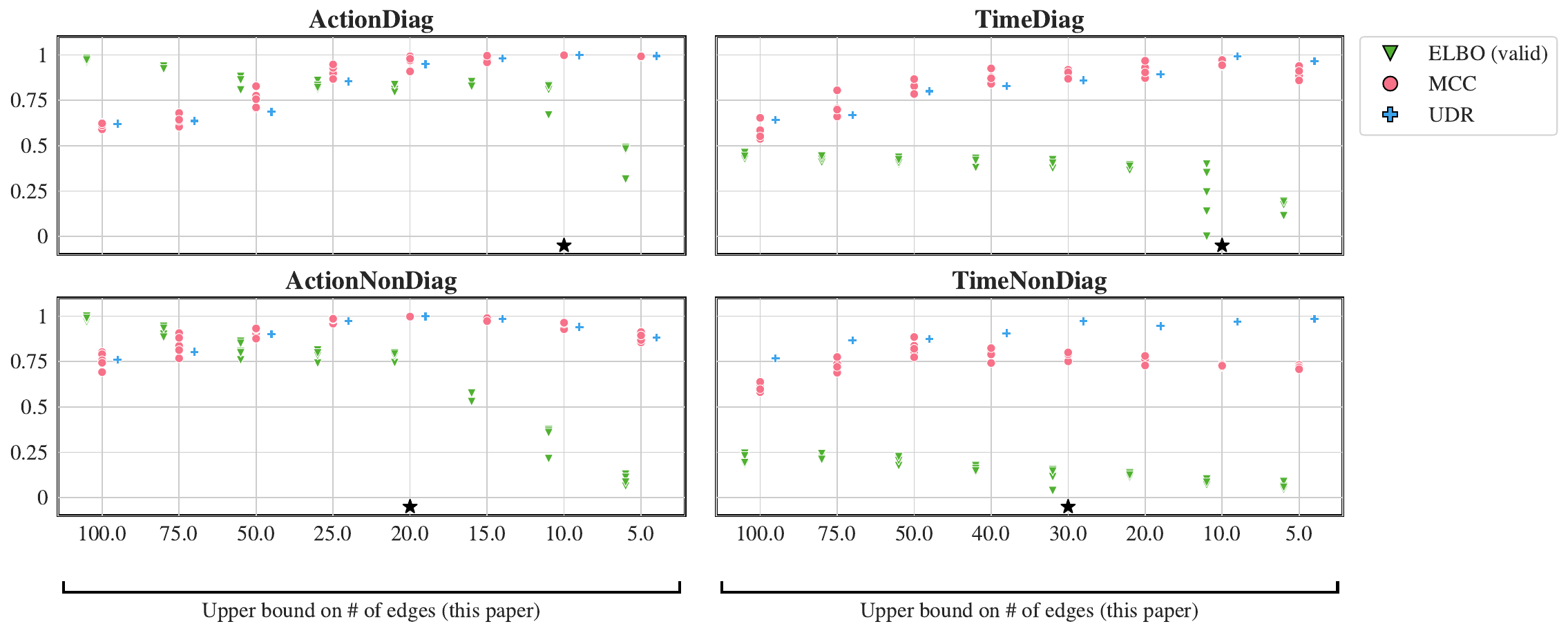}
    \caption{Investigating the link between goodness of fit (ELBO), disentanglement (MCC) and UDR. The ELBO is normalized so that it remains between 0 and 1.}
    \label{fig:elbo_mcc_udr}
\end{figure}

In practice, MCC cannot be measured since the ground-truth latent variables are not observed. Unlike in standard machine learning settings, hyperparameter selection for disentanglement cannot be performed simply by evaluating goodness of fit on a validation set and selecting the highest scoring model, since there is usually a trade-off between goodness of fit and disentanglement~\citep[Sec. 5.4]{pmlr-v97-locatello19a}. To avoid this problem, \cite{Duan2020UDR} introduced \textit{an unsupervised disentanglement ranking} (UDR), which, for all combinations of hyperparameters, measures how consistent are different random initializations of the algorithm. The authors argue that hyperparameters yielding disentangled representation typically yields consistent representations. In our experiments, the consistency of a given hyperparameter combination is measured as follows: for every pair of models, we compute the MCC between their representations. Then, we report the median of all pairwise MCC. This gives a UDR score for every hyperparameter values considered. Figure~\ref{fig:elbo_mcc_udr} reports the ELBO (normalized between zero and one), the MCC and the UDR score for the experiments of Figure~\ref{fig:clear_exp}. We can visualize the trade-off between ELBO and MCC. That being said, MCC and UDR correlate nicely except for the TimeNonDiag dataset, in which this correlation breaks for stronger regularization. We noticed that these specific runs correspond to excessively sparse graph, with fewer than 10 edges (out of 100 possible edges). The black star indicates the hyperparameter selected by UDR when excluding coefficient values which yields graphs with less than 10 edges (on average).

\textbf{Baselines.} Two of the baselines considered had hyperparameters to tune, SlowVAE~\citep{slowVAE} and TCVAE~\citep{tcvae}. For SlowVAE, we did a grid search on the following values, $\gamma \in \{ 1.0, 2.0, 4.0, 8.0, 16.0\}$ and $\alpha \in \{ 1, 3, 6, 10\}$. For TCVAE, we explored $\beta \in \{ 1, 2, 3, 4, 5\}$ but the optimal value in terms of disentanglement was almost always 1. Values of $\beta$ larger than 5 led to instabilities during training. The hyperparameters were selected using UDR, as described in the paragraph above. 

\section{Miscellaneous}
\subsection{On the Invertibility of the Mixing Function}\label{sec:diffeo_f}
Throughout this work as well as many others~\citep{TCL2016,PCL17, HyvarinenST19, iVAEkhemakhem20a, pmlr-v119-locatello20a, slowVAE}, it is assumed that the mixing function mapping the latent factors to the observation is a diffeomorphism onto its image. In this section, we briefly discuss the practical implications of this assumption.

Recall that a diffeomorphism is a differentiable bijective function with a differentiable inverse. We start by addressing the bijective part of the assumption. To understand it, we consider a plausible situation where the mapping $\bff$ is not invertible. Consider the minimal example of Fig.~\ref{fig:working_example} consisting of a tree, a robot and a ball. Assume that the ball can be hidden behind either the tree or the robot. Then, the mixing function $\bff$ is not invertible because, given only the image, it is impossible to know whether the ball is behind the tree or the robot. Thus, this situation is not covered by our theory. Intuitively, one could infer, at least approximately, where the ball is hidden based on previous time frames. Allowing for this form of occlusion is left as future work. See also \cite{mansouri2022objectcentric} for further discussion about how one can relax this assumption.

We believe the differentiable part of this assumption is only a technicality that could probably be relaxed to being piecewise differentiable. Our experiments were performed with data generated with a piecewise linear $\bff$, which is not differentiable only on a set of (Lebesgue) measure zero, but this was not an issue in practice.

\subsection{Contrasting with the Assumptions of \cite{iVAEkhemakhem20a} \& \cite{yao2022iclr}}\label{sec:implausible_ivae}

In this section, we discuss two identifiability results previously proposed in the literature that do not leverage sparsity~\citep{iVAEkhemakhem20a, yao2022neurips}. We show that these results do not apply to the simple homoscedastic Gaussian latent models of the form $p(\vz^t_i \mid \vz^{t-1}) = \mathcal{N}(\vz_i^t \mid \vmu_i(\vz^{t-1}), \sigma_i^2)$, contrary to our theory, as we saw in Examples~\ref{ex:a_target_one_z_cont_a}, \ref{ex:multi_target_a_cont_a} and \ref{ex:multinode_linear_gauss}. We will see that in the context of a Gaussian latent model, both results require the variance to vary sufficiently strongly. We believe that such a requirement is not well suited for nearly deterministic environments such as the one depicted in Figure~\ref{fig:working_example}.

\paragraph{\citet{iVAEkhemakhem20a}.} The most significant distinction between the theory of~\cite{iVAEkhemakhem20a} (iVAE) and ours is how identifiability up to permutation is obtained: Theorems~2~\&~3 from iVAE shows that if the assumptions of their Theorem~1 (which is essentially Theorem~\ref{thm:linear}) are satisfied and $\bfT_i$ has dimension $k>1$ \textbf{or} is non-monotonic, then the model is not just identifiable up to linear transformation but up to permutations (and rescalings). In contrast, our theory covers the case where $k=1$ and $\bfT_i$ is monotonic, like in the homoscedastic Gaussian case. Interestingly, \citet{iVAEkhemakhem20a} mentioned this specific case as a counterexample to their theory in their Proposition~3. %This is because, in that case, $k=1$ and the sufficient statistic is simply the identity, which is monotonic. 
The extra power of our theory comes from the extra \textit{structure} in the dependencies of the latent factors coupled with sparsity regularization. We note that, assuming the latent factors are Gaussian, the variability assumption of Theorem~\ref{thm:linear} combined with $k>1$ requires the variance to vary sufficiently, which is implausible in the nearly deterministic environment of Figure~\ref{fig:working_example}.

%We now argue that the assumptions of iVAE for disentanglement are less plausible in an environment such as the one of Fig.~\ref{fig:working_example}. Assuming the latent factors are Gaussian, the variability assumption of Thm.~\ref{thm:linear} combined with $k>1$ requires the variance to vary sufficiently, which is implausible in such a nearly deterministic environment. Assuming $k=1$ with non-monotonic $\bfT_i$ implies the conditional mean of $\vz^t$ does not depend on the past (since the sufficient statistic corresponding to the mean of a Gaussian is monotonous), which is also implausible in this environment. On the other hand, the case $k=1$ with monotonic $\bfT_i$ of Thm.~\ref{thm:combined} is well suited for the situation. Indeed, again in the Gaussian case, this would amount to predicting only the mean of the future positions of object. That being said, we also believe practical applications of these ideas will most likely require a combination of different identifiability results. How to formally combined these results is left as future work.

\paragraph{\citet{yao2022neurips}.} This work (Theorem 1) requires that, for each value of $\vz^t$, the $2d_z$ functions 
$$\frac{\partial^2}{\partial \vz^t_i\partial \vz^{t-1}}\log p(\vz_i^t \mid \vz^{t-1}) \text{ and } \frac{\partial^3}{(\partial \vz^t_i)^2\partial \vz^{t-1}}\log p(\vz_i^t \mid \vz^{t-1})\ \text{for }i=1 \dots d_z\,,$$
seen as functions from $\sR^{d_z}$ to $\sR^{d_z}$ are linearly independent. Indeed, if $p(\vz^t_i \mid \vz^{t-1}) = \mathcal{N}(\vz_i^t \mid \vmu_i(\vz^{t-1}), \sigma_i^2)$, one can easily derive that \
\begin{align*}
    \frac{\partial }{\partial \vz^t_i} \log p(\vz_i^t \mid \vz_i^{t-1}) &= - (\vz_i - \vmu_i(\vz^{t-1})) / \sigma_i^2 \\
    \frac{\partial^2 }{(\partial \vz^t_i)^2} \log p(\vz_i^t \mid \vz_i^{t-1}) &= - 1 / \sigma_i^2 \\
    \frac{\partial^2 }{(\partial \vz^t_i)^2\partial \vz^{t-1}} \log p(\vz_i^t \mid \vz_i^{t-1}) &= \bm 0 \,,
\end{align*}
which shows that the assumption of \citet[Theorem 1]{yao2022neurips} does not hold for homoscedastic Gaussian latent models. We further notice that, had the variance $\sigma^2_i$ depend on $\vz^{t-1}$, the identifiability result of \cite{yao2022iclr} could have applied.

\subsection{Derivation of the ELBO}\label{app:elbo}
In this section, we derive the evidence lower bound presented in Section~\ref{sec:estimation}.
\begin{align}
    \log &\ p(\vx^{\leq T} \mid \va^{<T}) =  \nonumber\\
    &\sE_{q(\vz^{\leq T} \mid \vx^{\leq T}, \va^{< T})}\left[\log \frac{q(\vz^{\leq T} \mid \vx^{\leq T}, \va^{<T})}{p(\vz^{\leq T} \mid \vx^{\leq T}, \va^{<T})} \right. \label{eq:kl_elbo} \\
    & \ \ \ \ \ \ \ \ \ \ \ \ \ \ \ \ \ \ \ \ \ \ \ \ \ \ \ \ \ \ \left.+ \log \frac{p(\vz^{\leq T}, \vx^{\leq T} \mid \va^{<T})}{q(\vz^{\leq T} \mid \vx^{\leq T}, \va^{<T})}\right] \nonumber \\
    &\geq \sE_{q(\vz^{\leq T} \mid \vx^{\leq T}, \va^{< T})}\left[\log \frac{p(\vz^{\leq T}, \vx^{\leq T} \mid \va^{<T})}{q(\vz^{\leq T} \mid \vx^{\leq T}, \va^{<T})}\right] \nonumber \\
    &= \sE_{q(\vz^{\leq T} \mid \vx^{\leq T}, \va^{<T})}\left[\log p(\vx^{\leq T} \mid \vz^{\leq T}, \va^{<T})\right]  \label{eq:rec_term}\\
    &\ \ \ \ \ - KL(q(\vz^{\leq T} \mid \vx^{\leq T}, \va^{<T}) || p(\vz^{\leq T} \mid \va^{<T})) \label{eq:kl_other}
\end{align}
where the inequality holds because the term at~\eqref{eq:kl_elbo} is a Kullback-Leibler divergence, which is greater or equal to 0. Notice that
\begin{align}
    p(\vx^{\leq T} \mid \vz^{\leq T}, \va^{<T}) &= p(\vx^{\leq T} \mid \vz^{\leq T}) =\prod_{t=1}^T p(\vx^{t} \mid \vz^{t}) \, . \label{eq:decoder_app}
\end{align}

Recall that we are considering a variational posterior of the following form:
\begin{align}
    q(\vz^{\leq T} \mid \vx^{\leq T}, \va^{<T}) := \prod_{t=1}^{T}q(\vz^t \mid \vx^t) \, . \label{eq:approx_post_app}
\end{align}

Equations~\eqref{eq:decoder_app}~\&~\eqref{eq:approx_post_app} allow us to rewrite the term in~\eqref{eq:rec_term} as
\begin{align*}
   \sum_{t=1}^T\mathop{\sE}_{\vz^t \sim q(\cdot| \vx^t)} [\log p(\vx^t \mid \vz^t) ]
\end{align*}

Notice further that
\begin{align}
    p(\vz^{\leq T} \mid \va^{<T}) &= \prod_{t=1}^T p(\vz^{t} \mid \vz^{<t}, \va^{<t})\, .\label{eq:prior}
\end{align}

Using~\eqref{eq:approx_post_app}~\&~\eqref{eq:prior}, the KL term~\eqref{eq:kl_other} can be broken down as a sum of KL as:
\begin{align*}
    \sum_{t=1}^T\mathop{\sE}_{\vz^{<t} \sim q(\cdot \mid \vx^{<t})} KL(q(\vz^t \mid \vx^t) || p(\vz^t \mid \vz^{<t}, \va^{<t}))
\end{align*}

Putting all together yields the desired ELBO:
\begin{align}
    &\log p(\vx^{\leq T}|\va^{<T}) \geq \sum_{t=1}^T\mathop{\sE}_{\vz^t \sim q(\cdot| \vx^t)} [\log p(\vx^t \mid \vz^t) ] \label{eq:elbo_app}\\ 
    \ \ \ &- \mathop{\sE}_{\vz^{<t} \sim q(\cdot \mid \vx^{<t})} KL(q(\vz^t \mid \vx^t) || p(\vz^t \mid \vz^{<t}, \va^{<t})) \, .\nonumber 
\end{align}

\section{Author Contributions}
%\begin{comment}
\subsection{Contributions to the Extended Version}
\textbf{Sébastien Lachapelle} developed the idea, the theory and proofs behind mechanism sparsity regularization for disentanglement, wrote the crux of the paper, developed the regularized VAE-based method, and performed most of the experiments. \textbf{R\'emi Le Priol} provided valuable feedback on the clarity of the manuscript. \textbf{Simon Lacoste-Julien} helped with the overall presentation of the paper, clarified the conceptual framework and the motivation, and provided supervision.

\subsection{Contributions to the CLeaR Version~\citep{lachapelle2022disentanglement}}
\textbf{Sébastien Lachapelle} developed the idea, the theory and proofs behind mechanism sparsity regularization for disentanglement, wrote the first draft of the paper, and designed and implemented the regularized VAE-based method. \textbf{Pau Rodr\'iguez L\'opez} ran all experiments appearing in the paper, produced associated figures, and ran experiments with image data that are still work in progress. \textbf{Yash Sharma} contributed to the research process, the experimental design in particular, implemented and ran experiments on image data that did not make it in the final version, and contributed to the writing and the literature review. \textbf{Katie Everett} implemented and ran experiments on image data that did not make it in the final version and contributed to the writing and figures. \textbf{R\'emi Le Priol} reviewed the proofs of main theorems, simplified some arguments and the general presentation of the proof, and contributed to the writing and figures. \textbf{Alexandre Lacoste} produced image datasets that did not make it into the final version and provided supervision. \textbf{Simon Lacoste-Julien} helped with the overall presentation of the paper, clarified the conceptual framework and the motivation, and provided supervision.
%\end{comment}

%\textbf{Sébastien Lachapelle} developed the idea, the theory and proofs behind mechanism sparsity regularization for disentanglement, wrote the crux of the paper, designed and implemented the regularized VAE-based method and performed most of the experiments. \textbf{Pau Rodr\'iguez L\'opez} ran all experiments appearing in the CLeaR version of this work~, produced associated figures and ran experiments with image data that did not make it into the final version. \textbf{Yash Sharma} contributed to the experimental design, the writing and the literature review of the CLeaR version and led experiments on image data that did not make it into the final version. \textbf{Katie Everett} implemented and started running experiments that did not make it in the final version and contributed to the writing and figures of the CLeaR version. \textbf{R\'emi Le Priol} reviewed the proofs of main theorems, simplified some arguments, helped with the overall proof presentation and contributed to the writing and figures of the CLeaR version. \textbf{Alexandre Lacoste} produced datasets that did not make it into the final version and provided supervision. \textbf{Simon Lacoste-Julien} helped with overall paper presentation, clarified the conceptual framework and the motivation and provided supervision.

\vskip 0.2in
\bibliography{biblio_thesis}
%\bibliography{biblio}

\end{document}